\documentclass{article}

\usepackage{titletoc}
\usepackage{microtype}
\usepackage{graphicx}
\usepackage{subcaption}
\usepackage{booktabs}
\usepackage{titletoc}
\usepackage{times}  
\usepackage{helvet}  
\usepackage{courier}  
\usepackage[hyphens]{url}  
\usepackage{graphicx}   
\usepackage{natbib}    
\usepackage{caption}     
\usepackage{hyperref}   
\usepackage{xspace}
\usepackage{verbatim}
\usepackage{thm-restate}  
\usepackage{newfloat}
\usepackage{listings}
\usepackage{enumitem}
\usepackage{mathtools}
\usepackage{amsthm}
\usepackage{thmtools}
\usepackage{amsmath}
\usepackage{amssymb}
\usepackage{float}

\usepackage[accepted]{icml2026}

\usepackage[capitalize,noabbrev]{cleveref}

\theoremstyle{plain}
\newtheorem{theorem}{Theorem}[section]

\newtheorem{lemma}{Lemma}
\newtheorem{corollary}{Corollary}
\theoremstyle{definition}
\newtheorem{definition}{Definition}
\newtheorem{assumption}{Assumption}
\theoremstyle{remark}
\newtheorem{remark}{Remark}
\icmltitlerunning{On the Variability of Concept Activation Vectors}

\begin{document}

\twocolumn[
\icmltitle{On the Variability of Concept Activation Vectors}
\icmlsetsymbol{equal}{*}
\begin{icmlauthorlist}
\icmlauthor{Julia Wenkmann}{yyy}
\icmlauthor{Damien Garreau}{yyy}
\end{icmlauthorlist}
\icmlaffiliation{yyy}{University of Würzburg, Center for Artificial Intelligence and Data Science (CAIDAS), Würzburg, Germany}
\icmlcorrespondingauthor{Julia Wenkmann}{julia.wenkmann@uni-wuerzburg.de}
\icmlkeywords{Machine Learning, Explainable AI}

\vskip 0.3in
]

\printAffiliationsAndNotice{}

\newcommand{\alphastar}{\alpha^\star}
\newcommand{\Astar}{A^\star}
\newcommand{\betastar}{\beta^\star}
\newcommand{\defeq}{\vcentcolon =}%
\newcommand{\Diff}{\mathrm{d}}%
\newcommand{\Reals}{\mathbb{R}}
\newcommand{\Proba}{\mathbb{P}}
\newcommand{\proba}[1]{\Proba(#1)}
\newcommand{\condproba}[2]{\Proba(#1|#2)}
\newcommand{\exps}[1]{\mathrm{e}^{#1}}
\newcommand{\norm}[1]{\left\lVert#1\right\rVert}%
\newcommand{\xbar}{\overline{x}}
\newcommand{\Var}{\mathrm{Var}}
\newcommand{\var}[1]{\Var(#1)}
\newcommand{\Cov}{\mathrm{Cov}}
\newcommand{\cov}[1]{\Cov(#1)}
\newcommand{\Expec}{\mathbb{E}}
\newcommand{\expec}[1]{\Expec\left[#1\right]}%
\newcommand{\E}{\mathbb{E}}
\newcommand{\PP}{\mathbb{P}}
\newcommand{\tr}{\operatorname{tr}}
\newcommand{\1}{\mathbf{1}}
\newcommand{\sigmoid}{\sigma}

\theoremstyle{definition}

\makeatletter
\def\th@plain{%
  \thm@notefont{}%
  \itshape %
}
\def\th@definition{%
  \thm@notefont{}%
  \normalfont %
}
\makeatother

\providecommand{\AClArC}{\textsc{A-ClArC}\xspace}
\providecommand{\ACAV}{\textsc{A-Cav}\xspace}
\providecommand{\ACDTE}{\textsc{Acdte}\xspace}
\providecommand{\ACE}{\textsc{Ace}\xspace}
\providecommand{\AI}{\textsc{Ai}\xspace}
\providecommand{\BRODEN}{\textsc{Broden}\xspace}
\providecommand{\CAR}{\textsc{Car}\xspace}
\providecommand{\CaCE}{\textsc{CaCE}\xspace}
\providecommand{\CXAI}{\textsc{C-Xai}\xspace}
\providecommand{\CART}{\textsc{Cart}\xspace}
\providecommand{\CAZ}{\textsc{Caz}\xspace}
\providecommand{\CAV}{\textsc{Cav}\xspace}
\providecommand{\CAVAZ}{\textsc{Cav-Az}\xspace}
\providecommand{\CAVLI}{\textsc{Cavli}\xspace}
\providecommand{\CBM}{\textsc{Cbm}\xspace}
\providecommand{\CCGLS}{\textsc{Ccgls}\xspace}
\providecommand{\CD}{\textsc{Cd}\xspace}
\providecommand{\CDEP}{\textsc{Cdep}\xspace}
\providecommand{\CG}{\textsc{Cg}\xspace}
\providecommand{\CLARC}{\textsc{ClArC}\xspace}
\providecommand{\CLIP}{\textsc{Clip}\xspace}
\providecommand{\CLIPDISSECT}{\textsc{Clip-Dissect}\xspace}
\providecommand{\CMU}{\textsc{Cmu}\xspace}
\providecommand{\CNN}{\textsc{Cnn}\xspace}
\providecommand{\COCOX}{\textsc{CoCoX}\xspace}
\providecommand{\ConceptSHAP}{\textsc{ConceptShap}\xspace}
\providecommand{\CPM}{\textsc{Cpm}\xspace}
\providecommand{\CRAFT}{\textsc{Craft}\xspace}
\providecommand{\CRP}{\textsc{Crp}\xspace}
\providecommand{\DCECA}{\textsc{Dceca}\xspace}
\providecommand{\DMA}{\textsc{Dma}\xspace}
\providecommand{\EDD}{\textsc{Edd}\xspace}
\providecommand{\EGEM}{\textsc{Egem}\xspace}
\providecommand{\FISTA}{\textsc{Fista}\xspace}
\providecommand{\GAN}{\textsc{Gan}\xspace}
\providecommand{\GNN}{\textsc{Gnn}\xspace}
\providecommand{\GNNCI}{\textsc{Gnn-Ci}\xspace}
\providecommand{\GRADCAM}{\textsc{Grad-Cam}\xspace}
\providecommand{\HINT}{\textsc{Hint}\xspace}
\providecommand{\IBD}{\textsc{Ibd}\xspace}
\providecommand{\ICA}{\textsc{Ica}\xspace}
\providecommand{\ICE}{\textsc{Ice}\xspace}
\providecommand{\IDS}{\textsc{Ids}\xspace}
\providecommand{\IIT}{\textsc{Iit}\xspace}
\providecommand{\IMA}{\textsc{Ima}\xspace}
\providecommand{\IMGTWOTAB}{\textsc{Img2Tab}\xspace}
\providecommand{\INV}{\textsc{Inv}\xspace}
\providecommand{\IoU}{\textsc{IoU}\xspace}
\providecommand{\Lrp}{\textsc{Lrp}\xspace}
\providecommand{\LIME}{\textsc{Lime}\xspace}
\providecommand{\LLM}{\textsc{Llm}\xspace}
\providecommand{\MCD}{\textsc{Mcd}\xspace}
\providecommand{\MCTS}{\textsc{Mcts}\xspace}
\providecommand{\MILAN}{\textsc{Milan}\xspace}
\providecommand{\MILANNOTATIONS}{\textsc{Milannotations}\xspace}
\providecommand{\MPO}{\textsc{Mpo}\xspace}
\providecommand{\NCA}{\textsc{Nca}\xspace}
\providecommand{\NCAVs}{\textsc{N-Cav}\xspace}
\providecommand{\ND}{\textsc{Nd}\xspace}
\providecommand{\NETTWOVEC}{\textsc{Net2Vec}\xspace}
\providecommand{\NLP}{\textsc{Nlp}\xspace}
\providecommand{\NN}{\textsc{NN}\xspace}
\providecommand{\NNMF}{\textsc{Nnmf}\xspace}
\providecommand{\OACAV}{\textsc{OA-Cav}\xspace}
\providecommand{\OD}{\textsc{Od}\xspace}
\providecommand{\PACE}{\textsc{Pace}\xspace}
\providecommand{\PATCAV}{\textsc{Pat-Cav}\xspace}
\providecommand{\PATCLARC}{\textsc{Pat-ClArC}\xspace}
\providecommand{\PCA}{\textsc{Pca}\xspace}
\providecommand{\PClArC}{\textsc{P-ClArC}\xspace}
\providecommand{\POEM}{\textsc{Poem}\xspace}
\providecommand{\PostHocCBM}{\textsc{Post-Hoc CBM}\xspace}
\providecommand{\RClArC}{\textsc{R-ClArC}\xspace}
\providecommand{\RF}{\textsc{RF}\xspace}
\providecommand{\RRClArC}{\textsc{RR-ClArC}\xspace}
\providecommand{\RtwoR}{\textsc{R2R}\xspace}
\providecommand{\RRR}{\textsc{Rrr}\xspace}
\providecommand{\SDC}{\textsc{Sdc}\xspace}
\providecommand{\SGD}{\textsc{Sgd}\xspace}
\providecommand{\SLIC}{\textsc{Slic}\xspace}
\providecommand{\SPRAY}{\textsc{SpRAy}\xspace}
\providecommand{\SSC}{\textsc{Ssc}\xspace}
\providecommand{\STCE}{\textsc{Stce}\xspace}
\providecommand{\SUMMIT}{\textsc{Summit}\xspace}
\providecommand{\SUN}{\textsc{Sun}\xspace}
\providecommand{\SVC}{\textsc{Svc}\xspace}
\providecommand{\TCAR}{\textsc{Tcar}\xspace}
\providecommand{\TCAV}{\textsc{Tcav}\xspace}
\providecommand{\TBD}{\textsc{Tbd}\xspace}
\providecommand{\VAE}{\textsc{Vae}\xspace}
\providecommand{\VICE}{\textsc{Vice}\xspace}
\providecommand{\XAI}{\textsc{Xai}\xspace}
\providecommand{\XGBOOST}{\textsc{XGBoost}\xspace}
\providecommand{\Var}{\text{Var}}
\providecommand{\Cov}{\text{Cov}}

\providecommand{\eat}[1]{}

\begin{abstract}
One of the most pressing challenges in artificial intelligence is to make models more transparent to their users. 
Recently, explainable artificial intelligence has come up with numerous methods to tackle this challenge. 
A promising avenue is to use concept-based explanations, that is, high-level concepts instead of plain feature importance scores. 
Among this class of methods, Concept Activation Vectors \citep[\text{\textsc{Cav}\xspace}s,][]{kim_interpretability_2018} stand out as one of the main protagonists. 
One interesting aspect of \text{\textsc{Cav}\xspace}s is that their computation requires sampling random examples from the train set. 
Therefore, the actual vectors obtained may vary depending on the randomness of this sampling. 
In this paper, we propose a fine-grained theoretical analysis of \text{\textsc{Cav}\xspace} construction in order to quantify their variability. 
Our results, confirmed by experiments on several real-life datasets of four different modalities, point to an universal result: the variance of \text{\textsc{Cav}\xspace}s declines roughly as $1/N$, where~$N$ is the number of random examples. 
Based on this, we give practical recommendations for a resource-efficient application of the method.
\end{abstract}

\section{Introduction}  
\label{sec:introduction}

As deep learning systems are increasingly deployed in high-stakes domains, from medical diagnosis to autonomous vehicles, the inability to understand why a model makes its predictions poses serious risks. 
A classifier might achieve high accuracy while relying on spurious correlations, or fail catastrophically on edge cases. 
Without interpretable explanations, practitioners cannot verify that models are learning the right patterns, regulators cannot assess safety, and users cannot calibrate their trust appropriately. 
Concept-based explanations offer a promising path forward by moving beyond low-level feature attributions to human-understandable abstractions \citep{poeta_concept-based_2023}. 
Among these methods, Testing with Concept Activation Vectors (\textsc{Tcav}), introduced by \citet{kim_interpretability_2018} has become a foundational framework that underpins much of the subsequent work in concept-based explainability. 
\text{\textsc{Tcav}\xspace} quantifies the influence of human-understandable concepts (\emph{e.g.}, ``stripes'') on a specific class prediction (\emph{e.g.}, ``zebra''). 
Central to the method are \text{\textsc{Cav}\xspace}s, which associate to a concept's direction in a latent layer by training a linear classifier to distinguish concept examples from random ones.  
The resulting \text{\textsc{Tcav}\xspace} score measures the sensitivity of a prediction to this concept direction, as described in detail in Section~\ref{sec:tcav_introduction}.

\begin{figure}[!ht]
  \centering
  \includegraphics[width=\columnwidth]{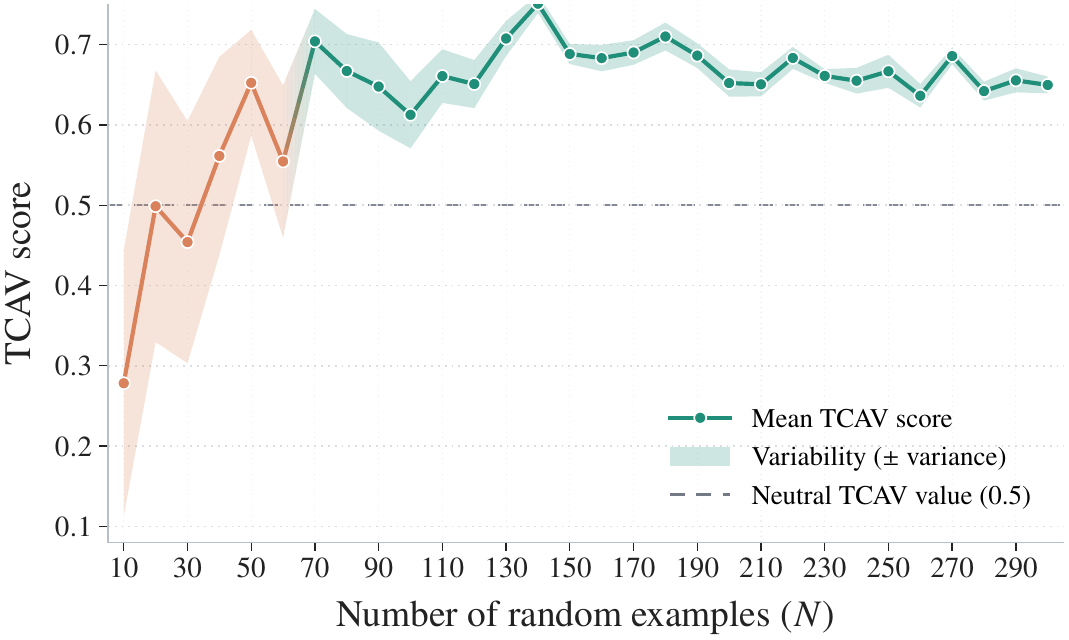}
  \caption[Example visualization of TCAV score as a function of the number of random examples $N$.]{Example visualization of the \TCAV score as a function of the number of random examples $N$ on the tabular dataset described in Section~\ref{sec:experiments}, computed from the \texttt{hidden\_layers.0} layer for the \texttt{sex=Female} concept. The line shows the mean \TCAV score across independently sampled random sets, and the shaded ribbon shows empirical run-to-run variability. The dashed line marks the neutral \TCAV value of $0.5$, which indicates no consistent directional influence of the concept and is the value expected by chance from a random direction.}
  \label{fig:example_introduction}
\end{figure}

However, \textsc{Tcav}'s reliability is compromised by an under-examined source of instability: the method requires randomly sampling reference examples to train the linear classifier that defines each \textsc{Cav}.
Different random samples yield different \textsc{Cav}s, and consequently different \textsc{Tcav} scores even when analyzing the same model with the same concept. 
This variability is not merely a theoretical concern. 
Sensitivity to data choices is a recognized issue for \textsc{Cav}s~\citep{ramaswamy_overlooked_2023, soni_adversarial_2020}. 
In real-world settings, practitioners may obtain a \textsc{Tcav} score of $0.7$ in one run and $0.5$ in another, potentially leading to contradictory conclusions about whether a concept positively or negatively influences a class (see Figure \ref{fig:example_introduction}).
\citet{kim_interpretability_2018} acknowledged this issue and recommended running the method multiple times and reporting average scores.
Yet this pragmatic fix raises immediate questions: \emph{How many runs are enough? How many random samples per run?} 
Without theoretical grounding, practitioners must resort to \emph{ad-hoc} choices that may be either wasteful or insufficient.

In this work, we conduct the first theoretical analysis of \textsc{Cav} variability.
Based on this, we derive recommendations for a resource-efficient application of the method. 
Our contributions are as follows:
\begin{enumerate}
    \item \textbf{Asymptotic normality of \textsc{Cav}s.} 
    We prove that penalized logistic regression \textsc{Cav}s are asymptotically normal in the infinitely imbalanced regime (Thm.~\ref{thm:asymptotic_normality_cav}), with covariance trace decaying as $\mathcal{O}(1/N)$. 
    We establish analogous results for hinge loss (Thm.~\ref{thm:asymptotic_normality_hinge}), Difference of Means (Thm.~\ref{thm:asymptotic_variance_dom}), and SVMs (Thm.~\ref{thm:car_variance}).
    
    \item \textbf{Variance transfer to downstream quantities.} 
    We show that variance of the derived measure of concept sensitivity inherits the same $\mathcal{O}(1/N)$ scaling, but \textsc{Tcav} scores may retain $\Theta(1)$ variance due to points with zero concept sensitivity (Section~\ref{sec:variance_tcav_scores}).
    
    \item \textbf{Multi-run averaging analysis.} 
    We analyze the variance reduction from averaging over $s$ independent runs, establishing $\mathcal{O}(1/s)$ scaling.
    
    \item \textbf{Practical guidelines.} 
    Based on our theoretical and empirical findings across four data modalities (images, tabular, text, and multimodal), we provide concrete recommendations for resource-efficient, stable \textsc{Tcav}.
\end{enumerate}

 The code for all experiments is publicly available at \href{https://github.com/juliawenkmann/variability-of-tcav}{https://github.com/juliawenkmann/variability-of-tcav}.

\section{Related Work}
\label{sec:related_work}

As mentioned earlier, \text{\textsc{Tcav}\xspace} is a very influential method among concept-based explainability approaches and has been extended and modified several times \citep{poeta_concept-based_2023}. 
We therefore focus on \textsc{Tcav}\xspace, as it serves as a natural starting point for studying stability in this family of methods.

\noindent 
\textbf{Influence of the (T)CAV method.} 
Many subsequent works use both {\textsc{Cav}\xspace}s and \textsc{Tcav}\xspace scores to represent concepts in latent space and analyze class-concept relationships.
For example, \label{acro:STCE} \text{\textsc{STCE}\xspace}~\citep{ji_spatial-temporal_2023} transfers the \textsc{Tcav}\xspace method to video data and thus allows a temporal view of concepts. Other methods, such as \label{acro:ACE}\text{\textsc{ACE}\xspace}~\citep{ghorbani_towards_2019}, \label{acro:ICE}\text{\textsc{Ice}\xspace}~\citep{zhang_invertible_2021} and \label{acro:COCOX}\text{\textsc{CoCoX}\xspace}~\citep{akula_cocox_2020}, adapt the calculation of the \textsc{Tcav}\xspace score, but in an unsupervised setting.
Although \label{acro:TCAV}\textsc{Tcav}\xspace provides mainly global class-concept relationships, \label{acro:CALVI} \textsc{Cavli}\xspace \citep{shukla_cavli_2023} and Visual-\textsc{Tcav}\xspace \citep{santis_visual-tcav_2024} adapts \TCAV to produce local explanations. 
Specifically, \textsc{Cavli}\xspace combines \textsc{Tcav}\xspace with \label{acro:LIME}\textsc{Lime}\xspace \citep{ribeiro_why_2016} for instance-level interpretations, and Visual-\textsc{Tcav}\xspace adds saliency maps to localize concepts in the input, leveraging an Integrated Gradients \citep{sundararajan_axiomatic_2017} approach. 
Beyond interpretability, \text{\textsc{Cav}\xspace}s have also been utilized for model steering, where concept vectors are added to or subtracted from internal activations during inference to shift model output toward desired behaviors. This has been applied to guide text generation \citep{subramani_extracting_2022} and more broadly to control model traits like honesty or sentiment via representation engineering \citep{zou_representation_2023}.

Subsequent work has also focused on improving the accuracy of {\textsc{Cav}\xspace}s and accommodating non-linearly separable concepts, which can not be fully captured by {\textsc{Cav}\xspace}s. Methods such as \emph{Concept Activation Regions} \label{acro:CAR}(\textsc{Car}s\xspace)~\citep{crabbe_concept_2022} and \emph{Concept Gradient} \label{acro:CG}(\textsc{Cg}\xspace)~\citep{bai_concept_2024} generalize {\textsc{Cav}\xspace}s to capture more complex concept boundaries, representing concepts through kernel-based regions or non-linear functions. \citet{soni_adversarial_2020} improve \textsc{Cav}\xspace robustness in two ways. First, \emph{Adversarial} \textsc{Cav}\xspace introduces small adversarial perturbations to concept samples. This leads to more stable concept vectors. Second, \emph{Orthogonal Adversarial} \textsc{Cav}\xspace applies a Gram–Schmidt-like orthogonalization to further separate concept and non-concept subspaces, thus improving \textsc{Cav}\xspace separability.
Pattern-based $\text{\textsc{Cav}\xspace}s $, introduced by \citet{pahde_navigating_2024}, address the issue of noise when learning {\textsc{Cav}\xspace}s. Instead of learning to separate positive and negative examples, they find a direction in activation space that best correlates with a concept's intensity. This yields concept representations that are more precise and robust against noise.

\noindent 
\textbf{Prior Analysis of Robustness and Consistency in XAI.}
An important part of evaluating \text{\textsc{Xai}\xspace} methods is examining their consistency, \emph{i.e.}, the extent to which an explanation method provides  deterministic explanations for the same inputs to be explained \citep{alvarezmelis_robustness_2018}. 
Such an analysis has already been carried out for some established \text{\textsc{Xai}\xspace} methods, such as \text{\textsc{Lime}\xspace}. \citet{garreau_what_2021} show that the randomness and instability observed in \text{\textsc{Lime}\xspace} explanations are a direct result of not using enough samples to fit the linear model. When the number of generated samples is very large, the explanations converge to an explicit ``limit explanation.'' 
\citet{visani_statistical_2021} take a different approach by introducing two new indices to measure the instability and reliability of \text{\textsc{Lime}\xspace}, which give practitioners a tool to assess the trustworthiness of \text{\textsc{Lime}\xspace}'s outputs.
\newpage
\textbf{Prior Analysis of \text{\textsc{Cav}\xspace}s.}
Although \text{\textsc{Cav}\xspace}s are widely used, their fundamental properties and limitations have been comparatively underexplored, with only a few works systematically addressing this gap~\citep{nicolson_explaining_2025}.
At the representation level, \citet{nicolson_explaining_2025} show that \text{\textsc{Cav}\xspace}s suffer from three distinct issues: they are \emph{inconsistent across layers} (the same concept yields different directions depending on the layer probed), \emph{entangled} with semantically related concepts, and \emph{spatially dependent}, meaning that the direction learned for a concept can vary with where in the image the concept appears.
At the data level, \citet{ramaswamy_overlooked_2023} demonstrate that the choice of probe dataset has a significant impact on the resulting explanations: \text{\textsc{Cav}\xspace}s learned for the same concept on different probe sets can have low cosine similarity, calling into question how far an explanation generalises beyond the dataset it was estimated on. \citet{raman_understanding_2024} further highlight a dependence on the spatial location of the concept in the input. Together, these works establish that \text{\textsc{Cav}\xspace}s are sensitive to layer choice, probe dataset, and spatial structure.

We focus on a different issue: the variance introduced by randomly sampling reference examples.

\begin{figure*}[!ht]
  \centering
  \includegraphics[scale=0.35]{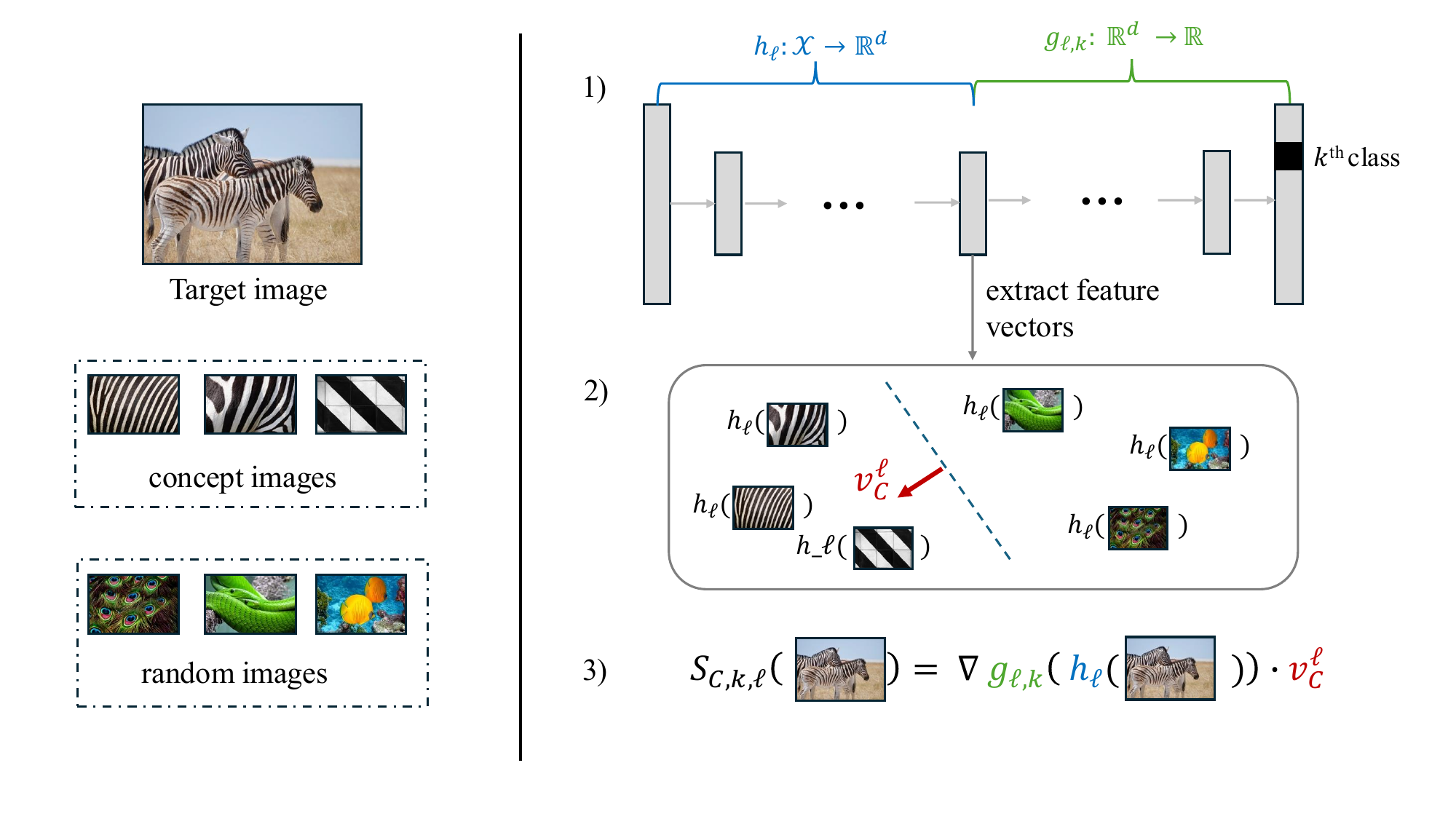}
  \caption[An overview of the workflow of the first part of \textsc{Tcav}\xspace]{\label{fig:tcav}
    An overview of the first part of the \textsc{Tcav}\xspace{} operating procedure: 
    1) \textsc{Tcav}\xspace{} extracts the activation vectors at a specific layer~$\ell$. 
    2) The \textsc{Cav}\xspace{} $v^{\ell}_{C}$, shown by the red arrow, is learned by training a binary linear classifier to differentiate between concept examples and random examples. 
    3) \textsc{Tcav}\xspace{} calculates the directional derivative $S_{C,k,\ell}(x)$ to measure how sensitive the model’s class prediction is to the Concept Activation Vector~$v^{\ell}_{C}$.
  }
  \vskip -0.2in
\end{figure*}

\section{Preliminaries and Notation}  
\label{sec:tcav_introduction}
In this section, we describe the concept activation vectors and \text{\textsc{Tcav}} values used by \citet{kim_interpretability_2018},
as well as generalizations introduced by \citet{schmalwasser_fastcav_2025} and \citet{crabbe_concept_2022}. Hereby we introduce our notation.

\subsection{Generating Concept Activation Vectors}
To formalize concept-based analysis, consider a model \emph{e.g.} a neural network $f \colon \mathcal{X} \to \mathbb{R}^K$ , that classifies inputs into one of
$K$ classes. At a chosen layer $\ell$, we  decomposed $f$ as $f \defeq g_{\ell} \circ h_{\ell} $. 
Here $h_{\ell}\colon \mathcal{X} \to \Reals^d$ maps inputs to the latent representation at layer~$\ell$, and $g_{\ell}\colon \mathbb{R}^d \to \mathbb{R}^K$ transforms these embeddings into class logits. 
In other words, for a fixed class $k \in \{1,\dots,K\}$, the function
$
g_{\ell,k} \colon \mathbb{R}^d \to \mathbb{R}
$
maps a latent embedding $v \in \mathbb{R}^d$ to the logit (or score) of class $k$.
This is illustrated in Figure~\ref{fig:tcav}. 
The \text{\textsc{Cav}\xspace} in layer~$\ell$ is computed by collecting:
\begin{enumerate}
\item \textit{Positive examples} of concept~$C$, \emph{i.e.}, a set of $n$ inputs $\{x_i\}_{i=1}^n \subset \mathcal{X} $ that clearly exhibit the concept $C$. 
\item  \textit{Random examples}, \emph{i.e.}, a set of $N$ inputs $\{z_j\}_{j=1}^N\subset \mathcal{X}$ selected uniformly from the training set. 
\end{enumerate}
\begin{remark}
Positive examples can be sourced from annotated datasets such as the Broden dataset \citep{bau_network_2017}, or be custom-curated, as done e.g. by \citet{kokhlikyan_captum_2025}. While \citet{kim_interpretability_2018} also propose using examples from contrary concepts as references, we use randomly selected samples throughout this work, as they are often more accessible. 
\end{remark}
We then compute the latent embeddings of the concepts and the random examples. 
Because the presence of a concept $C$ is generally not directly readable from raw embeddings, \text{\textsc{Tcav}\xspace} trains a linear classifier to discriminate between positive $\bigl\{h_\ell(x_i)\bigr\}_{i=1}^n$ and negative $\bigl\{h_\ell(z_j)\bigr\}_{j=1}^N$ embeddings. 
With a slight abuse of notation, we use $\{x_i\}_{i=1}^n$ and $\{z_j\}_{j=1}^N$ to refer to the latent representations $\{h_\ell(x_i)\}_{i=1}^n$ and $\{h_\ell(z_j)\}_{j=1}^N$ throughout the paper.
The normal vector to the decision boundary learned by the classifier, oriented towards the concept examples, is called the \textbf{Concept Activation Vector} $v_C^\ell$ associated to concept~$C$. 
Note that the \text{\textsc{Cav}\xspace} depends on both the random sampling and the stochastic nature of the classifier training process. 
Several different classifiers can be employed:
\begin{itemize}
\itemsep0pt
\item The original \textsc{Tcav}\xspace implementation defaults to a linear Support Vector Machine (\texttt{SGDClassifier} with hinge loss~\citep{pedregosa_scikit-learn_2011}), as detailed in Appendix~\ref{appendix:hinge_loss}.    
\item A simpler approach, \texttt{Difference of Means}, proposed by ~\citet{martin_interpretable_2019}, is adopted in Visual-\textsc{Tcav}\xspace~\citep{santis_visual-tcav_2024}. 
It computes the \textsc{Cav}\xspace as the  difference between the mean embeddings of the concept and random samples. \citet{santis_visual-tcav_2024} report that this estimator yields better \textsc{Cav}\xspace quality. 
Moreover, this method also yields an average speed-up of $46.4\times$ according to~\citet{schmalwasser_fastcav_2025}. We discuss this case in Appendix~\ref{appendix:difference_of_means}.
\item \citet{crabbe_concept_2022} propose an alternative that does not require concepts to be linearly separable. They use \texttt{Kernel SVMs} to learn Concept Activation \emph{Regions} rather than Concept Activation \emph{Vectors}. We assess the variability of this method using the same metric, extended to the induced Hilbert spaces. See Appendix~\ref{appendix:concept_activation_regions}.
\end{itemize}
Despite their differences, these methods exhibit broadly similar
stabilization trends, as we demonstrate experimentally in Appendix~\ref{sec:additional-results} and theoretically in Appendix~\ref{sec:appendix-proof}.
However, in the main part of this paper, we focus on \texttt{LogisticRegression} with binary cross-entropy loss as it presents a more amenable mathematical analysis and is the prominent classifier available in the official \textsc{Tcav}\xspace repository \citep{kokhlikyan_captum_2025, kim_code_2025}.

\subsection{Calculating TCAV Scores}

Once the \text{\textsc{Cav}\xspace} $v_C^{\ell}$ has been determined for the concept~$C$ in the layer~$\ell$, \text{\textsc{Tcav}\xspace}\ then evaluates the relevance of this concept for a class~$k$ by comparing the gradient of $g_{\ell,k}$ with the respective \CAV $v_C^{\ell}$ using the dot product (denoted by~$\cdot$). 
Specifically,  for an input $\mathbf{x} \in \mathcal{X}_k$, that is, an input $\mathbf{x}$ of class~$k$, \text{\textsc{Tcav}\xspace} defines the \textit{concept sensitivity score} of $\mathbf{x}$ as
\begin{equation}
S_{C,k,\ell}(\mathbf{x}) 
  \;\defeq \; 
  \nabla g_{\ell,k}\bigl( h_{\ell}(\mathbf{x}) \bigr) 
  \;\cdot\;
  v_C^{\ell}
\, .
\end{equation}
The score $S_{C,k,\ell}(\mathbf{x})$ measures a concept's influence on the classification of $\mathbf{x}$. A positive score indicates that the concept positively influences the prediction toward class $k$, while a negative score indicates that the concept opposes this prediction.
To aggregate these local sensitivities, the \text{\textsc{Tcav}\xspace} \emph{score} is defined as the proportion of samples $\mathbf{x}$ from the class~$k$ whose concept sensitivity $S_{C,k,\ell}(\mathbf{x})$ is positive, that is,
\begin{equation}
\label{eq:tcav_score_algo}
\mathrm{\text{\textsc{Tcav}\xspace}}_{C,k,\ell}
  \;\defeq \;
  \frac{ \bigl|\{\mathbf{x} \in \mathcal{X}_k \mid S_{C,k,\ell}(\mathbf{x}) > 0 \}\bigr| }{ \bigl|\mathcal{X}_k\bigr| }
\, .
\end{equation}
A score close to $1$ means concept $C$ shifts the classification toward class $k$, a low score indicates a strong negative effect, while a score near $0.5$ implies no consistent effect. 

Thus, \citet{kim_interpretability_2018} suggest to run the \text{\textsc{Tcav}\xspace} algorithm multiple times with $s$ different random sets and then compute the mean value of the obtained \text{\textsc{Tcav}\xspace} scores in their official implementation. 
Precisely, the multi-run approach partitions a set of $R$ random samples into $s$ disjoint subsets, each of size $N = R/s$. For each subset, a separate logistic regression is trained to find a \text{\textsc{Cav}\xspace} $v_C^{(j)}$, with $ j \in [1,\dots,s]$. This yields $s$ individual \text{\textsc{Tcav}\xspace} scores 
\[
T_{j} := \mathrm{TCAV}\bigl(v_C^{(j)}\bigr) =
\frac{ \bigl|\{\mathbf{x} \in \mathcal{X}_k \mid S(\mathbf{x}, v_C^{(j)}) > 0 \}\bigr| }{ \bigl|\mathcal{X}_k\bigr| } \;.
\]
The final \textbf{multi-run \text{\textsc{Tcav}\xspace} score}, $T_{\mathrm{multi}}$, is the average of these individual scores $ T_{\mathrm{multi}} \defeq \frac{1}{s}\sum_{j=1}^{s}T_{j}.$
Finally, a two-tailed $t$-test (see, \emph{e.g.}, \citep{hogg_introduction_2019}, Chapter~4.5) is performed to ensure that the effect captured by the \text{\textsc{Tcav}\xspace} score is not due to random variation.

\section{Theoretical Analysis}
\label{sec:theoretical-analysis}

We now present our main theoretical results regarding the variability of \text{\textsc{Cav}\xspace}s obtained by logistic regression when the number of random examples goes to infinity.
This reflects a common practical scenario: annotated concept examples are costly to obtain, a typical concept class in the widely used Broden dataset has on average $50$ examples \citep{bau_network_2017}, whereas random samples can be drawn freely from large unlabelled datasets such as ImageNet \citep{deng_imagenet_2009}. 
Understanding the asymptotic behaviour in this regime allows practitioners to answer the key question: \emph{how many random samples are sufficient} to achieve stable \text{\textsc{Cav}\xspace}s or \TCAV scores?

\subsection{Theoretical Setting}
\label{sec:theoretical-setting}

To formalize the algorithmic computation of \text{\textsc{Cav}\xspace}s described just above, we consider two sets of samples:
\begin{itemize}
    \item $n$ \textbf{fixed points}, $\{x_i\}_{i=1}^n$, representing our ``concept'' samples and are assigned the class label $Y=1$.
    \item $N$ \textbf{random points}, $\{z_j\}_{j=1}^N$, drawn from the distribution $F_0$ of training samples and assigned $Y=0$.
\end{itemize}
We denote the \emph{intercept} by $\alpha\in\Reals$ and the \emph{coefficients} by $\beta \in\Reals^d$. 
Following \citet{owen_2007}, for any input $w\in\Reals^d$, we write the logistic regression model used in \TCAV as
\begin{equation}
\label{eq:logistic-regression-model}
\condproba{Y=1}{X=w} = \sigma(\alpha + \beta^{\top}w)
\, ,
\end{equation}
where $\sigma(u) \defeq (1+\exps{-u})^{-1}$, with  $u \in\Reals$. 
For any given~$N$, the default \TCAV implementation fits the logistic regression model Eq.~\eqref{eq:logistic-regression-model} by maximizing the $L^2$-penalized log-likelihood $\mathcal{L}_N^{(\lambda)}(\alpha, \beta)$
, where $\alpha \in \Reals$, $\beta \in \Reals^d$ and $\lambda$ is a positive regularization parameter. 
We denote by $(\alpha_N, \beta_N)$ the unique maximizer of this function. 
The coefficient vector $\beta_N$ is our mathematical modelisation of the empirical \text{\textsc{Cav}\xspace} $v_C^{\ell}$ computed with $N$ random samples. We adopt this notation to distinguish the theoretical quantity under analysis from the algorithmic \text{\textsc{Cav}\xspace} introduced in Section~\ref{sec:tcav_introduction}.

\subsection{Variability of Concept Activation Vectors}
\label{sec:main-results}

In this section, we prove asymptotic normality of \text{\textsc{Cav}\xspace}s in the special case of infinite imbalance. 
This specific setting, where the number of samples in one class grows infinitely large while the other is held fixed, was previously analyzed for logistic regression by \citet{owen_2007} and \citet{goldman_zhang_2022}.
Our result, similar to that of~\citet{owen_2007}, is conditional upon the ``Surrounded Mean'' Assumption. 
This is a weak condition, postulating that infinite random samples surround the concept's mean  $\bar{x}$  in the latent space, which generally holds in practice (see Appendix~\ref{appendix:cross_entropy}).

\begin{assumption}[Surrounded Mean]
\label{assumption:surround}
The distribution $F_0$ on $\mathbb{R}^d$ has the point $\bar{x}$ \emph{surrounded}, that is
\begin{equation}
\int_{(z-\bar{x})^\top\omega>\varepsilon} \Diff F_0(z) > \delta
\end{equation}
holds for some $\varepsilon > 0$, some $\delta > 0$ and all $\omega \in \Omega$ where $\Omega = \{\omega \in \mathbb{R}^d | \omega^\top\omega = 1\}$ is the unit sphere in $\mathbb{R}^d$ and $\bar{x}$ is the mean of the concept embeddings $\{x_i\}_{i=1}^n$.
\end{assumption}

Intuitively, this condition requires that infinite random samples surround the concept's mean $\bar{x}$ in the latent space: in every direction, the random references must populate a non-trivial cap around $\bar{x}$. This mild geometric assumption generally holds in practice (see Appendix~\ref{appendix:cross_entropy}).
The second assumption concerns the behavior of the intercept as the class imbalance grows.

\begin{assumption}[Intercept Scaling Limit]
\label{assumption:limit_A}
Let $A_N := N e^{\alpha_N}.$
Then there exists a deterministic finite constant $A_\infty$ such that
\[
    A_N \xrightarrow{p} A_\infty
    \qquad\text{as }N\to\infty.
\]
\end{assumption}
In this parametrization, $\alpha_N\to-\infty$ in probability
compensates for the growing class imbalance, and the assumption
requires $e^{\alpha_N}=O_p(N^{-1})$. This mild technical condition holds whenever the regression problem remains well-posed in the limit, and is illustrated in Appendix~\ref{appendix:cross_entropy}. The accompanying notebooks provide empirical checks.

The third assumption ensures that the \text{\textsc{Cav}\xspace} itself stabilizes.

\begin{assumption}[Consistency of $\beta_N$]
\label{assumption:beta_convergence}
There exists a deterministic vector $\beta_\infty\in\Reals^d$ such that
$\beta_N\xrightarrow{p}\beta_\infty$ as $N\to\infty$.
\end{assumption}
This is the central assumption of our analysis: as we increase the number of random samples, the estimated \text{\textsc{Cav}\xspace} converges to a well-defined limiting direction $\beta_\infty \in \Reals^d$. Note, that this assumption can be relaxed to almost sure convergence under stronger moment conditions on $F_0$.
Our last assumption is a technical moment condition on $F_0$. 

\begin{assumption}[Exponential moments under $F_0$]
\label{assumption:exp_moment}
For every $\beta\in\Reals^d$,
\label{eq:exp_moment}
\(\mathbb{E}_{F_0} \left[\|Z\|^2\,\exp \bigl(2\beta^\top Z\bigr)\right] < \infty.\)
\end{assumption}
This assumption ensures that the random samples $Z \sim F_0$ have sufficiently light tails. The condition is satisfied by any distribution with bounded support, as well as by Gaussian. In practice, neural network activations are typically bounded, so this assumption holds for most architectures encountered in deep learning.

We provide empirical verification of these assumptions in the accompanying notebooks under ``Checking Assumptions.''
We can now state our main result:

\begin{theorem}[Asymptotic Normality of \CAV s]
\label{thm:asymptotic_normality_cav}
Assume that Assumption~\ref{assumption:surround}, \ref{assumption:limit_A}, \ref{assumption:beta_convergence}, and~\ref{assumption:exp_moment} hold. 
Then
\begin{equation}
\label{eq:asymptotic_normality}
\sqrt{N}\,(\beta_N-\beta_\infty)\ \xrightarrow{D}\ \mathcal{N}(0,\Sigma),
\end{equation}
where $\xrightarrow{D}$ denotes convergence in distribution. A precise expression for $\Sigma$ is given in Eq.~\eqref{eq:sigma}.
\end{theorem}

\begin{figure}[t!]
\centering
\includegraphics[width=0.45\textwidth]{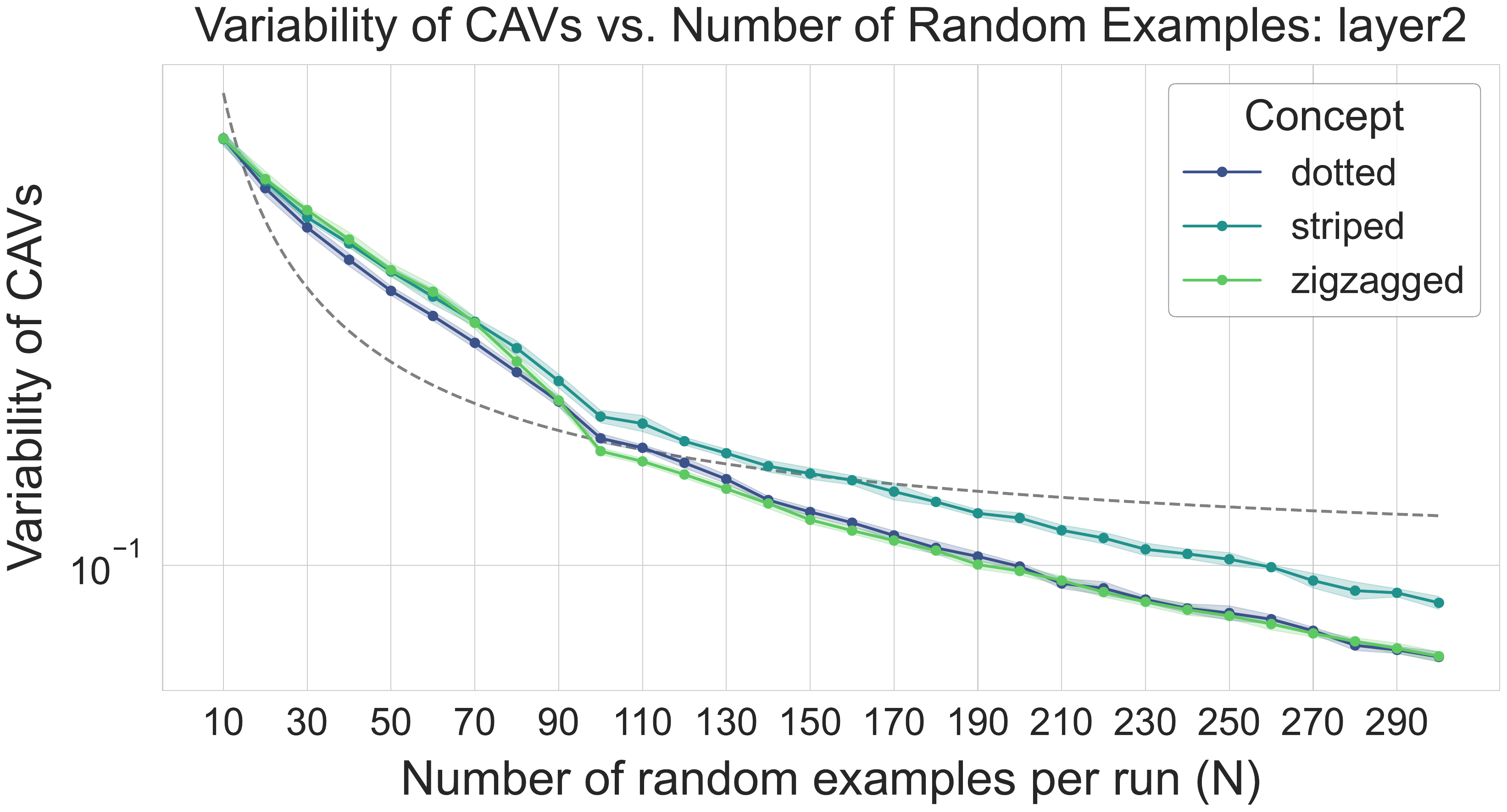}
\caption[Variability of \text{\textsc{Cav}\xspace}s vs. Number of Random Samples Used]{\label{fig:variance_of_cavs} Mean variability of \text{\textsc{Cav}\xspace}s for \texttt{ResNet-50} at \texttt{`layer2`} as a function of the number of random examples per run ($N$), shown for ``striped''-, ``zigzagged''-, and ``dotted'' concepts. Error bars indicate $\pm 1$ standard deviation; the $y$-axis is log-scaled. 
Variance of \text{\textsc{Cav}\xspace}s is estimated by taking the sum of per-feature variances across ten independent runs. The dashed line shows a fit of the form $y = a/N + b$ estimated on the average of all concepts.
}
\end{figure}

Informally, Theorem~\ref{thm:asymptotic_normality_cav} states that, for large $N$, the random fluctuations of 
$\sqrt{N}(\beta_N-\beta_\infty)$ are well-approximated by a centered multivariate normal distribution with covariance matrix $\Sigma$.
This matches the usual $\mathcal{O}(1/N)$ variance scaling from classical asymptotic theory: the
\text{\textsc{Cav}\xspace} estimator $\beta_N$ is closely related to a logistic regression estimator
\citep{goldman_zhang_2022}, and such estimators typically have covariance that decreases
proportionally to $1/N$. 
We show that this also holds in the infinitely imbalanced setting.

Let us now return to our main topic:  the variability of the \text{\textsc{Cav}\xspace}s. 
For this we choose the trace of the covariance matrix of the \text{\textsc{Cav}\xspace}s, \emph{i.e.}, the sum of the per-feature variance, as a variability measure.  
This measure equals zero, \emph{if and only if} the \text{\textsc{Cav}\xspace}s have no variability at all.

The following definition is standard in multivariate statistics; we restate it here for completeness and to fix notation.
\begin{definition}[Variability of CAVs]
\label{def:var-cav}
We define the variability of~$\beta_N$ as the trace of the covariance matrix
\begin{equation}
\var{\beta_N} \defeq \sum_j \cov{\beta_N}_{jj}.
\end{equation}
\vskip -0.1in
\end{definition}
\vskip -0.1in

We are now able to examine the variability of \text{\textsc{Cav}\xspace}s .
\begin{corollary}[Variability decay of CAVs]\label{cor:covariance_decay}
Under the assumptions of Theorem~\ref{thm:asymptotic_normality_cav} and uniform integrability
$\beta_N$ satisfies
\[
\Cov(\beta_N)=\mathcal{O}\left(1/N\right)
\qquad\text{as } N\to\infty.
\]
In particular, each component has $\Var \big((\beta_N)_j\big)=\mathcal{O}(1/N)$.
\label{cor:covariance_trace}
\end{corollary}
\vskip -0.1in
See Appendix~\ref{appendix:convergence_rate_of_covariance_trace} for a precise statement and proof. 
In words: as the random dataset size ($N$) grows, the variability in the direction of the Concept Activation Vector ($\beta_N$) shrinks at a rate proportional to $1/N$. For instance, multiplying the number of random examples $N$ by ten reduces the variance by a factor of ten.
Figure~\ref{fig:variance_of_cavs} confirms this scaling empirically; additional results across datasets and models appear in Section~\ref{sec:experiments}.

\subsection{Variance of Concept Sensitivity Scores}
\label{sec:variance_sensitivity_scores}

We now analyze whether the asymptotic normality of \text{\textsc{Cav}\xspace}s transfers to the concept sensitivity scores. 
In the notation of this section, they are written
\begin{equation}
\label{eq:taylor_expansion}
S(\mathbf{x},  {\beta}_N) \;:=\; 
  \nabla g_{\ell,k}\bigl( h_{\ell}(\mathbf{x}) \bigr) 
  \;\cdot\;
   {\beta}_N
  \, .
\end{equation}
The following corollary shows that we can transfer the asymptotic convergence of the \text{\textsc{Cav}\xspace}s.  
The variance of concept sensitivity  decreases at the rate $\mathcal{O}(1/N)$, like the variance of $\beta_N$, as demonstrated in Figure~\ref{fig:variance_of_sensitivity_scores}.

\begin{figure}[h!]
\centering
\includegraphics[width=0.45\textwidth]{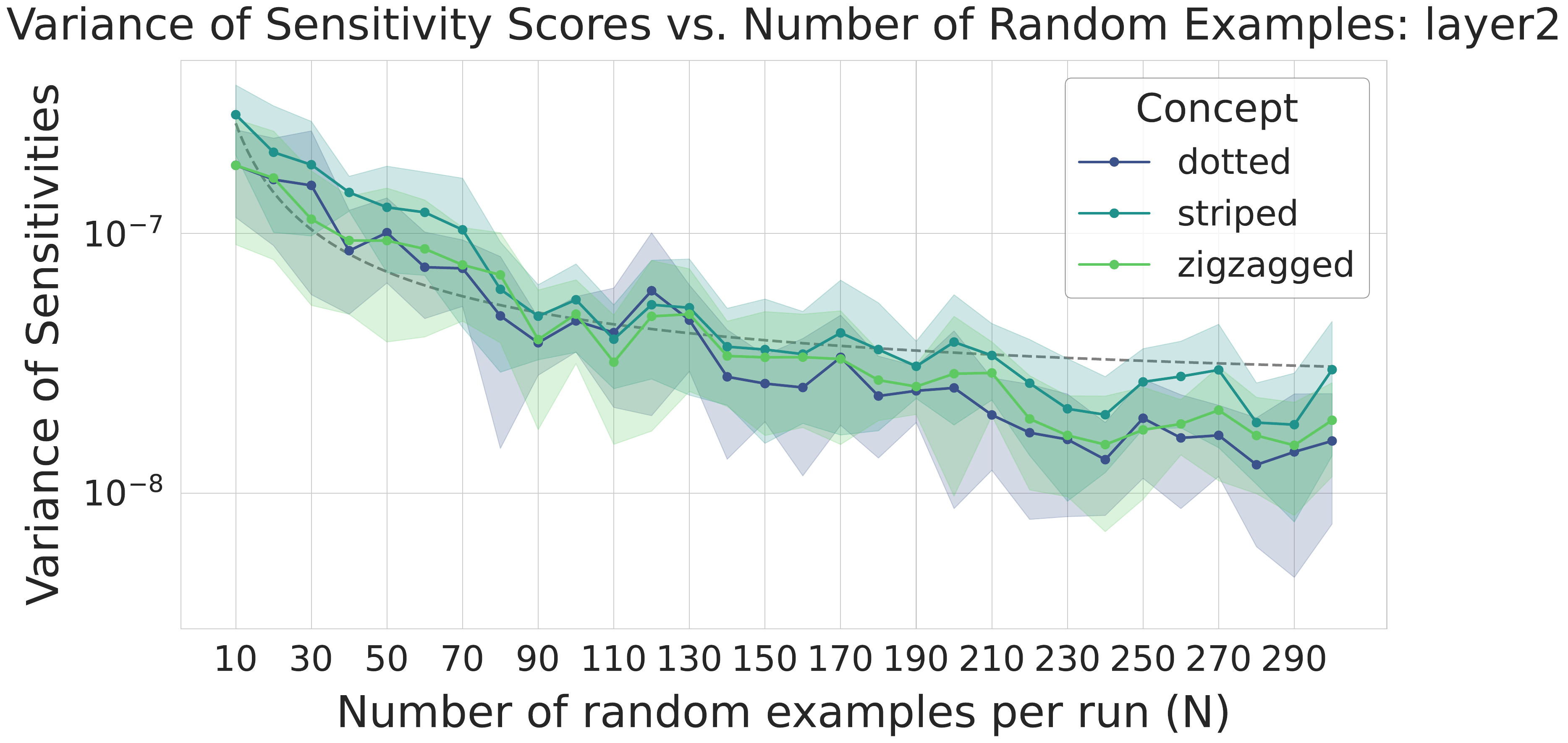}
\caption[Mean Variance of Concept Sensitivity Scores vs. Number of Random Samples Used]{\label{fig:variance_of_sensitivity_scores}Mean variance of concept sensitivity scores \emph{vs.} number~$N$ of random samples used for the class ``zebra'' in the Imagenet classification setting for one input. The $y$-value shows the geometric mean variance of these concept sensitivity scores $\pm 1$ standard deviation, averaged over all $r=10$ runs for a fixed positive input of a zebra. Log scale cut at $\geq 3 \times 10^{-9}$ so near‑zero bands do not dominate the picture.}
\end{figure}

\begin{corollary}[Asymptotic Normality of Concept Sensitivity]
\label{prop:asymptotic_normality_sensitivity}
Under the conditions specified in Corollary~\ref{cor:covariance_trace}, the concept sensitivity score $ S(\mathbf{x},  {\beta}_N) $ satisfies
\begin{equation}
Z_N := \sqrt{N} \left( S(\mathbf{x}, \beta_N) - S(\mathbf{x}, \beta_\infty) \right) \xrightarrow{D} \mathcal{N}\left( 0, V(\mathbf{x}) \right),
\label{eq:Z_N}
\end{equation}
where the asymptotic variance $ V(\mathbf{x}) $ is given by
\begin{equation}
V(\mathbf{x}) := \nabla g_{\ell,k}\bigl( h_{\ell}(\mathbf{x}) \bigr)^\top \cdot \Sigma \cdot \nabla g_{\ell,k}\bigl( h_{\ell}(\mathbf{x}) \bigr) \, .
\label{eq:asymptotic_variance_sensitivity}
\end{equation}
\end{corollary}

This corollary follows directly from Theorem~\ref{thm:asymptotic_normality_cav}: since $S(\mathbf{x}, \beta_N)$ is a linear function of $\beta_N$, the asymptotic normality transfers immediately. The $\mathcal{O}(1/N)$ variance decay is confirmed empirically in Figure~\ref{fig:variance_of_sensitivity_scores} (see also Appendix~\ref{appendix:sensitivity} for the full proof).

\subsection{Variance of TCAV Scores}
\label{sec:variance_tcav_scores}

Extending our variance analysis to the \text{\textsc{Tcav}\xspace} scores in Eq.~\eqref{eq:tcav_score_algo}, we make a surprising observation. In practice, the variance does not always decrease with more random samples, as one might assume. 
Experimentally, we observe that the variance of the \text{\textsc{Tcav}\xspace} scores initially decreases and then plateaus in many settings for medium to high values of $N$ of random embedding vectors used (see Figure~\ref{fig:variance_tcav}).

\begin{figure}[ht!]
\centering
\includegraphics[width=0.48 \textwidth]{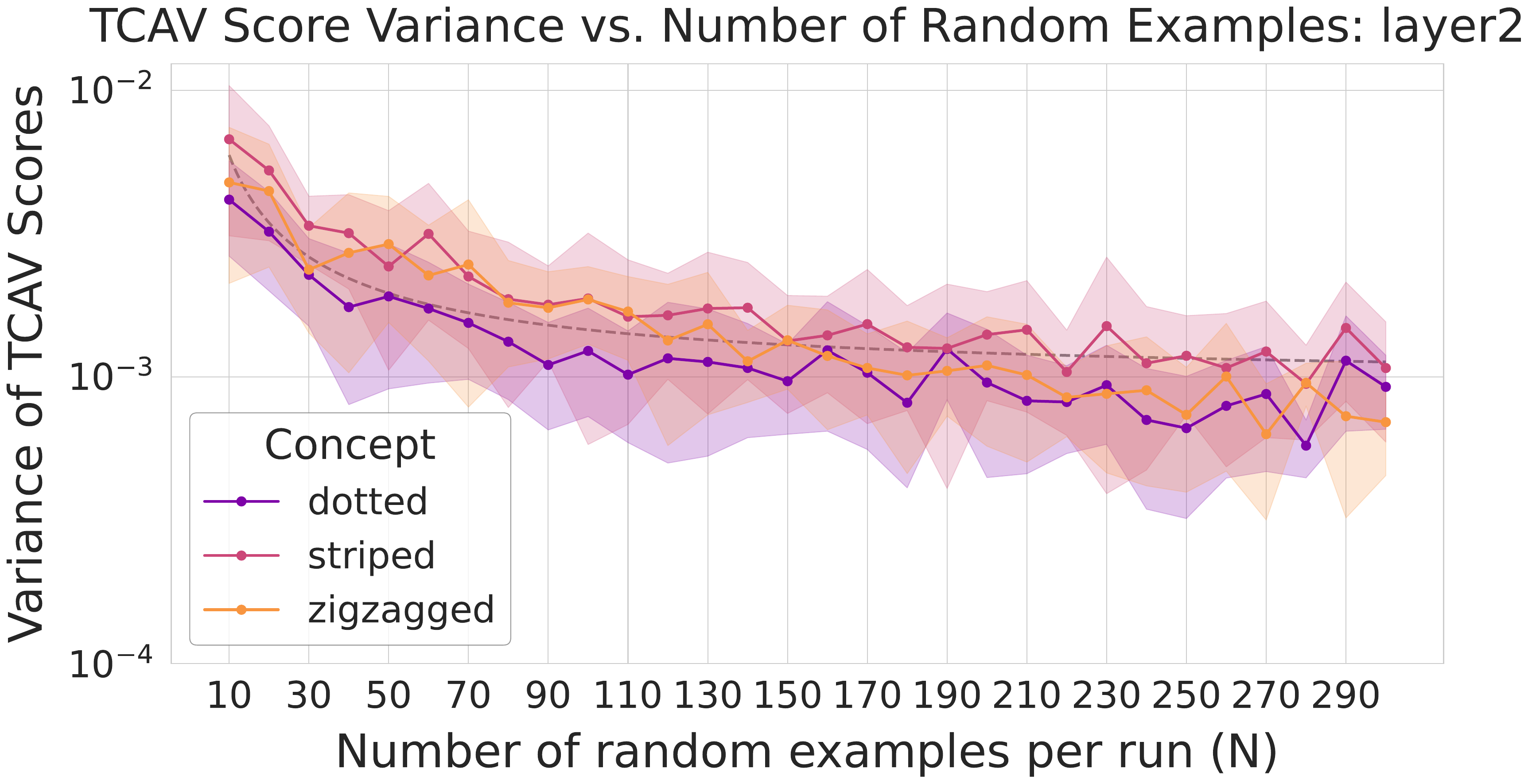}
\caption[Variance of the \text{\textsc{Tcav}\xspace} scores vs. Number of Random Samples Used ]{\label{fig:variance_tcav}Mean variance of \text{\textsc{Tcav}\xspace} scores at layer \texttt{`layer2`} vs. the number of random examples per concept set $N$ for ``striped''-, ``zigzagged''-, and ``dotted'' concepts on the Imagenet dataset; error bars denote $\pm 1$ standard deviation. $Y$‑axis clipped at $\geq 10^{-4}$ to reduce distortion from near‑zero standard deviations. 
} 
\end{figure}

Our intuition for this quite surprising fact is the following: across the embeddings of samples from a specific class over which we compute the \text{\textsc{Tcav}\xspace} score, a subset of them may lie \textbf{on} or \textbf{near} the decision boundary. 
This makes their classification highly sensitive to small changes in the model. 
Those samples, which we call ``borderline points,'' therefore still contribute to the variance of the \text{\textsc{Tcav}\xspace} score. This holds even when the variance of the \text{\textsc{Cav}\xspace}s vanishes asymptotically. For any other ``non-borderline'' point, the classification is asymptotically stable, and its contribution to the variance is negligible.
The covariance between any two borderline points, on the other hand, is a constant value, $\Theta(1)$, that does not decrease as the \text{\textsc{Cav}\xspace}  estimate improves. Therefore, the total variance is dominated by the sum over the pairs of  ``borderline points'' and $\operatorname{Var}(\mathrm{\TCAV}) = \mathcal{O}(1).$ More examples of this behavior are provided in Appendix~\ref{sec:additional-results}. 
Based on our empirical and theoretical observations we can now give concrete recommendations for action.

\subsection{Variability of the Multi-Run Method}
\label{sec:multi_run}

Based on our observations, averaging \textsc{Tcav}\xspace scores over multiple runs, as done in the original implementation, is indeed the most favourable method. Our analysis shows that as $s$ increases, the variance decreases as $\mathcal{O}(1/s)$ (see Figure~\ref{fig:var_tcav_multi}). This aligns with theoretical expectations: since the individual scores $T_1, \ldots, T_s$ are independent, averaging $s$ i.i.d.\ estimates with common variance $\sigma^2$ yields variance $\sigma^2/s$~\citep{hogg_introduction_2019}. Hence, assuming $\operatorname{Var}(\mathrm{TCAV}) = \mathcal{O}(1)$, the variance of the multi-run score satisfies $\operatorname{Var}(T_{\mathrm{multi}}) = \mathcal{O}(1/s)$.

\begin{figure}[h!]
    \centering
    \includegraphics[width=0.45\textwidth]{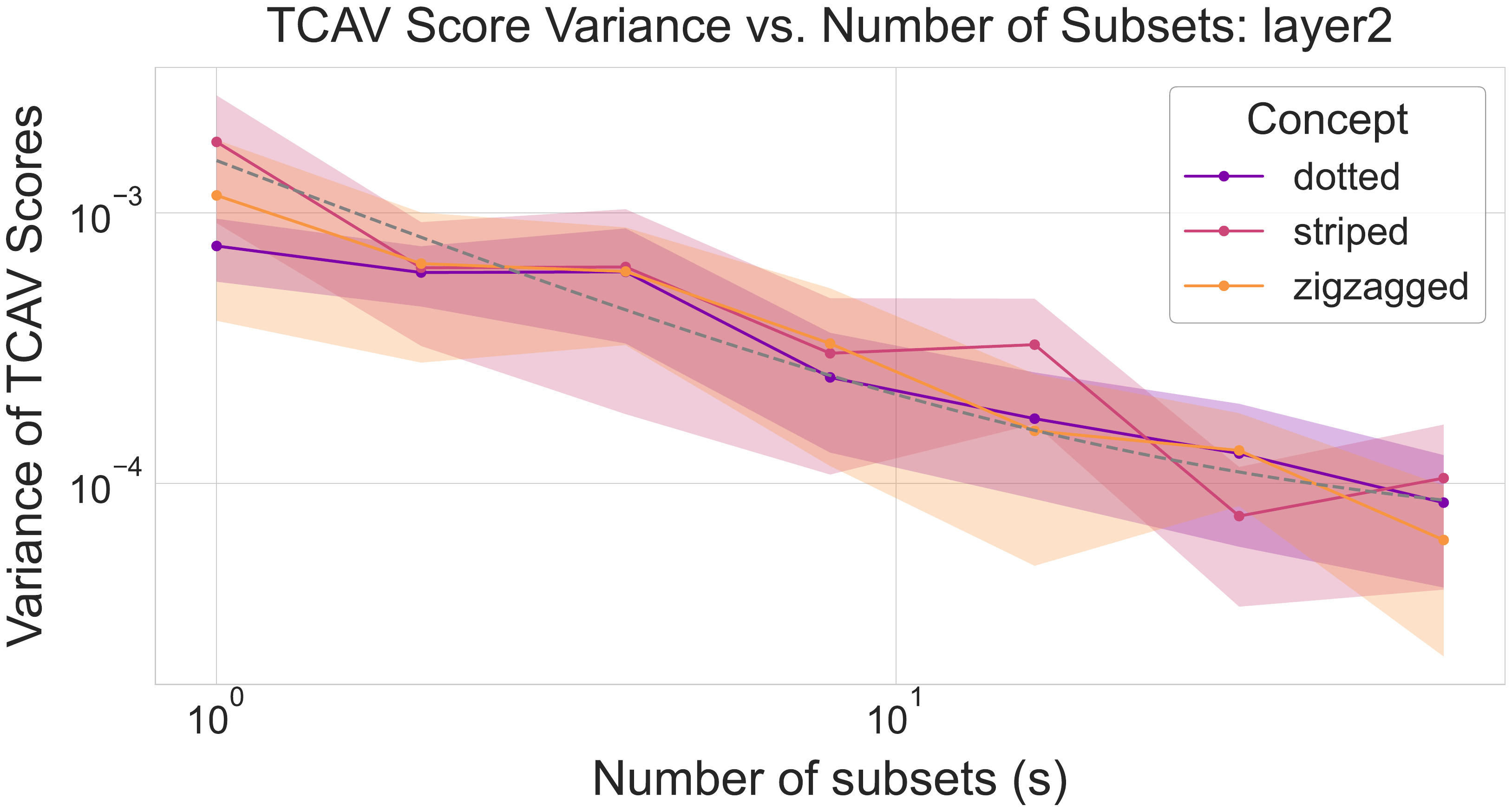} 
    \caption[Variance of \text{\textsc{Tcav}\xspace} vs. Number of Subsets]{\label{fig:var_tcav_multi}Variance of multi-run \text{\textsc{Tcav}\xspace} scores in the image setting. We use a fixed number $R = 2000$ of random samples and divide them over a varying number of subsets $s$, with each subset containing $N = R/s$ samples. To get the variance of the mean over all $s$ \text{\textsc{Tcav}\xspace} scores, we repeat this $r=10$ times and calculate the variance for each $s$ over the $r$ runs. The plot shows that the variance of the mean \text{\textsc{Tcav}\xspace} score decreases as $s$ increases, confirming our theoretical analysis. Finally, we repeat this experiment $e=10$ times and report the mean variance over those $e$ runs $\pm 1$ geometric standard deviation.
    }
\end{figure}

\subsection{Recommendations for Practice}
\label{sec:recommendations}

In Section~\ref{sec:introduction} we ask: \emph{``How can the variance of \text{\textsc{Cav}\xspace}s be minimized within a limited resource budget?''} Our theoretical and empirical analysis provides concrete guidance.

\textbf{Compute-variability trade-off.}  
Empirically, we find that compute cost scales linearly with $N$ for all \CAV estimators examined (logistic regression, hinge loss SVM, and Difference of Means) and quadratically for \CAR estimators (see Figure~\ref{fig:compute_vs_n_images} in the Appendix).

\begin{figure}[ht]
\includegraphics[width=\linewidth]{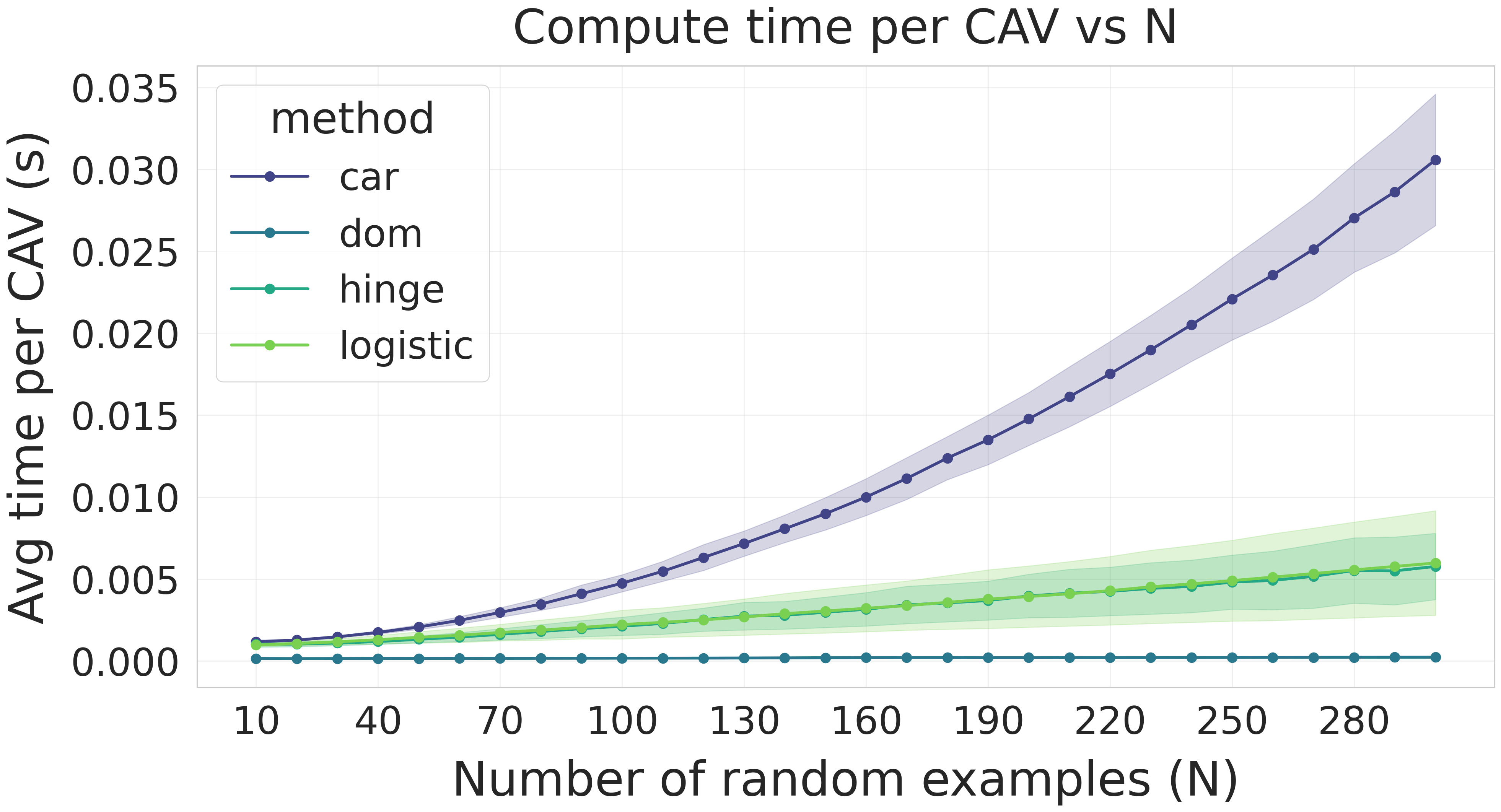}
        \caption{Average wall-clock time per concept representation as a function of the number of random reference examples per run ($N$). We compare logistic-regression \textsc{Cav}s (\textsc{Logistic}), hinge-loss \textsc{Cav}s (\textsc{Hinge}), Difference-of-Means \textsc{Cav}s (\textsc{DoM}), and Concept Activation Regions (\textsc{CAR}). Shaded bands denote $\pm 1$ standard deviation across repeated runs.}
\end{figure}

By Corollary~\ref{cor:covariance_decay}, the variance of \text{\textsc{Cav}\xspace}s decreases as $\mathcal{O}(1/N)$.  
However, the constant factors differ substantially across estimators.
In the example of tabular data (see Section~\ref{sec:experiments}), we summarize the resulting trade-off between stability and computational effort in Figure~\ref{fig:compute-variability-tradeoff}. 
Each point represents a different sample size $N$, showing average compute time per \textsc{Cav}\xspace versus \textsc{Cav}\xspace variability. 
The \texttt{Difference-of-Means} estimator clearly achieves the lowest variability for any given compute budget. 
This aligns with Theorem~\ref{thm:asymptotic_variance_dom}, which shows \texttt{Difference-of-Means} exhibits the same $\mathcal{O}(1/N)$ convergence rate while requiring only a single mean computation rather than iterative optimization.

 We can now present our practical recommendations given a fixed compute budget. 
\paragraph{Recommendation 1: For stable \text{\textsc{Cav}\xspace}s, use \texttt{Difference-of-Means} and maximize $\mathbf{N}$.}
According to~\citet{schmalwasser_fastcav_2025}, this estimator is  on average about ${\sim}46\times$ faster than logistic regression. We show that it also yields comparable or better stability for a given compute budget. Independent of the method, we recommend selecting $N$ based on your target stability level: by Equation~\eqref{eq:asymptotic-variance-dom}, $\mathrm{Var}(\beta_N) = \mathrm{tr}(\Sigma_z)/N$, so \textbf{doubling $\mathbf{N}$ halves the variance}.

\paragraph{Recommendation 2: For stable \TCAV scores, use multiple runs with modest $\mathbf{N}$.} 
Our surprising finding (Section~\ref{sec:variance_tcav_scores}, Figure~\ref{fig:variance_tcav}) is that \TCAV score variance does not always decrease with $N$. 
Instead, averaging over $s$ independent runs reduces variance as $\mathcal{O}(1/s)$. Given a fixed budget of $R$ total random samples, allocating them as $s$ runs of $N = R/s$ samples each is more effective for \TCAV stability than a single run with all $R$ samples. Note that this increases compute time: doubling $s$ approximately doubles the number of classifier trainings (with a factor of $\approx 1.5\times$ in our experiments due to overhead).

In conclusion, the optimal strategy depends on the use:
\begin{itemize}[nosep]
    \item \textbf{For \textsc{Cav}\xspace direction stability} (\emph{e.g.}, bias mitigation, feature steering): maximize $N$ per run with respect to your compute budget, use \texttt{Difference-of-Means}.
    \item \textbf{For \TCAV score stability} (\emph{e.g.}, concept importance rankings): use multiple modest-$N$ runs and average the scores.
\end{itemize}

\begin{figure}[ht]
  \centering
  \includegraphics[width=\linewidth]{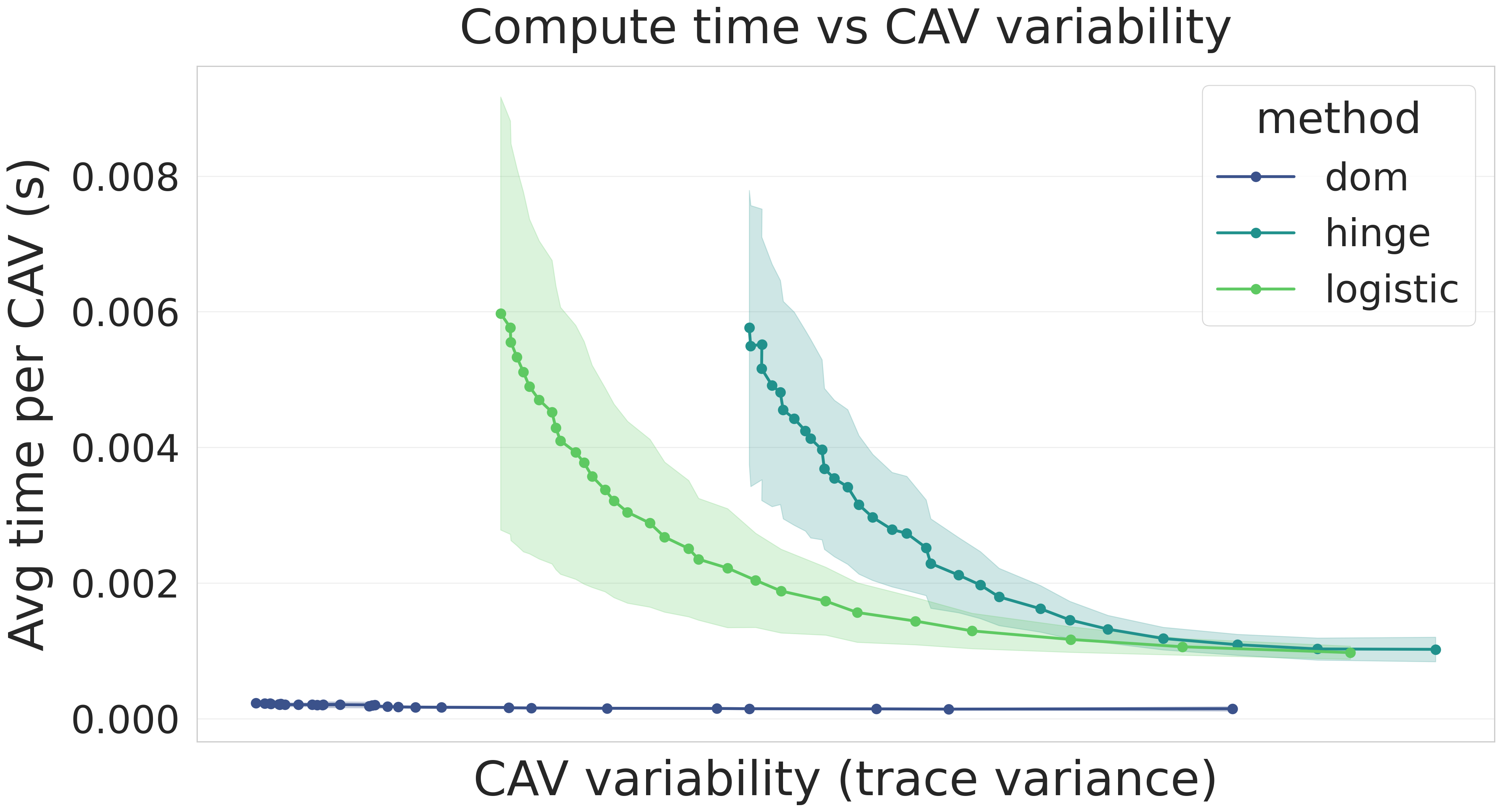}
  \caption{Compute-variability trade-off for \textsc{Cav}\xspace estimators on tabular data. We show average compute time per \textsc{Cav}\xspace vs.\ the \textsc{Cav}\xspace variability across estimators. Each point corresponds to a different random-sample size $N$ averaged over all layers.}
  \label{fig:compute-variability-tradeoff}
\end{figure}

\section{Experiments}
\label{sec:experiments}

We validate our theoretical findings and assumptions for the four estimation methods using four data modalities: images, tables, text, and multimodal data.  
For each, we use appropriate datasets and concepts (some from existing benchmarks, others newly created).
We verify all theorems and corollaries and empirically substantiate our practical recommendations.
We summarize the main findings below; Appendix~\ref{sec:additional-results} contains the complete results for all estimators and modalities, concept-set-size variation in Appendix~\ref{appendix:concept_size_variation}, and implementation ablations in Appendix~\ref{appendix:implementation_ablations}.   

\textbf{Experimental protocol.}
For a given run, we vary the number of random examples, $N$ from $10$ to $300$. At each value of $N$, we compute $s=50$ separate \textsc{Cav}s. Each of these \textsc{Cav}s is trained on a set of $N$ examples, which are sampled with replacement from a large pool of $1{,}000$ (for images) to $50{,}000$ (for text) samples. We then calculate the variability of the resulting $50$ \textsc{Cav}s.
We repeat this experiment $r=10$ times for statistical significance and report the mean and standard deviation of the traces collected across all $r$ runs.

\textbf{Image Data. } 
Following the original \textsc{Tcav} setup, we use ImageNet~\citep{deng_imagenet_2009} with concept definitions from Broden~\citep{bau_network_2017}.
Our main results use a pre-trained ResNet-50~\citep{he_deep_2016} (\texttt{layer2} and \texttt{layer3}), but we also run experiments on GoogLeNet \citep{szegedy_going_2014}, MobileNetV3 \citep{howard_searching_2019}, and ViT-B/16 \citep{dosovitskiy_an_2021}. All show comparable variance scaling.
We focus on ResNet-50 for clarity; full results appear in Appendix~\ref{appendix:images}.

\textbf{Tabular Data. } 
While image classification is the typical use case for concept-based explanations, our theoretical results apply to any differentiable model. To verify this generality, we adapt the framework to income prediction on the \emph{UCI Adult} dataset~\citep{becker_adult_1996}.
Unlike vision tasks where concepts require external datasets, tabular data often contains interpretable attributes that can serve directly as concept definitions. We leverage this by defining two concepts, ``male'' and ``female,'' from the \texttt{sex} attribute.
The model is a small network trained to predict whether income exceeds $\$50{,}000$.
Results are provided in Appendix~\ref{appendix:tabular}.

\paragraph{Text Data.}
We apply \textsc{Cav}s in the \NLP setting using the IMDB sentiment dataset
\citep{maas_learning_2011}. We consider two text model families. First, we use
the pre-trained BERT-based sentiment classifier from Captum
\citep{kokhlikyan_captum_2025} and extract representations from layers
\texttt{bert\_encoder\_2} and \texttt{bert\_encoder\_4}. Concepts are defined
by hand-picked sets of positive, negative, and neutral adjectives, extending the
Captum notebooks. Second, we evaluate \texttt{SmolLM2-1.7B-Instruct}~\citep{allal_smollm2_2025} and
extract \text{\textsc{Cav}\xspace}s from decoder layers
\texttt{decoder\_hidden\_6} and \texttt{decoder\_hidden\_12}. For both text
model families, we report results for binary cross-entropy, hinge loss,
Concept Activation Regions, and Difference of Means in
Appendix~\ref{appendix:text}.

\paragraph{Multimodal Data.}
To evaluate our framework on vision-language models, we use \CLIP
\citep{radford_learning_2021} and \textsc{LLaVA-1.5} \citep{liu_improved_2024}. For \CLIP,
we pair zebra images with the prompt ``a photo of a zebra'' and use image-text
cosine similarity as the scalar output instead of a target class logit. Concepts
are defined by image sets, as in the image experiments. We extract \textsc{Cav}s from
\CLIP vision layers \texttt{vision\_hidden\_4} and \texttt{vision\_hidden\_8}.
For \textsc{LLaVA-1.5}, we use the same visual concept protocol and extract \textsc{Cav}s from
vision layers \texttt{vision\_hidden\_7} and \texttt{vision\_hidden\_15}. For
both vision-language model families, we report results for binary
cross-entropy, hinge loss, Concept Activation Regions, and Difference of Means
in Appendix~\ref{appendix:multimodal}.

\textbf{Results.} 
Across all modalities, model families, and estimators, \text{\textsc{Cav}\xspace} variability decays
approximately as $1/N$, consistent with our theoretical predictions. Implementation ablations further show that layer choice, feature normalization, regularization, initialization, and optimizer choice mainly affect finite-sample constants rather than the central stabilization trend.
The \TCAV score tells a different story. 
In most settings, \TCAV score variance remains approximately constant even as the underlying \text{\textsc{Cav}\xspace} variability decreases with $N$. We attribute this apparent paradox largely to the borderline points discussed in Section~\ref{sec:variance_tcav_scores}: the \TCAV score depends on the \emph{sign} of concept sensitivities, and points near the decision boundary can flip classification regardless of how precisely the \text{\textsc{Cav}\xspace} direction is estimated.\label{sec:results}

\section{Limitations}
\label{sec:limitations}

Our primary focus is on stability. 
Our analysis does not determine which method performs best on downstream tasks, since such comparisons are difficult without available ground-truth. Nevertheless, we find that the resulting \textsc{Cav}s are very  similar across methods and converge toward comparable estimates. This is consistent with prior work demonstrating that different \CAV \ computation methods yield similar results under common conditions~\citep{soudry_the_2018, schmalwasser_fastcav_2025, schnoor_cav_2026}.

Moreover, the ablations show that the constants in the variance decay can depend on the representation layer and implementation choices. Thus, the qualitative recommendation to increase the number of random reference examples is robust, but numerical budgets should be recalibrated when moving to a new architecture, layer, or \text{\textsc{Tcav}\xspace} implementation.

\section{Conclusion}

In this paper, we analyze the stability of the \text{\textsc{Tcav}\xspace} method. To this end, we introduce a mathematical framework to theoretically analyze variability in the limit of infinitely imbalanced logistic regression, which applies beyond the scope of this work. Building on this analysis and extensive experiments, we provide practical recommendations on how many random samples to use and how to allocate them for optimal stability in \text{\textsc{Tcav}\xspace}. We further present a thorough exploration of the compute-consistency trade-off, quantifying how a fixed budget is best split between (i) more independent runs with fewer random samples per run and (ii) fewer runs with larger per-run sample sizes. 
\emph{Crucially, this trade-off is method- and implementation-dependent}: it differs for \text{\textsc{Tcav}\xspace} scores versus \text{\textsc{Cav}\xspace}s and varies with the logistic-regression solver, regularization, and feature normalization.
We demonstrate that, for stable \text{\textsc{Tcav}\xspace} scores, a modest number of random samples distributed across multiple independent runs is sufficient and compute-efficient. 
By contrast, applications that require stable \text{\textsc{Cav}\xspace}s, such as bias mitigation or feature steering, benefit more from larger per-run sample sizes. We emphasize again that there is no single best setting: the optimal compute allocation should be re-evaluated for the specific method and implementation at hand.

Building on our findings, we identify the development of more stable \text{\textsc{Tcav}\xspace} variants as an important direction for future work. Rather than only quantifying variability or mitigating it through repeated runs and larger reference sets, future methods could incorporate stability directly into the construction of concept directions and scores.  Moreover, further work should investigate implementation-specific compute-variability trade-offs, since our results show that stability depends not only on the number of random samples, but also on choices such as the estimator, solver, regularization, feature normalization, and stopping criteria.

\section*{Impact Statement}
Reliable concept-based explanations are only useful if they are reproducible. Yet, in practice \TCAV can shift across runs because the concept direction is learned from random reference samples, leaving teams guessing whether a reported effect is real or just sampling noise. Our results turn ``how many random examples do I need?'' into concrete, budgetable rules. We hope this strengthens \TCAV and related methods by improving reproducibility, enabling fairer comparisons across studies, and reducing wasted compute by investing resources where they actually buy stability.

\section*{Acknowledgements}
We thank Gabriele Ciravegna for his valuable feedback. Prototyping and parts of the development were carried out on the Julia 2 cluster. Julia 2 was funded as DFG project as ``Forschungsgroßgerät nach Art 91b GG'' under INST 93/1145-1 FUGG.

\bibliography{refs}
\bibliographystyle{icml2026}
\appendix
\onecolumn

\section*{Appendix}

\subsection*{Table of Contents}
\begin{enumerate}[label=\Alph*.]
    \item \hyperref[sec:additional-results]{Appendix A: Additional Results}
    \begin{enumerate}[label=A.\arabic*.]
        \item  \hyperref[appendix:images]{TCAV for Images}
        \item  \hyperref[appendix:tabular]{TCAV for Tabular}
        \item  \hyperref[appendix:text]{TCAV for Text}
        \item  \hyperref[appendix:multimodal]{TCAV for Multimodal}
        \item  \hyperref[appendix:concept_size_variation]{Concept-Set-Size Variation}
        \item  \hyperref[appendix:implementation_ablations]{Implementation Ablations}
        
    \end{enumerate}
    \item \hyperref[sec:appendix-proof]{Appendix B: Proofs of the Main Results}
    \begin{enumerate}[label=B.\arabic*.]
        \item  \hyperref[appendix:cross_entropy]{ Asymptotic Normality with Binary Cross-Entropy Loss}
        \item  \hyperref[appendix:hinge_loss]{Asymptotic Normality with Hinge Loss}
        \item  \hyperref[appendix:difference_of_means]{Variability of Difference of Means}
        \item  \hyperref[appendix:concept_activation_regions]{Variability of Concept Activation Regions}
    \end{enumerate}
\end{enumerate}

\section{Appendix A: Additional Results}
\label{sec:additional-results}

In Appendix~A, we provide the full experimental evidence behind Section~\ref{sec:experiments}. We report results across the four modalities studied in the paper and across the four \text{\textsc{Cav}\xspace} estimation methods: binary cross-entropy, hinge loss, Concept Activation Regions, and Difference of Means.
For each setting, we report the relevant variability quantities: \text{\textsc{Cav}\xspace}/\text{\textsc{Car}\xspace} variability, concept-sensitivity variance where available, and \text{\textsc{Tcav}\xspace}/\text{\textsc{Tcar}\xspace} score variance.
We also include additional SmolLM2 and LLaVA experiments, concept-set-size variation, and implementation ablations over layer choice, feature normalization, regularization, initialization, and optimizer choice.

\textbf{Layer selection.}
Unless otherwise stated, we report two representative intermediate layers per architecture, one earlier/middle layer and one later layer. This keeps the main comparison readable while still testing layer dependence; a dedicated layer ablation is provided in Appendix~\ref{appendix:implementation_ablations}.

\paragraph{Note.}
All experiments were conducted on a MacBook Pro running macOS Sonoma $14.6$. The system was equipped with an Apple $M3$ Max chip ($14$-core CPU) and $36$ GB of unified memory.

\setcounter{assumption}{0}
\renewcommand{\theassumption}{\arabic{assumption}}
\setcounter{theorem}{0}
\renewcommand{\thetheorem}{\arabic{theorem}}

\subsection{TCAV for Images}
\label{appendix:images}

Finally, we evaluate \TCAV on images. 
We report results for ResNet50~\citep{he_deep_2016}, but the method is readily applicable to a wide range of architectures; we have also implemented it for GoogLeNet~\citep{szegedy_going_2014}, EfficientNet~\citep{efficientnet_tan_2019}, MobileNetV3~\citep{howard_searching_2019}, and ViT-B/16~\citep{dosovitskiy_an_2021}. 
Across models, the empirical behavior aligns with our propositions and theorems, though to varying degrees. Some settings converge more slowly, which is expected given the asymptotic nature of our results.

\subsubsection{Empirical Findings with Binary Cross-Entropy Loss}

Finally, we present the results for image data using the logistic regression classifier with binary cross-entropy loss. 

\begin{figure}[!ht]
    \centering
    \begin{minipage}{0.48\textwidth}
        \centering
        \includegraphics[width=\linewidth]{figures/img/cav_vector_variance_resnet50_layer2_logistic_old.pdf}
    \end{minipage}
    \hfill 
    \begin{minipage}{0.48\textwidth}
        \centering
        \includegraphics[width=\linewidth]{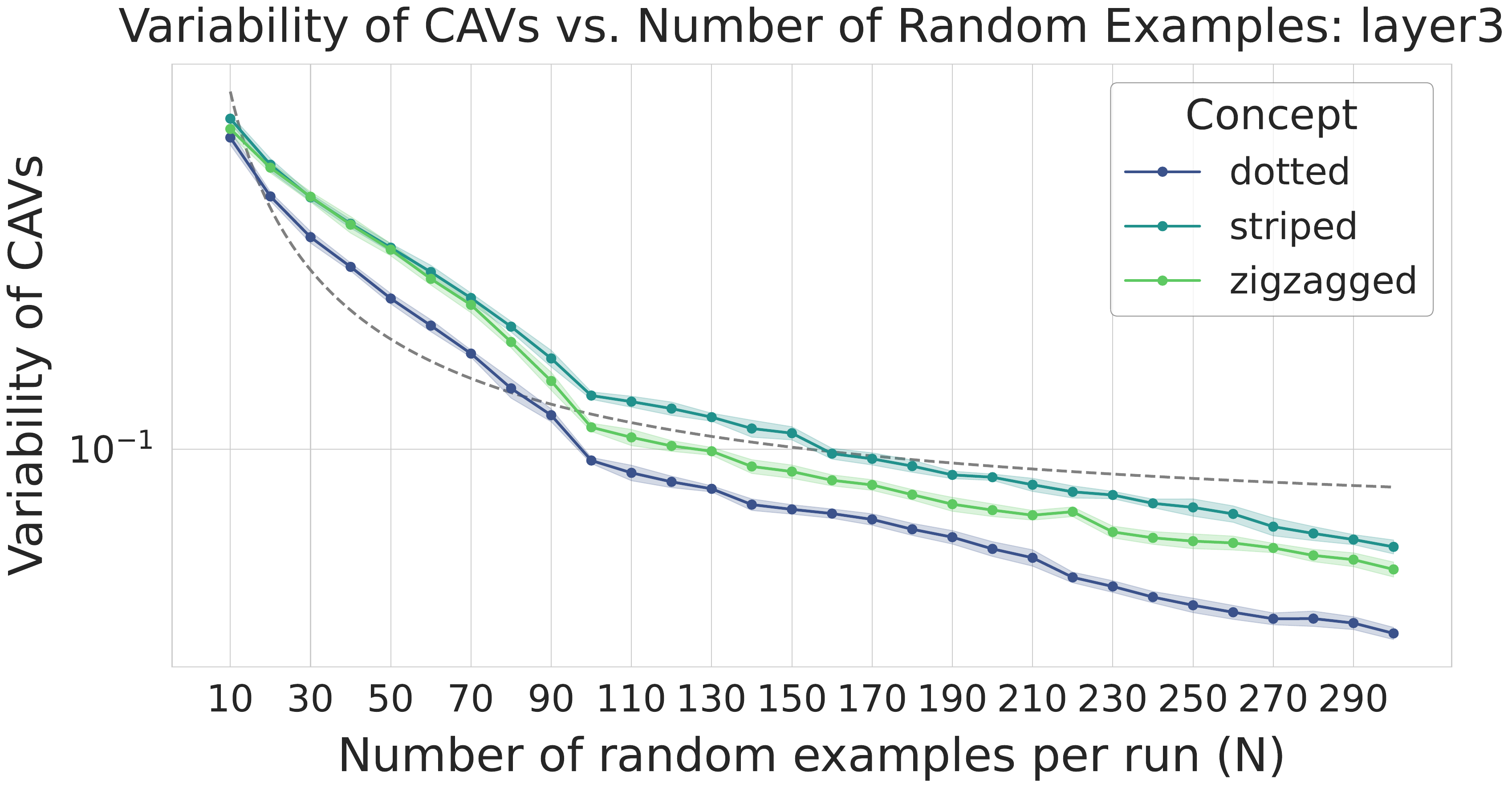}
    \end{minipage}
    \caption[Variance of CAVs on ImageNet]{ \label{fig:variance_of_cavs_log_images}
    Mean variability of \text{\textsc{Cav}\xspace}s for ``striped'', ``zigzagged,'' and ``dotted'' as a function of random examples ($N$) for \texttt{'layer2'} and \texttt{'layer3'} of the \textbf{Resnet50 model}. The classifiers were trained using \textbf{binary cross-entropy loss}. Error bars indicate $\pm 1$ SD; the $y$-axis is log-scaled. We fitted a curve of the form $f(N) = a/N + b$ to it. For \texttt{layer2} the parameters were $a=7.71, b=0.0999$, for \texttt{layer3} they were $a=5.58, b=0.0643$.
}
\end{figure}

\begin{figure}[H]
    \centering
    \begin{minipage}{0.48\textwidth}
        \centering
        \includegraphics[width=\linewidth]{figures/img/tcav_score_variance_resnet50_layer2_logistic.pdf}
    \end{minipage}
    \hfill 
    \begin{minipage}{0.48\textwidth}
        \centering
        \includegraphics[width=\linewidth]{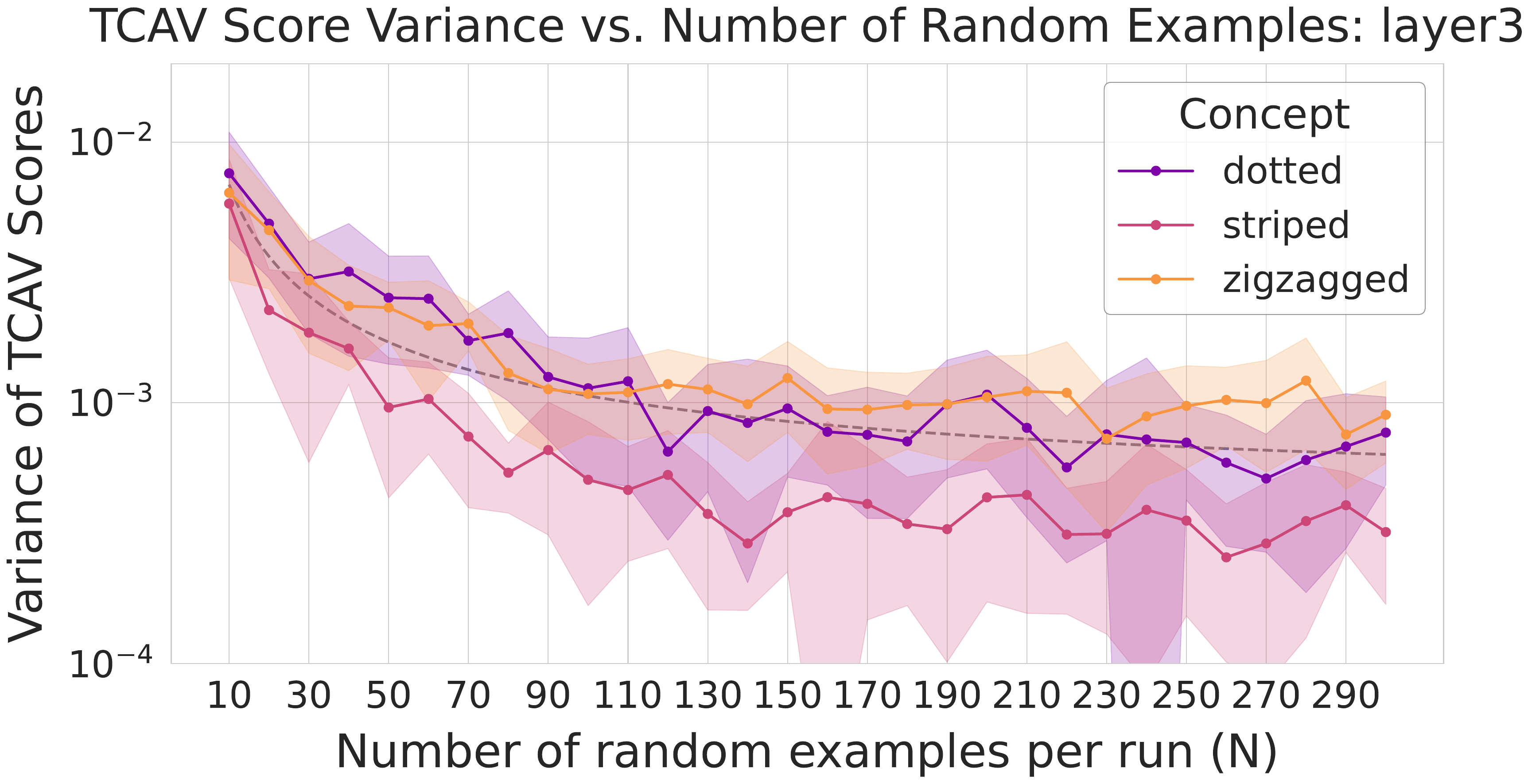}
    \end{minipage}
    \caption[Variance of TCAV scores (Log Loss) on ImageNet]{\label{fig:variance_of_tcavs_log_images}
    Variance of \text{\textsc{Tcav}\xspace} scores at \texttt{'layer2'} and \texttt{'layer3'} of the \textbf{Resnet50 model} vs. the number of examples per concept set ($N$) for the visual concepts ``striped'', ``zigzagged'', and ``dotted''.  Error bars denote $\pm 1$ standard deviation. We fitted a curve of the form $f(N) = a/N + b$ to it. For \texttt{layer2} the parameters were $a=0.0396, b=0.00134$, for \texttt{layer3} they were $a=0.0784, b=5.83\times 10^{-4}$. For \texttt{layer2} the $y$‑axis was clipped at $\geq 10^{-4}$. 
}
\end{figure}

\subsubsection{Empirical Findings with Hinge Loss}

Next, we show the corresponding results with hinge loss.

\begin{figure}[H]
    \centering
    \begin{minipage}{0.48\textwidth}
        \centering
        \includegraphics[width=\linewidth]{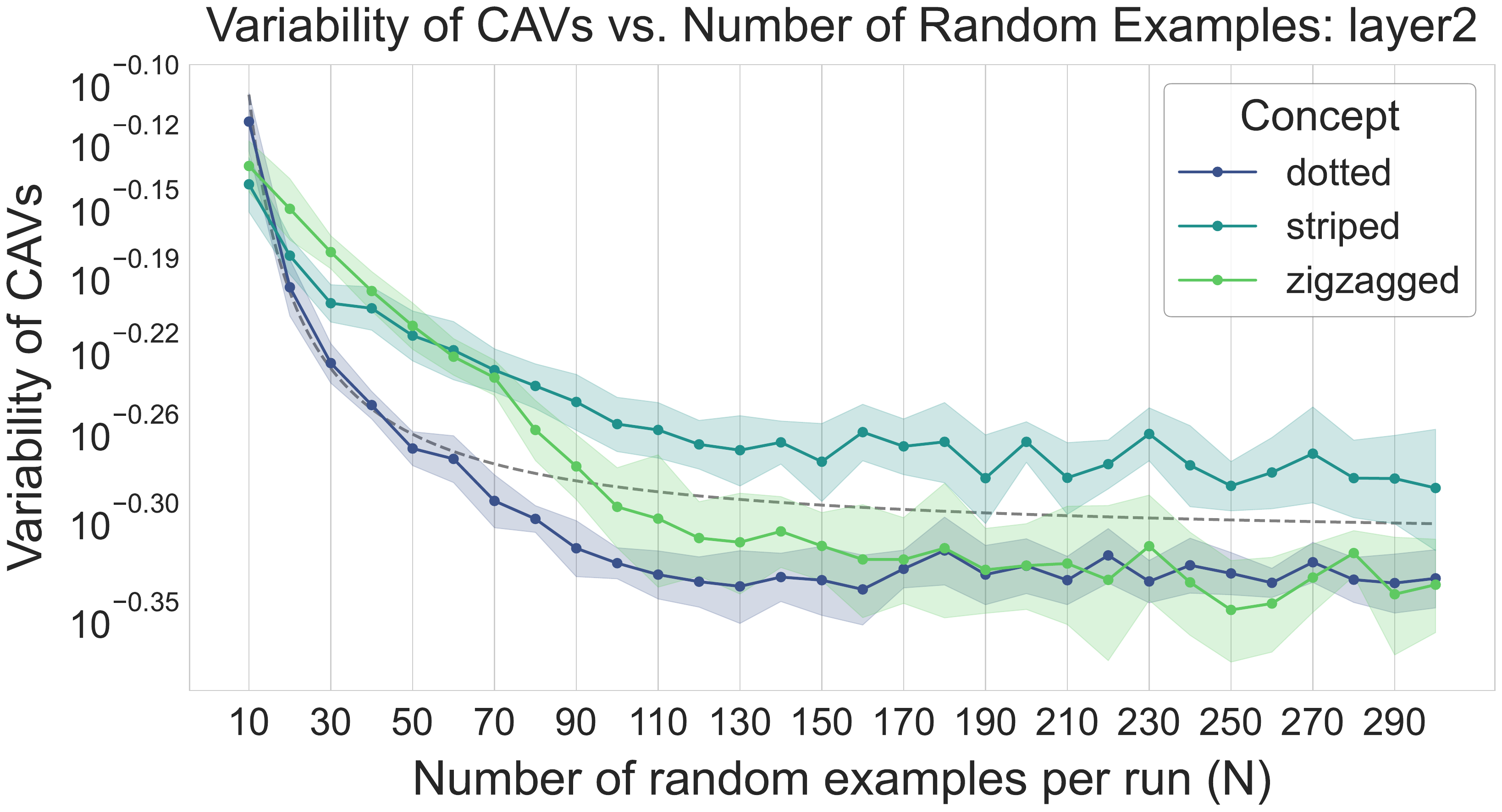}
    \end{minipage}
    \hfill
    \begin{minipage}{0.48\textwidth}
        \centering
        \includegraphics[width=\linewidth]{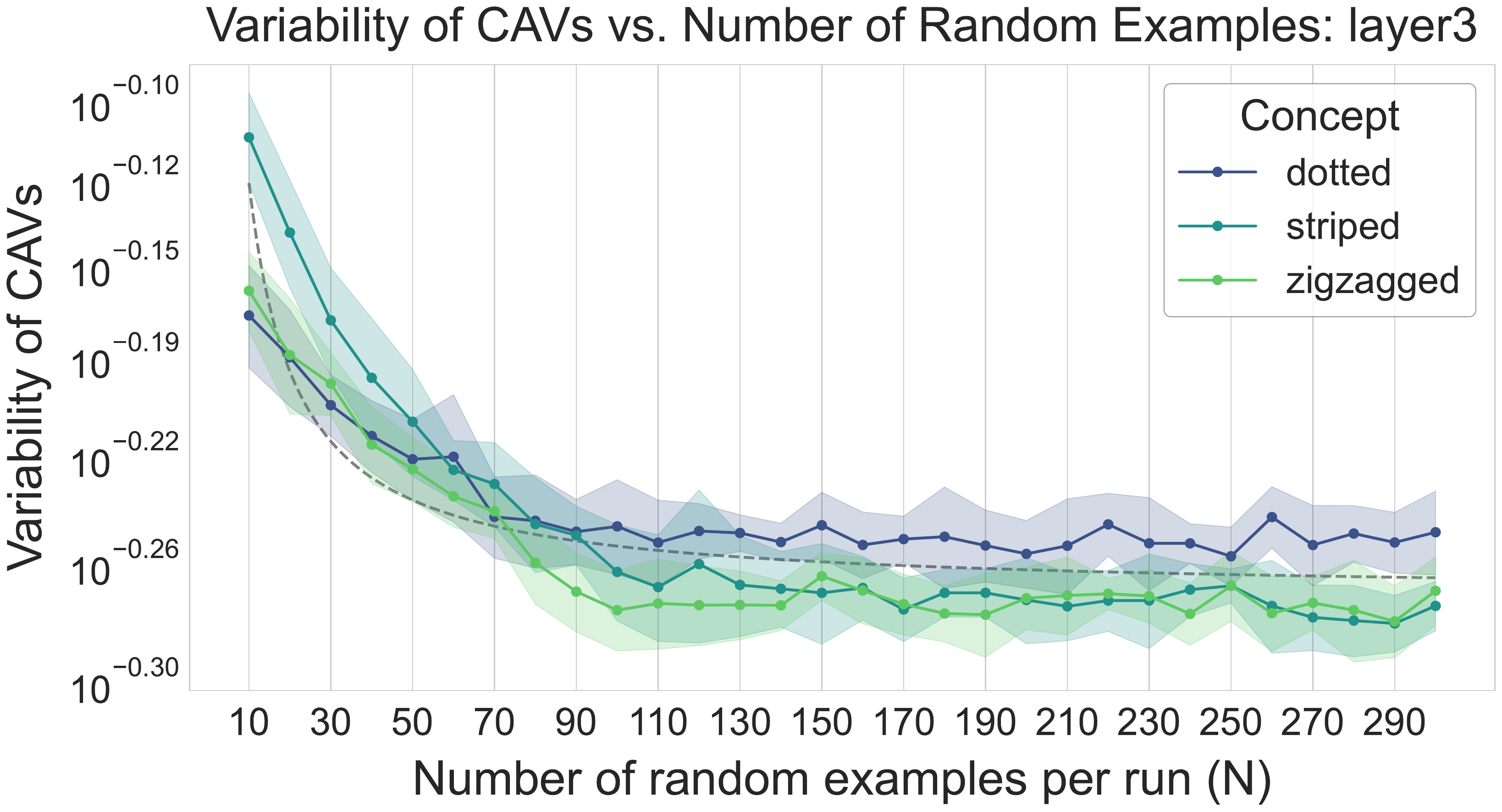}
    \end{minipage}

    \caption[Variance of CAVs (Hinge Loss) on ImageNet]{\label{fig:variance_of_cavs_hinge_images}
    Mean variability of \text{\textsc{Cav}\xspace}s for the visual concepts ``striped'', ``zigzagged'', and ``dotted'' as a function of the number of random examples per run ($N$). Results are shown for two different layers of the \textbf{Resnet50 model} (\texttt{'layer2'} on the left, \texttt{'layer3'} on the right). Error bars indicate $\pm 1$ SD; the $y$-axis is log-scaled. Variance is estimated by the sum of per-feature variances across runs. We fitted a curve of the form $f(N) = a/N + b$ to it. For \texttt{layer2} the parameters were $a=3.01, b=0.488$, for \texttt{layer3}  $a=2.12, b=0.538$.
}
\end{figure}

\begin{figure}[H]
    \centering  
    \begin{minipage}{0.48\textwidth}
        \centering
        \includegraphics[width=\linewidth]{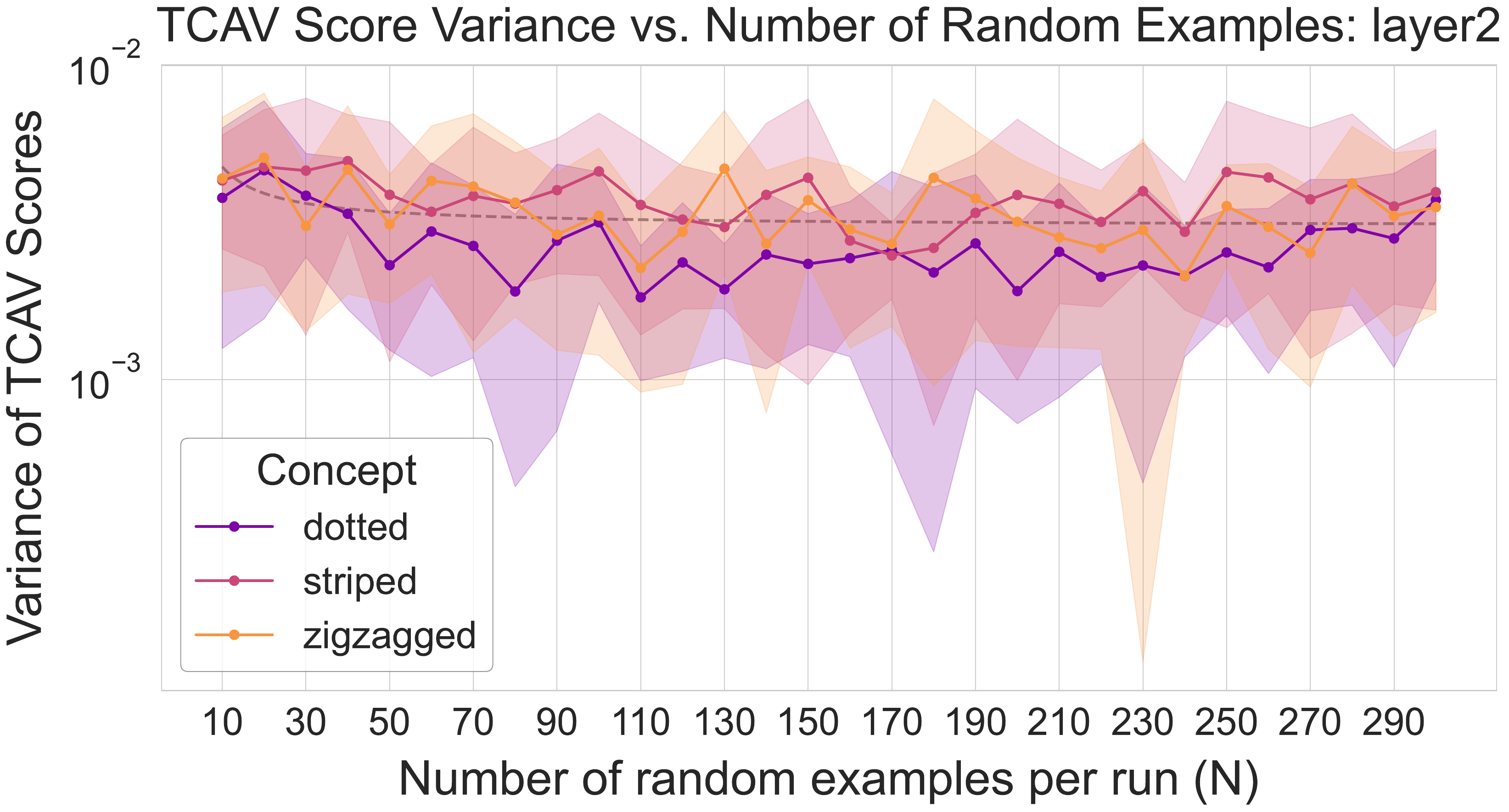}
    \end{minipage}
    \hfill
    \begin{minipage}{0.48\textwidth}
        \centering
        \includegraphics[width=\linewidth]{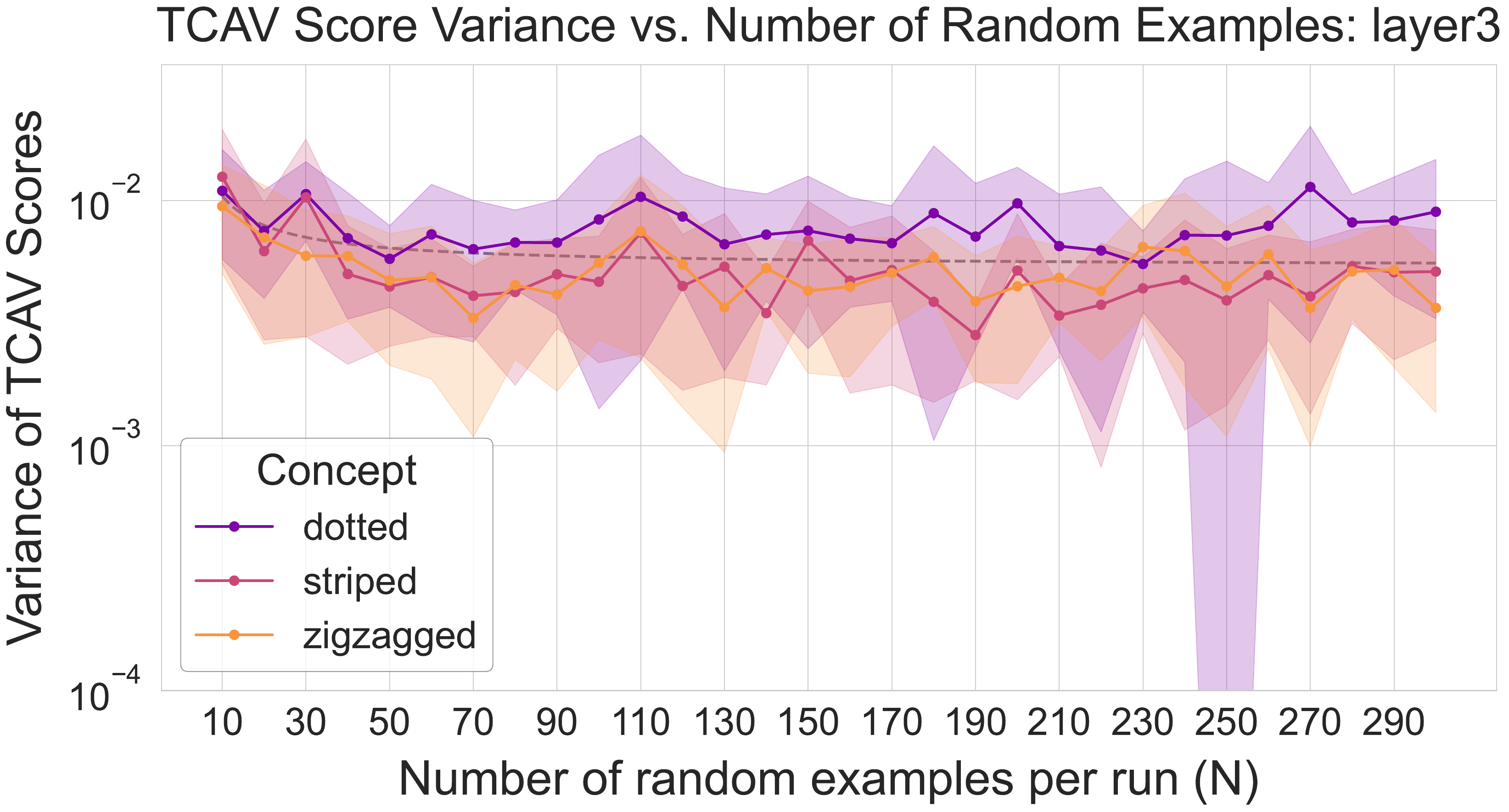}
    \end{minipage}
    \caption[Variance of TCAV scores (Hinge Loss) on ImageNet]{
    Variance of \text{\textsc{Tcav}\xspace} scores at \texttt{'layer2'} and \texttt{'layer3'} of the \textbf{ResNet50 model} vs. the number of examples per concept set ($N$) for the visual concepts ``striped'', ``zigzagged'', and ``dotted''. The underlying \text{\textsc{Cav}\xspace}s were trained using \textbf{hinge loss}. Error bars denote $\pm 1$ standard deviation. We fitted a curve of the form $f(N) = a/N + b$ to it. For \texttt{layer2} the parameters were $a=0.0168, b=0.00307$, for \texttt{layer3} they were $a=0.0507, b=0.00539$. For \texttt{layer3} the scale is trimmed at $\geq 10^{-4}$ to prevent small‑variance tails collapsing the plot. $Y$‑axis clipped at $\geq 10^{-6}$ for \texttt{'layer2'} and $\geq 10^{-4}$ for \texttt{layer3} to reduce distortion from near‑zero standard deviations.
}
    \label{fig:variance_of_tcavs_hinge_images}
\end{figure}
\subsubsection{Empirical Findings for Concept Activation Regions}

Second, we establish that the same theoretical guarantees apply to the Concept Activation Region framework \citep{crabbe_concept_2022}.

\begin{figure}[H]
    \centering
    \begin{minipage}{0.48\textwidth}
        \centering
        \includegraphics[width=\linewidth]{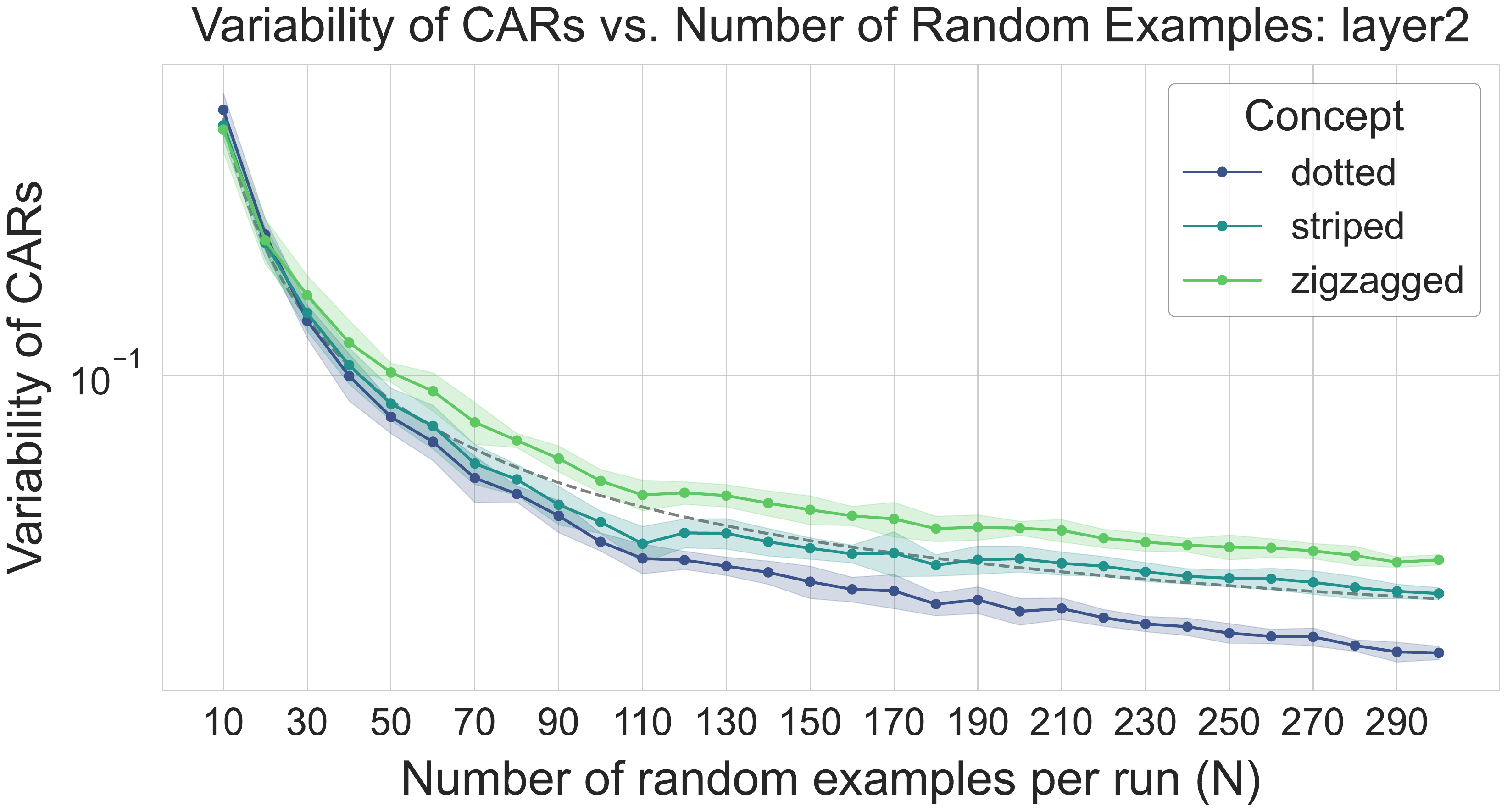}
    \end{minipage}
    \hfill
    \begin{minipage}{0.48\textwidth}
        \centering
        \includegraphics[width=\linewidth]{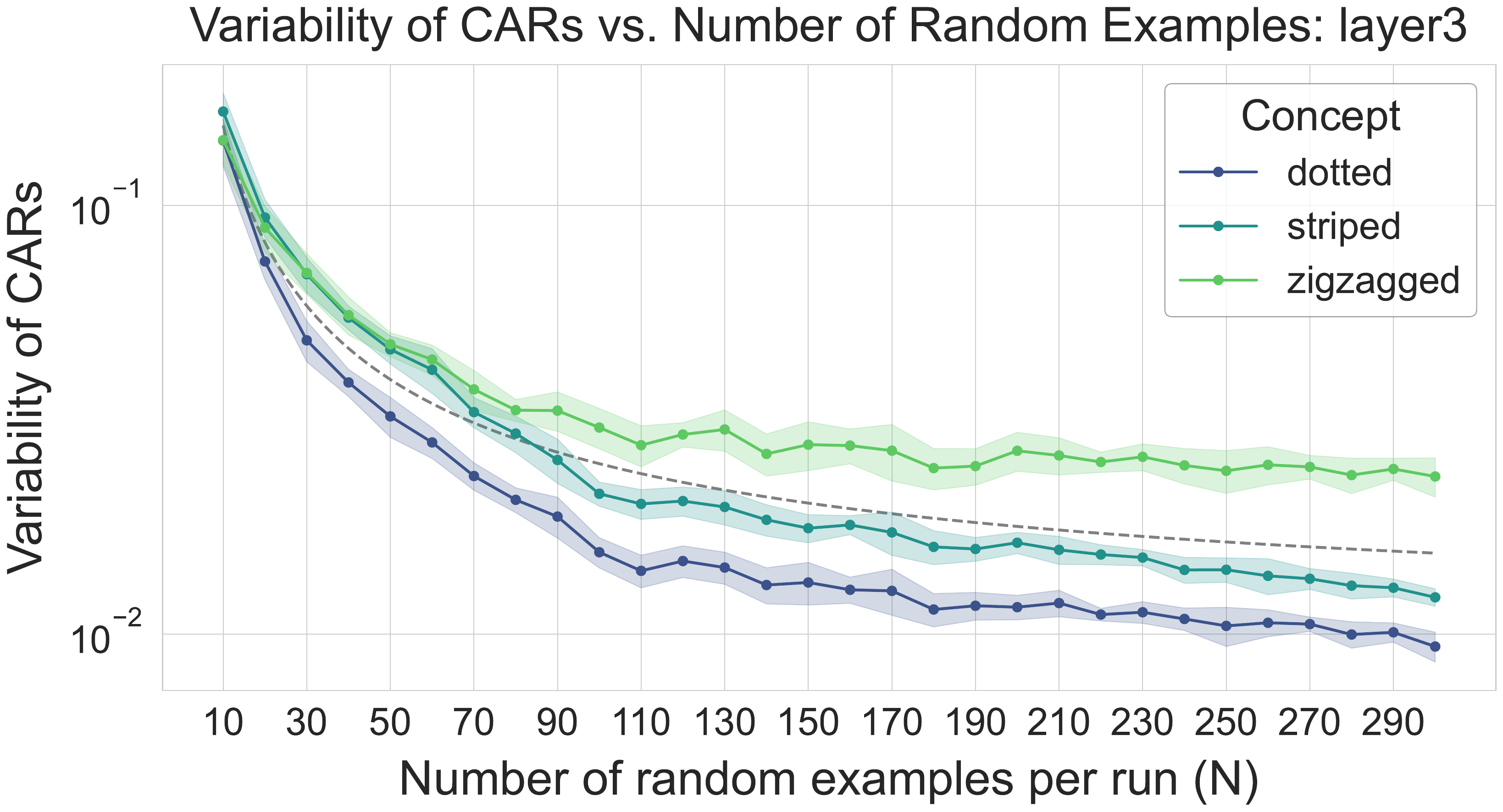}
    \end{minipage}

    \caption[Variance of CARs on ImageNet]{ \label{fig:variance_of_cavs_svm_images}
    Mean variability of \text{\textsc{Car}\xspace}s for the visual concepts ``striped'', ``zigzagged'', and ``dotted'' as a function of the number of random examples per run ($N$). Results are shown for two different layers of the \textbf{Resnet50 model} (\texttt{'layer2'} on the left, \texttt{'layer3'} on the right). Error bars indicate $\pm 1$ SD; the $y$-axis is log-scaled. Variance is estimated by the sum of per-feature variances across five independent runs. We fitted a curve of the form $f(N) = a/N + b$ to it. For \texttt{layer2} the parameters were $ a=3.3, b=0.0225 $, for \texttt{layer3} they were $a=1.43, b=0.0107$.
}
\end{figure}

\begin{figure}[H]
    \centering  
    \begin{minipage}{0.48\textwidth}
        \centering
        \includegraphics[width=\linewidth]{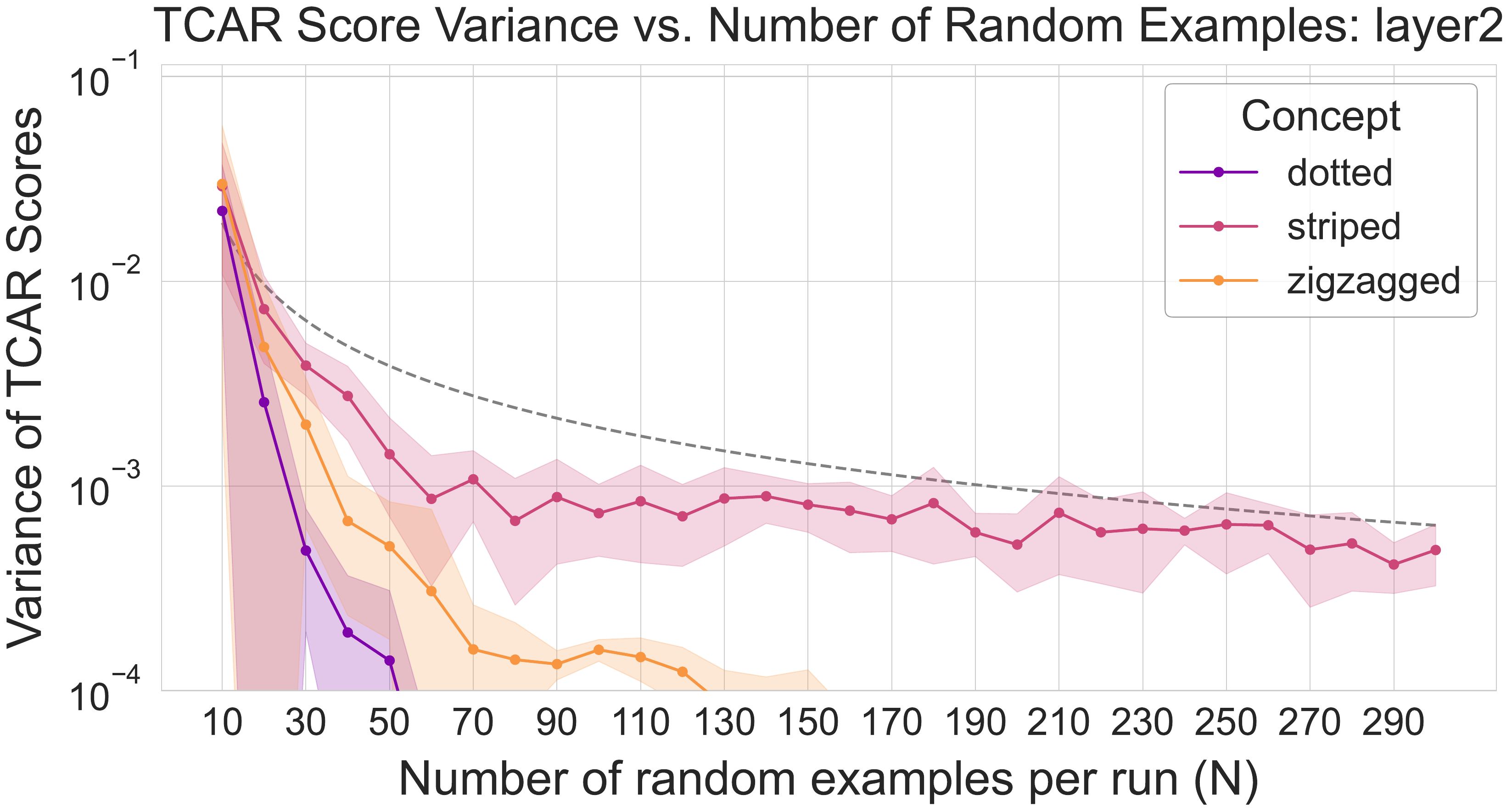}
    \end{minipage}
    \hfill
    \begin{minipage}{0.48\textwidth}
        \centering
        \includegraphics[width=\linewidth]{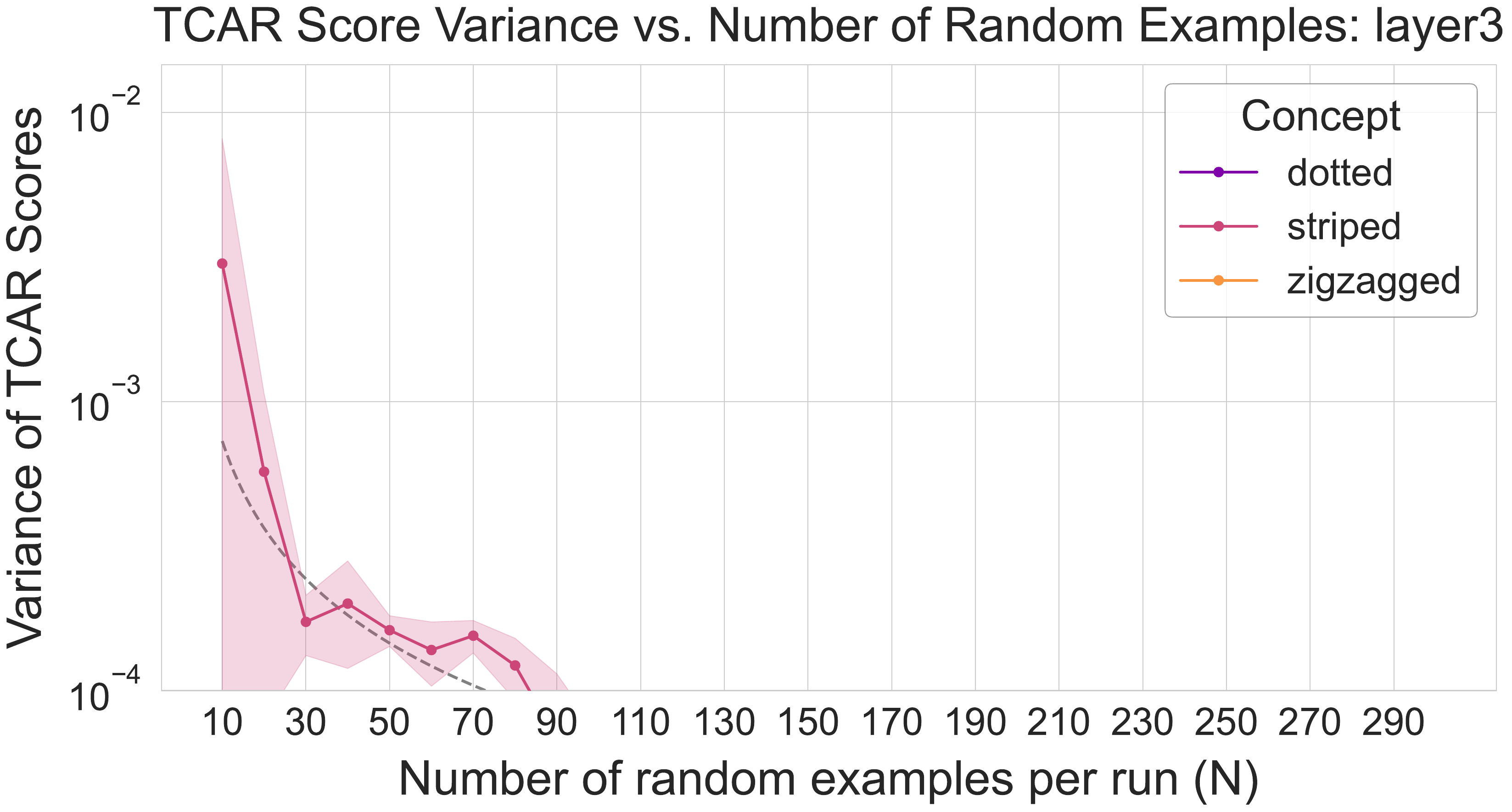}
    \end{minipage}
    \caption[Variance of TCAR scores on ImageNet]{\label{fig:variance_of_tcavs_svm_images}
    Variance of \text{\textsc{Tcar}\xspace} scores at \texttt{'layer2'} and \texttt{'layer3'} of the \textbf{ResNet50 model} vs. the number of examples per concept set ($N$) for the visual concepts ``striped'', ``zigzagged'', and ``dotted''. The underlying \text{\textsc{Car}\xspace}s were trained using non-linear SVM. Error bars denote $\pm 1$ standard deviation. We fitted a curve of the form $f(N) = a/N + b$ to it. For \texttt{layer2} the parameters were $ a=0.193, b=4.55\times 10^{-11}$ , for \texttt{layer3} they were $a=0.0507, b=0.00539$.
}
\end{figure}

\subsubsection{Empirical Findings with Difference of Means}

Again, we found that  \text{\textsc{Cav}\xspace}s computed via the \emph{Difference of Means}-method most closely follow a variance decline of $\mathcal{O}(1/N)$, compared to the other two methods.

\begin{figure}[H]
    \centering
    \begin{minipage}{0.48\textwidth}
        \centering
        \includegraphics[width=\linewidth]{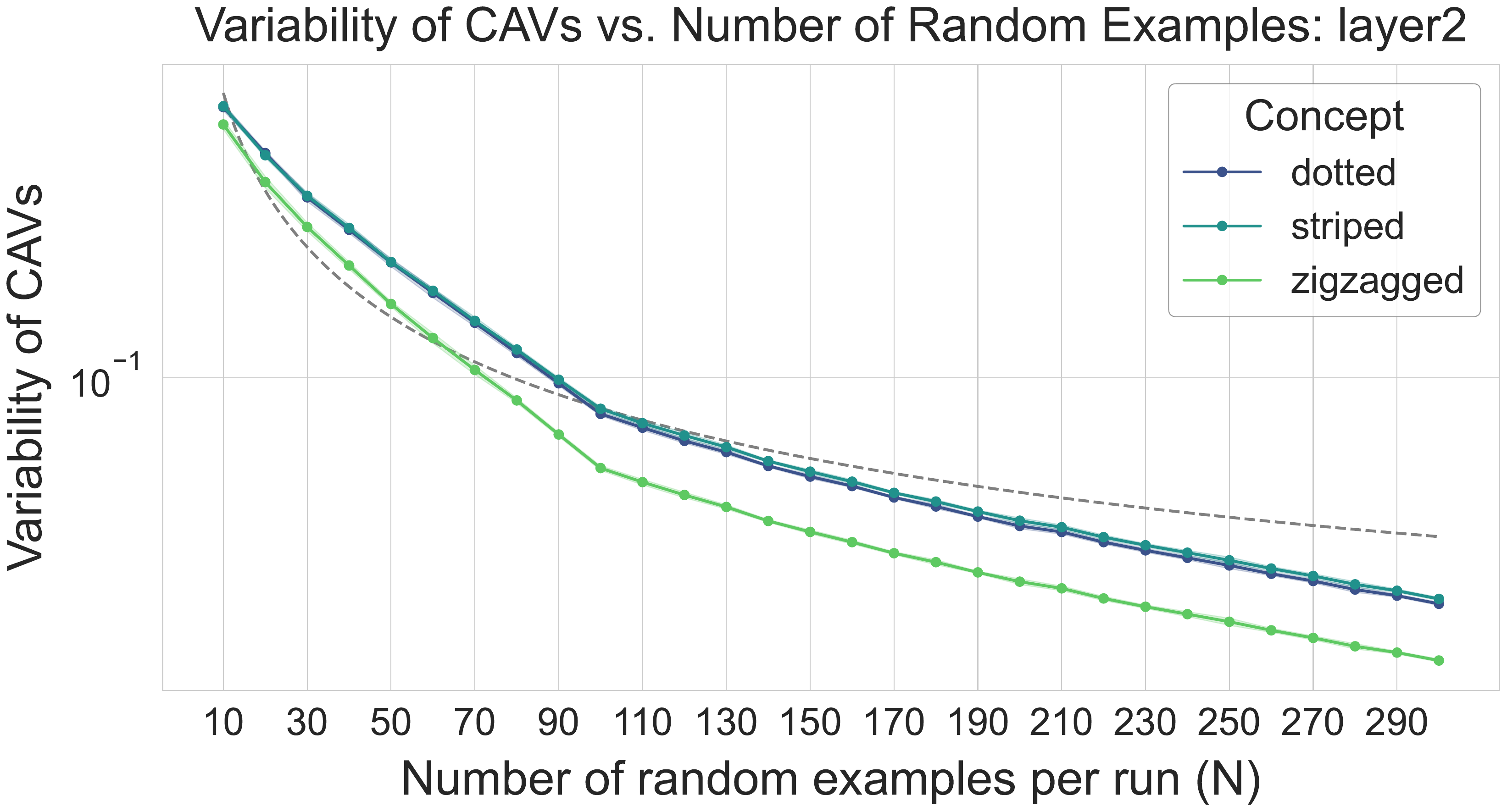}
    \end{minipage}
    \hfill
    \begin{minipage}{0.48\textwidth}
        \centering
        \includegraphics[width=\linewidth]{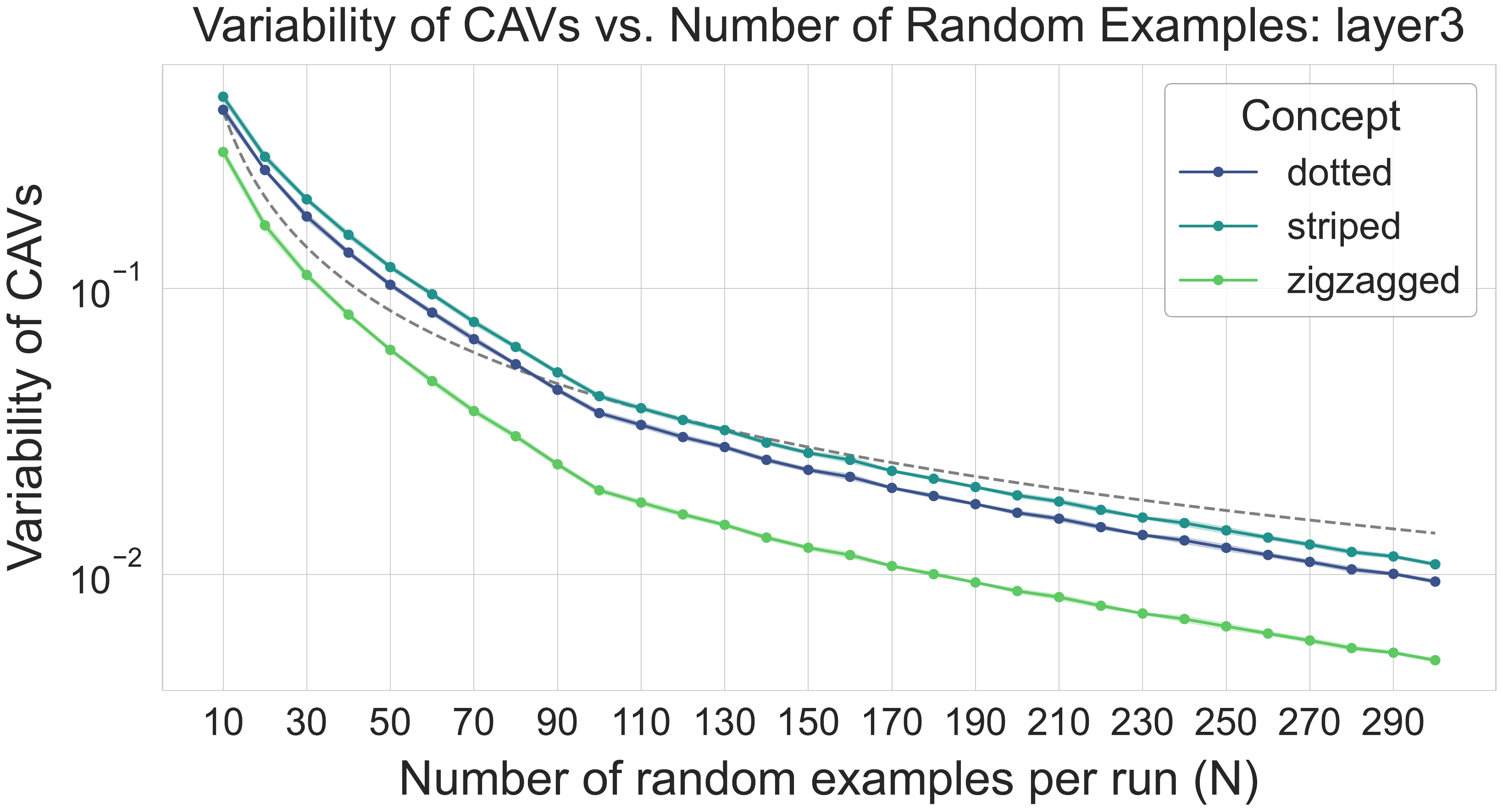}
    \end{minipage}
    \caption[Variance of CAVs (Difference of Means) on Imagenet]{\label{fig:variance_of_cavs_dom_images}
    Mean variability of \text{\textsc{Cav}\xspace}s in dependance of the number of random examples per run $N$ for \textbf{ResNet50 model}. The \text{\textsc{Cav}\xspace}s were generated using \textbf{difference-of-means}. Error bars indicate $\pm 1$ SD; the $y$-axis is log-scaled. Variance is estimated by the sum of per-feature variances across five independent runs. We fitted a curve of the form $f(N) = a/N + b$ to it. For \texttt{layer2} the parameters were $a=7.24, b=0.0106$, for \texttt{layer3} they were $a=4.13, b=3.18\times 10^{-8}$.
}
\end{figure}

\begin{figure}[H]
    \centering
    \begin{minipage}{0.48\textwidth}
        \centering
        \includegraphics[width=\linewidth]{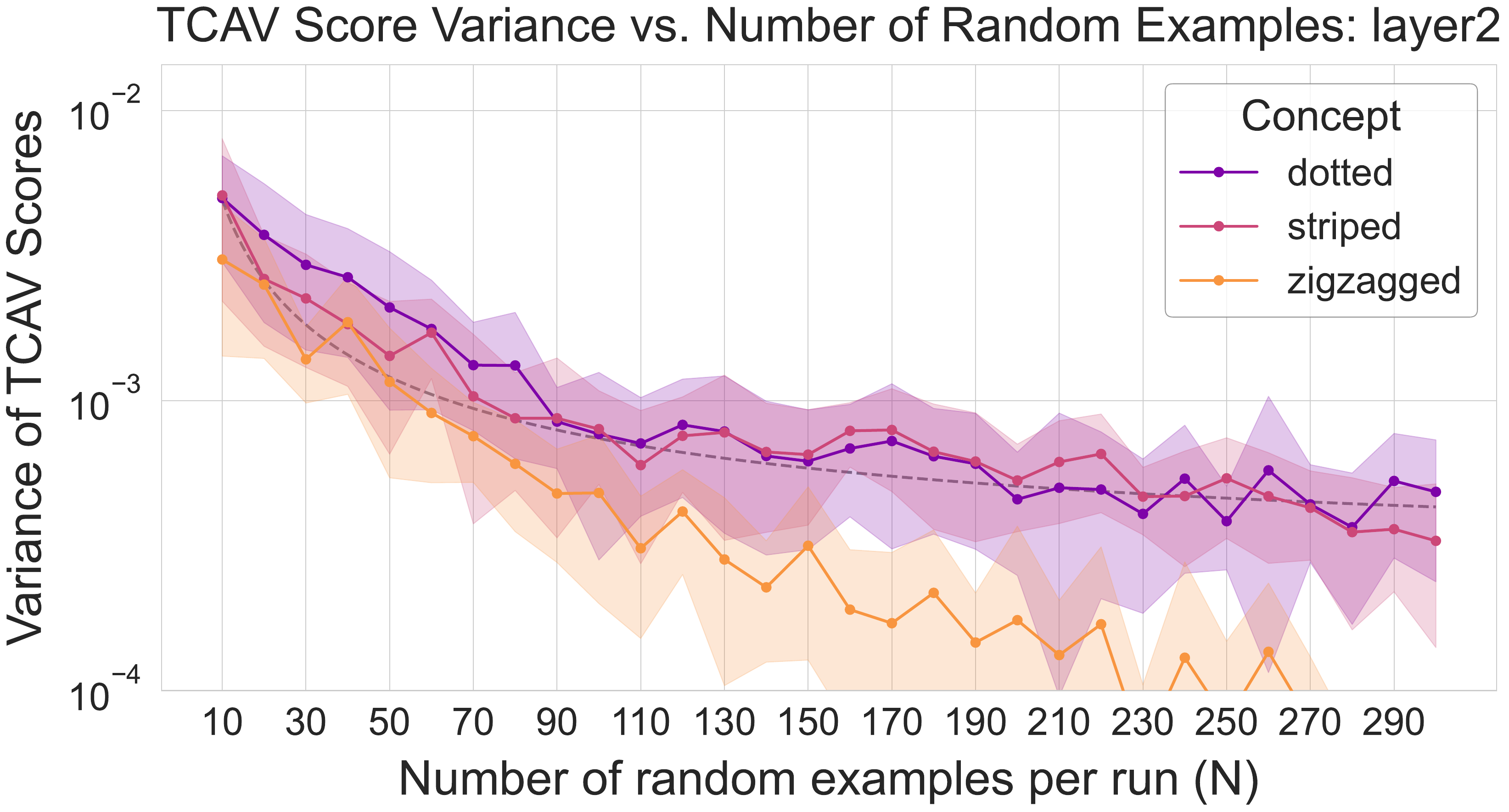}
    \end{minipage}
    \hfill 
    \begin{minipage}{0.48\textwidth}
        \centering
        \includegraphics[width=\linewidth]{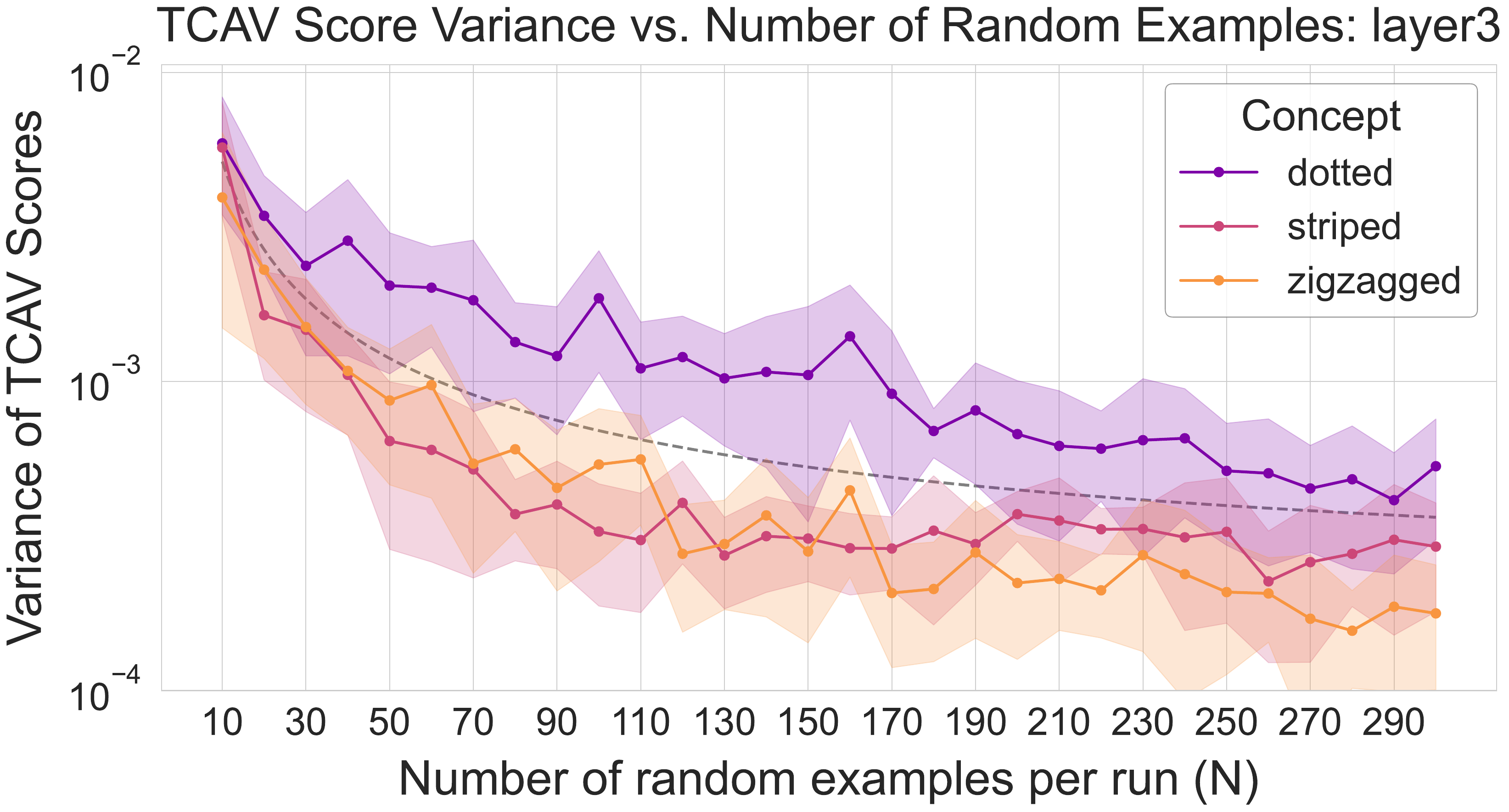}
    \end{minipage}

    \caption[Variance of TCAV scores (Difference of Means) on Imagenet]{ \label{fig:variance_of_tcavs_dom_images}
    Variance of \text{\textsc{Tcav}\xspace} scores at \texttt{'layer2'} and \texttt{'layer3'} of the \textbf{ReseNet50 model} vs. the number $N$ of random examples used. The underlying \text{\textsc{Cav}\xspace}s were generated using \textbf{difference-of-means}. Error bars denote $\pm 1$ standard deviation. We fitted a curve of the form $f(N) = a/N + b$ to it. For \texttt{layer2} the parameters were $a=0.0429, b=3.29\times 10^{-4}$, for \texttt{layer3} they were $a=0.0548, b=1.5\times 10^{-4}$.
}
\end{figure}

\subsubsection{Compute Trade-off}

\begin{figure}[H]
    \centering

    \begin{subfigure}[t]{0.48\textwidth}
        \centering
        \includegraphics[width=\linewidth]{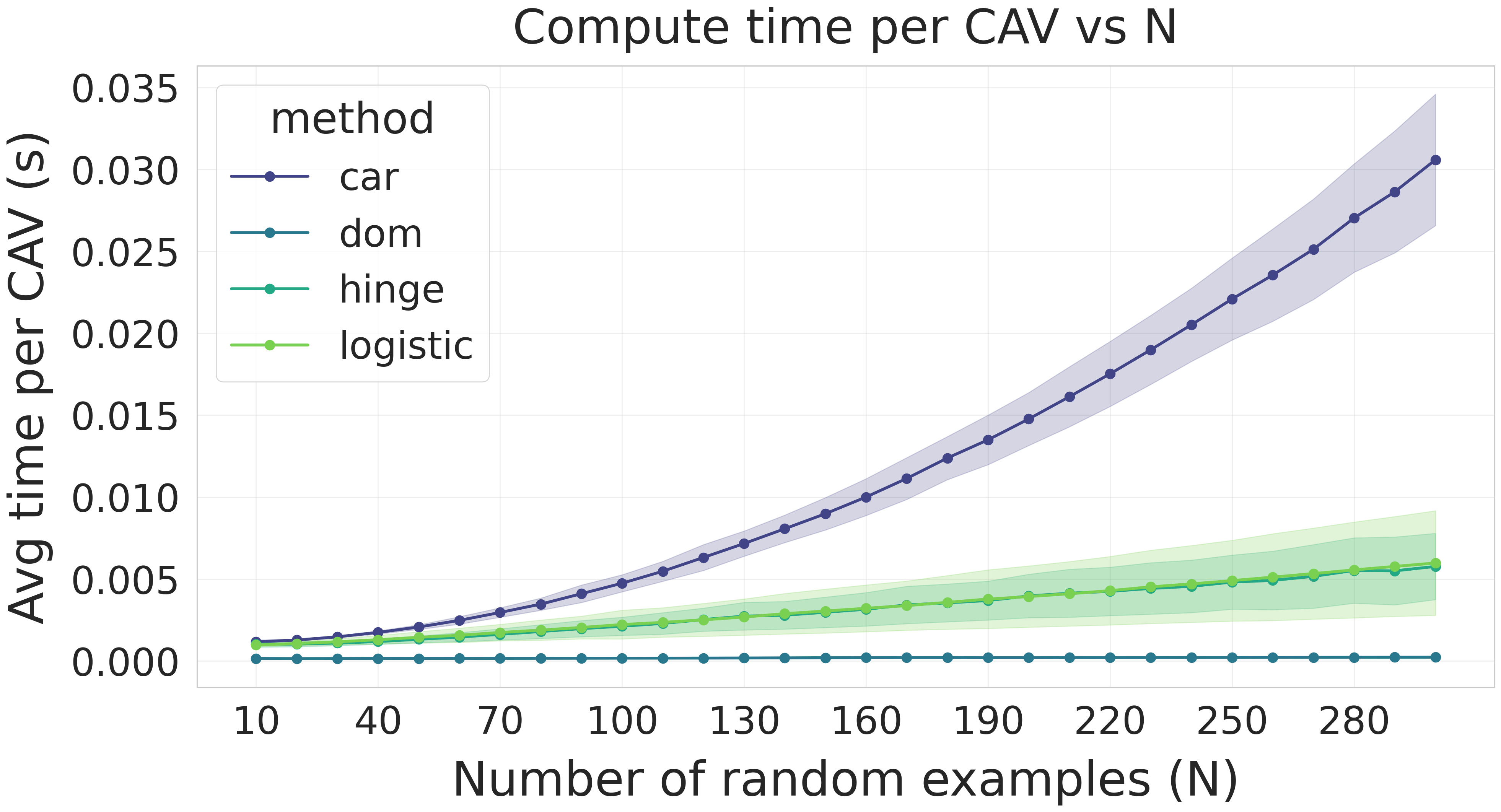}
        \caption{\textbf{Compute vs.\ $N$.} Average wall-clock time per \text{\textsc{Cav}\xspace} as a function of the number of random examples per run ($N$). Shaded band denotes $\pm 1$ SD across runs.}
        \label{fig:compute_vs_n_images}
    \end{subfigure}
    \hfill
    \begin{subfigure}[t]{0.48\textwidth}
        \centering
        \includegraphics[width=\linewidth]{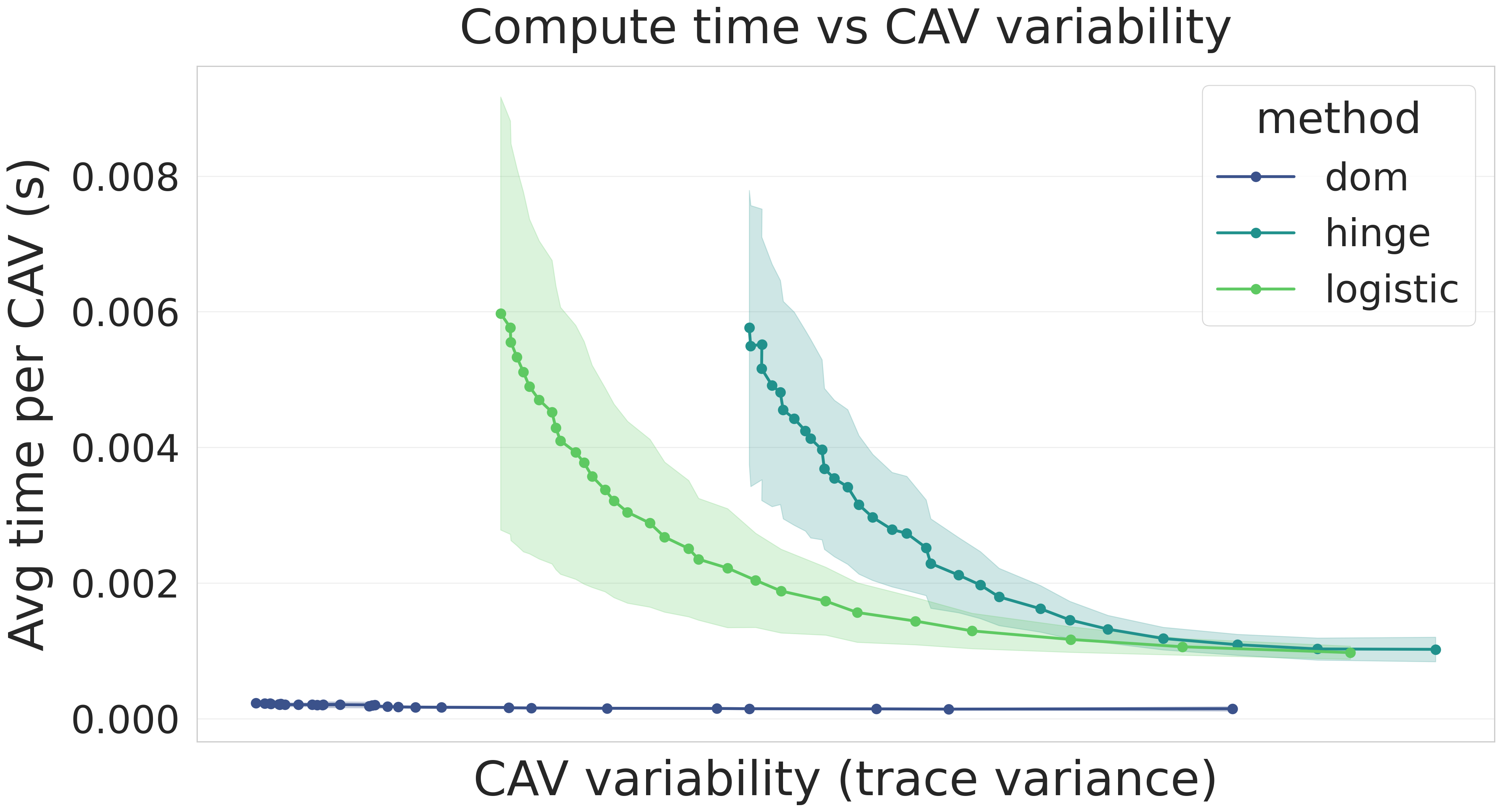}
        \caption{\textbf{Compute vs.\ variability.} Compute--stability trade-off: average time per \text{\textsc{Cav}\xspace} versus \text{\textsc{Cav}\xspace} variability (trace variance). Points correspond to different $N$ values (connected in increasing $N$). Shaded band denotes $\pm 1$ SD for compute time.}
        \label{fig:compute_vs_variability_images}
    \end{subfigure}

    \caption[Compute cost vs.\ stability of CAVs]{\label{fig:compute_tradeoff_tab}
    Compute cost and stability of \text{\textsc{Cav}\xspace} generation (see legend for the selected setting, e.g.\ layer/method). Left: compute time per \text{\textsc{Cav}\xspace} vs.\ $N$. Right: compute time vs.\ resulting \text{\textsc{Cav}\xspace} variability (trace variance), illustrating the compute--stability trade-off.}
\end{figure}

\subsection{TCAV for Tabular Data}
\label{appendix:tabular}

Because this tabular task is straightforward, we train several simple models from scratch. 
Using the Adult dataset\citep{becker_adult_1996}, we define two concepts from the gender field—male and female. 
\subsubsection{Empirical Findings with Binary Cross-Entropy Loss}
First, we analyze  the binary cross-entropy loss from the logistic regression classifier. 
\begin{figure}[H]
    \centering
    \begin{minipage}{0.48\textwidth}
        \centering
        \includegraphics[width=\linewidth]{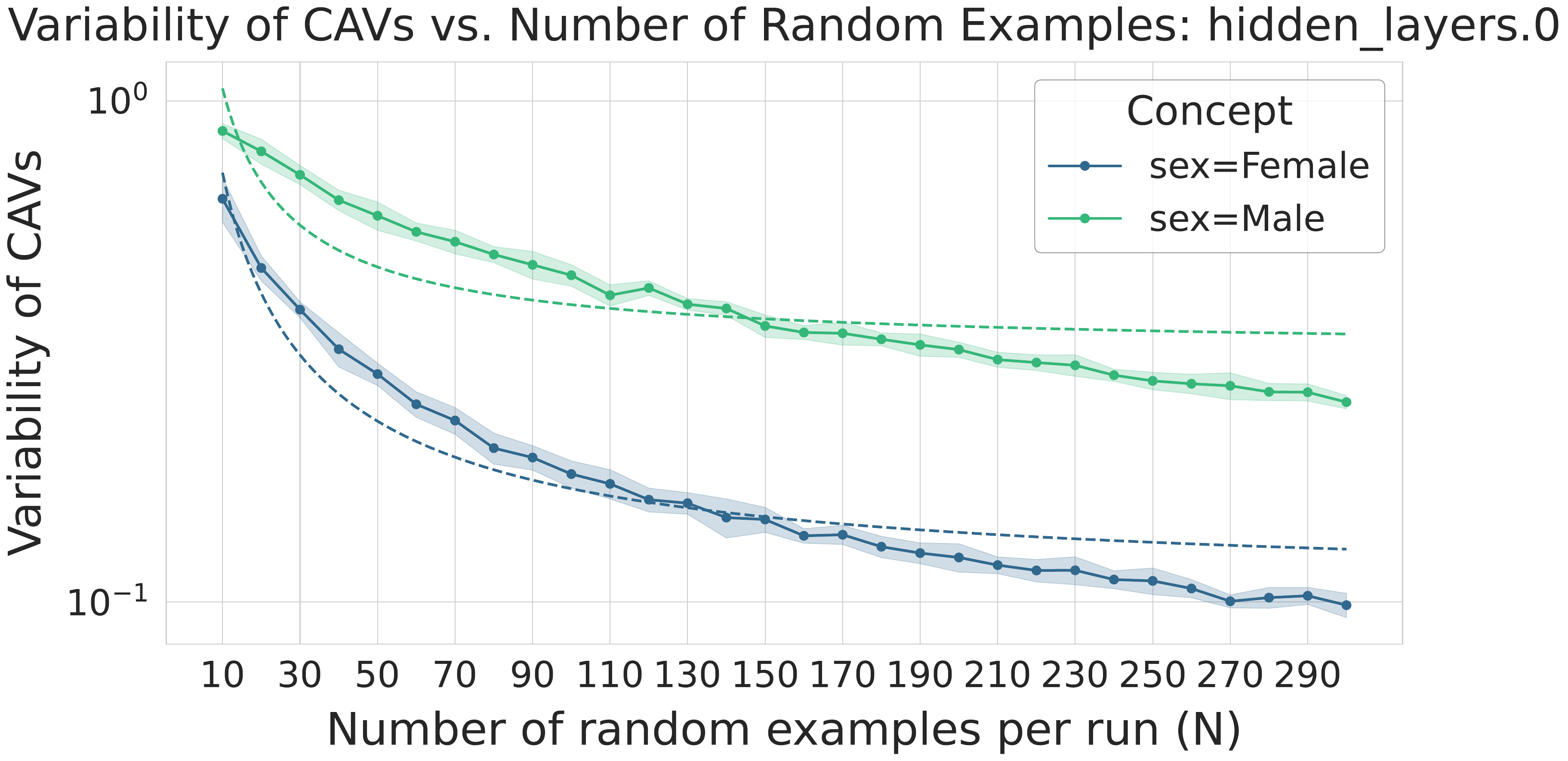}
    \end{minipage}
    \hfill
    \begin{minipage}{0.48\textwidth}
        \centering
        \includegraphics[width=\linewidth]{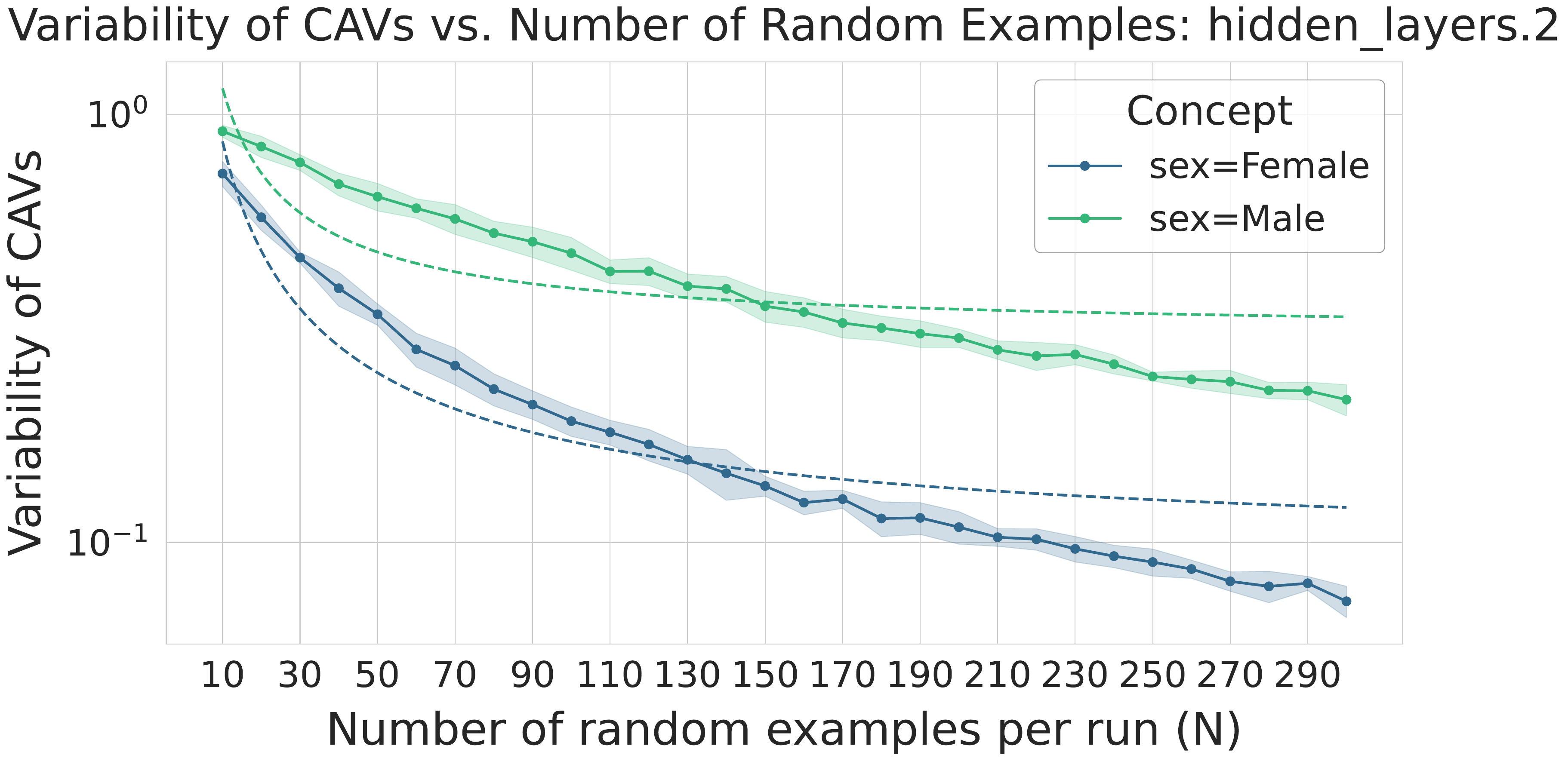}
    \end{minipage}
    \caption[Variance of CAVs (Log Loss) on UCI Adult]{\label{fig:variance_of_cavs_log_tab}
    Mean variability of \text{\textsc{Cav}\xspace}s as a function of the number of random examples per run ($N$) on the \textbf{UCI Adult dataset}. Results are shown for two different hidden layers. Error bars indicate $\pm 1$ SD; the $y$-axis is log-scaled. Variance is estimated by the sum of per-feature variances across ten independent runs. We fitted a curve of the form $f(N) = a/N + b$ to it. For \texttt{hidden\_layers\_0} they were $ 4.75/N  +  2.18\times 10^{-13} $ for ``sex=Female'' and $ 6.74/N  + 0.281 $ for ``sex=Male''. For \texttt{hidden\_layers\_2} they were $ 6.82/N \; + \; 0.074 $ \emph{i.e.} $ 8.46/N \; + \; 0.228 $.}
\end{figure}

\begin{figure}[H]
    \centering
    \begin{minipage}{0.48\textwidth}
        \centering
        \includegraphics[width=\linewidth]{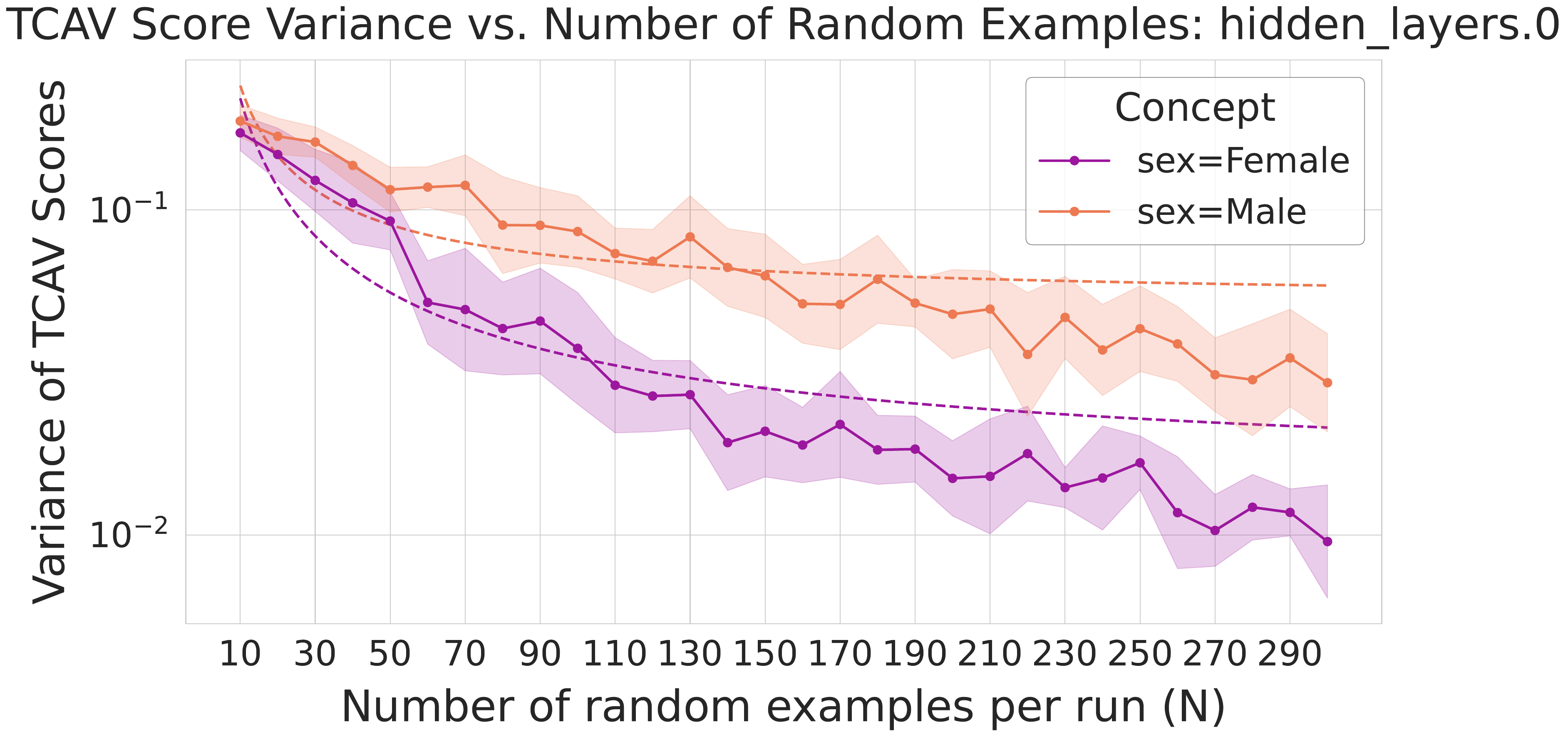}
    \end{minipage}
    \hfill
    \begin{minipage}{0.48\textwidth}
        \centering
        \includegraphics[width=\linewidth]{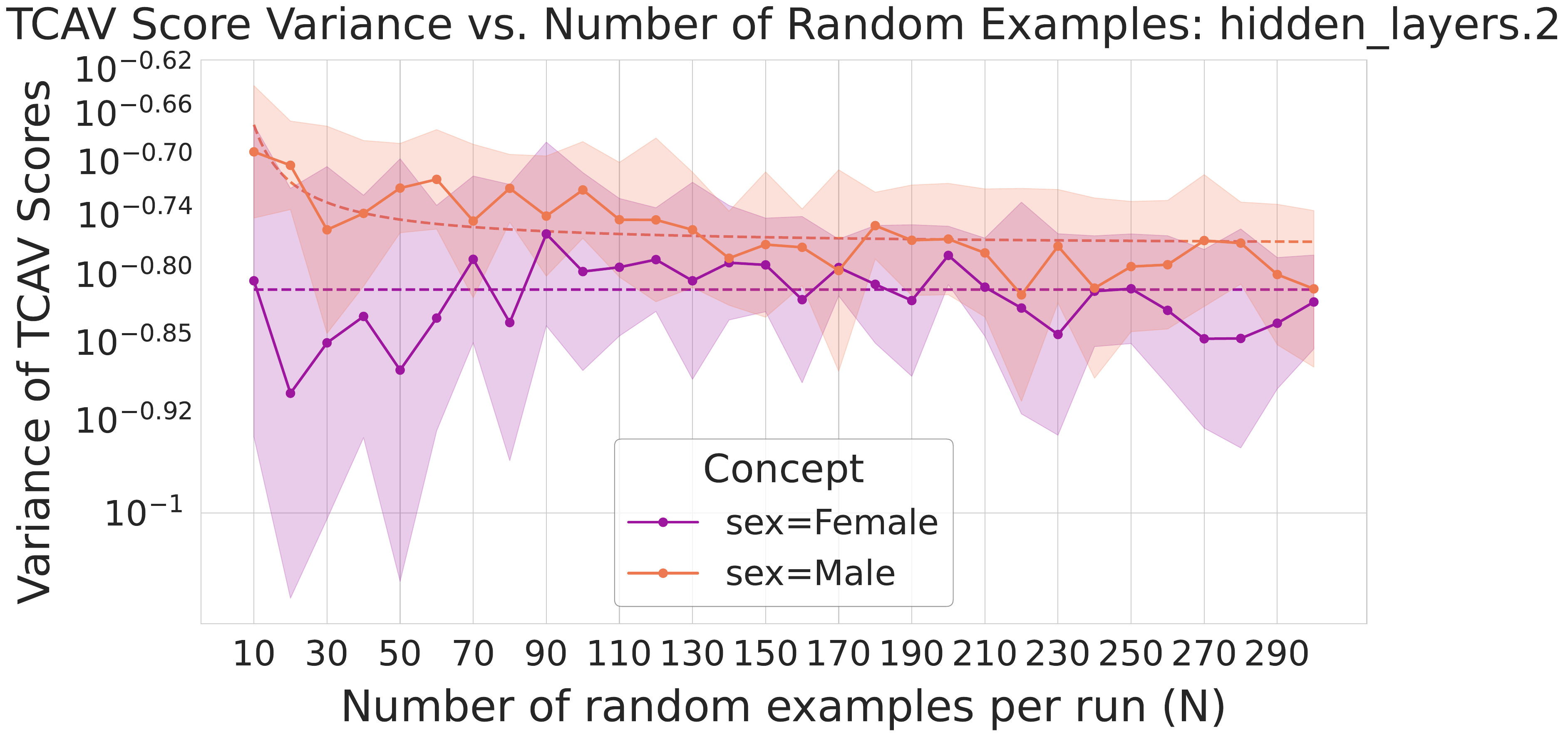}
    \end{minipage}
\caption[Variance of TCAV scores (Log Loss) on UCI Adult]{\label{fig:variance_of_tcavs_log_tab}
    Variance of \text{\textsc{Tcav}\xspace} scores  vs. the number of examples per concept set ($N$) for the concepts ``male'' and ``female'' on the \textbf{UCI Adult dataset}. Error bars denote $\pm 1$ standard deviation.}
\end{figure}

\subsubsection{Empirical Findings with Hinge Loss}
\begin{figure}[H]
  \centering
  \begin{minipage}{0.48\textwidth}
    \centering
    \includegraphics[width=\linewidth]{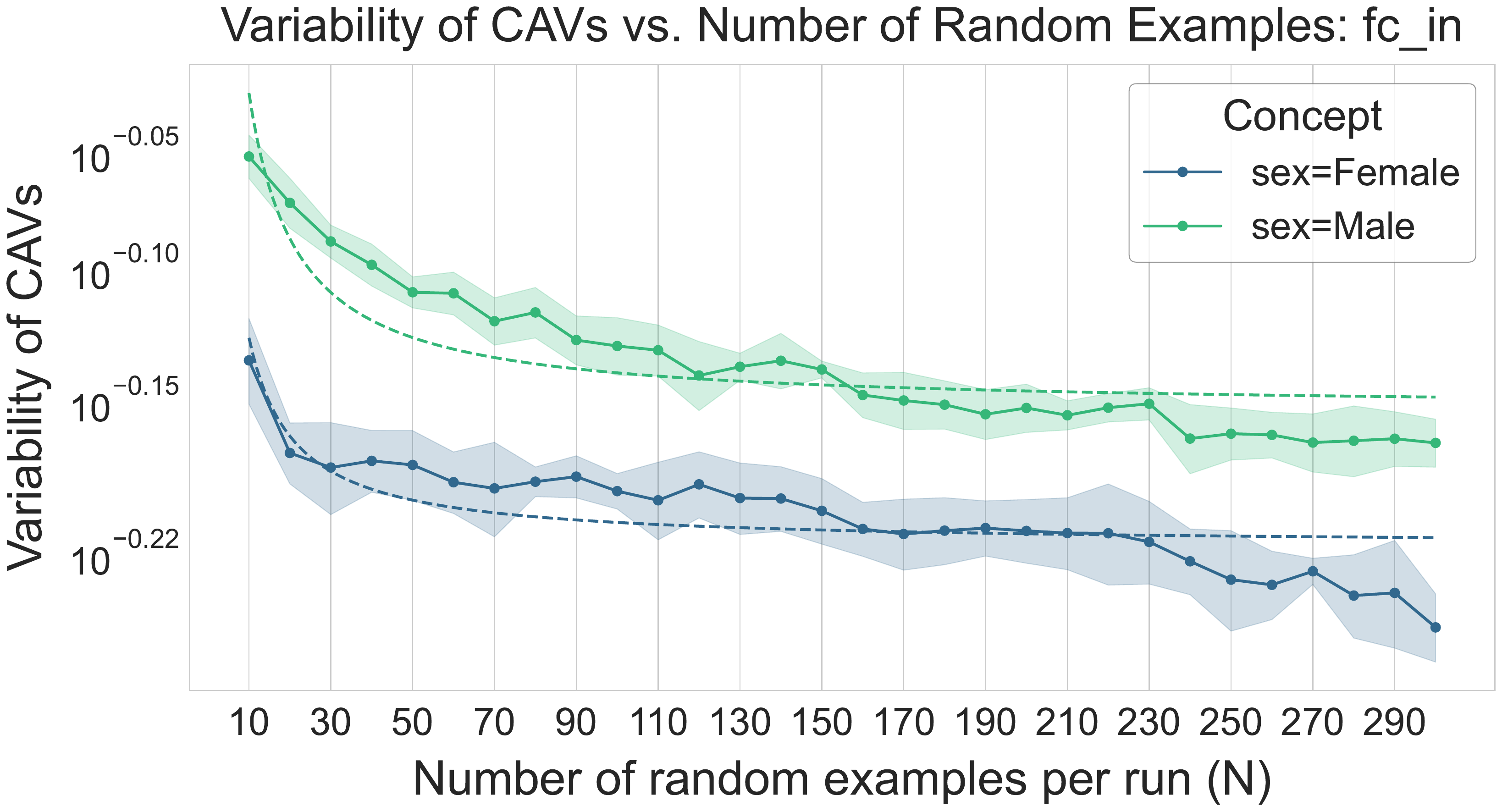}
  \end{minipage}\hfill
  \begin{minipage}{0.48\textwidth}
    \centering
    \includegraphics[width=\linewidth]{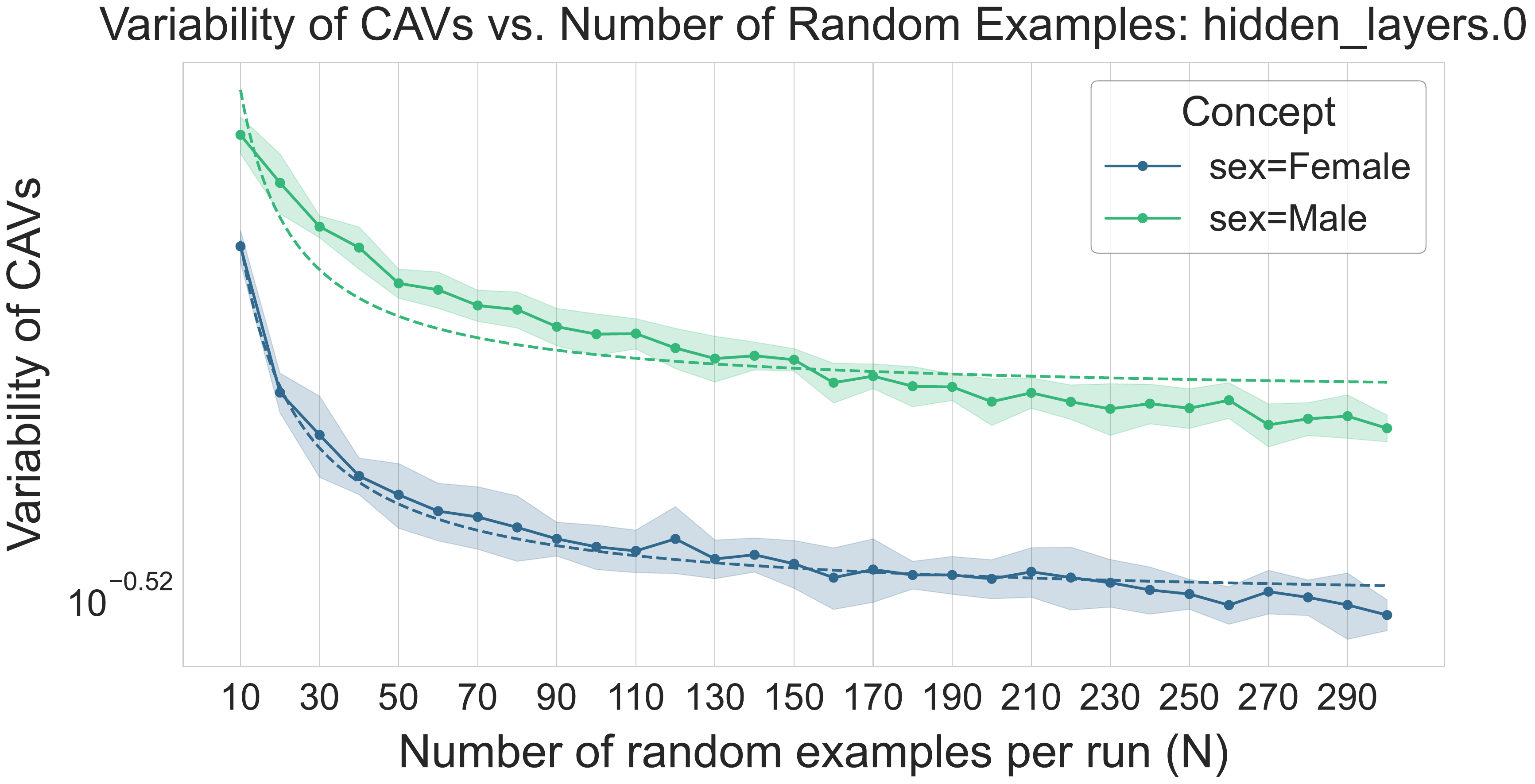}
  \end{minipage}
  \caption[Variance of CAVs (Hinge Loss) on UCI Adult]{\label{fig:variance_of_cavs_hinge_tab}
    Mean variability of \textsc{CAV}s as a function of the number of random examples per run ($N$) on the \textbf{UCI Adult dataset}. Results are shown for two different hidden layers. The linear classifiers for \textsc{CAV} generation were trained using \textbf{hinge loss}. Variance is estimated by the sum of per-feature variances across ten independent runs. We fitted a curve of the form $f(N) = a/N + b$ to it. For \texttt{fc\_in} the parameters were $ 1.41/N \; + \; 0.606 $ for ``sex=Female'', and $ 2.61/N \; + \; 0.695 $ for ``sex=Male''. For \texttt{hidden\_layers\_0} they were $ 3.7/N \; + \; 0.296 $ \emph{i.e.} $ 4.63/N \; + \; 0.47 $.}
\end{figure}

\begin{figure}[H]
  \centering
  \begin{minipage}{0.48\textwidth}
    \centering
    \includegraphics[width=\linewidth]{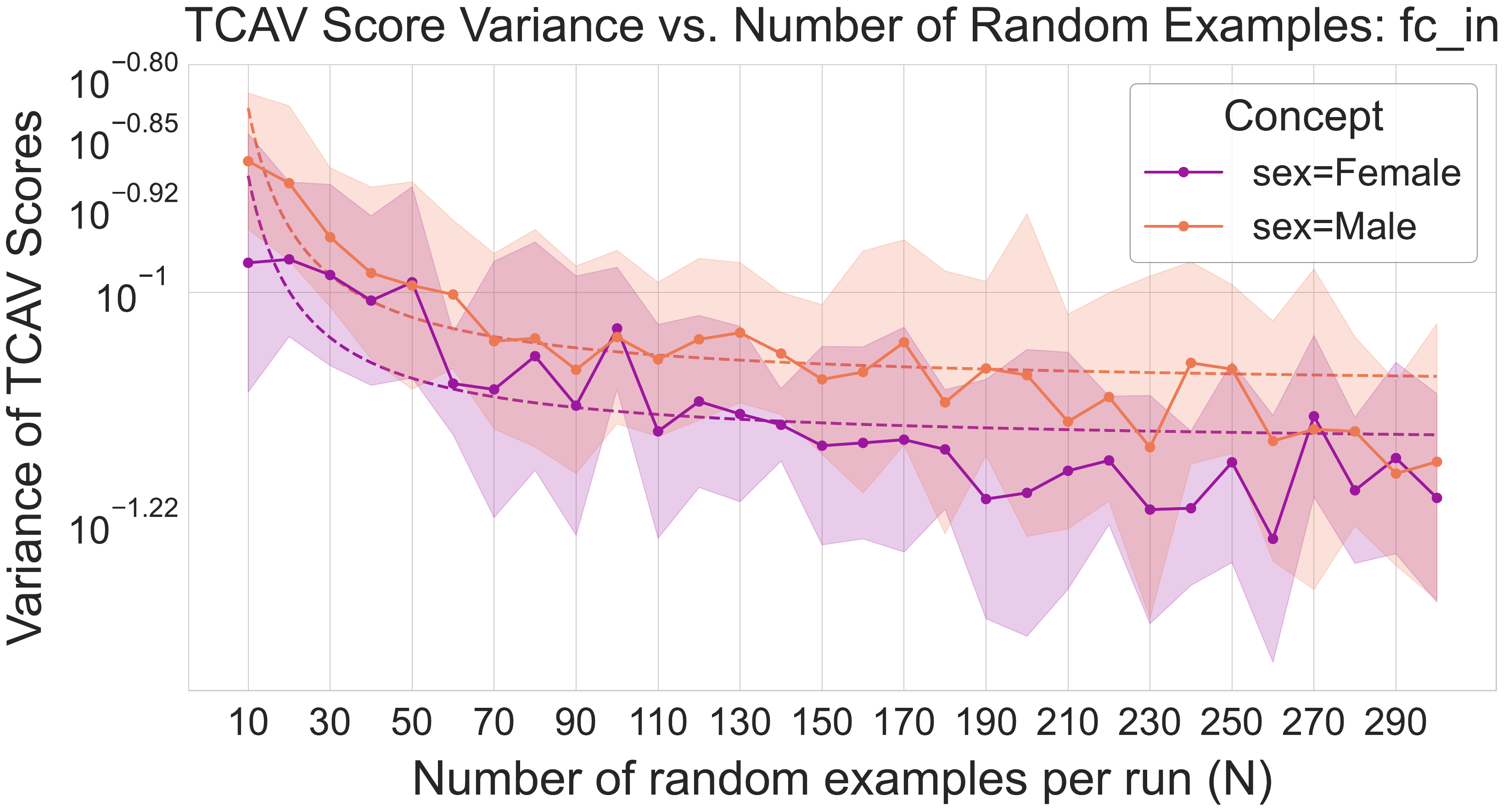}
  \end{minipage}\hfill
  \begin{minipage}{0.48\textwidth}
    \centering
    \includegraphics[width=\linewidth]{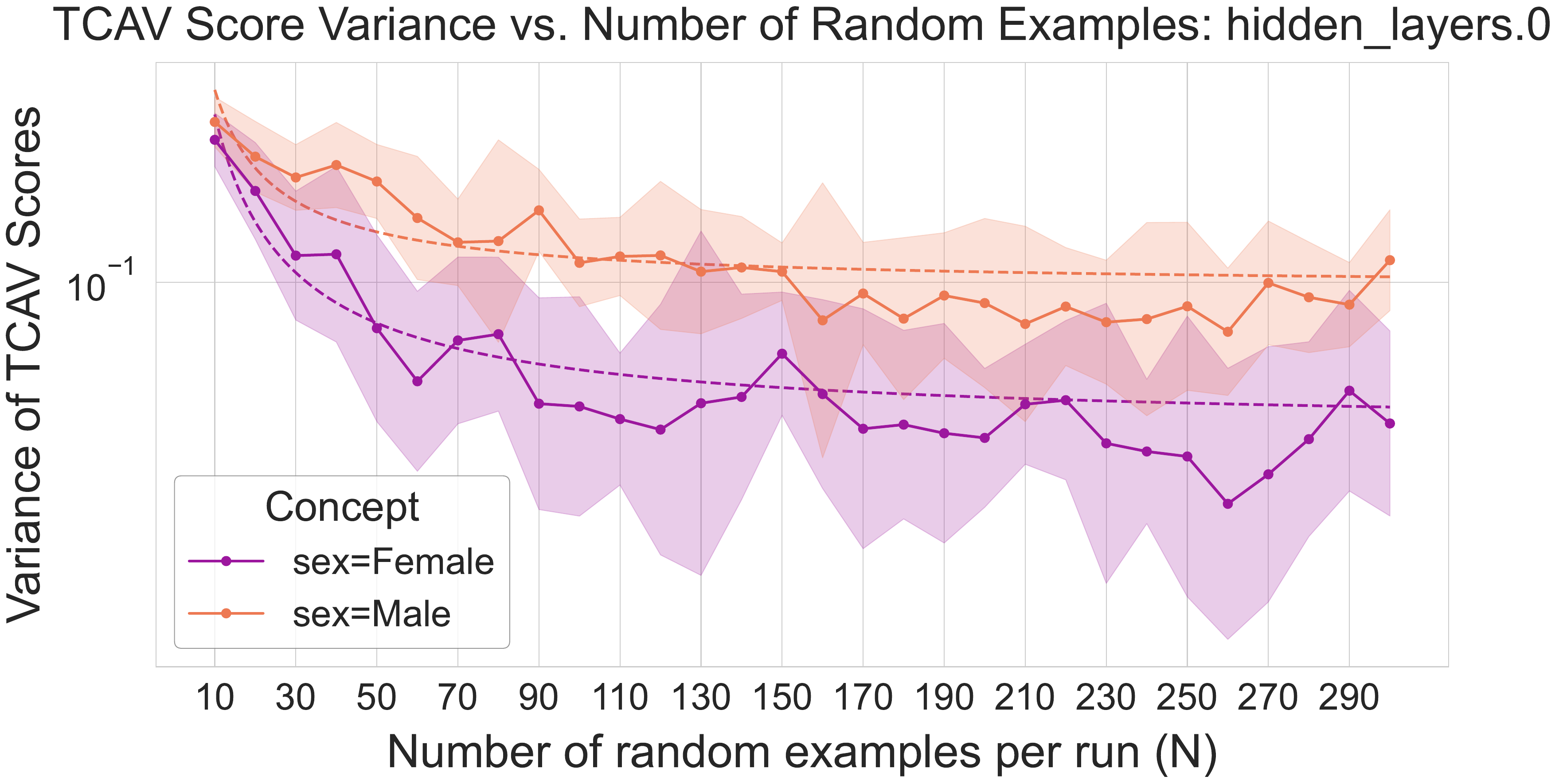}
  \end{minipage}
  \caption[Variance of TCAV scores (Hinge Loss) on UCI Adult]{\label{fig:variance_of_tcavs_hinge_tab}
    Variance of \textsc{TCAV} scores vs.\ the number of examples per concept set ($N$) for the concepts ``male'' and ``female'' on the \textbf{UCI Adult dataset}. The underlying \textsc{CAV}s were trained using \textbf{hinge loss}. Error bars denote $\pm 1$ standard deviation. We fitted a curve of the form $f(N) = a/N + b$ to it. For \texttt{fc\_in} the parameters were $ a=0.581, b=0.0712 $ for``sex=Female'' and $ a=0.693, b=0.0808 $ for``sex=Male''. For \texttt{hidden\_layers\_0} they were $ a=1.53, b=0.0533 $ \emph{i.e.} $ a=1.31, b=0.098 $.}
\end{figure}

\subsubsection{Empirical Findings with Concept Activation Regions}

Third, we demonstrate that our theoretical results extend to the generalization of hinge loss: Concept Activation Regions.

\begin{figure}[H]
  \centering
  \begin{minipage}{0.48\textwidth}
    \centering
    \includegraphics[width=\linewidth]{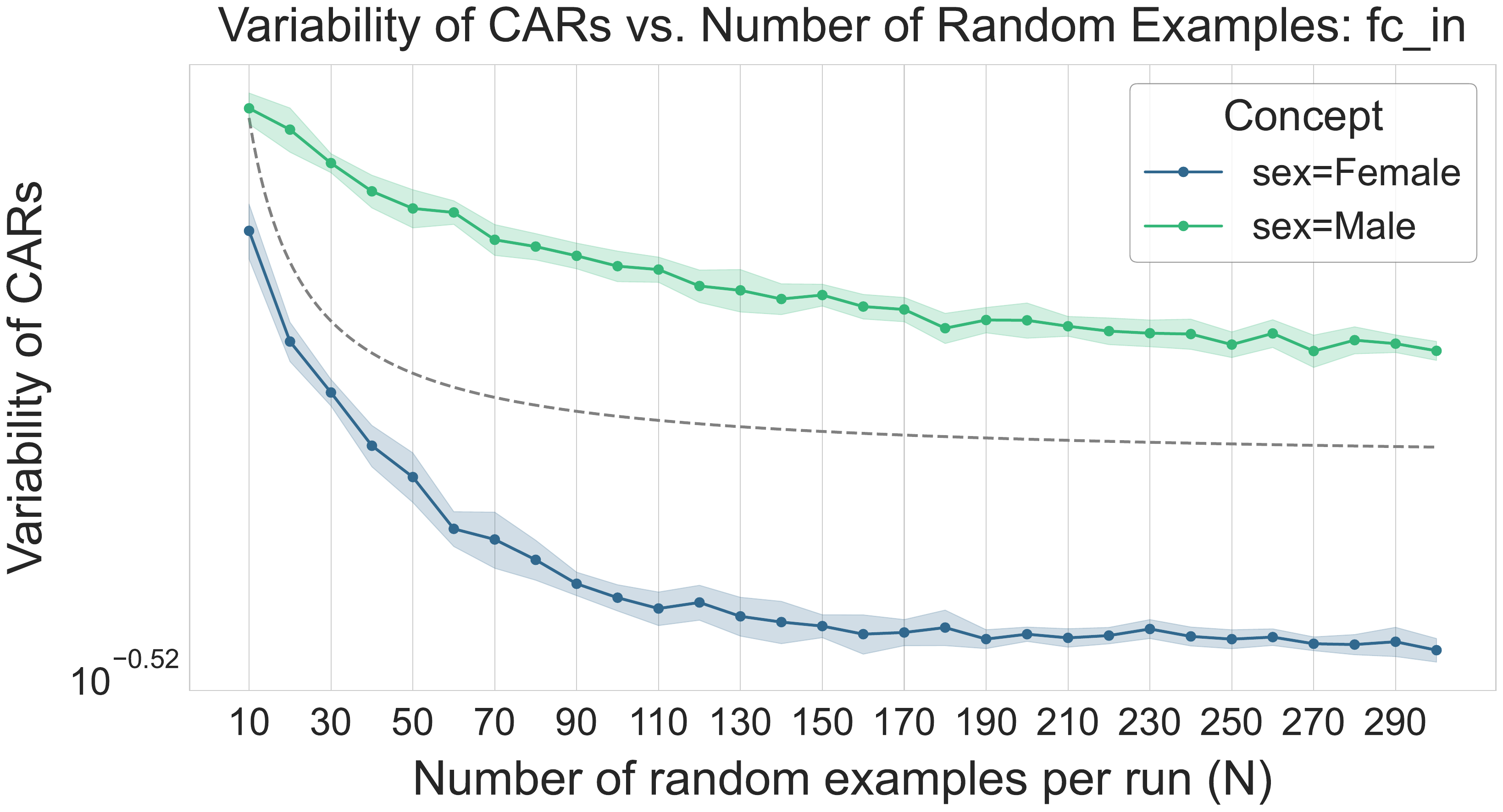}
  \end{minipage}\hfill
  \begin{minipage}{0.48\textwidth}
    \centering
    \includegraphics[width=\linewidth]{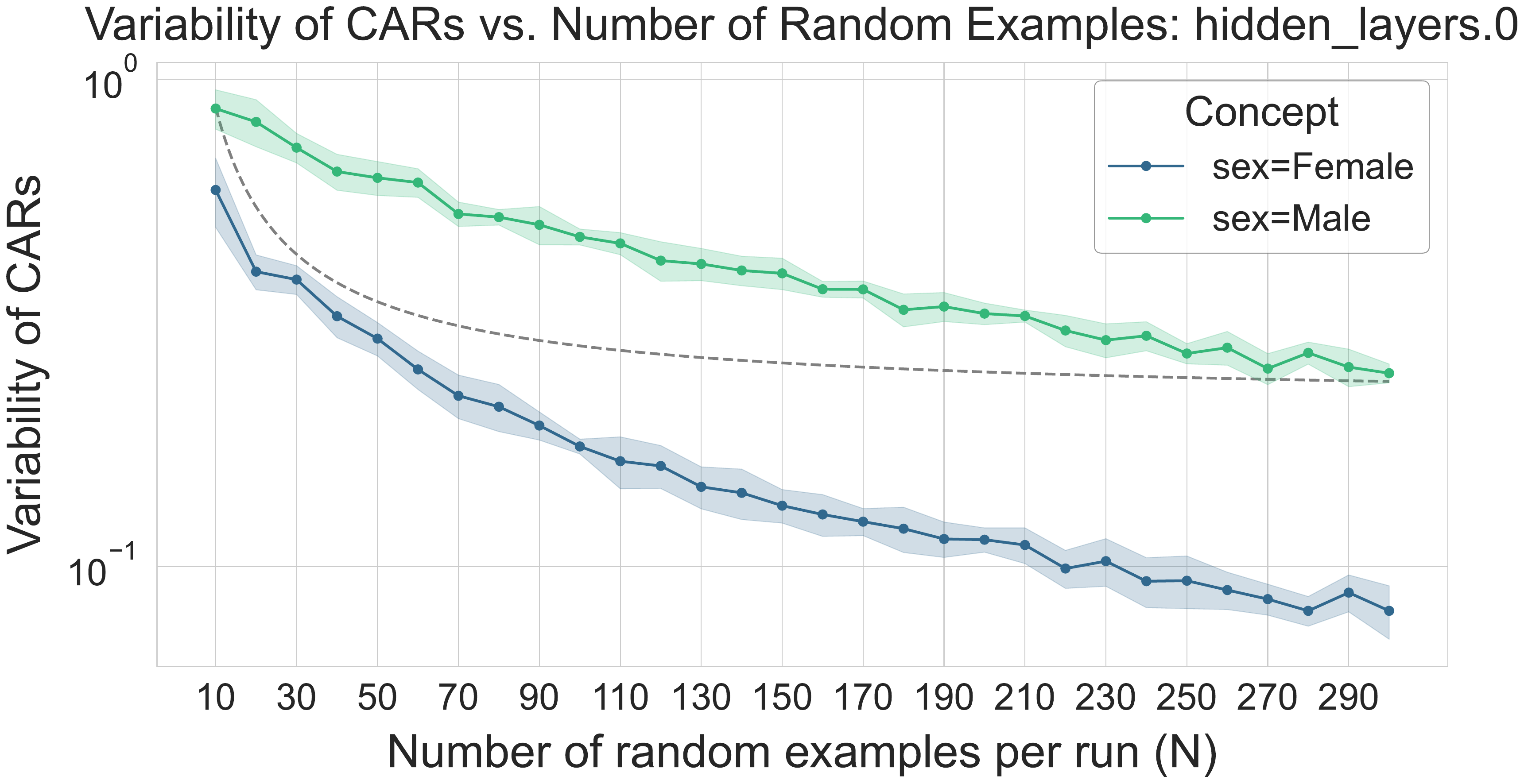}
  \end{minipage}
  \caption[Variance of CARs on UCI Adult]{\label{fig:variance_of_cars_svm_tab}
    Mean variability of \textsc{CAR}s as a function of the number of random examples per run ($N$) on the \textbf{UCI Adult dataset}. Results are shown for two different hidden layers. The \textsc{CAR}s  were obtained using a non-linear SVC. Variance is estimated by the sum of per-feature variances across ten independent runs. We fitted a curve of the form $f(N) = a/N + b$ to it. For \texttt{fc\_in} the parameters were $ 2.78/N  +  0.451 $ for ``sex=Female'' and $ 2.76/N  +  0.681 $ for  ``sex=Male''. For \texttt{hidden\_layers\_0} they were $ 6.03/N  +  0.172 $ \emph{i.e.} $ 6.37/N  +  0.412 $.}
\end{figure}

\begin{figure}[H]
  \centering
  \begin{minipage}{0.48\textwidth}
    \centering
    \includegraphics[width=\linewidth]{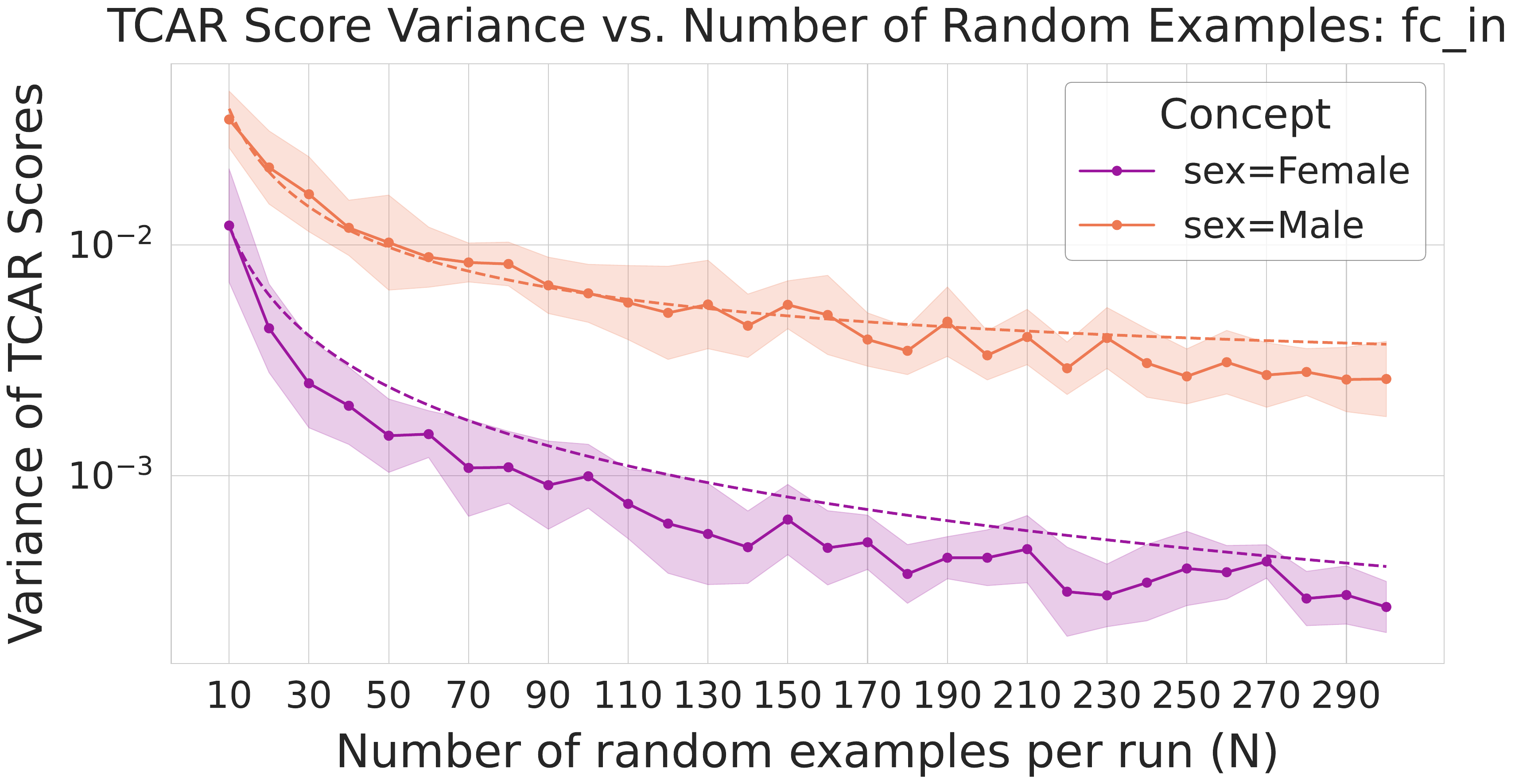}
  \end{minipage}\hfill
  \begin{minipage}{0.48\textwidth}
    \centering
    \includegraphics[width=\linewidth]{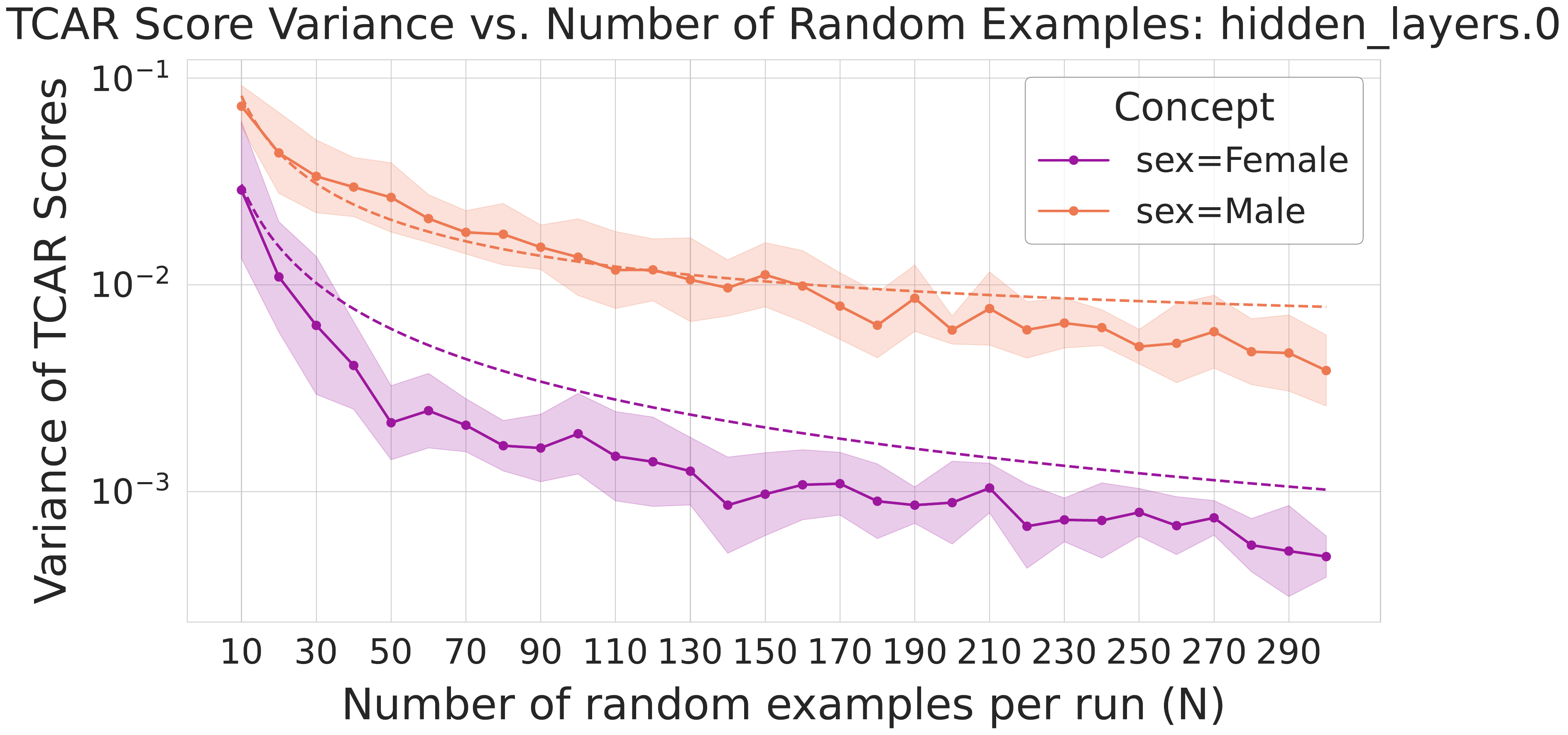}
  \end{minipage}
  \caption[Variance of TCAR scoreson UCI Adult]{ \label{fig:variance_of_tcavs_svm_tab}
    Variance of \textsc{TCAR} scores vs.\ the number of examples per concept set ($N$) for the concepts ``male'' and ``female'' on the \textbf{UCI Adult dataset}. The \textsc{TCAR} scores were obtained following the framework proposed by \citet{crabbe_concept_2022}. Error bars denote $\pm 1$ standard deviation. We fitted a curve of the form $f(N) = a/N + b$ to it. For \texttt{fc\_in} the parameters were $ a=0.121, b=2\times 10^{-8} $ for concept ``sex=Female'' and $ a=0.364, b=0.0025 $for concept ``sex=Male''. For \texttt{hidden\_layers\_0} they were $ a=0.306, b=7.97\times 10^{-10} $ \emph{i.e.} $ a=0.769, b=0.00526 $.}
\end{figure}

\subsubsection{Empirical Findings with Difference of Means}

When \text{\textsc{Cav}\xspace}s are computed via the \emph{Difference of Means}-method, the variance decline best resembles an $\mathcal{O}(1/N)$ rate, compared to the other two methods we evaluated.

\begin{figure}[H]
    \centering
    \begin{minipage}{0.48\textwidth}
        \centering
        \includegraphics[width=\linewidth]{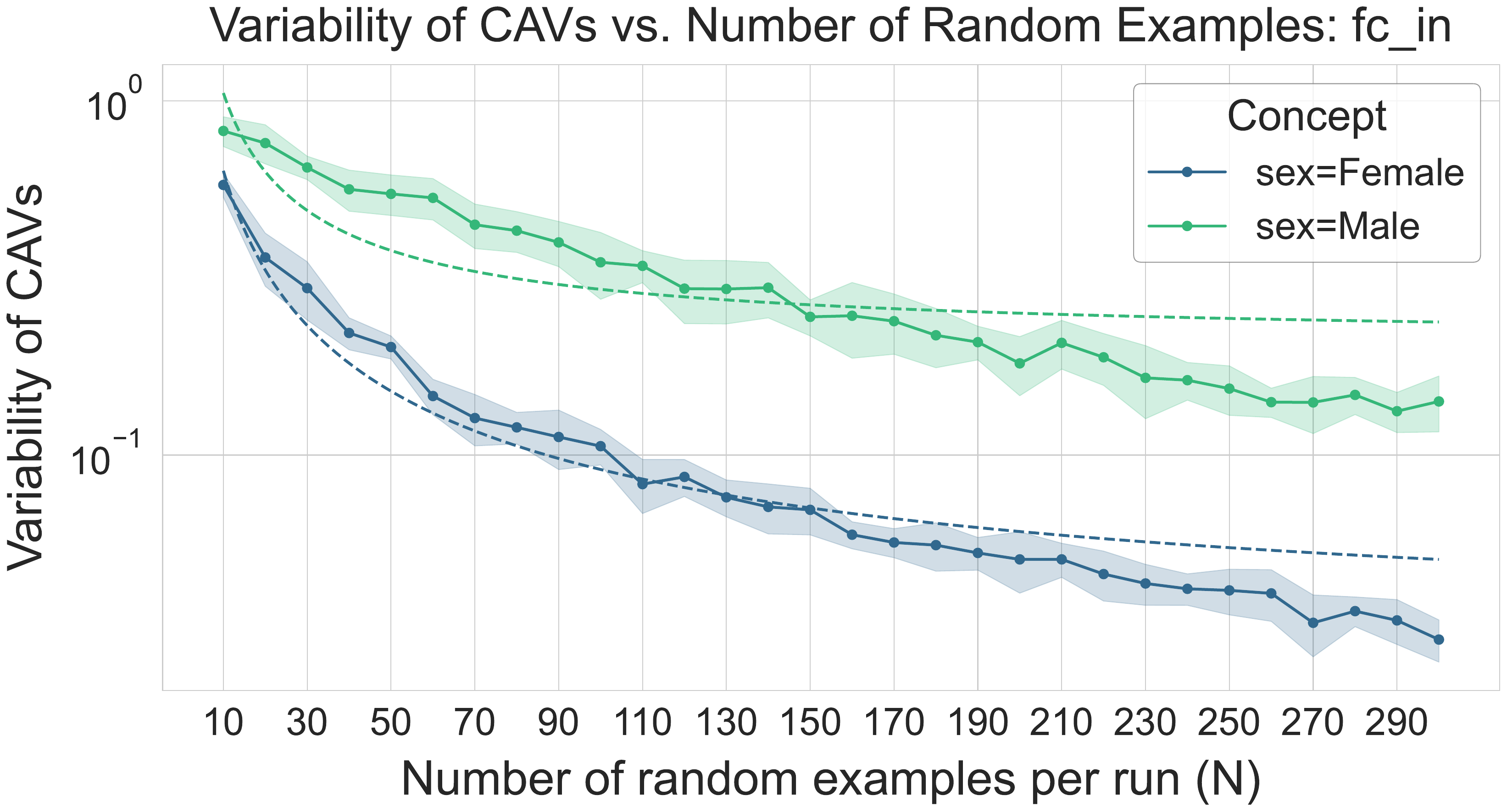}
    \end{minipage}
    \hfill
    \begin{minipage}{0.48\textwidth}
        \centering
        \includegraphics[width=\linewidth]{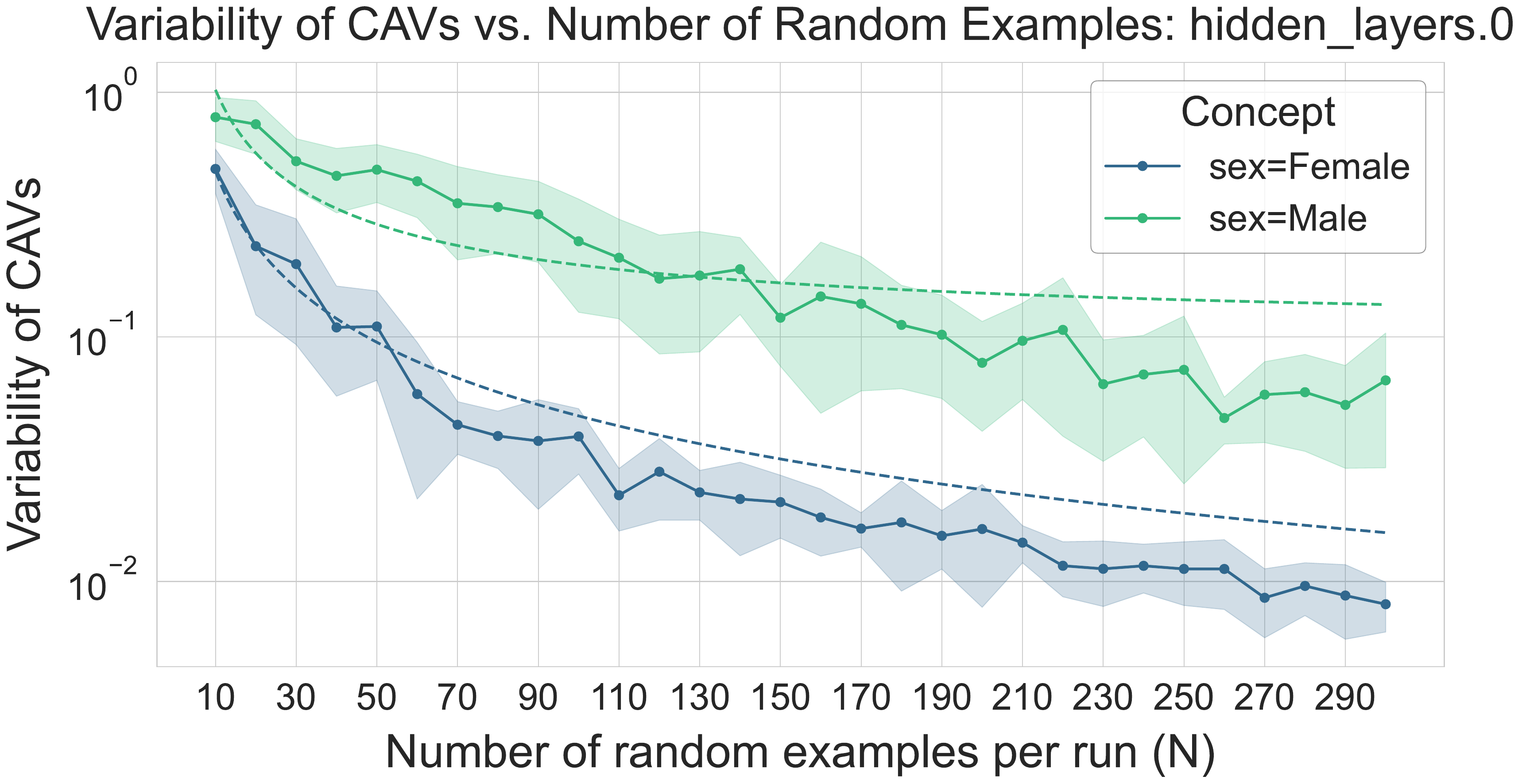}
    \end{minipage}
    \caption[Variability of CAVs (Difference of Means) on UCI Adult]{\label{fig:variance_of_cavs_dom_tab}The linear classifiers for \text{\textsc{Cav}\xspace} generation were trained using \textbf{difference-of-means}. Error bars indicate $\pm 1$ SD; the $y$-axis is log-scaled. Variance is estimated by the sum of per-feature variances across ten independent runs.  We fitted a curve of the form $f(N) = a/N + b$ to it. For \texttt{fc\_in} the parameters were $ 6.05/N  +  0.0306 $ for ``sex=Female'',$ 8.44/N  +  0.209 $ for  ``sex=Male''. For \texttt{hidden\_layers\_0} they were $ 4.75/N  +  2.18\times 10^{-13} $ \emph{i.e.} $ 9.19/N  +  0.105 $. Moreover, for \texttt{hidden\_layers\_0} we clipped below $3 \times 10^{-3}$ to avoid visual distortion from extreme lows.
    }
\end{figure}

\begin{figure}[H]
    \centering
    \begin{minipage}{0.48\textwidth}
        \centering
        \includegraphics[width=\linewidth]{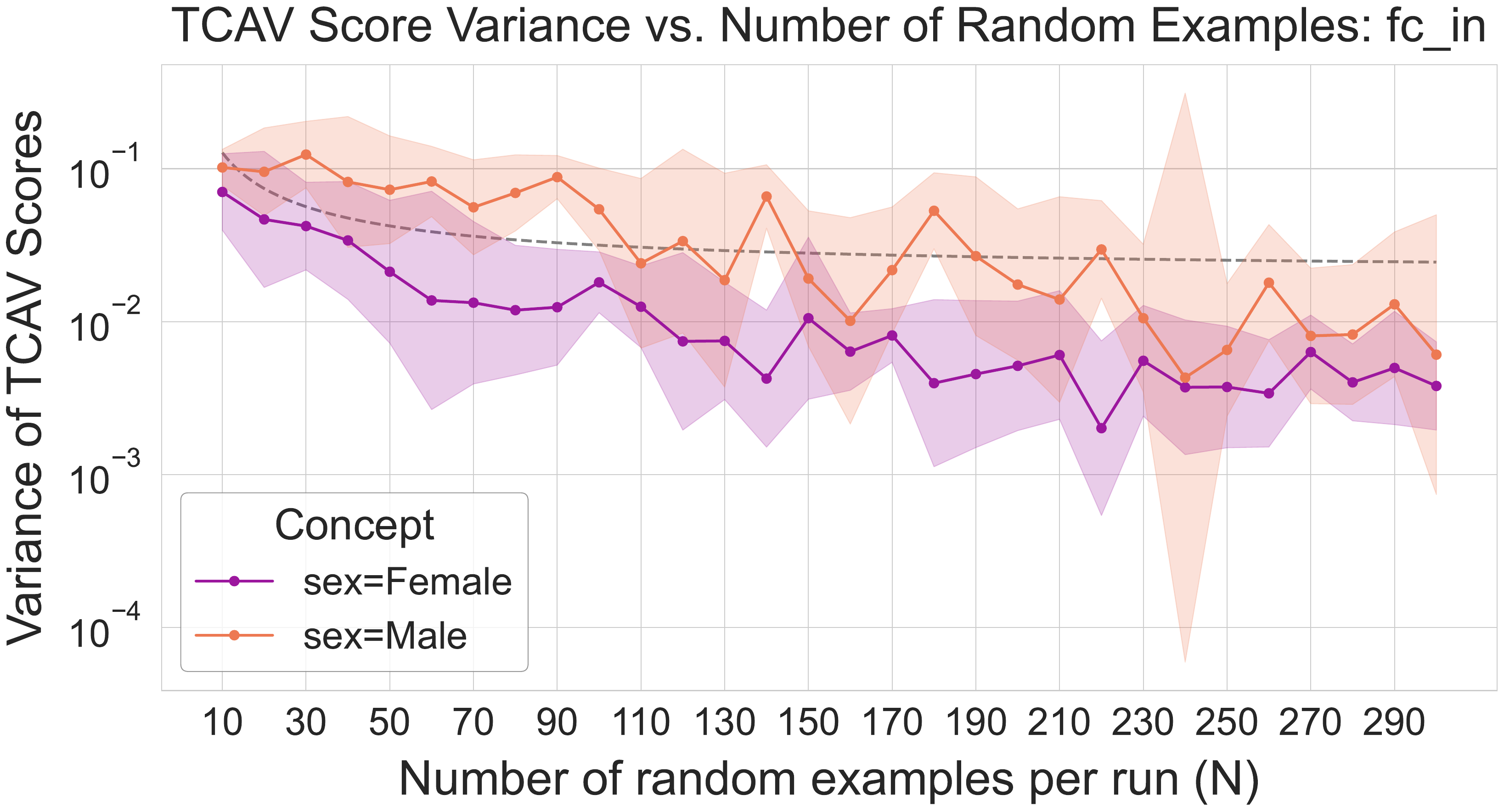}
    \end{minipage}
    \hfill 
    \begin{minipage}{0.48\textwidth}
        \centering
        \includegraphics[width=\linewidth]{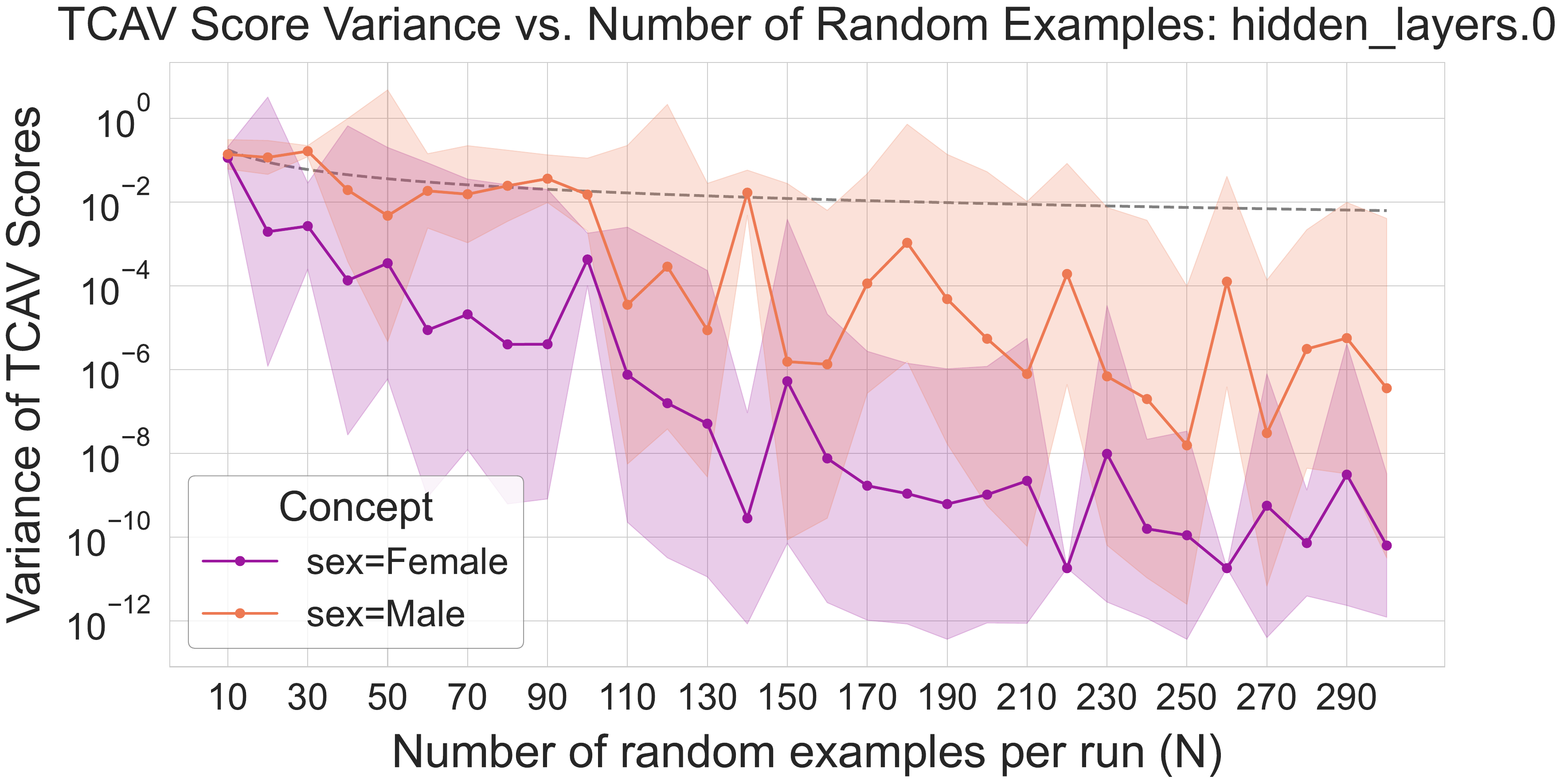}
    \end{minipage}
    \caption[Variance of TCAV scores (Difference of Means) on UCI Adult]{\label{fig:variance_of_tcavs_dom_tab}
    Mean geometrical variance of \text{\textsc{Tcav}\xspace} scores vs. the number of examples per concept set ($N$) for the concepts ``male'' and ``female'' on the \textbf{UCI Adult dataset}. The underlying \text{\textsc{Cav}\xspace}s were generated using \textbf{difference-of-means}. Error bars denote $\pm 1$ geometrical standard deviation.  We fitted a curve of the form $f(N) = a/N + b$ to it. For \texttt{fc\_in} the parameters were $a=1.06, b=0.0209$, for \texttt{hidden\_layers\_0} they were $a=1.81, b=3.05\times 10^{-4}$.
}
\end{figure}

\subsubsection{Compute Trade-off}

\begin{figure}[H]
    \centering

    \begin{subfigure}[t]{0.48\textwidth}
        \centering
        \includegraphics[width=\linewidth]{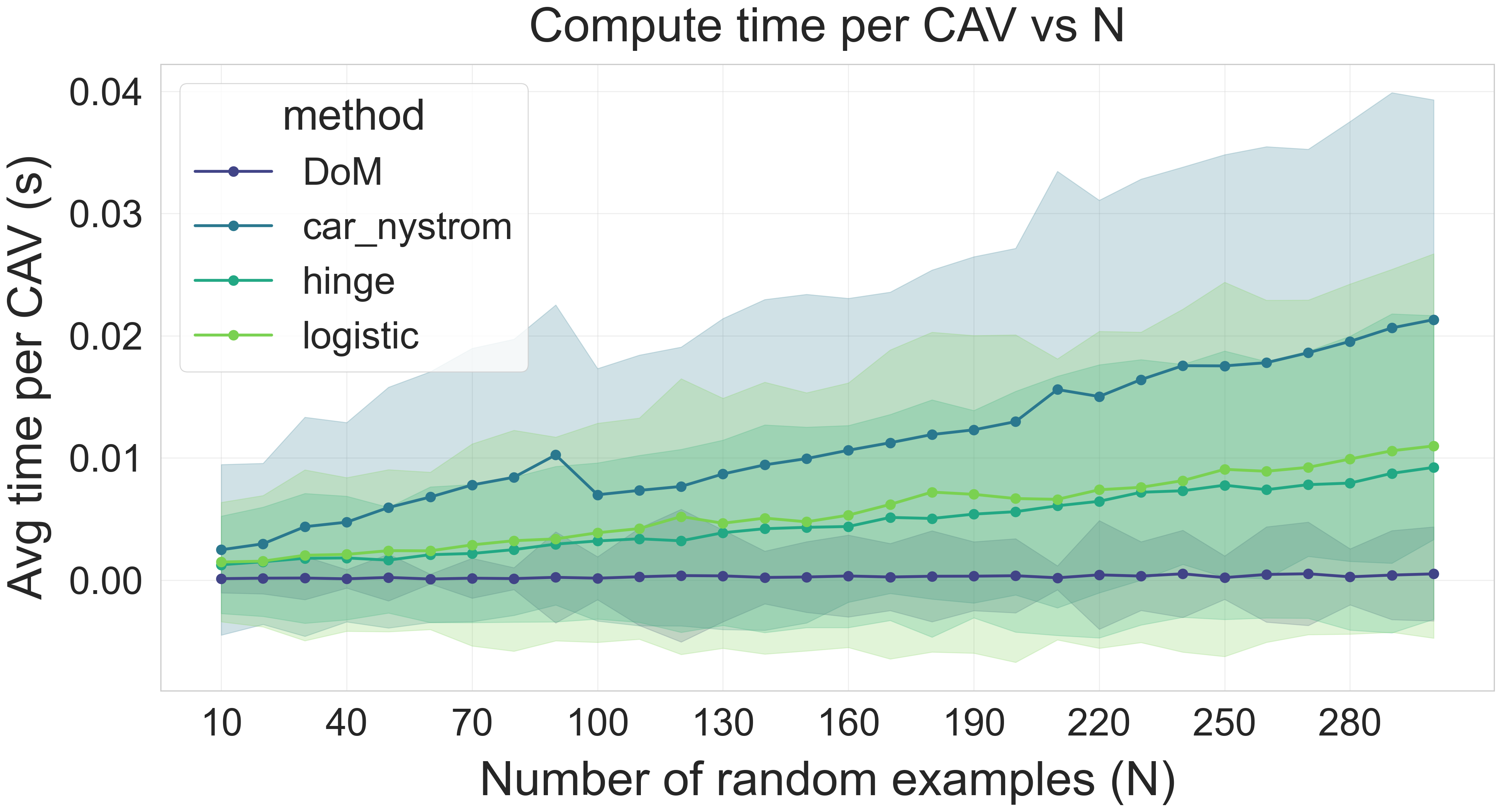}
        \caption{\textbf{Compute vs.\ $N$.} Average wall-clock time per \text{\textsc{Cav}\xspace} as a function of the number of random examples per run ($N$). Shaded band denotes $\pm 1$ SD across runs.}
        \label{fig:compute_vs_n_tab_appendix}
    \end{subfigure}
    \hfill
    \begin{subfigure}[t]{0.48\textwidth}
        \centering
        \includegraphics[width=\linewidth]{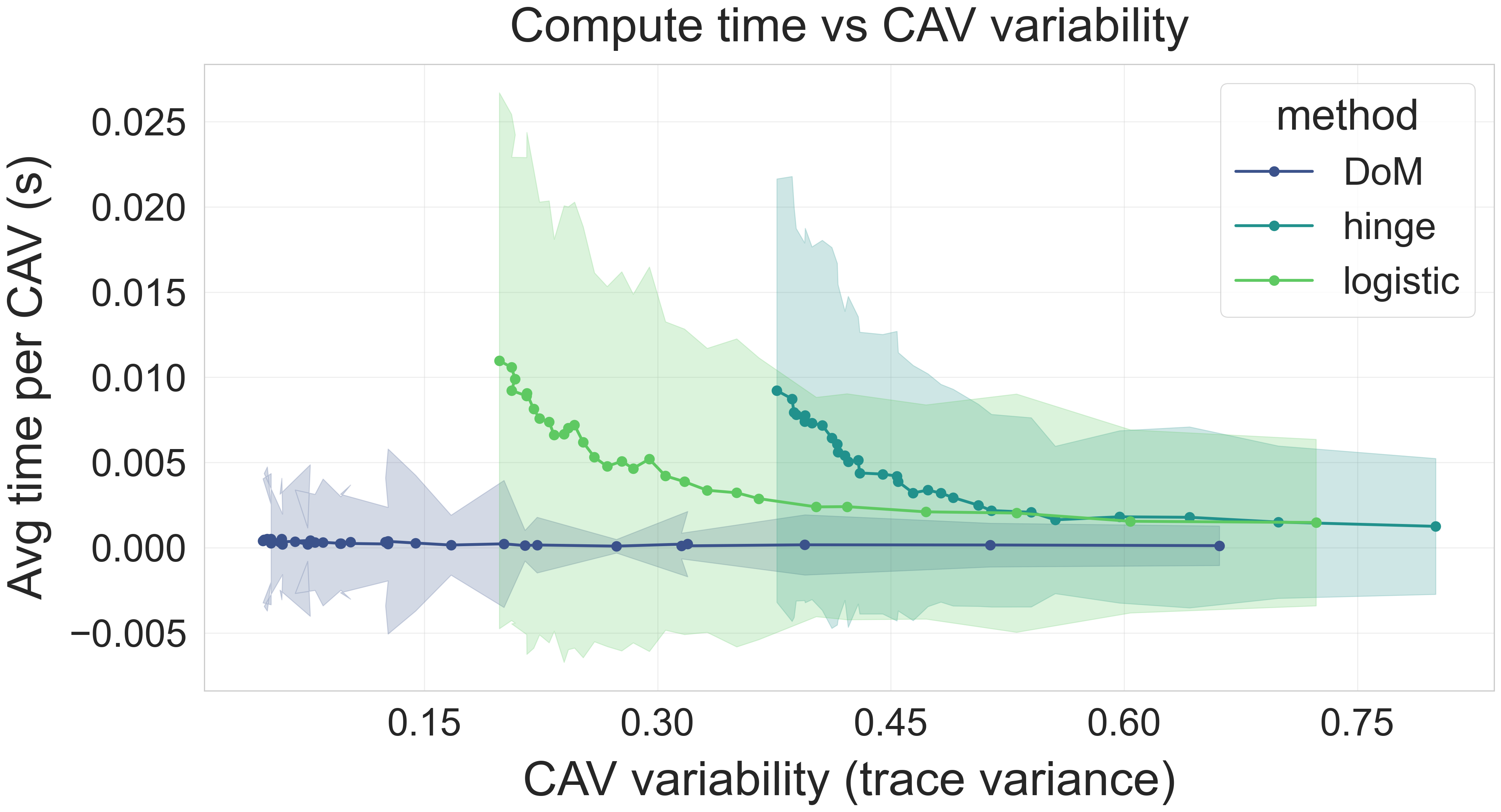}
        \caption{\textbf{Compute vs.\ variability.} Compute-stability trade-off: average time per \text{\textsc{Cav}\xspace} versus \text{\textsc{Cav}\xspace} variability (trace variance). Points correspond to different $N$ values (connected in increasing $N$). Shaded band denotes $\pm 1$ SD for compute time.}
        \label{fig:compute_vs_variability_tab_appendix}
    \end{subfigure}

    \caption[Compute cost vs.\ stability of CAVs]{\label{fig:compute_tradeoff_tab_appendix}
    Compute cost and stability of \text{\textsc{Cav}\xspace} generation (see legend for the selected setting, e.g.\ layer/method). Left: compute time per \text{\textsc{Cav}\xspace} vs.\ $N$. Right: compute time vs.\ resulting \text{\textsc{Cav}\xspace} variability (trace variance), illustrating the compute-stability trade-off.}
\end{figure}

\subsection{TCAV for Text Data}
\label{appendix:text}

\subsubsection{Empirical Findings with Binary Cross-Entropy Loss}

For the text data (\textbf{IMDB}), we first report the results from the logistic regression classifier using binary cross-entropy loss.
We also report the same two quantities for \texttt{SmolLM2-1.7B-Instruct} \citep{allal_smollm2_2025} at \texttt{decoder\_hidden\_6} and \texttt{decoder\_hidden\_12}.
\begin{figure}[H]
    \centering
    \begin{minipage}{0.48\textwidth}
        \centering
        \includegraphics[width=\linewidth]{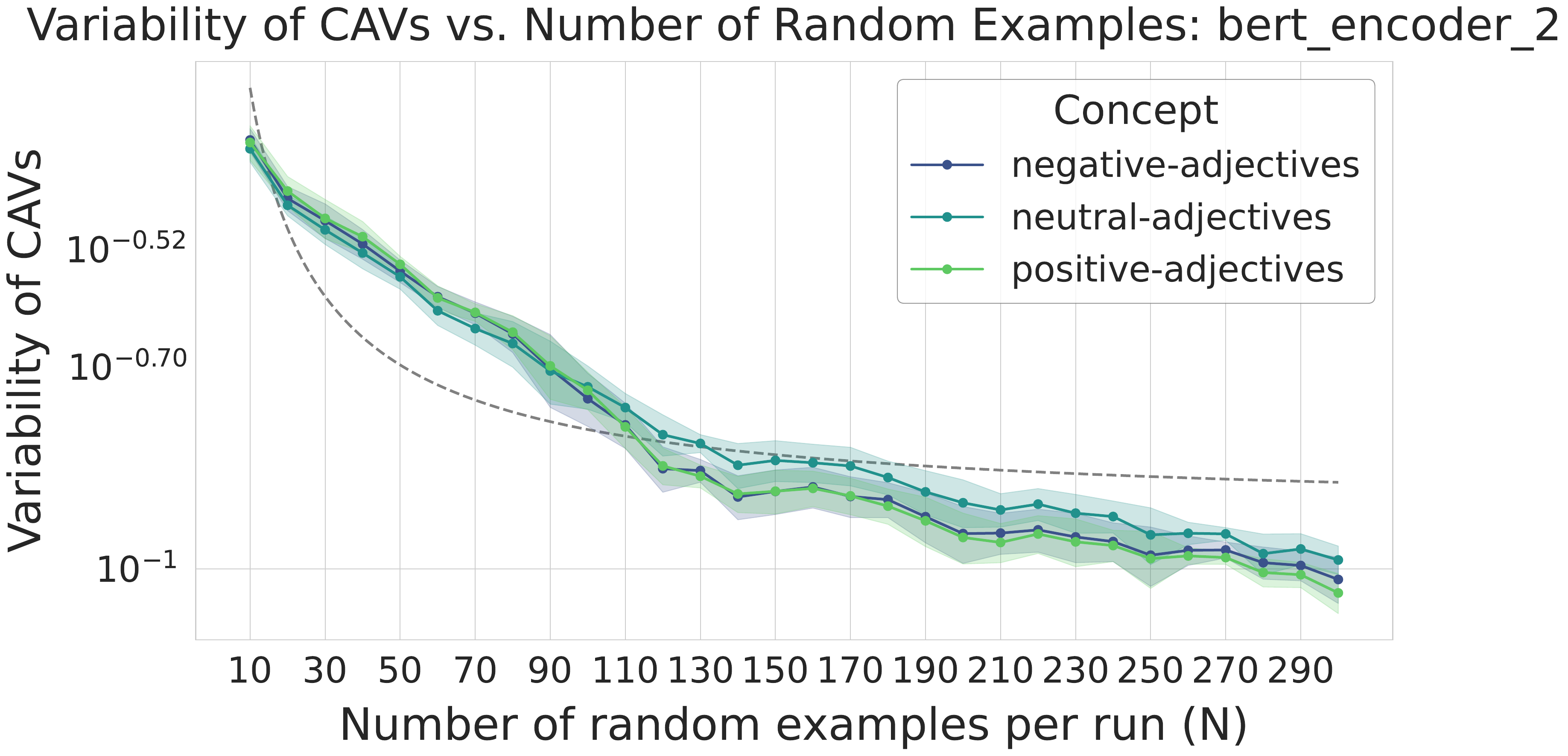}
    \end{minipage}
    \hfill 
    \begin{minipage}{0.48\textwidth}
        \centering
        \includegraphics[width=\linewidth]{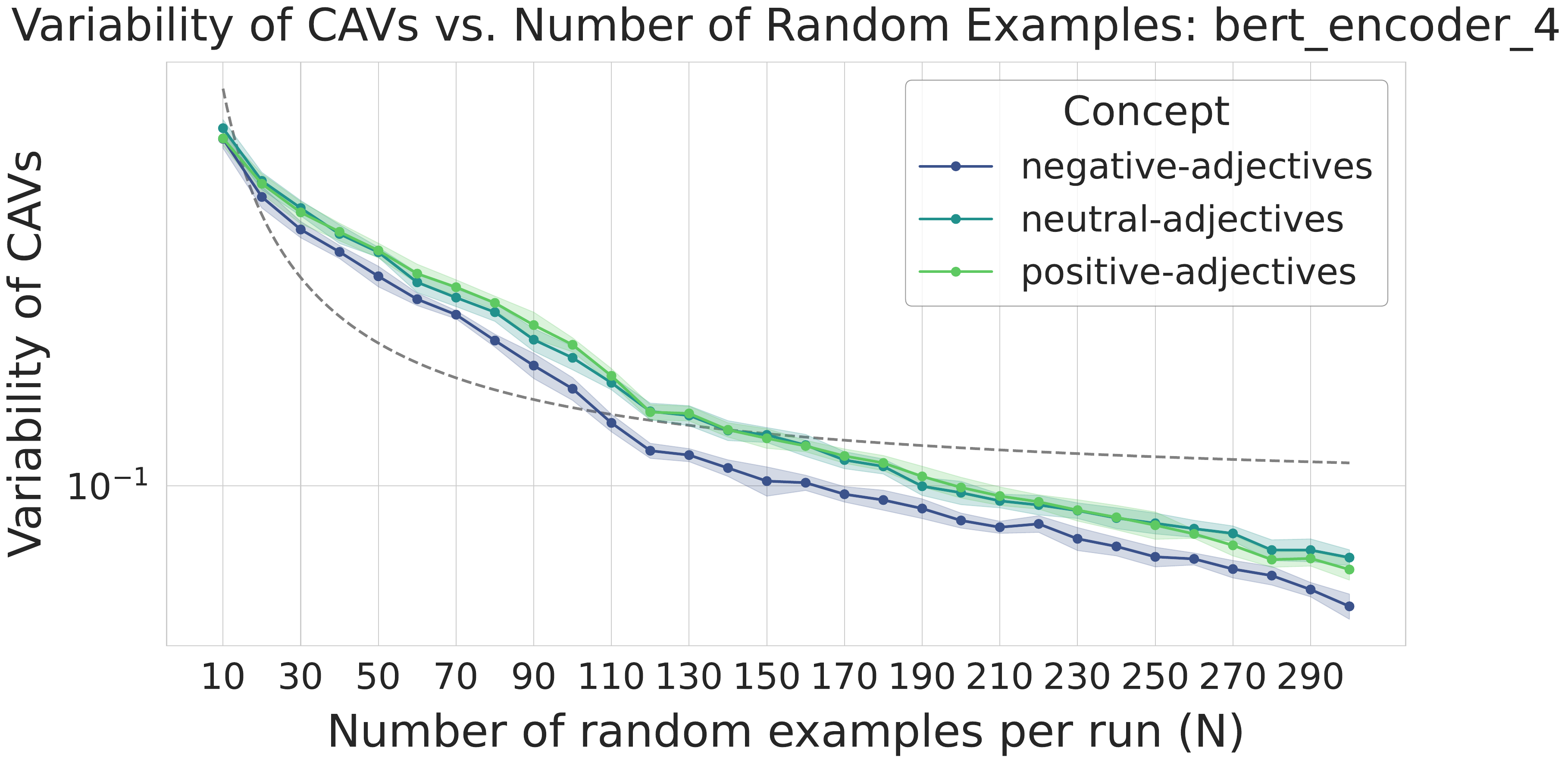}
    \end{minipage}

    \caption[Variability of CAVs vs. Number of Samples]{\label{fig:variance_of_cavs_log_text}
        Mean variability of \text{\textsc{Cav}\xspace}s as a function of the number of random examples per run ($N$), shown for concepts at two different layers (\texttt{bert\_encoder\_2} on the left, \texttt{bert\_encoder\_4} on the right). Error bars indicate $\pm 1$ SD; the $y$-axis is log-scaled. Variance is estimated by the sum of per-feature variances across ten runs. We fitted a curve of the form $f(N) = a/N + b$ to it. For \texttt{bert\_encoder\_2} the parameters were $a=4.02, b=0.121$, for \texttt{bert\_encoder\_4} they were $a=4.27, b=0.0958$.
    }
\end{figure}

\begin{figure}[h!]
    \centering
    \begin{minipage}{0.48\textwidth}
        \centering
        \includegraphics[width=\linewidth]{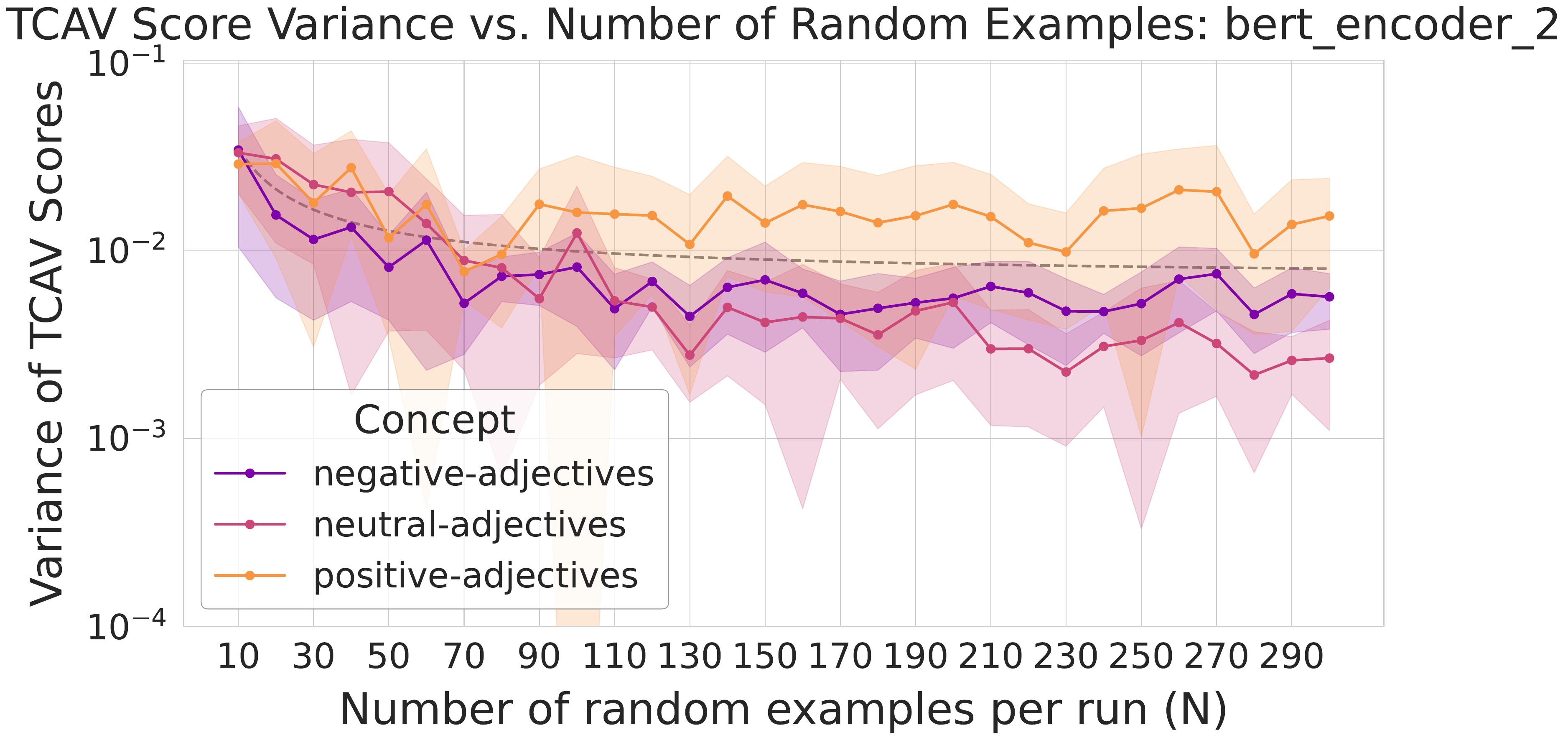}
    \end{minipage}
    \hfill 
    \begin{minipage}{0.48\textwidth}
        \centering
        \includegraphics[width=\linewidth]{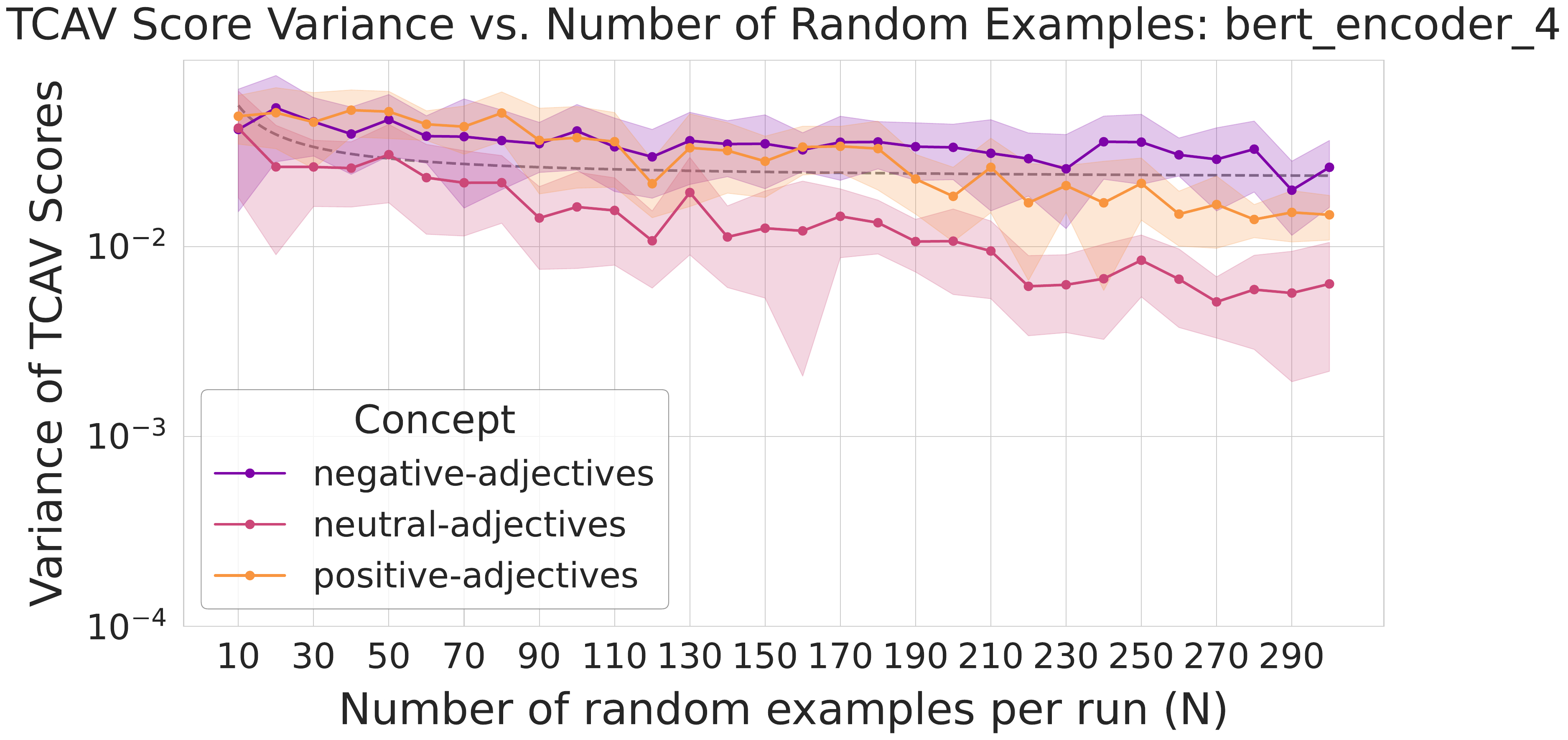}
    \end{minipage}
    \caption[Variance of TCAV scores vs. Number of Samples]{\label{fig:variance_of_tcavs_log_text}
        Variance of \text{\textsc{Tcav}\xspace} scores at layers \texttt{bert\_encoder\_2} and \texttt{bert\_encoder\_4} vs. the number of examples per concept set ($N$) for ``positive''-, ``negative''-, and ``neutral''-adjective concepts on the IMDB dataset; error bars denote $\pm 1$ standard deviation. We fitted a curve of the form $f(N) = a/N + b$ to it. For \texttt{bert\_encoder\_2} the parameters were $a=0.295, b=0.00654$, for \texttt{bert\_encoder\_4} they were $a=0.329, b=0.0226$. For \texttt{bert\_encoder\_2} we cut the $log‑y$-axis at $\geq 10^{-4}$ for readability to omit tiny standard deviation values.
    }
\end{figure}

\begin{figure}[H]
    \centering
    \begin{minipage}{0.48\textwidth}
        \centering
        \includegraphics[width=\linewidth]{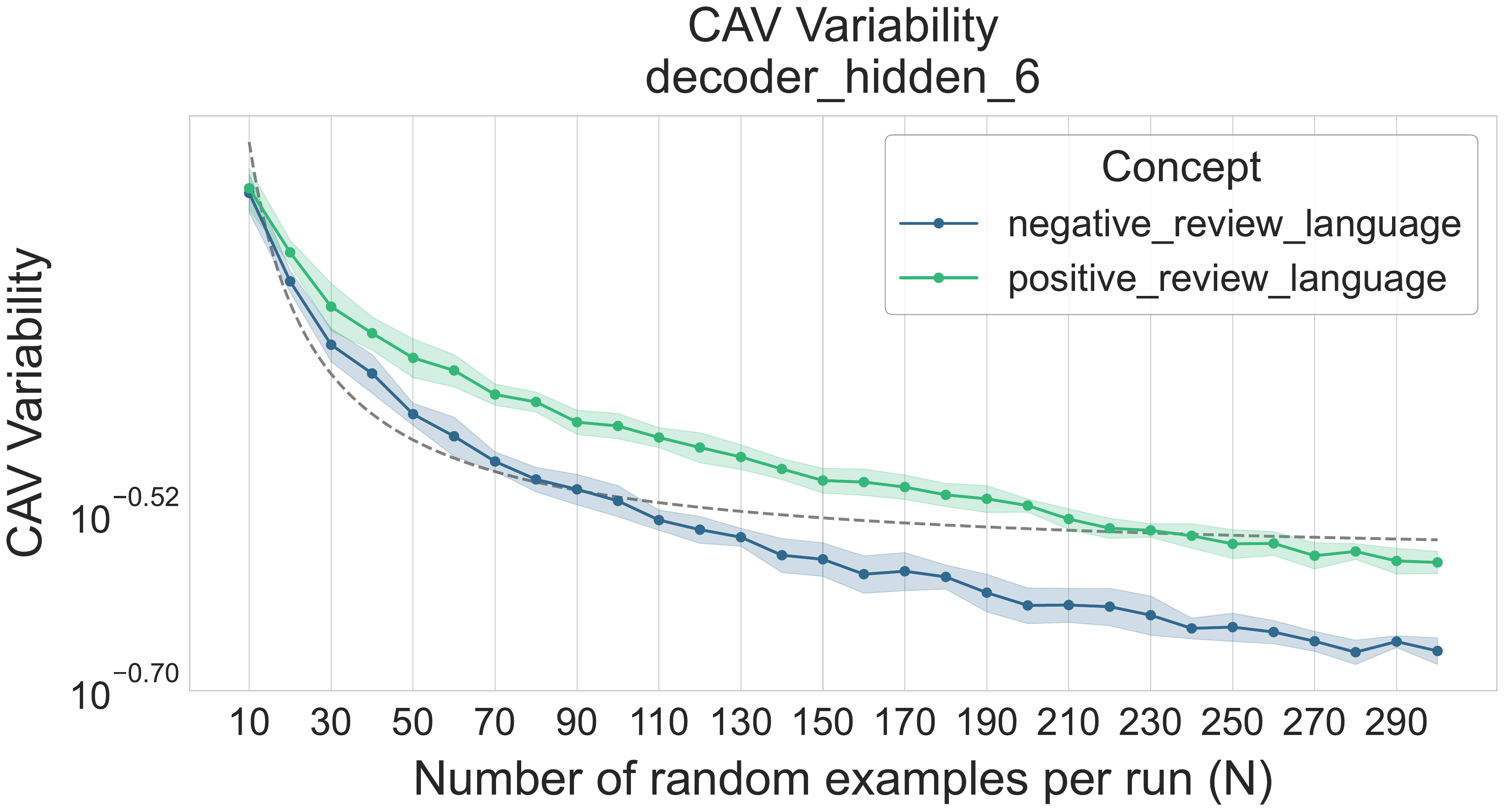}
    \end{minipage}
    \hfill
    \begin{minipage}{0.48\textwidth}
        \centering
        \includegraphics[width=\linewidth]{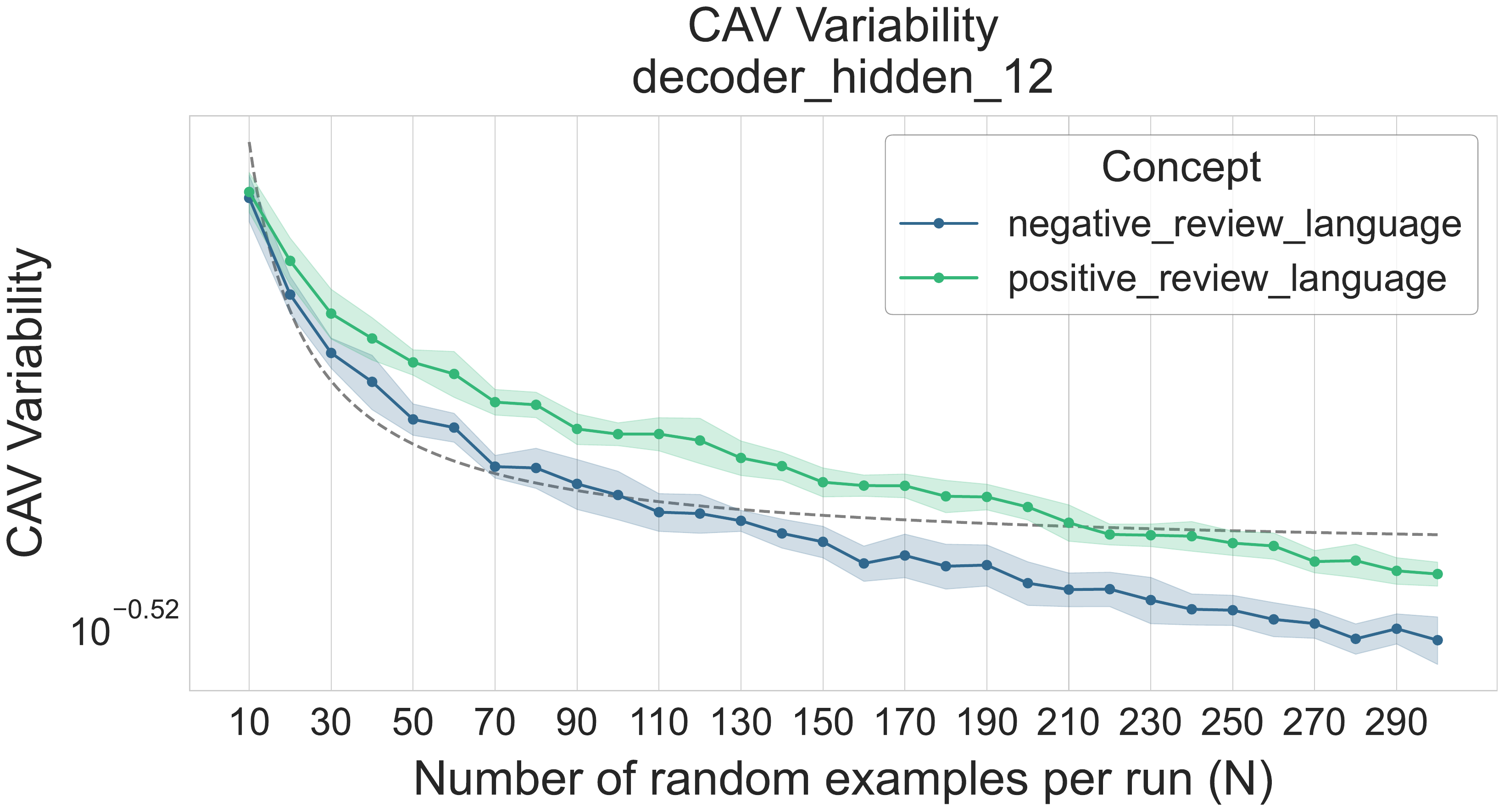}
    \end{minipage}
    \caption[Variability of CAVs (Log Loss) for SmolLM2]{\label{fig:variance_of_cavs_logistic_smollm2_text}
    Mean variability of \text{\textsc{Cav}\xspace}s for \texttt{SmolLM2-1.7B-Instruct} as a function of the number of random examples per run ($N$) on the IMDB setting. Results are shown for decoder layers \texttt{decoder\_hidden\_6} and \texttt{decoder\_hidden\_12}. Error bars indicate $\pm 1$ SD; the $y$-axis is log-scaled.}
\end{figure}

\begin{figure}[H]
    \centering
    \begin{minipage}{0.48\textwidth}
        \centering
        \includegraphics[width=\linewidth]{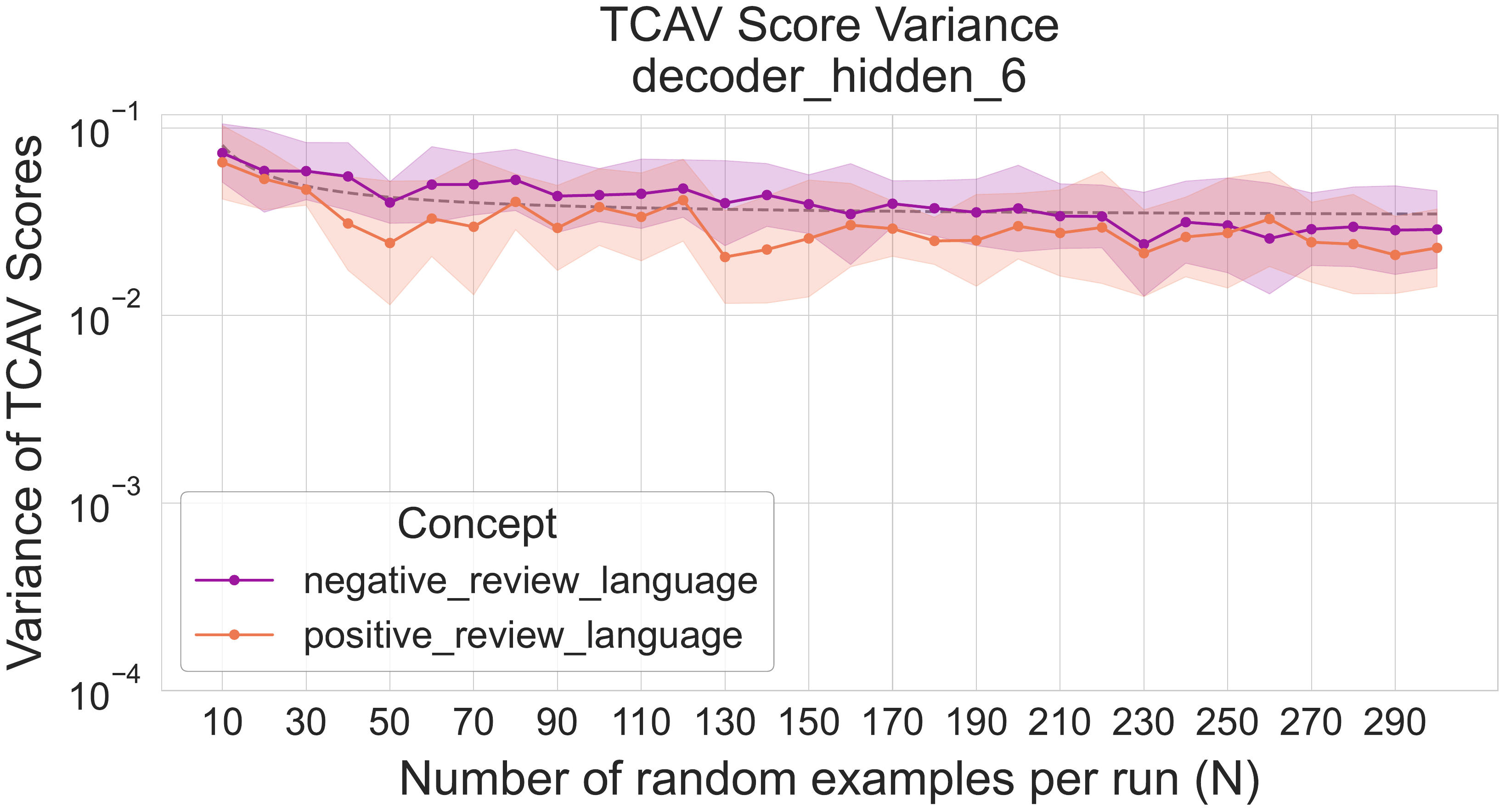}
    \end{minipage}
    \hfill
    \begin{minipage}{0.48\textwidth}
        \centering
        \includegraphics[width=\linewidth]{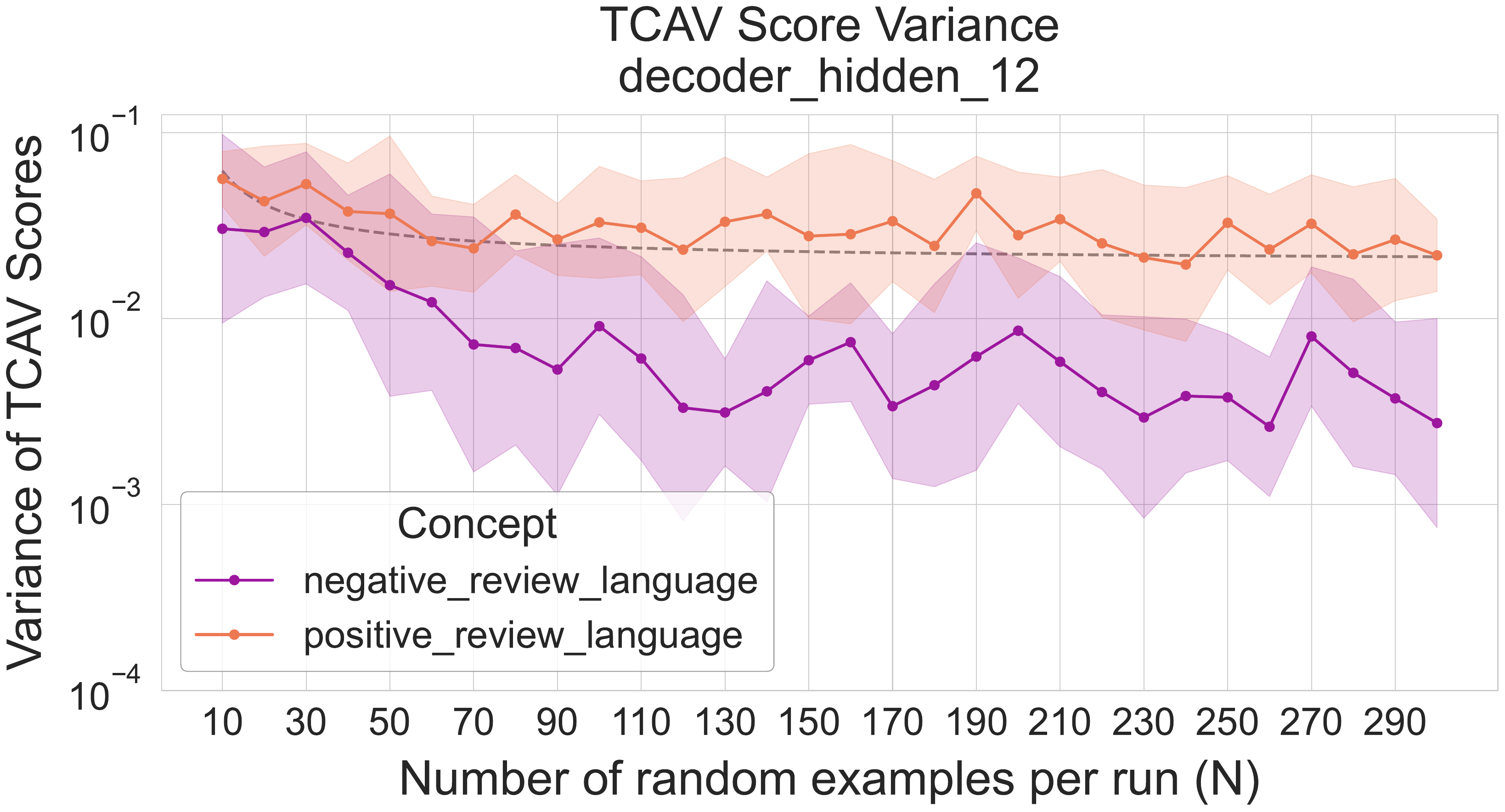}
    \end{minipage}
    \caption[Variance of TCAV scores (Log Loss) for SmolLM2]{\label{fig:variance_of_tcavs_logistic_smollm2_text}
    Variance of \text{\textsc{Tcav}\xspace} scores for \texttt{SmolLM2-1.7B-Instruct} at decoder layers \texttt{decoder\_hidden\_6} and \texttt{decoder\_hidden\_12} vs. the number of examples per concept set ($N$) on the IMDB setting. Error bars denote $\pm 1$ standard deviation.}
\end{figure}

\subsubsection{Empirical Findings with Hinge Loss}

We now present the results for the \texttt{SGDLinearModel}, which uses  hinge loss.

\begin{figure}[h!]
    \centering
    \begin{minipage}{0.48\textwidth}
        \centering
        \includegraphics[width=\linewidth]{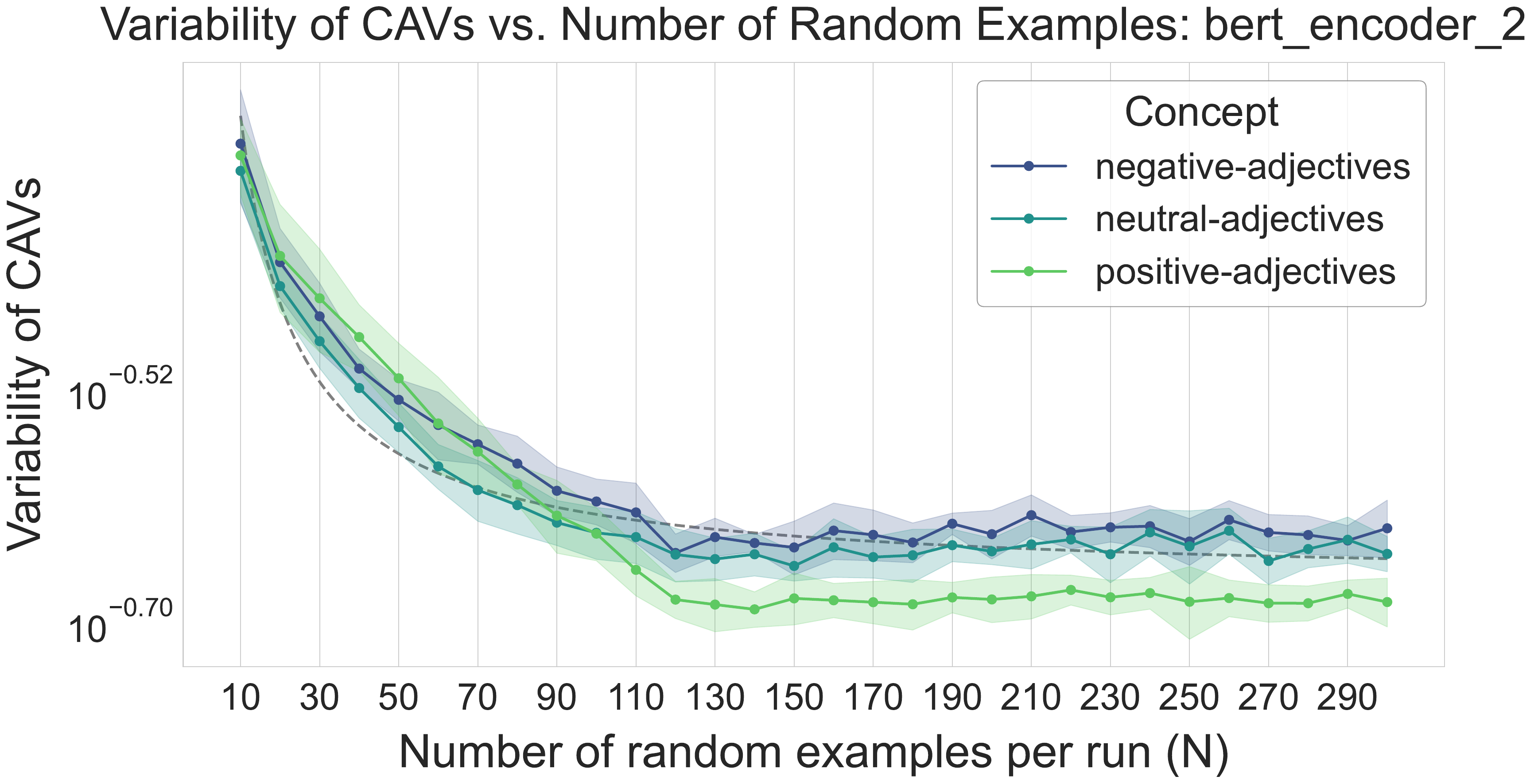}
    \end{minipage}
    \hfill
    \begin{minipage}{0.48\textwidth}
        \centering
        \includegraphics[width=\linewidth]{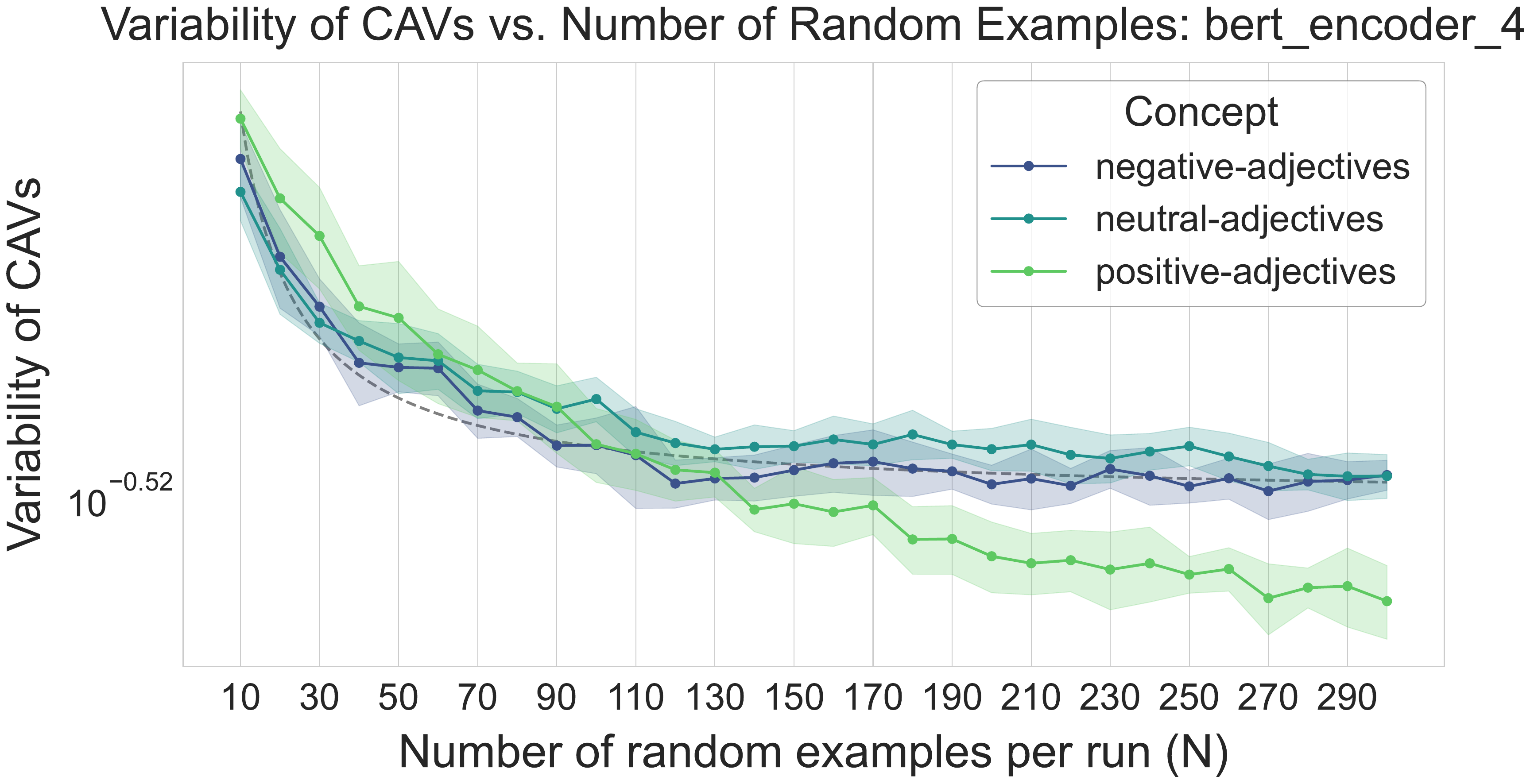}
    \end{minipage}
    \caption[Variability of CAVs (Hinge Loss) on UCI Adult]{\label{fig:variance_of_cavs_hinge_text}
        Mean variability of \text{\textsc{Cav}\xspace}s as a function of the number of random examples per run ($N$), shown for concepts at two different layers. Error bars indicate $\pm 1$ SD; the $y$-axis is log-scaled. Variance is estimated by the sum of per-feature variances across ten runs. We fitted a curve of the form $f(N) = a/N + b$ to it. For \texttt{bert\_encoder\_2} the parameters were $a=2.69, b=0.215$, for \texttt{bert\_encoder\_4} they were $a=3.02, b=0.298$.
    }
\end{figure}

\begin{figure}[H]
    \centering
    \begin{minipage}{0.48\textwidth}
        \centering
        \includegraphics[width=\linewidth]{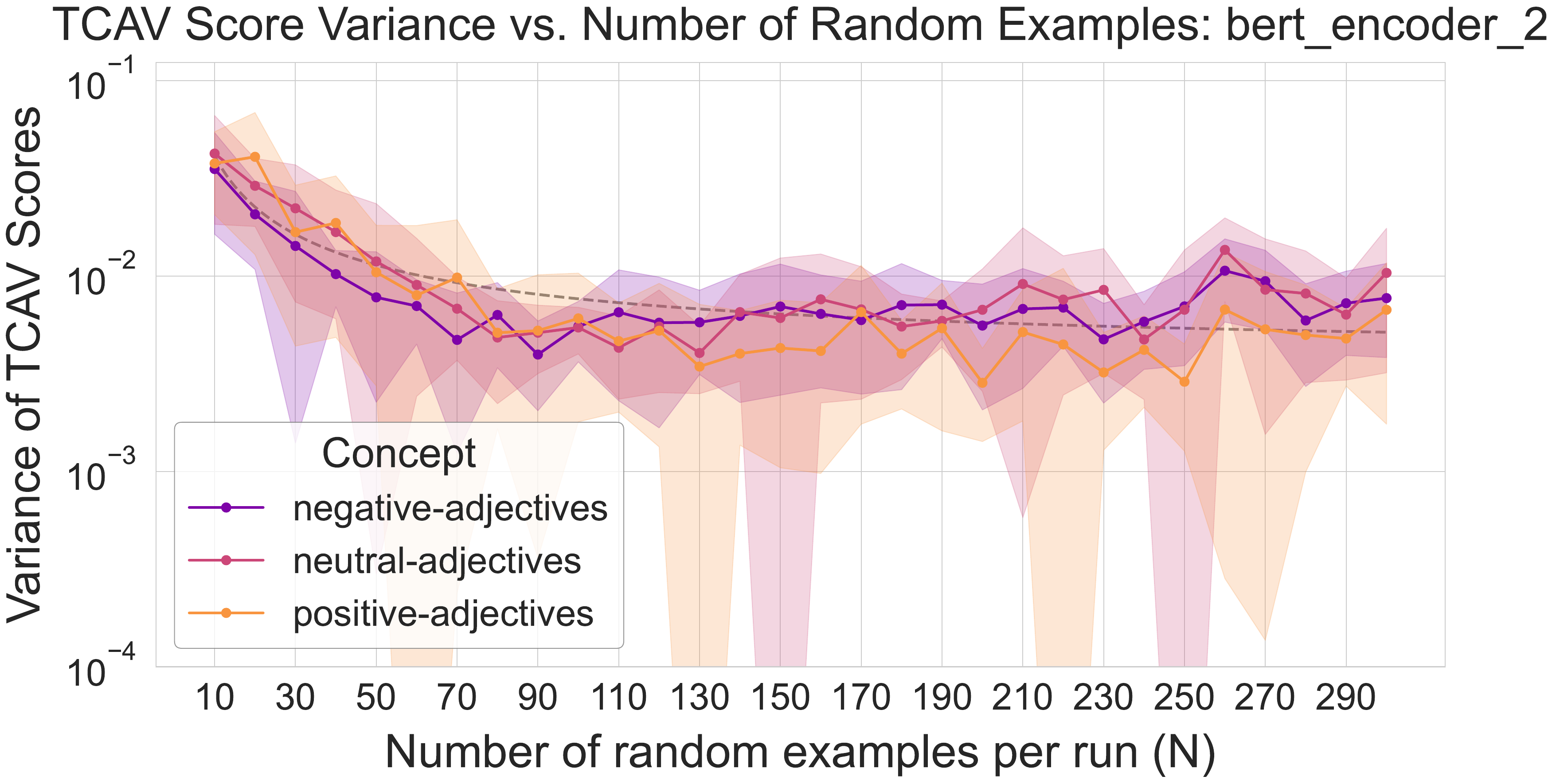}
    \end{minipage}
    \hfill 
    \begin{minipage}{0.48\textwidth}
        \centering
        \includegraphics[width=\linewidth]{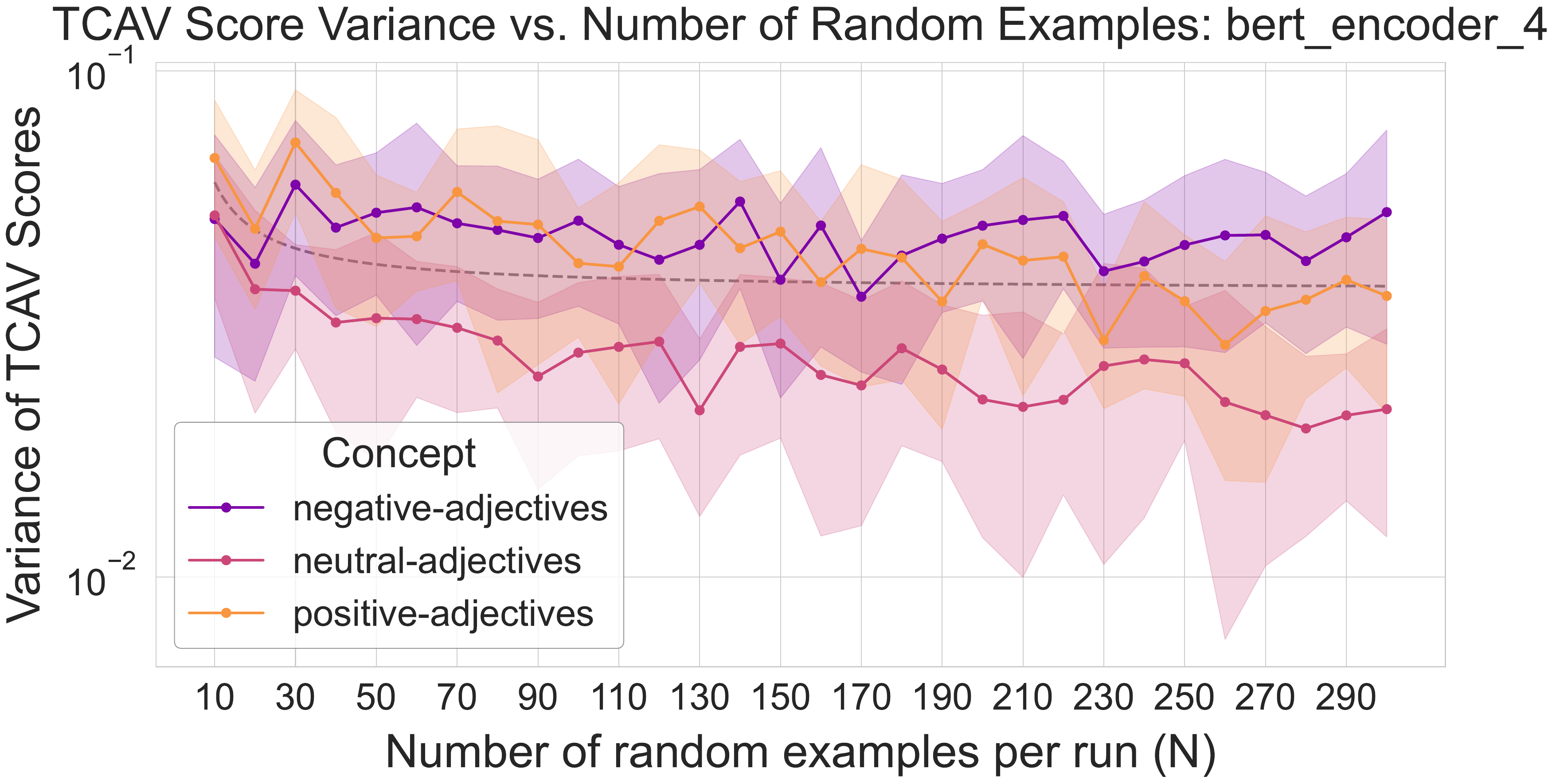}
    \end{minipage}
    \caption[Variance of TCAV scores vs. Number of Samples]{\label{fig:variance_of_tcavs_hinge_text}
        Variance of \text{\textsc{Tcav}\xspace} scores at layers \texttt{bert\_encoder\_2} and \texttt{bert\_encoder\_4} vs. the number of examples per concept set ($N$) for ``positive''-, ``negative''-, and ``neutral''-adjective concepts on the IMDB dataset; error bars denote $\pm 1$ standard deviation. We fitted a curve of the form $f(N) = a/N + b$ to it. For \texttt{bert\_encoder\_2} the parameters were $a=0.372, b=0.00391$ and cut the log scale at $\geq 10^{-4}$ so near‑zero bands do not dominate. For \texttt{bert\_encoder\_4} the fitted parameters were $a=0.235$ and $b=0.0368$.
    }
\end{figure}

\begin{figure}[H]
    \centering
    \begin{minipage}{0.48\textwidth}
        \centering
        \includegraphics[width=\linewidth]{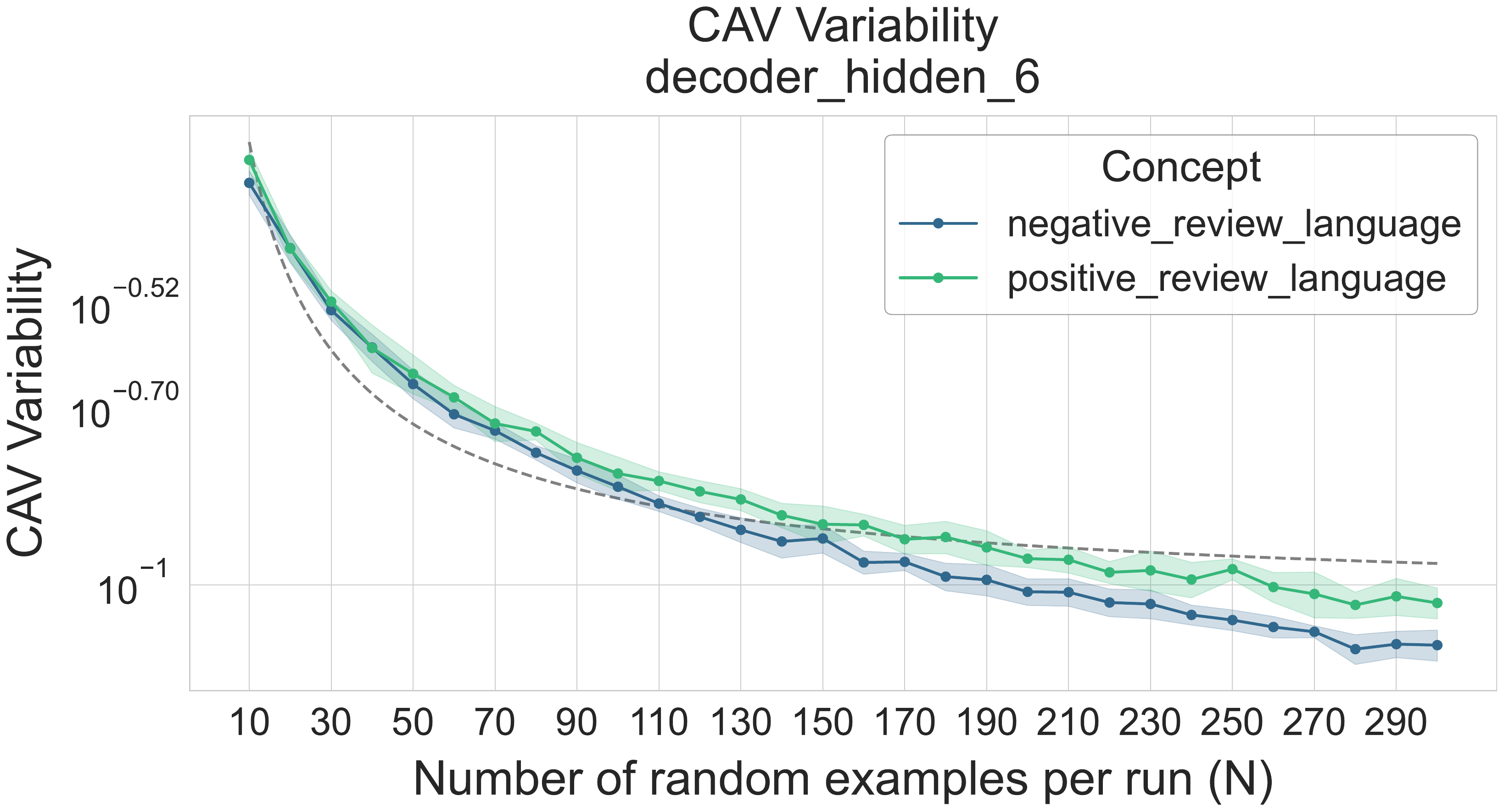}
    \end{minipage}
    \hfill
    \begin{minipage}{0.48\textwidth}
        \centering
        \includegraphics[width=\linewidth]{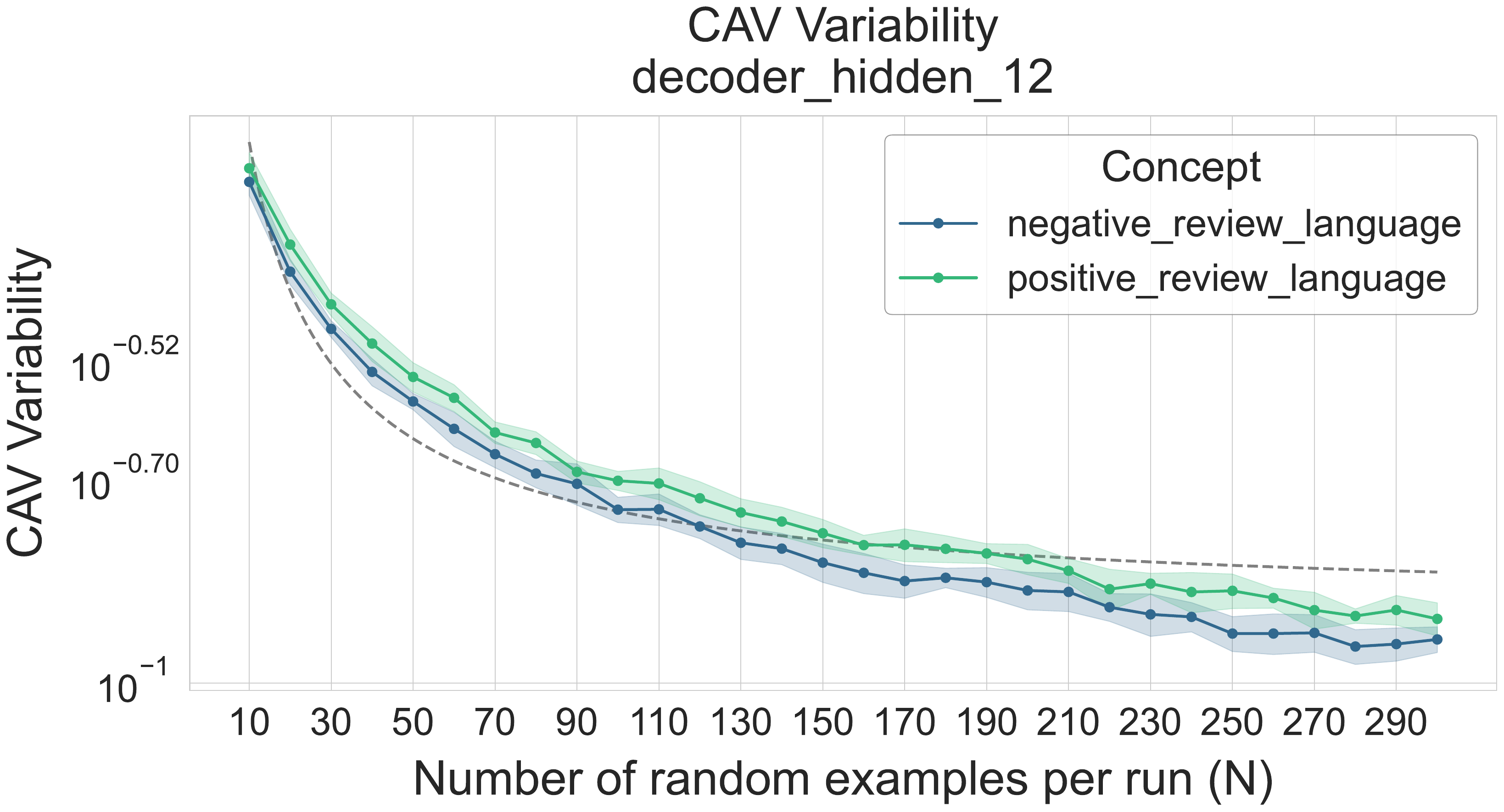}
    \end{minipage}
    \caption[Variability of CAVs (Hinge Loss) for SmolLM2]{\label{fig:variance_of_cavs_hinge_smollm2_text}
    Mean variability of \text{\textsc{Cav}\xspace}s for \texttt{SmolLM2-1.7B-Instruct} as a function of the number of random examples per run ($N$) on the IMDB setting. Results are shown for decoder layers \texttt{decoder\_hidden\_6} and \texttt{decoder\_hidden\_12}. Error bars indicate $\pm 1$ SD; the $y$-axis is log-scaled.}
\end{figure}

\begin{figure}[H]
    \centering
    \begin{minipage}{0.48\textwidth}
        \centering
        \includegraphics[width=\linewidth]{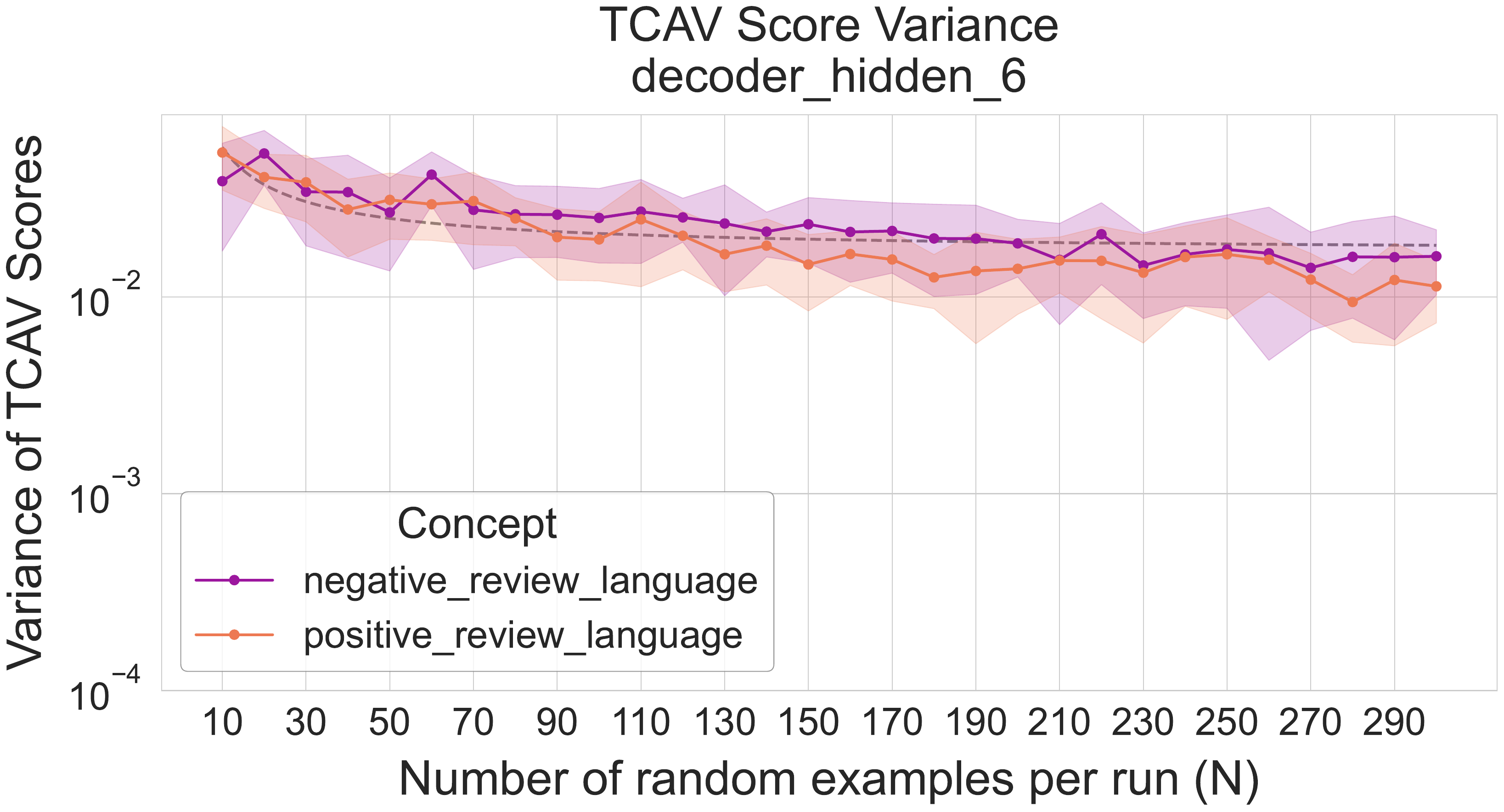}
    \end{minipage}
    \hfill
    \begin{minipage}{0.48\textwidth}
        \centering
        \includegraphics[width=\linewidth]{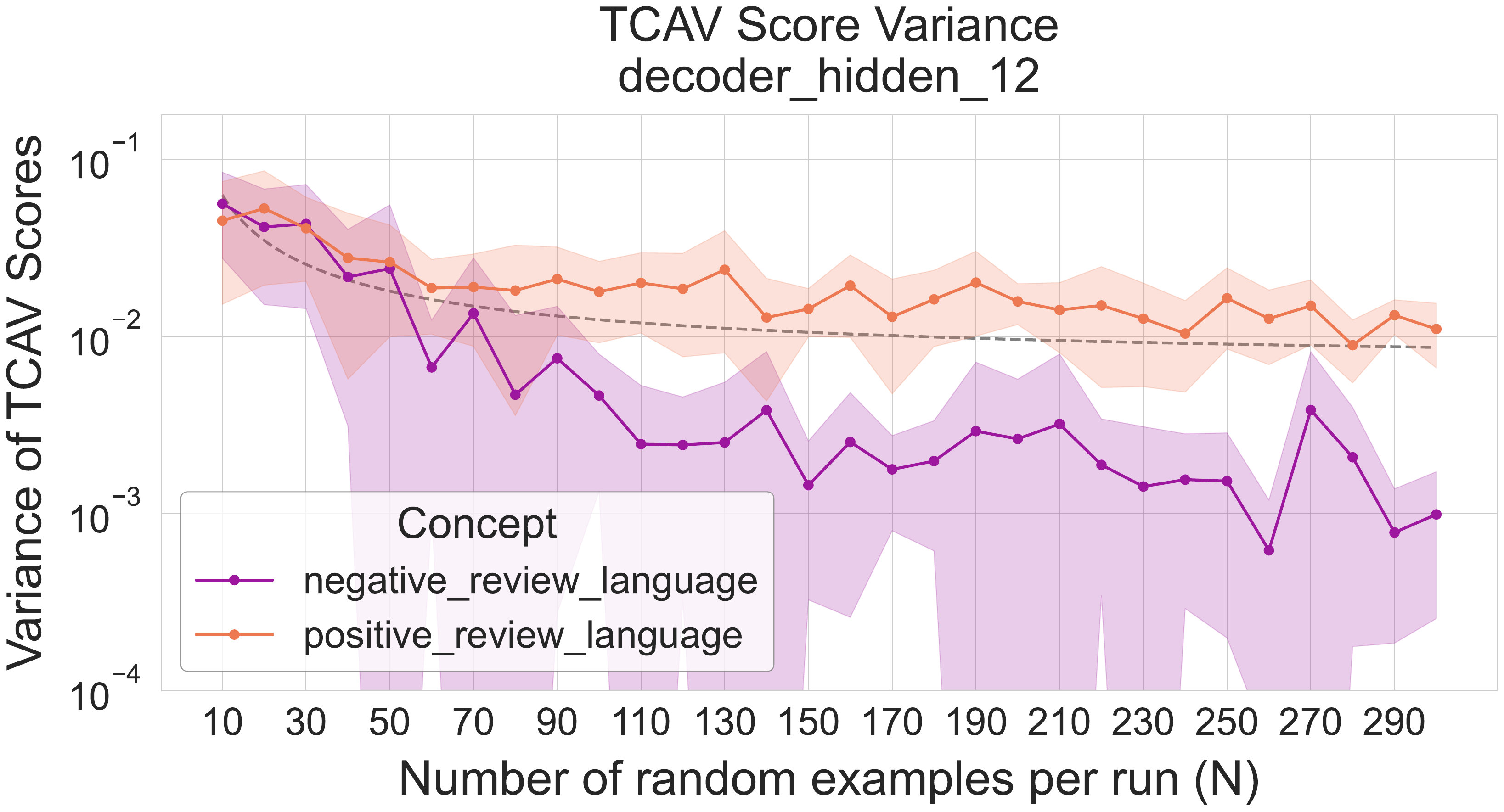}
    \end{minipage}
    \caption[Variance of TCAV scores (Hinge Loss) for SmolLM2]{\label{fig:variance_of_tcavs_hinge_smollm2_text}
    Variance of \text{\textsc{Tcav}\xspace} scores for \texttt{SmolLM2-1.7B-Instruct} at decoder layers \texttt{decoder\_hidden\_6} and \texttt{decoder\_hidden\_12} vs. the number of examples per concept set ($N$) on the IMDB setting. Error bars denote $\pm 1$ standard deviation.}
\end{figure}

\subsubsection{Empirical Findings with Concept Activation Regions}
 
In the following, we report results for \text{\textsc{Car}\xspace}s and \TCAR scores obtained with non-linear Support Vector Machines.

\begin{figure}[h!]
    \centering
    \begin{minipage}{0.48\textwidth}
        \centering
        \includegraphics[width=\linewidth]{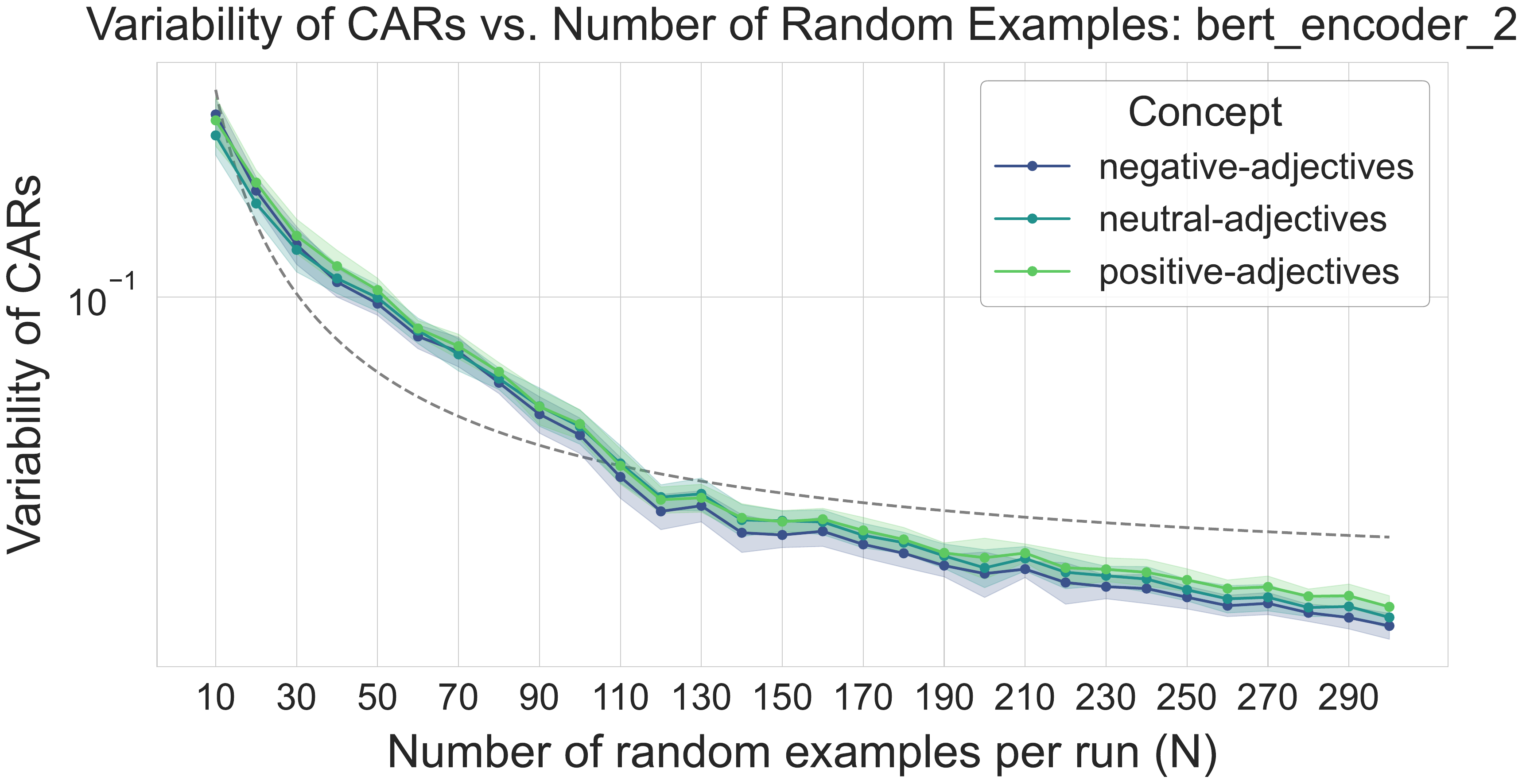}
    \end{minipage}
    \hfill
    \begin{minipage}{0.48\textwidth}
        \centering
        \includegraphics[width=\linewidth]{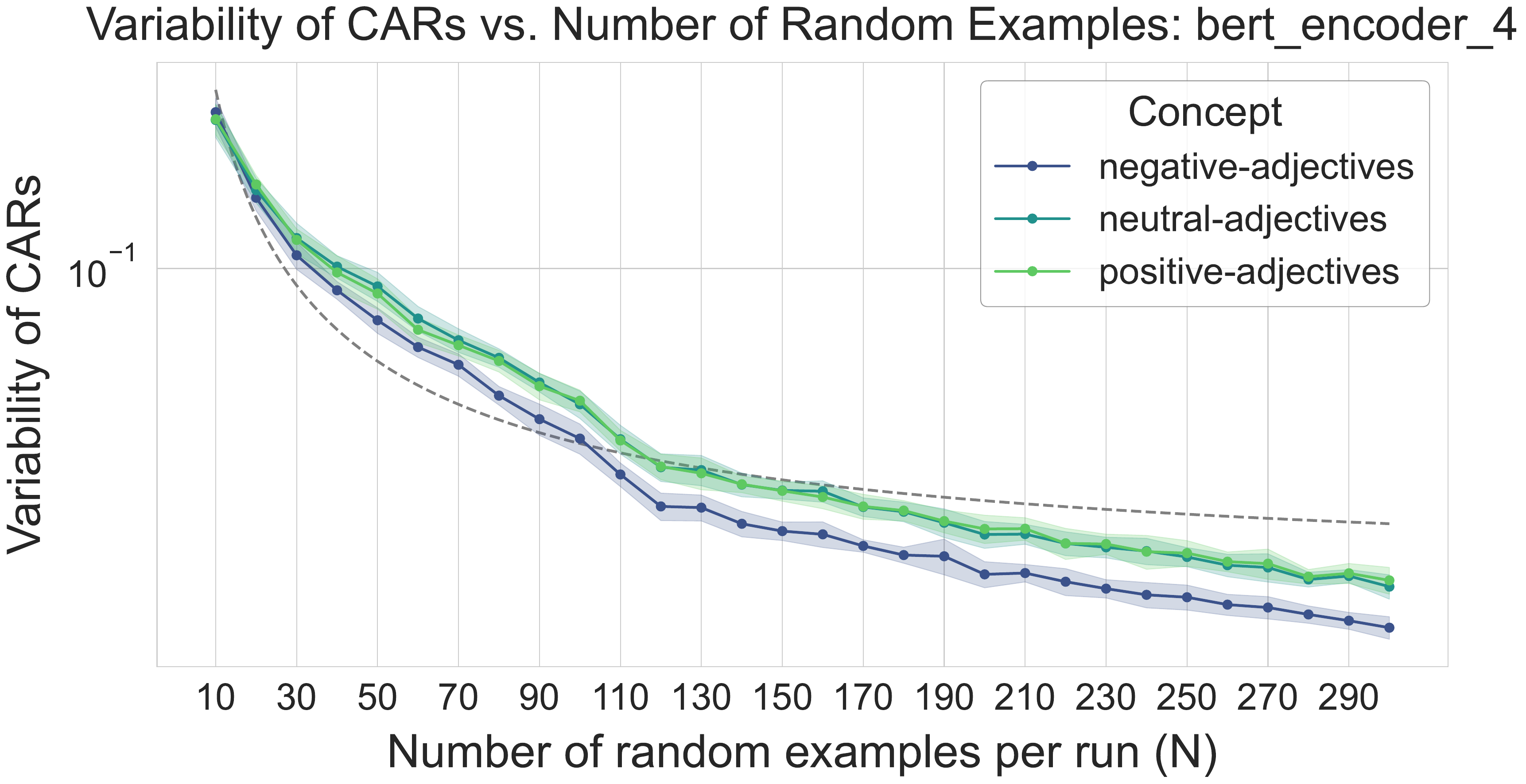}
    \end{minipage}
    \caption[Variability of CARs on UCI Adult]{\label{fig:variance_of_cavs_svm_text}
        Mean variability of \text{\textsc{Car}\xspace}s as a function of the number of random examples per run ($N$), shown for concepts at two different layers. Error bars indicate $\pm 1$ SD; the $y$-axis is log-scaled. Variance is estimated by the sum of per-feature variances across ten runs.
    }
\end{figure}

\begin{figure}[H]
    \centering
    \begin{minipage}{0.48\textwidth}
        \centering
        \includegraphics[width=\linewidth]{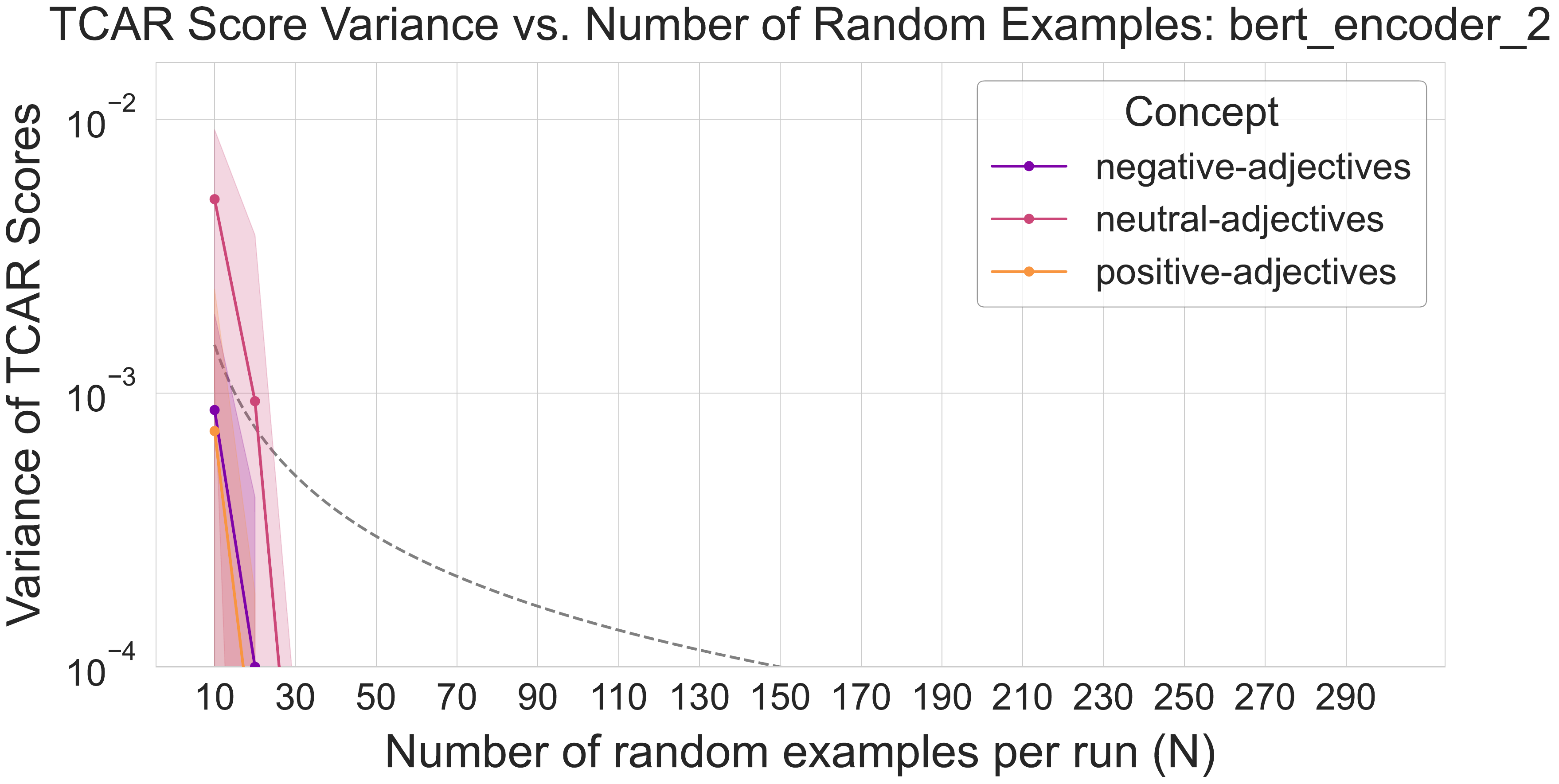}
    \end{minipage}
    \hfill 
    \begin{minipage}{0.48\textwidth}
        \centering
        \includegraphics[width=\linewidth]{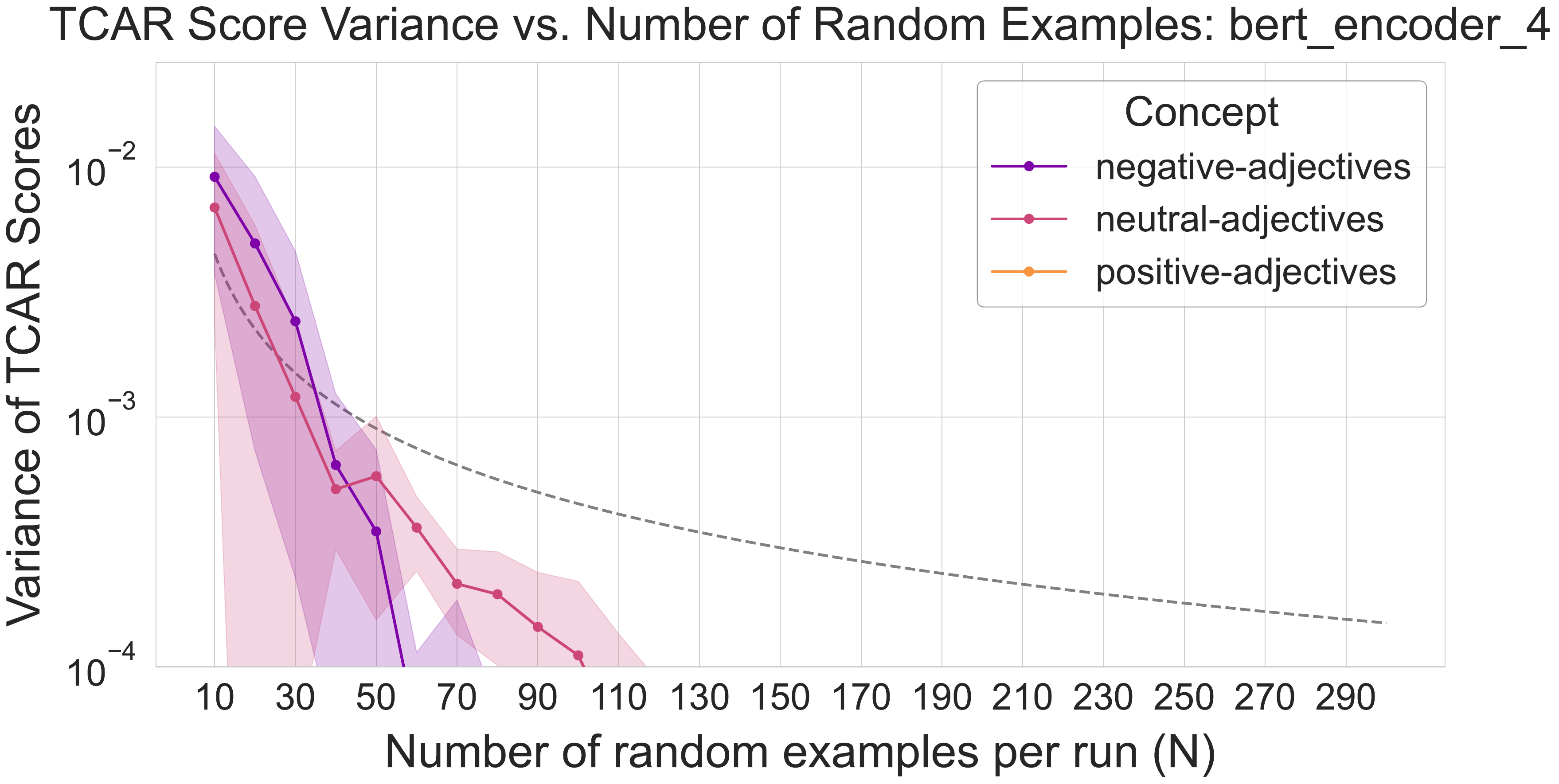}
    \end{minipage}
    \caption[Variance of TCAR scores vs. Number of Samples]{\label{fig:variance_of_tcavs_svm_text}
        Variance of \text{\textsc{Tcar}\xspace} scores at layers \texttt{bert\_encoder\_2} and \texttt{bert\_encoder\_4} vs. the number of examples per concept set ($N$) for ``positive''-, ``negative''-, and ``neutral''-adjective concepts on the IMDB dataset; error bars denote $\pm 1$ standard deviation. We fitted a curve of the form $f(N) = a/N + b$ to it. For \texttt{bert\_encoder\_2} the parameters were $a=0.372, b=0.00391$ and cut the log scale at $\geq 10^{-4}$ so near‑zero bands do not dominate. For \texttt{bert\_encoder\_4} the fitted parameters were $a=0.235$ and $b=0.0368$.
    }
\end{figure}

\begin{figure}[H]
    \centering
    \begin{minipage}{0.48\textwidth}
        \centering
        \includegraphics[width=\linewidth]{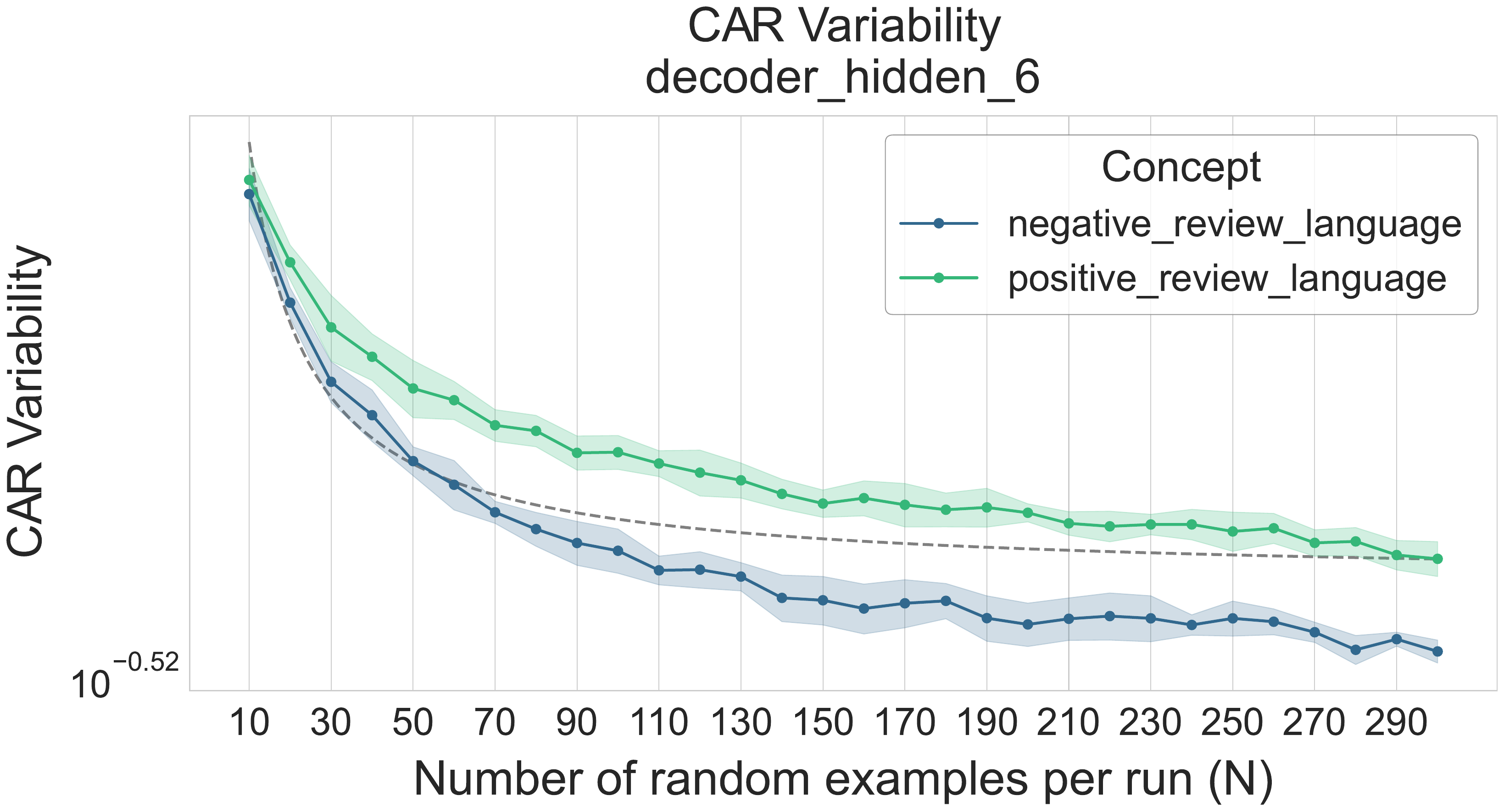}
    \end{minipage}
    \hfill
    \begin{minipage}{0.48\textwidth}
        \centering
        \includegraphics[width=\linewidth]{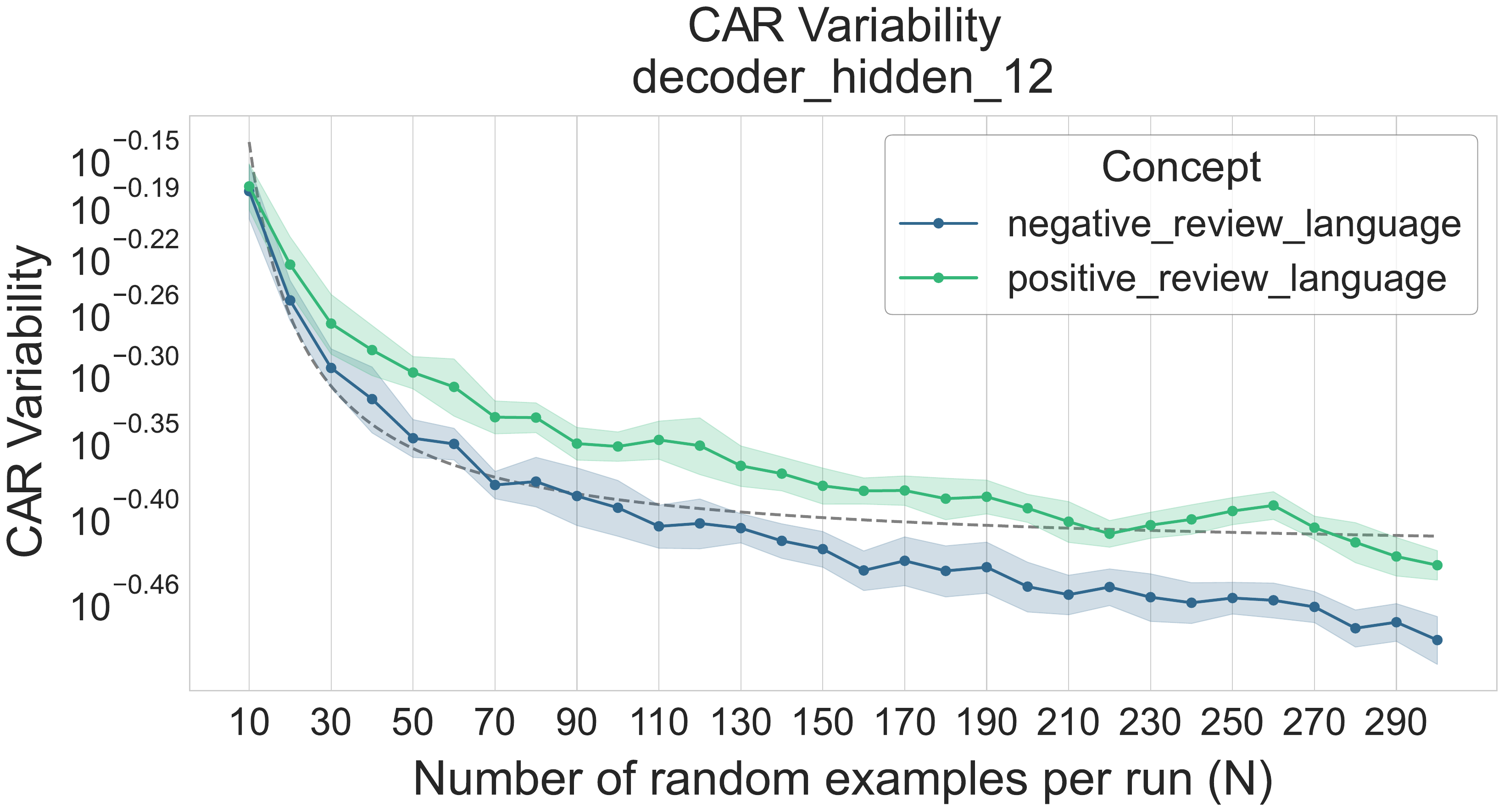}
    \end{minipage}
    \caption[Variability of CARs (Concept Activation Regions) for SmolLM2]{\label{fig:variance_of_cars_smollm2_text}
    Mean variability of \text{\textsc{Car}\xspace}s for \texttt{SmolLM2-1.7B-Instruct} as a function of the number of random examples per run ($N$) on the IMDB setting. Results are shown for decoder layers \texttt{decoder\_hidden\_6} and \texttt{decoder\_hidden\_12}. Error bars indicate $\pm 1$ SD; the $y$-axis is log-scaled.}
\end{figure}

\begin{figure}[H]
    \centering
    \begin{minipage}{0.48\textwidth}
        \centering
        \includegraphics[width=\linewidth]{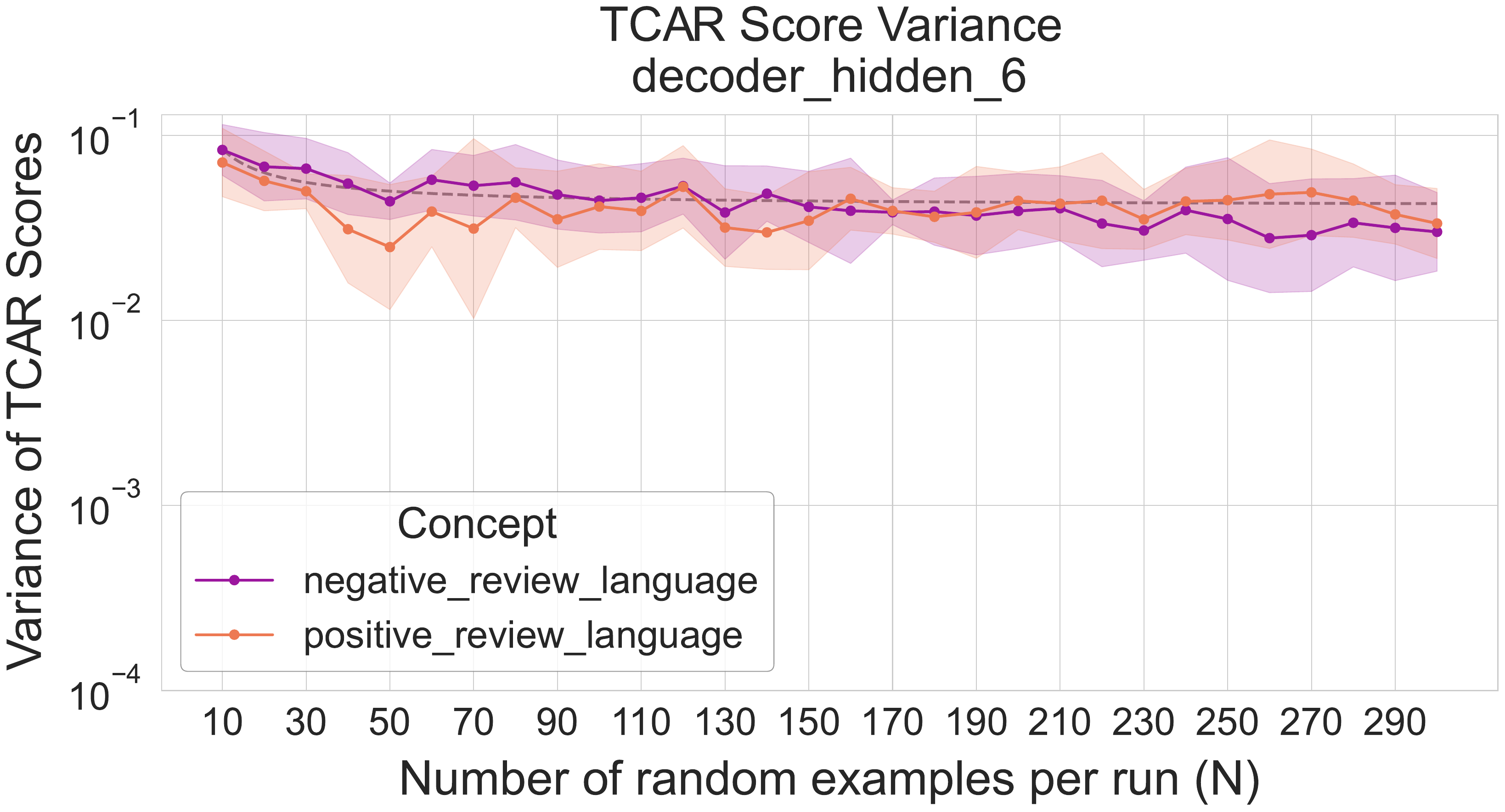}
    \end{minipage}
    \hfill
    \begin{minipage}{0.48\textwidth}
        \centering
        \includegraphics[width=\linewidth]{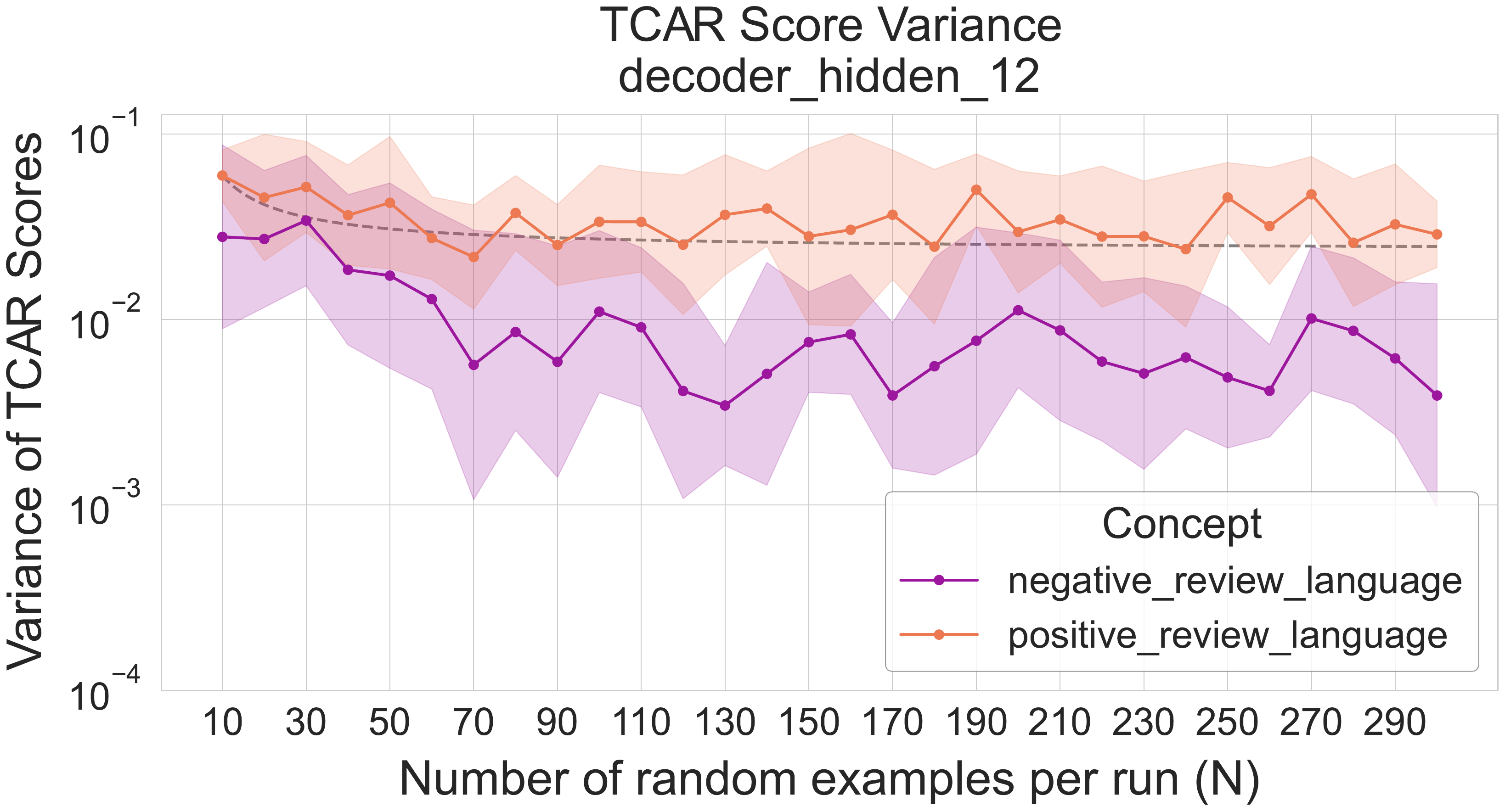}
    \end{minipage}
    \caption[Variance of TCAR scores (Concept Activation Regions) for SmolLM2]{\label{fig:variance_of_tcars_smollm2_text}
    Variance of \text{\textsc{Tcar}\xspace} scores for \texttt{SmolLM2-1.7B-Instruct} at decoder layers \texttt{decoder\_hidden\_6} and \texttt{decoder\_hidden\_12} vs. the number of examples per concept set ($N$) on the IMDB setting. Error bars denote $\pm 1$ standard deviation.}
\end{figure}

\subsubsection{Empirical Findings with Difference of Means}
Again, we found that  \text{\textsc{Cav}\xspace}s computed via the \emph{Difference of Means}-method most closely follow a variance decline of $\mathcal{O}(1/N)$.

\begin{figure}[H]
    \centering
    \begin{minipage}{0.48\textwidth}
        \centering
        \includegraphics[width=\linewidth]{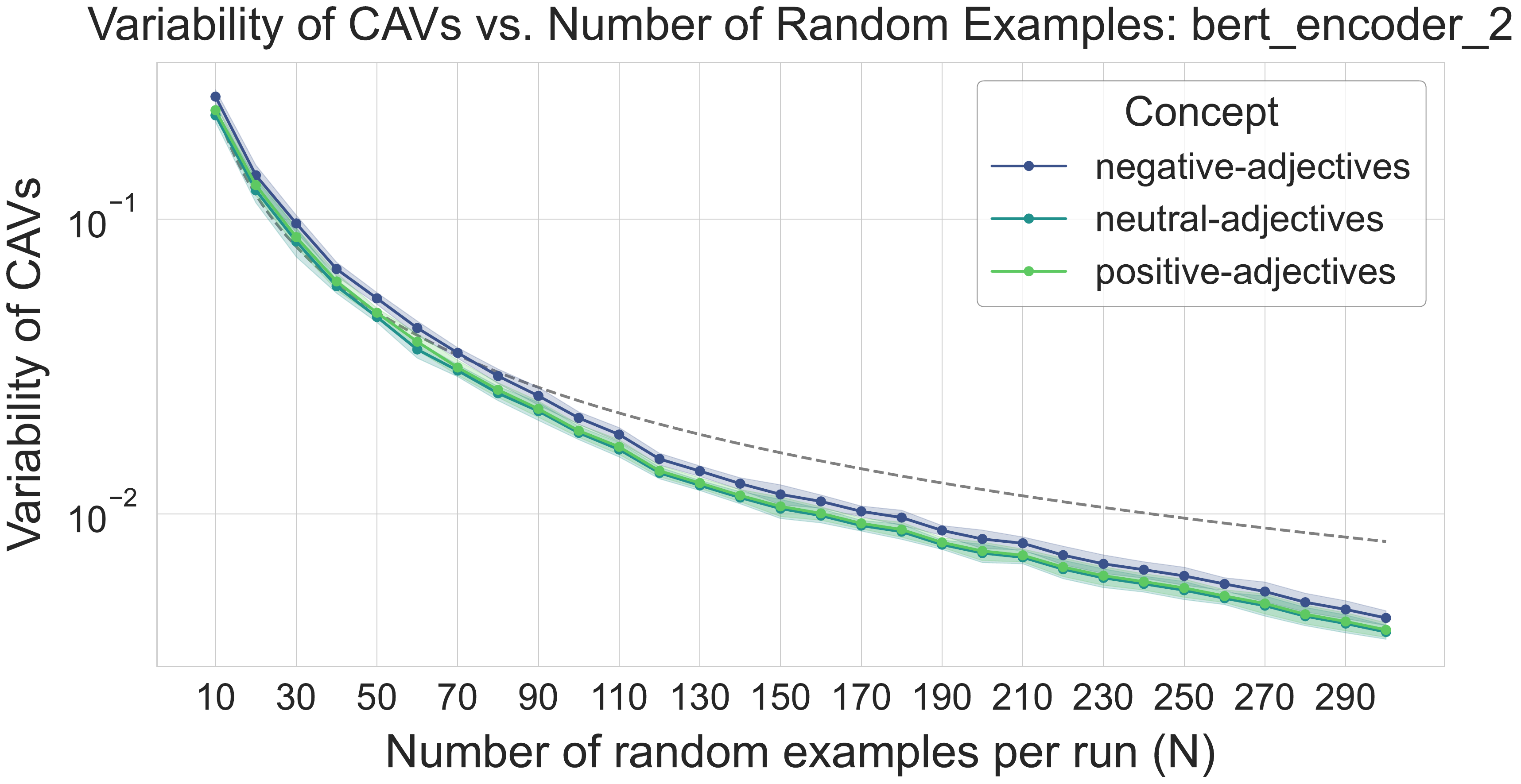}
    \end{minipage}
    \hfill
    \begin{minipage}{0.48\textwidth}
        \centering
        \includegraphics[width=\linewidth]{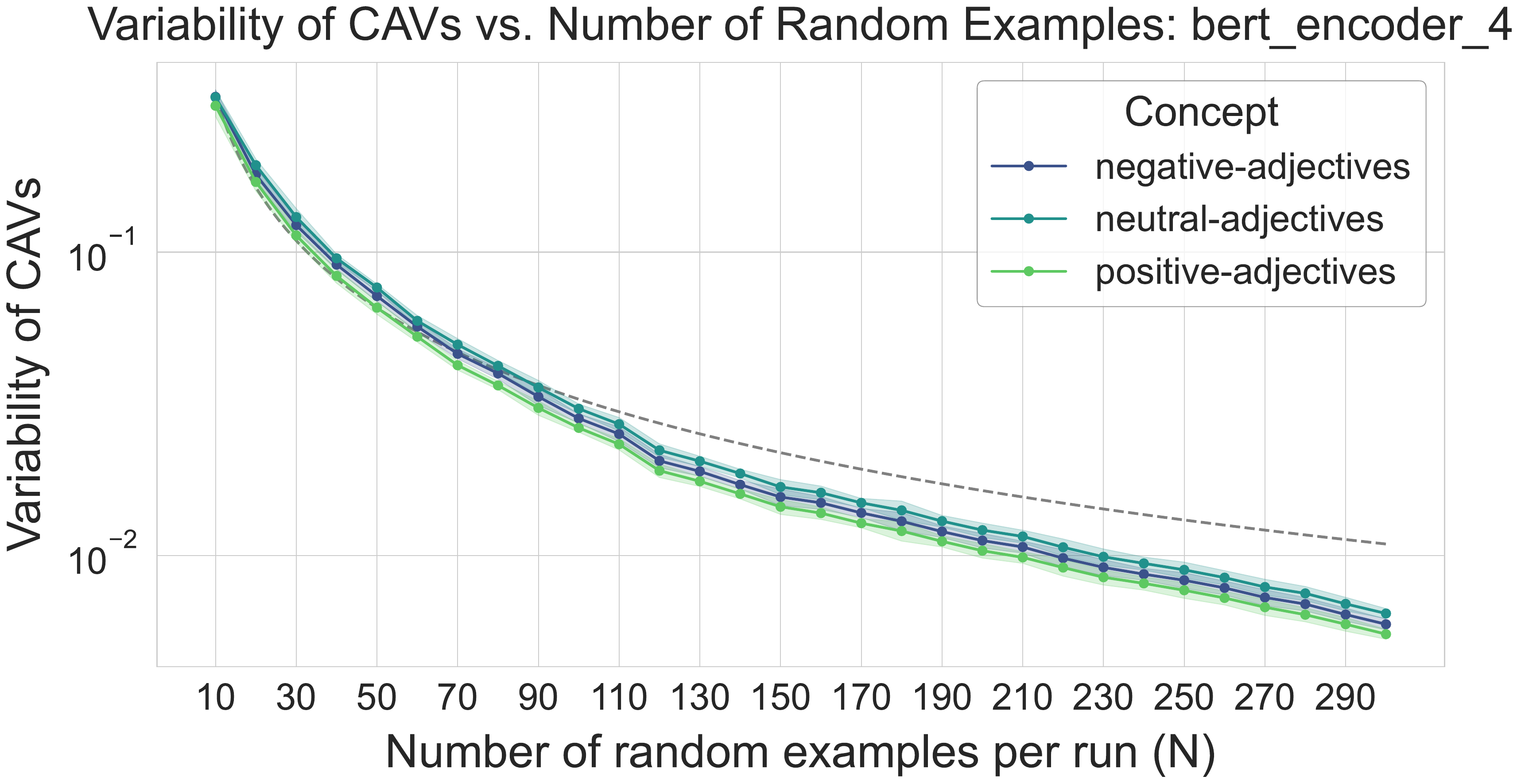}
    \end{minipage}

    \caption[Variability of CAVs (Difference of Means) on UCI Adult]{\label{fig:variance_of_cavs_dom_text}
        Mean variability of \text{\textsc{Cav}\xspace}s for the concepts ``positive'', ``negative'' and ``neutral'' adjectives as a function of the number of random examples per run ($N$). Results are shown for two different hidden layers. The \text{\textsc{Cav}\xspace} are generated using \textbf{difference-of-means}. Error bars indicate $\pm 1$ SD; the $y$-axis is log-scaled. Variance is estimated by the sum of per-feature variances across ten independent runs. We fitted a curve of the form $f(N) = a/N + b$ to it. For \texttt{bert\_encoder\_2} the parameters were $a=2.42, b=1.19\times 10^{-7}$, for \texttt{bert\_encoder\_4} they were $a=3.27, b=3.24\times 10^{-10}$.}
\end{figure}

\begin{figure}[H]
    \centering
    
    \begin{minipage}{0.48\textwidth}
        \centering
        \includegraphics[width=\linewidth]{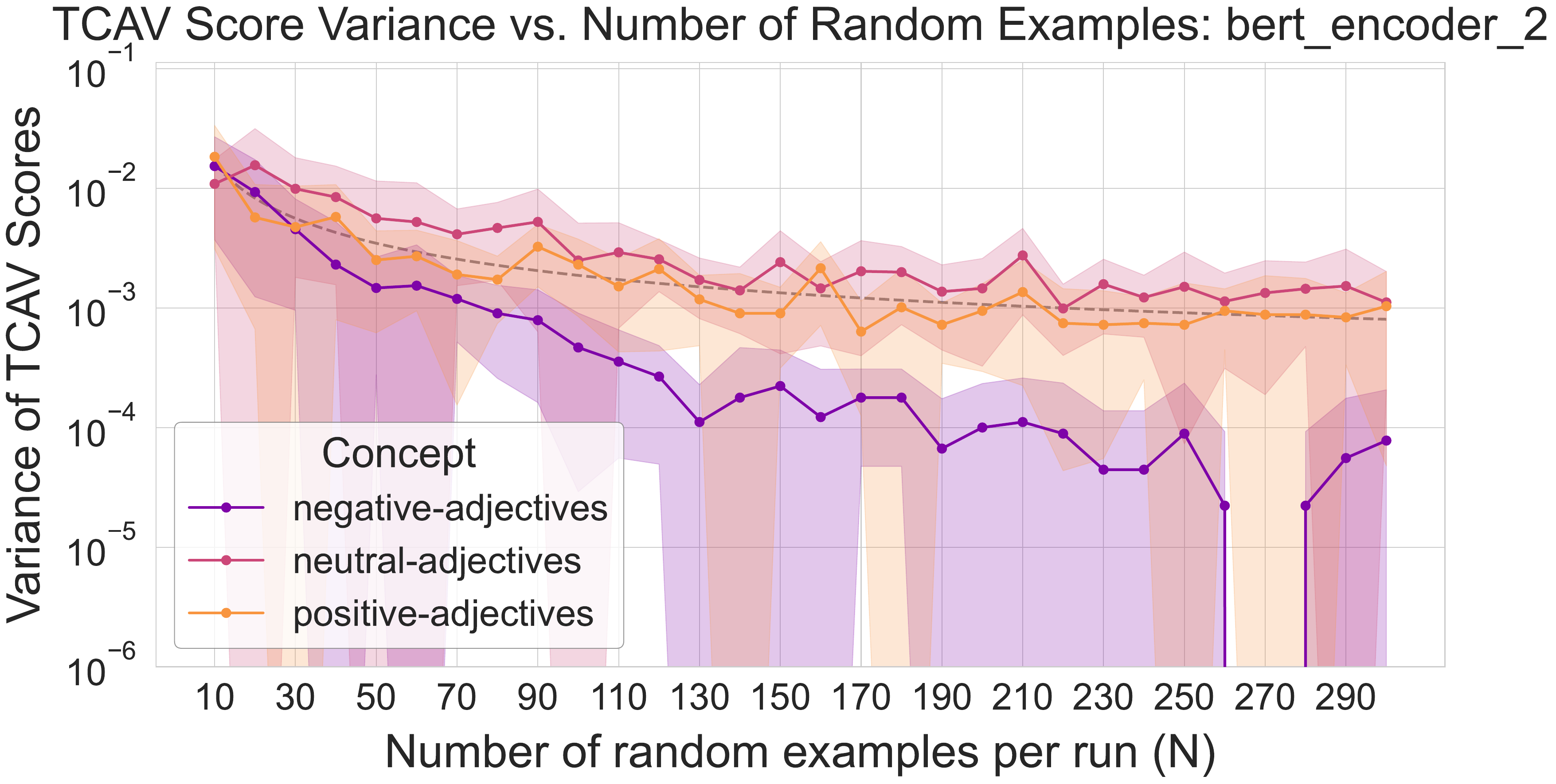}
    \end{minipage}
    \hfill
    \begin{minipage}{0.48\textwidth}
        \centering
        \includegraphics[width=\linewidth]{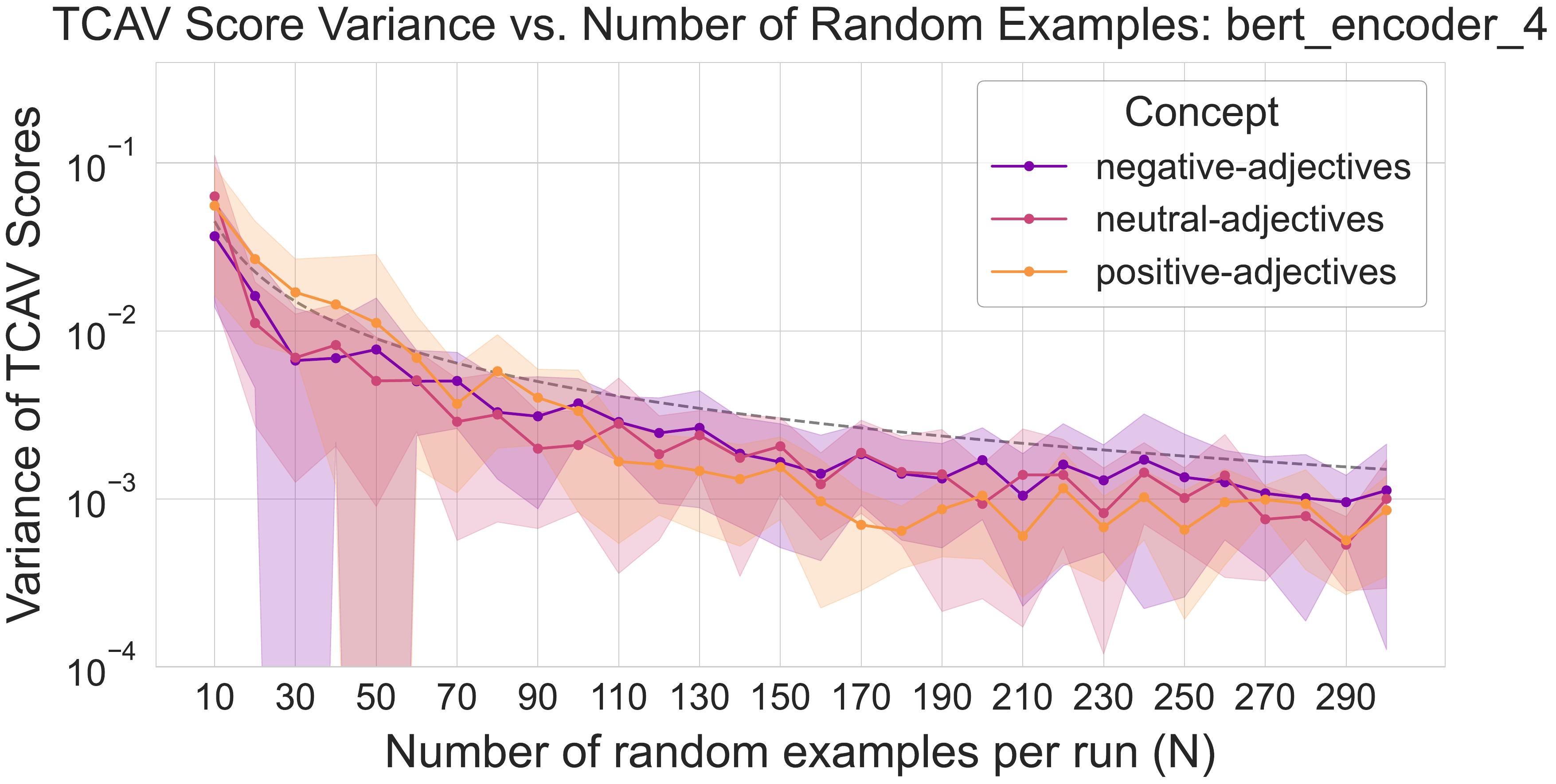}
    \end{minipage}
    \caption[Variance of TCAV scores (Hinge Loss) on UCI Adult]{\label{fig:variance_of_tcavs_dom_text}
    Variance of \text{\textsc{Tcav}\xspace} scores  vs. the number of examples per concept set ($N$) for the concepts ``positive'', ``negative'' and ``neutral'' adjectives. The underlying \text{\textsc{Cav}\xspace}s were generated using \textbf{difference-of-means}. Error bars denote $\pm 1$ standard deviation. We fitted a curve of the form $f(N) = a/N + b$ to it. For \texttt{bert\_encoder\_2} the parameters were $a=0.16, b=2.7\times 10^{-4}$, for \texttt{bert\_encoder\_4} they were $a=0.449, b=1.59\times 10^{-8}$. We clipped below $10^{-6}$ respectively $10^{-4}$ to avoid visual distortion from extreme lows.}
\end{figure}

\begin{figure}[H]
    \centering
    \begin{minipage}{0.48\textwidth}
        \centering
        \includegraphics[width=\linewidth]{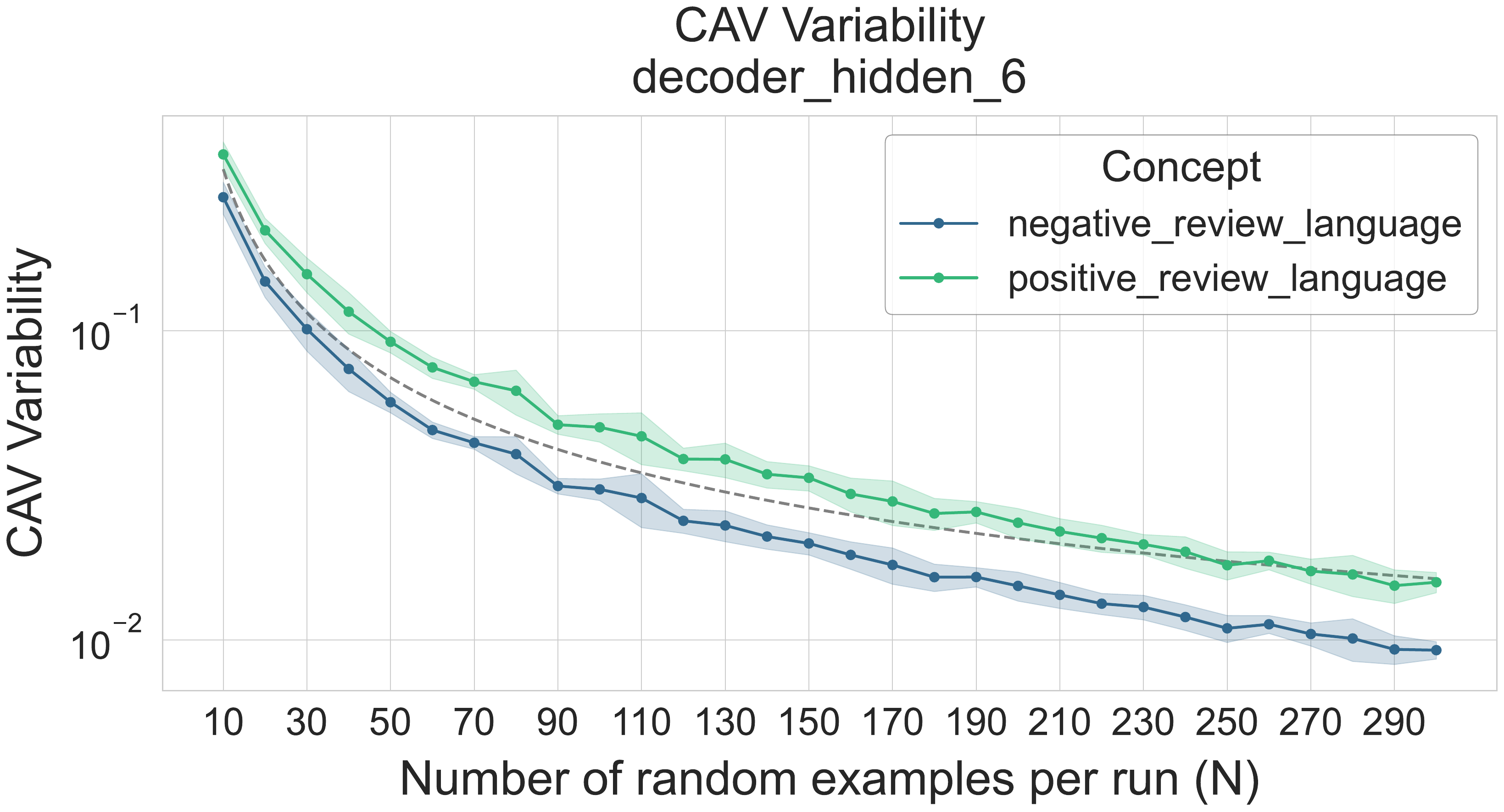}
    \end{minipage}
    \hfill
    \begin{minipage}{0.48\textwidth}
        \centering
        \includegraphics[width=\linewidth]{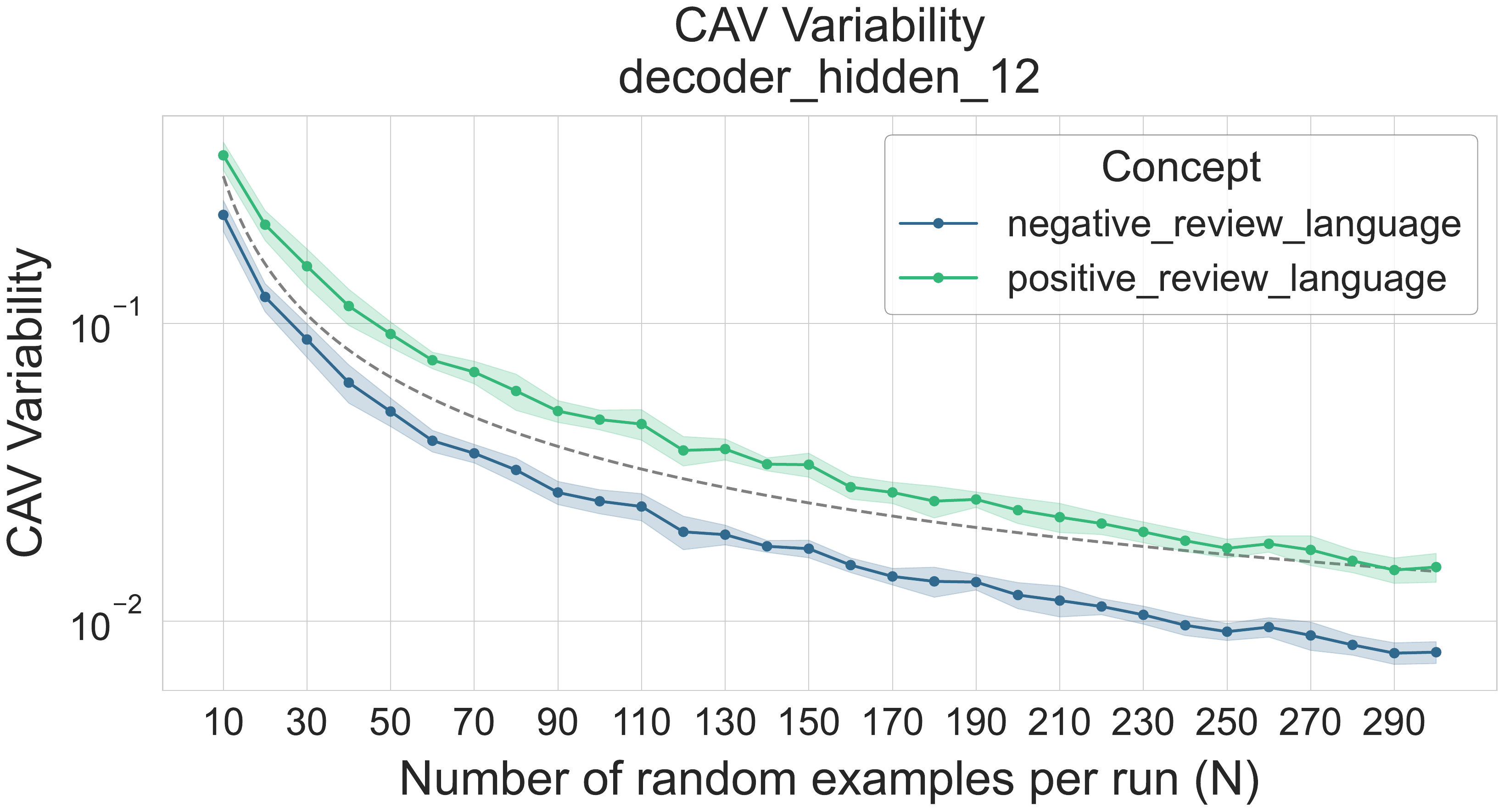}
    \end{minipage}
    \caption[Variability of CAVs (Difference of Means) for SmolLM2]{\label{fig:variance_of_cavs_dom_smollm2_text}
    Mean variability of \text{\textsc{Cav}\xspace}s for \texttt{SmolLM2-1.7B-Instruct} as a function of the number of random examples per run ($N$) on the IMDB setting. Results are shown for decoder layers \texttt{decoder\_hidden\_6} and \texttt{decoder\_hidden\_12}. Error bars indicate $\pm 1$ SD; the $y$-axis is log-scaled.}
\end{figure}

\begin{figure}[H]
    \centering
    \begin{minipage}{0.48\textwidth}
        \centering
        \includegraphics[width=\linewidth]{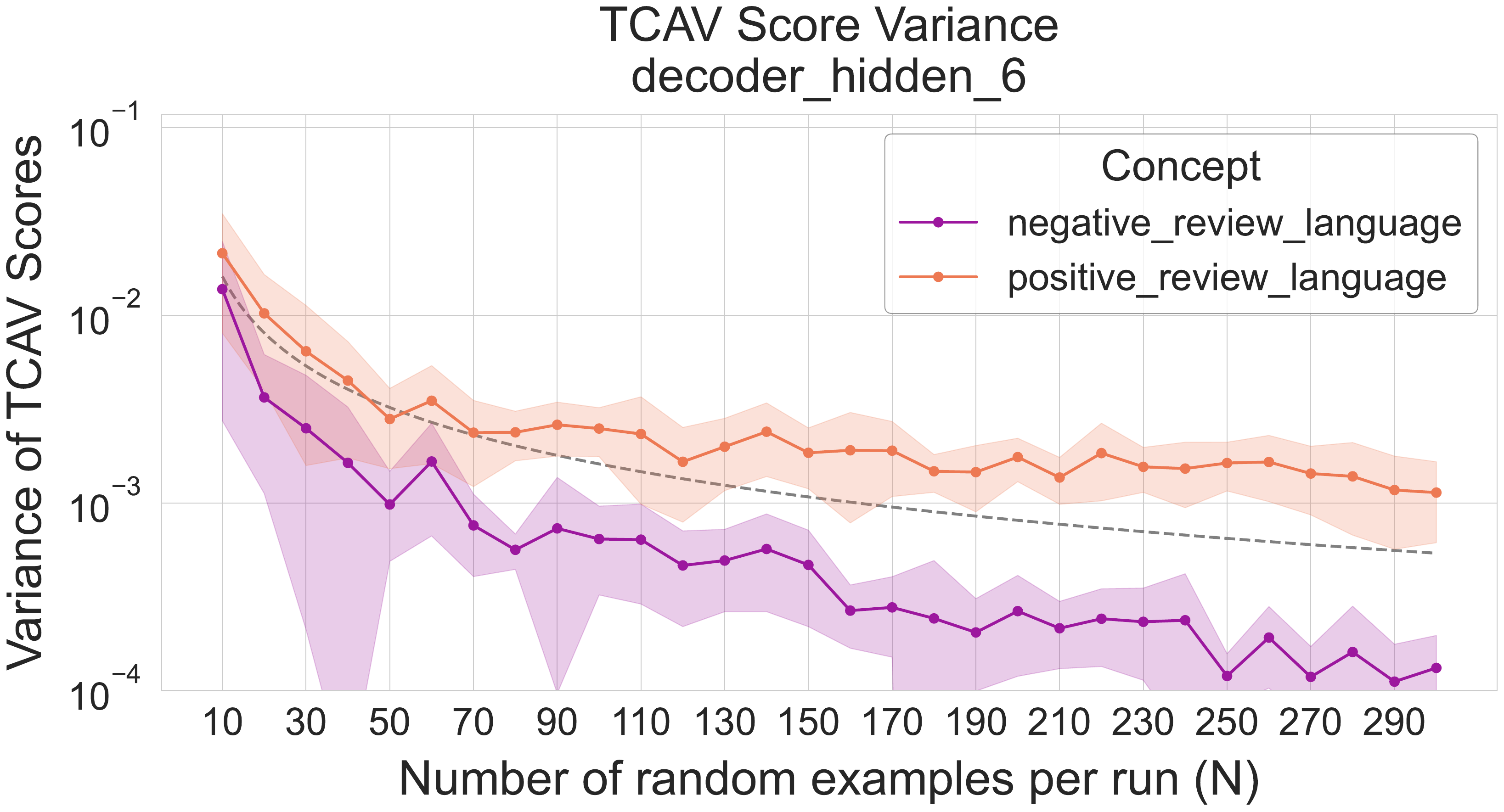}
    \end{minipage}
    \hfill
    \begin{minipage}{0.48\textwidth}
        \centering
        \includegraphics[width=\linewidth]{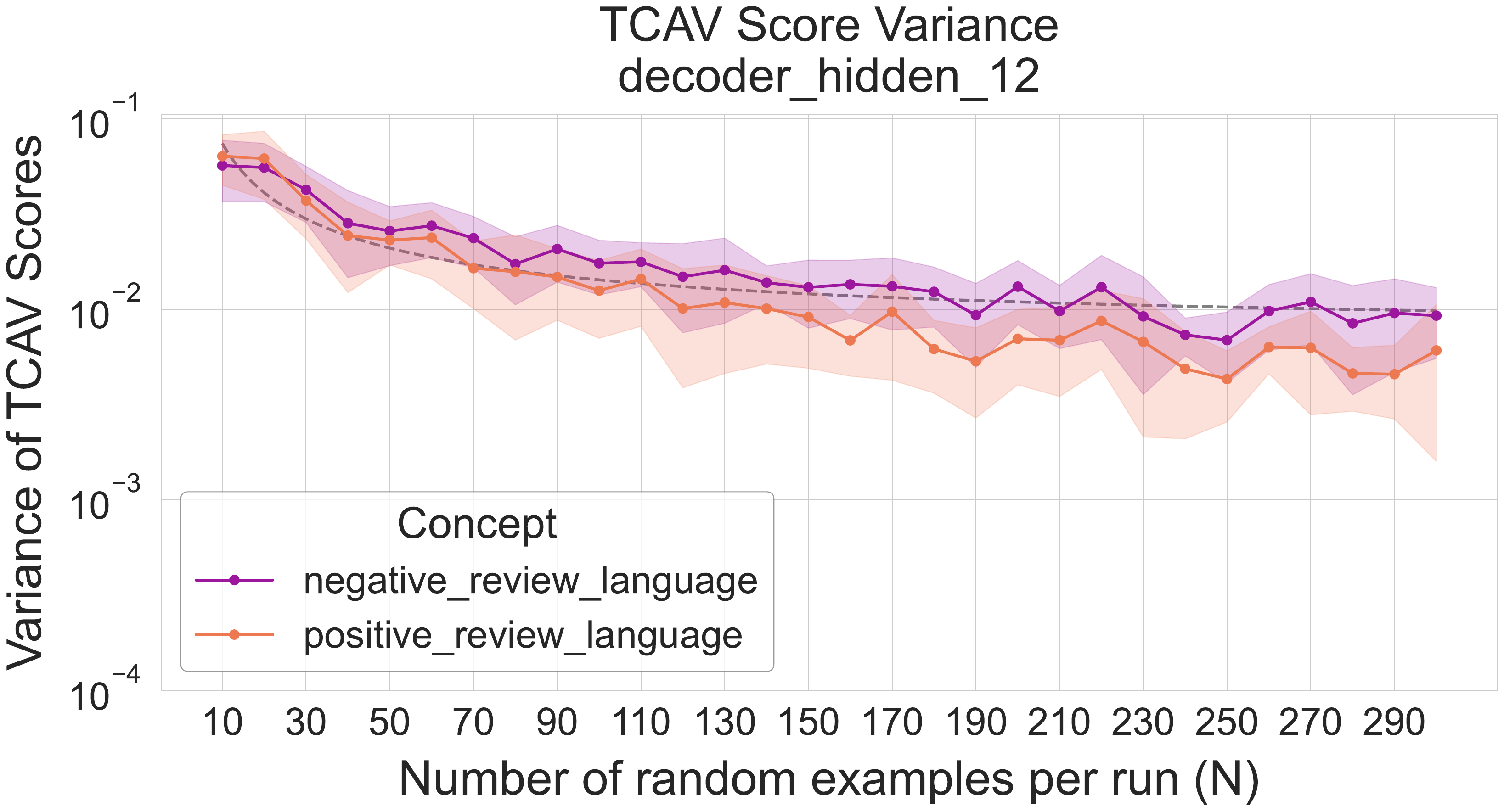}
    \end{minipage}
    \caption[Variance of TCAV scores (Difference of Means) for SmolLM2]{\label{fig:variance_of_tcavs_dom_smollm2_text}
    Variance of \text{\textsc{Tcav}\xspace} scores for \texttt{SmolLM2-1.7B-Instruct} at decoder layers \texttt{decoder\_hidden\_6} and \texttt{decoder\_hidden\_12} vs. the number of examples per concept set ($N$) on the IMDB setting. Error bars denote $\pm 1$ standard deviation.}
\end{figure}

\subsection{TCAV for Multimodal Data}
\label{appendix:multimodal}

We evaluate \TCAV on \textbf{multimodal} (vision-language) data using \textbf{CLIP}~\citep{radford_learning_2021} and LLaVA-1.5 \citep{liu_improved_2024}.
We pair zebra images with the prompt ``a photo of a zebra'' and use image-text similarity as the scalar output.
Concepts are defined by \textbf{image sets}, as in the vision experiments (e.g., ``striped'', ``zigzagged'', ``dotted'').
We extract \text{\textsc{Cav}\xspace}s from CLIP at hidden layers $4$ and $8$ and from LLaVA-1.5 at vision layers $7$ and $15$.
We present the results for this final dataset below.

\subsubsection{Empirical Findings with Binary Cross-Entropy Loss}

We first report results for multimodal data using a logistic regression classifier trained with \textbf{binary cross-entropy loss}.

\begin{figure}[!ht]
    \centering
    \begin{minipage}{0.48\textwidth}
        \centering
        \includegraphics[width=\linewidth]{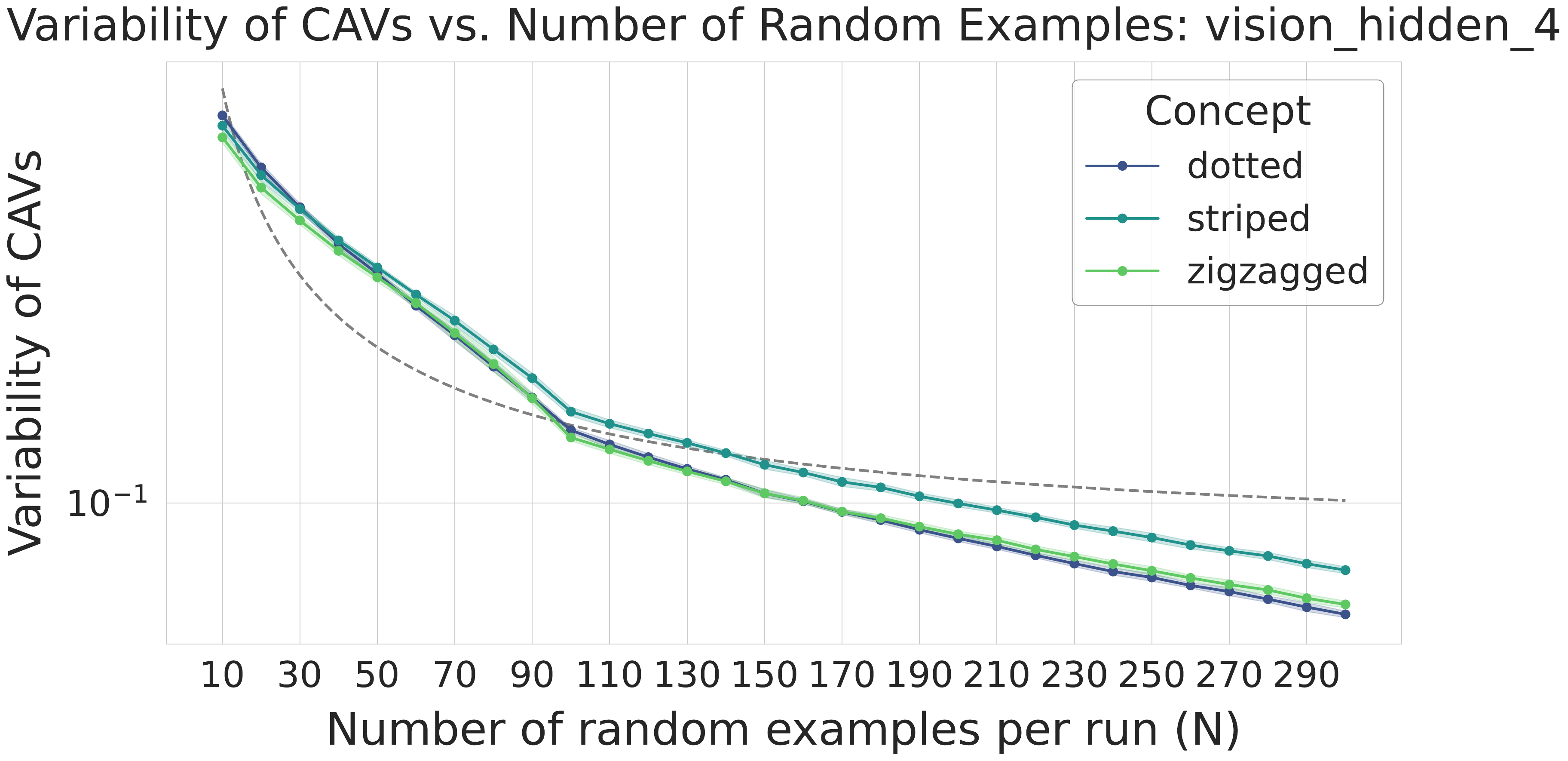}
    \end{minipage}
    \hfill 
    \begin{minipage}{0.48\textwidth}
        \centering
        \includegraphics[width=\linewidth]{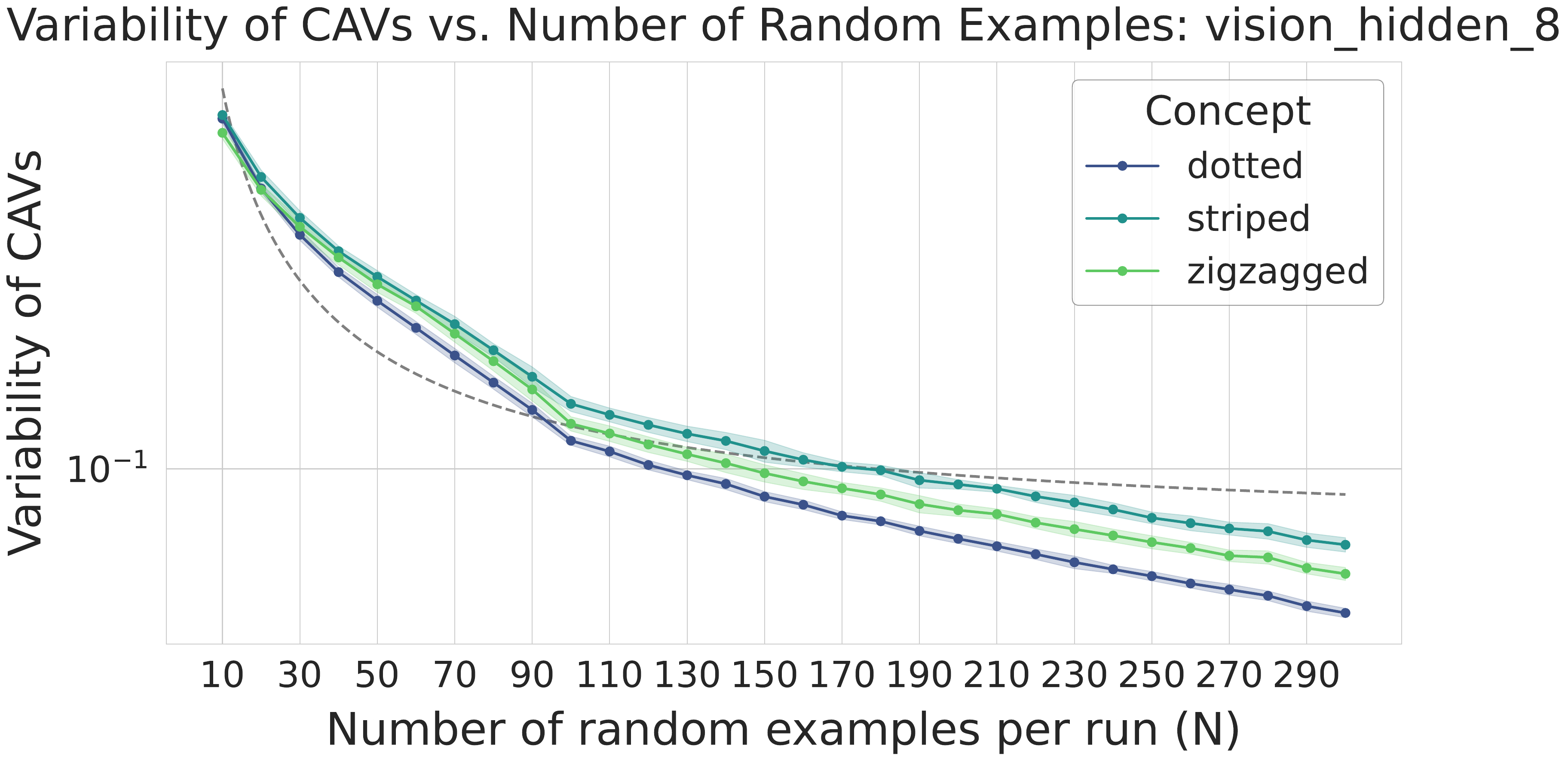}
    \end{minipage}
    \caption[Variance of CAVs (Log Loss) for CLIP]{\label{fig:variance_of_cavs_log_multimodal}
    Mean variability of \text{\textsc{Cav}\xspace}s for the visual concepts ``striped'', ``zigzagged'', and ``dotted'' as a function of random examples ($N$) for CLIP hidden layers $4$ and $8$.  Error bars indicate $\pm 1$ SD; the $y$-axis is log-scaled.
    We fitted a curve of the form $f(N) = a/N + b$ to it. For the layer \texttt{vision\_hidden\_4} we found $6.72/N + 0.0788$, for \texttt{vision\_hidden\_8} $4.81/N + 0.0731$.}
\end{figure}

\begin{figure}[H]
    \centering
    \begin{minipage}{0.48\textwidth}
        \centering
        \includegraphics[width=\linewidth]{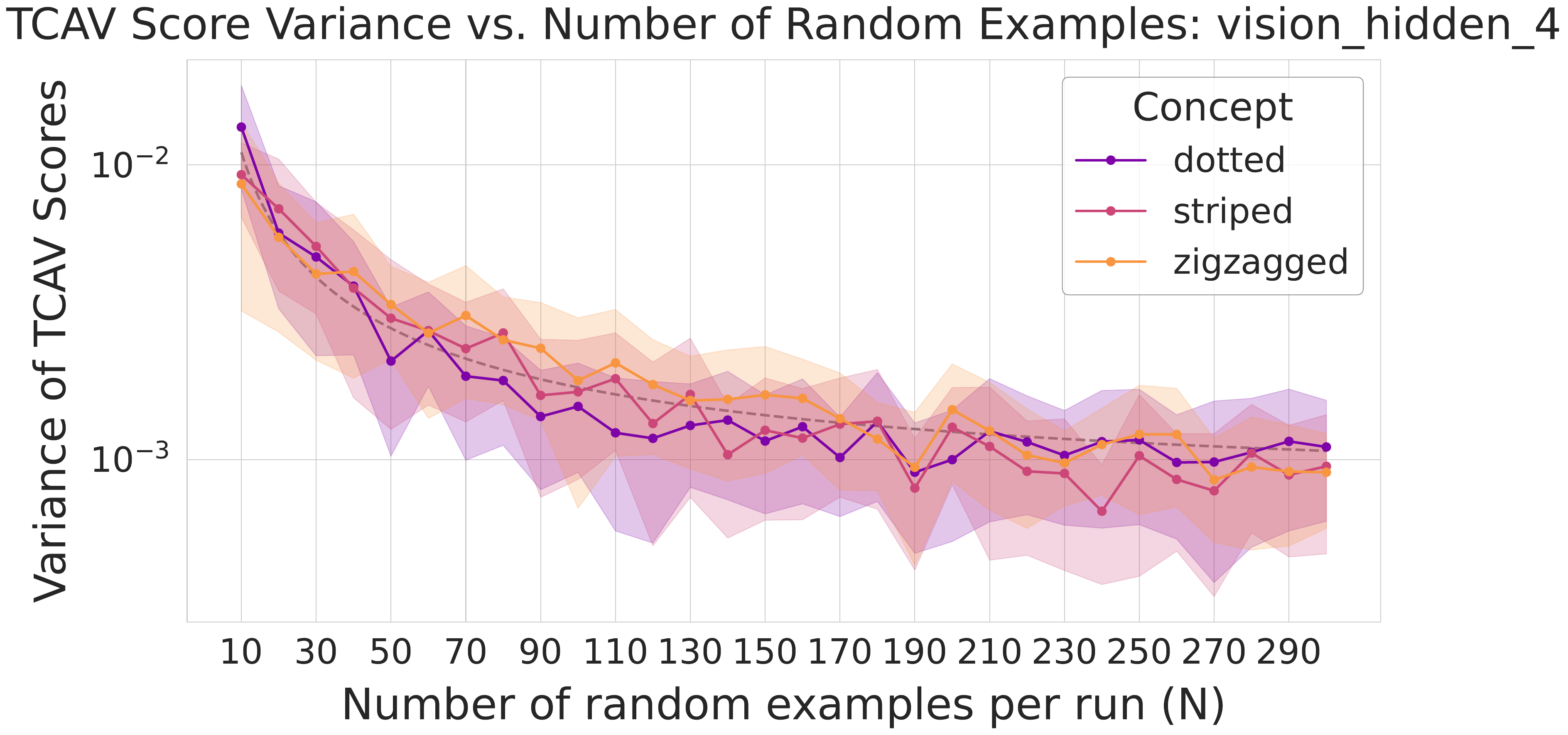}
    \end{minipage}
    \hfill 
    \begin{minipage}{0.48\textwidth}
        \centering
        \includegraphics[width=\linewidth]{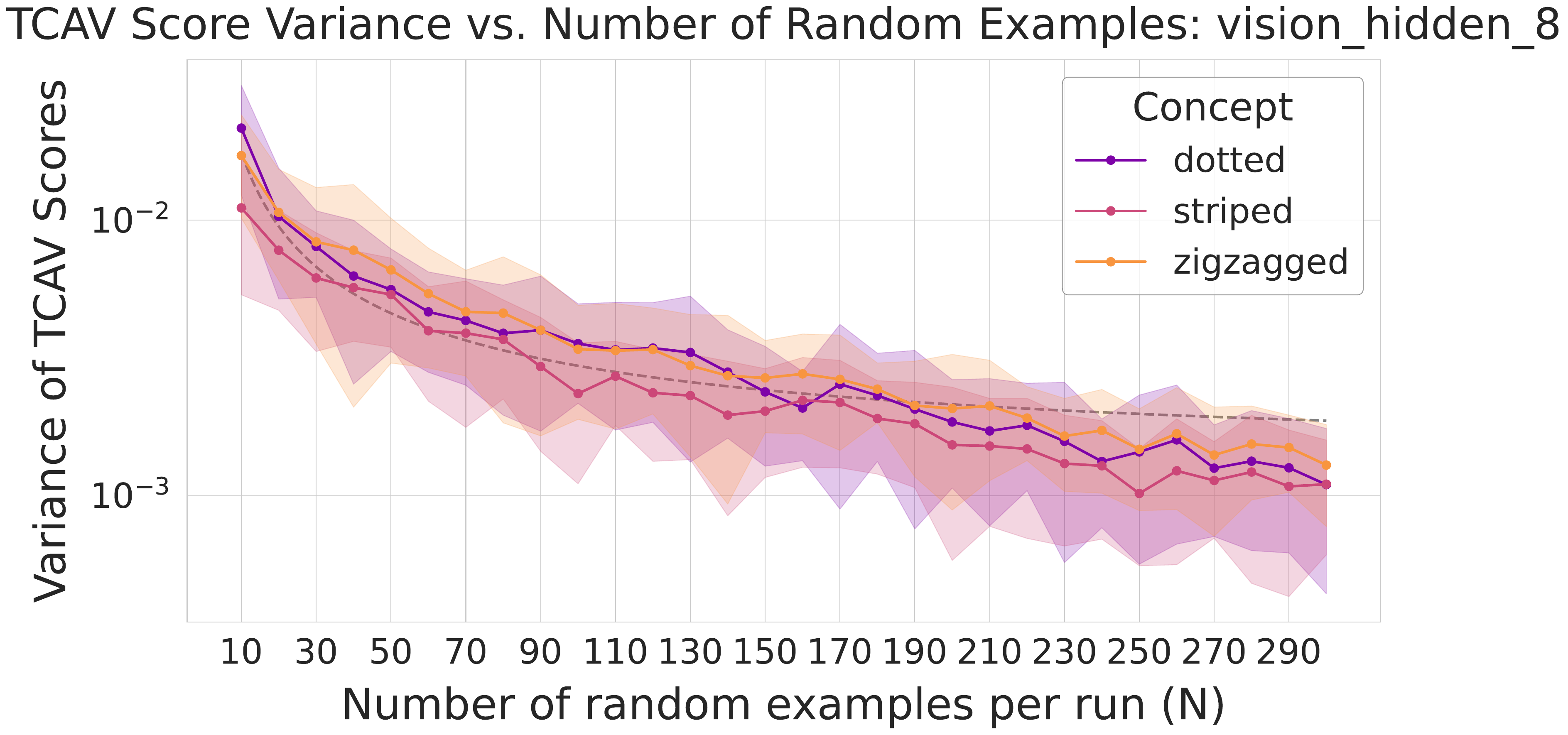}
    \end{minipage}
    \caption[Variance of TCAV scores (Log Loss) for CLIP]{\label{fig:variance_of_tcavs_log_multimodal}
    Variance of \text{\textsc{Tcav}\xspace} scores at CLIP hidden layers $4$ and $8$ vs.\ the number of examples per concept set ($N$) for the visual concepts ``striped'', ``zigzagged'', and ``dotted''. 
    Error bars denote $\pm 1$ standard deviation. We fitted a curve of the form $f(N) = a/N + b$ to it. For the layer \texttt{vision\_hidden\_4} we found $0.103/N + 7.27\times 10^{-4}$, for \texttt{vision\_hidden\_8} we found $0.163/N + 0.00133$.}
\end{figure}

\begin{figure}[H]
    \centering
    \begin{minipage}{0.48\textwidth}
        \centering
        \includegraphics[width=\linewidth]{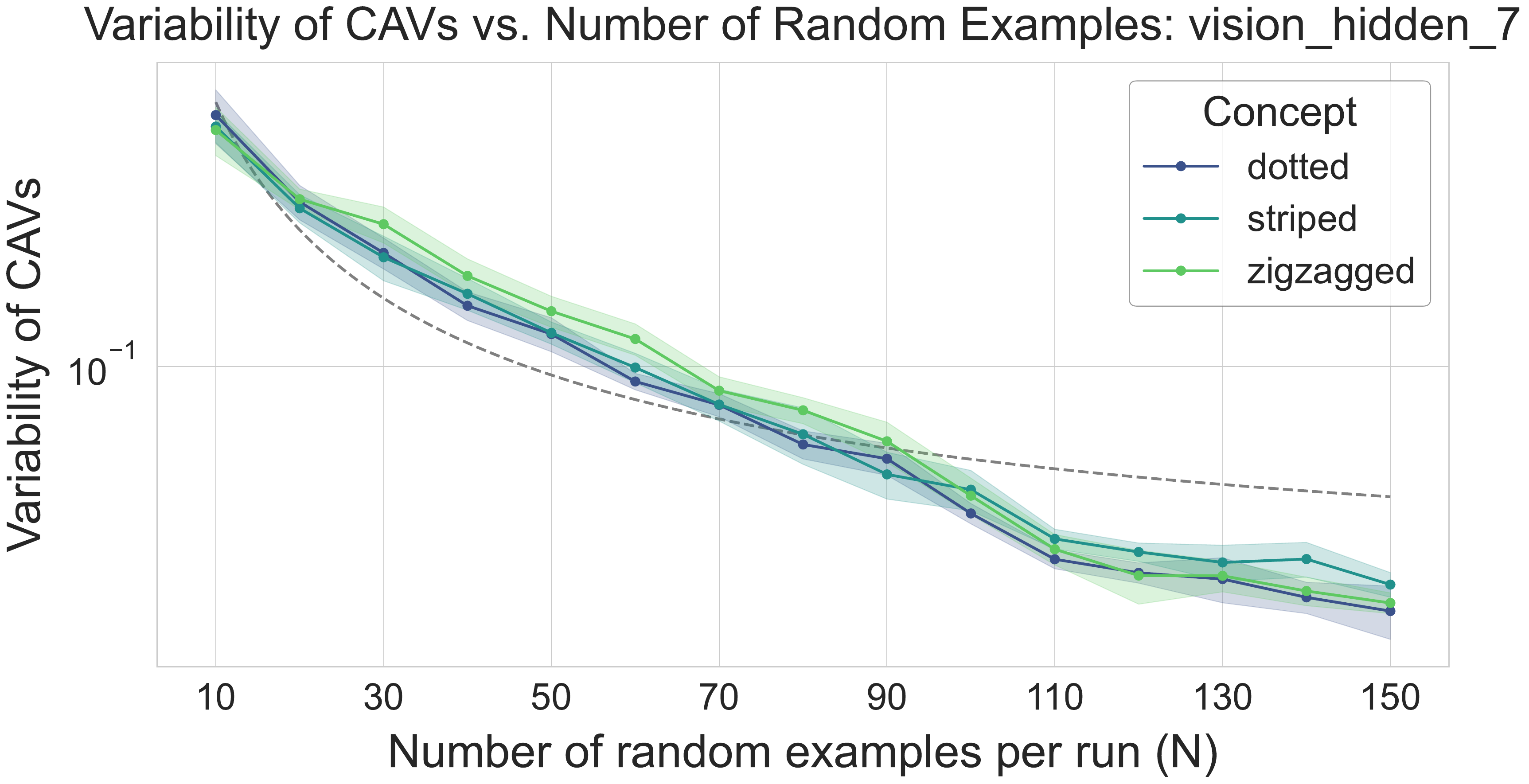}
    \end{minipage}
    \hfill
    \begin{minipage}{0.48\textwidth}
        \centering
        \includegraphics[width=\linewidth]{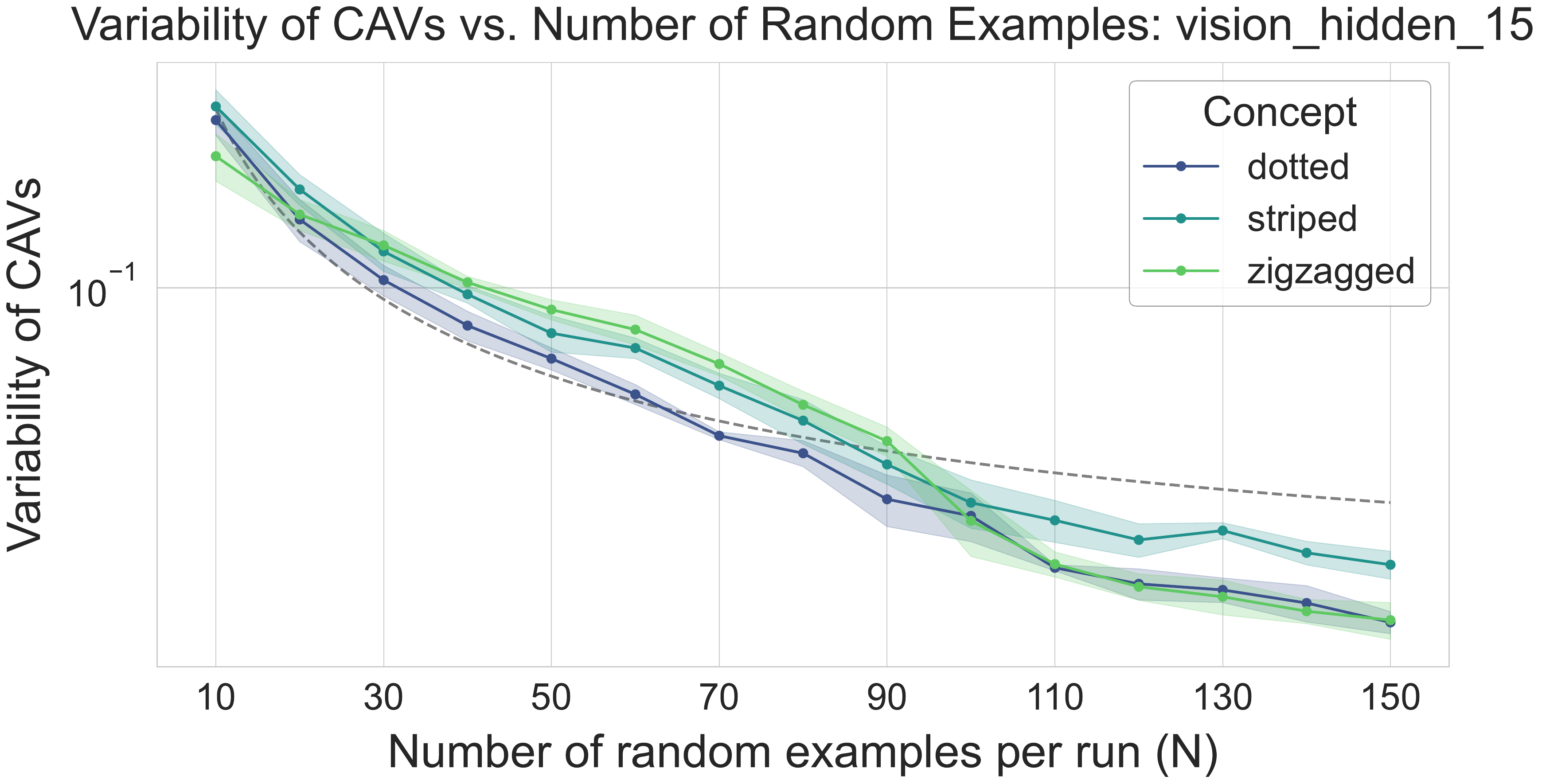}
    \end{minipage}
    \caption[Variability of CAVs (Log Loss) for LLaVA-1.5]{\label{fig:variance_of_cavs_logistic_llava_multimodal}
    Mean variability of \text{\textsc{Cav}\xspace}s for LLaVA-1.5 as a function of the number of random examples per run ($N$) for Broden-style visual concepts. Results are shown for vision layers \texttt{vision\_hidden\_7} and \texttt{vision\_hidden\_15}. Error bars indicate $\pm 1$ SD; the $y$-axis is log-scaled.}
\end{figure}

\begin{figure}[H]
    \centering
    \begin{minipage}{0.48\textwidth}
        \centering
        \includegraphics[width=\linewidth]{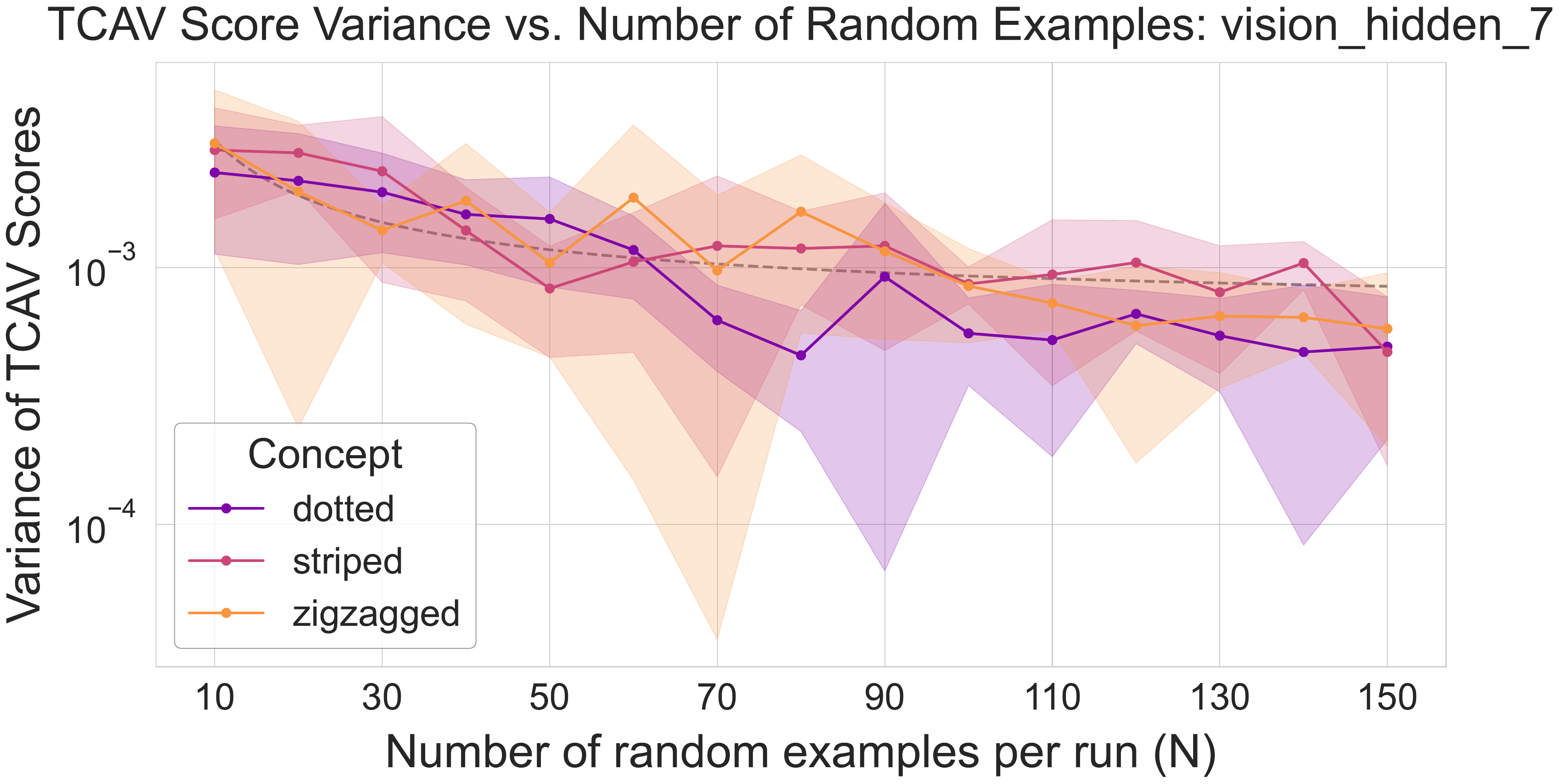}
    \end{minipage}
    \hfill
    \begin{minipage}{0.48\textwidth}
        \centering
        \includegraphics[width=\linewidth]{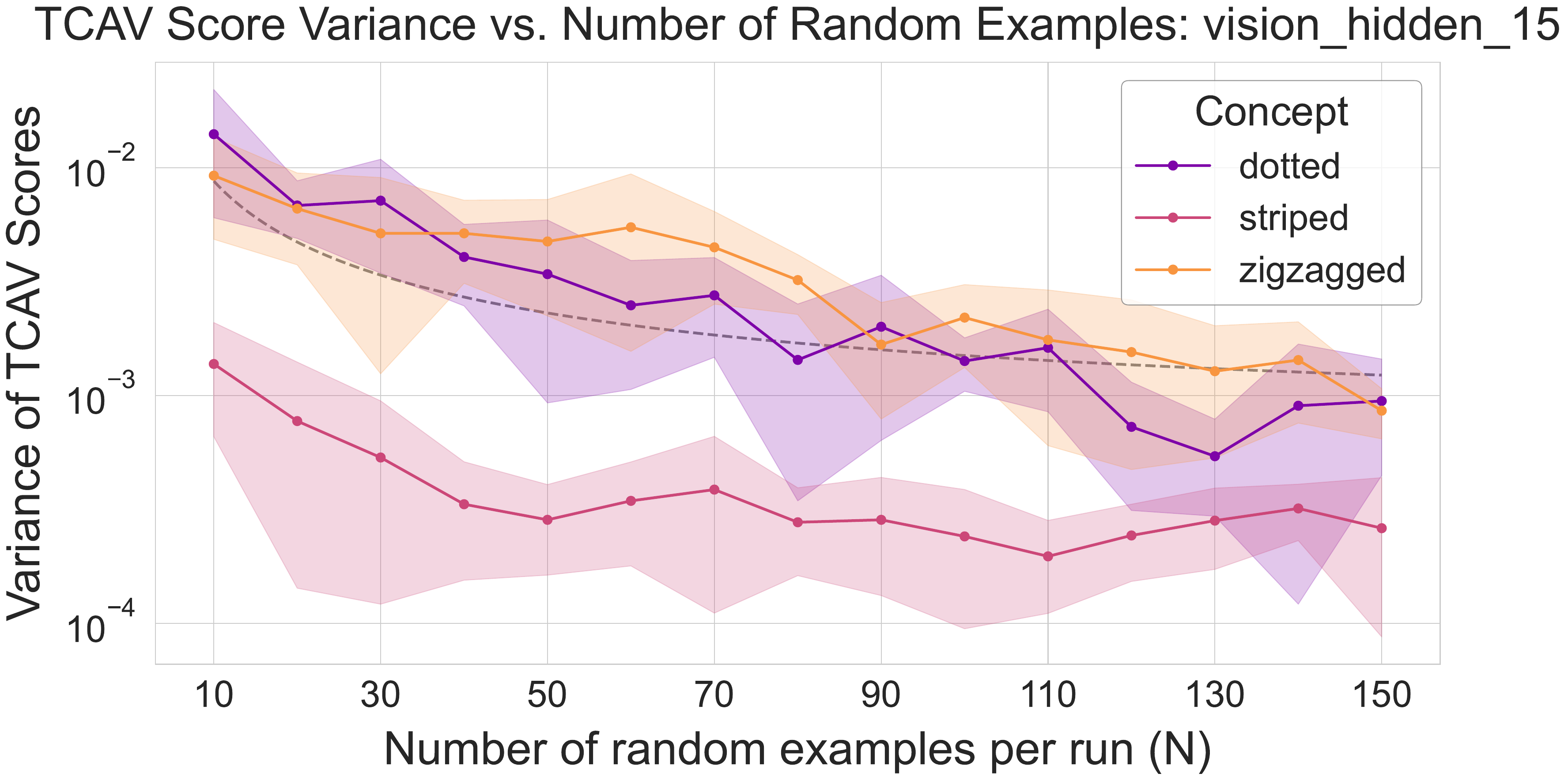}
    \end{minipage}
    \caption[Variance of TCAV scores (Log Loss) for LLaVA-1.5]{\label{fig:variance_of_tcavs_logistic_llava_multimodal}
    Variance of \text{\textsc{Tcav}\xspace} scores for LLaVA-1.5 at vision layers \texttt{vision\_hidden\_7} and \texttt{vision\_hidden\_15} vs. the number of examples per concept set ($N$) for Broden-style visual concepts. Error bars denote $\pm 1$ standard deviation.}
\end{figure}

\subsubsection{Empirical Findings with Hinge Loss}

Next, we show the corresponding results with \textbf{hinge loss}.

\begin{figure}[H]
    \centering
    \begin{minipage}{0.48\textwidth}
        \centering
        \includegraphics[width=\linewidth]{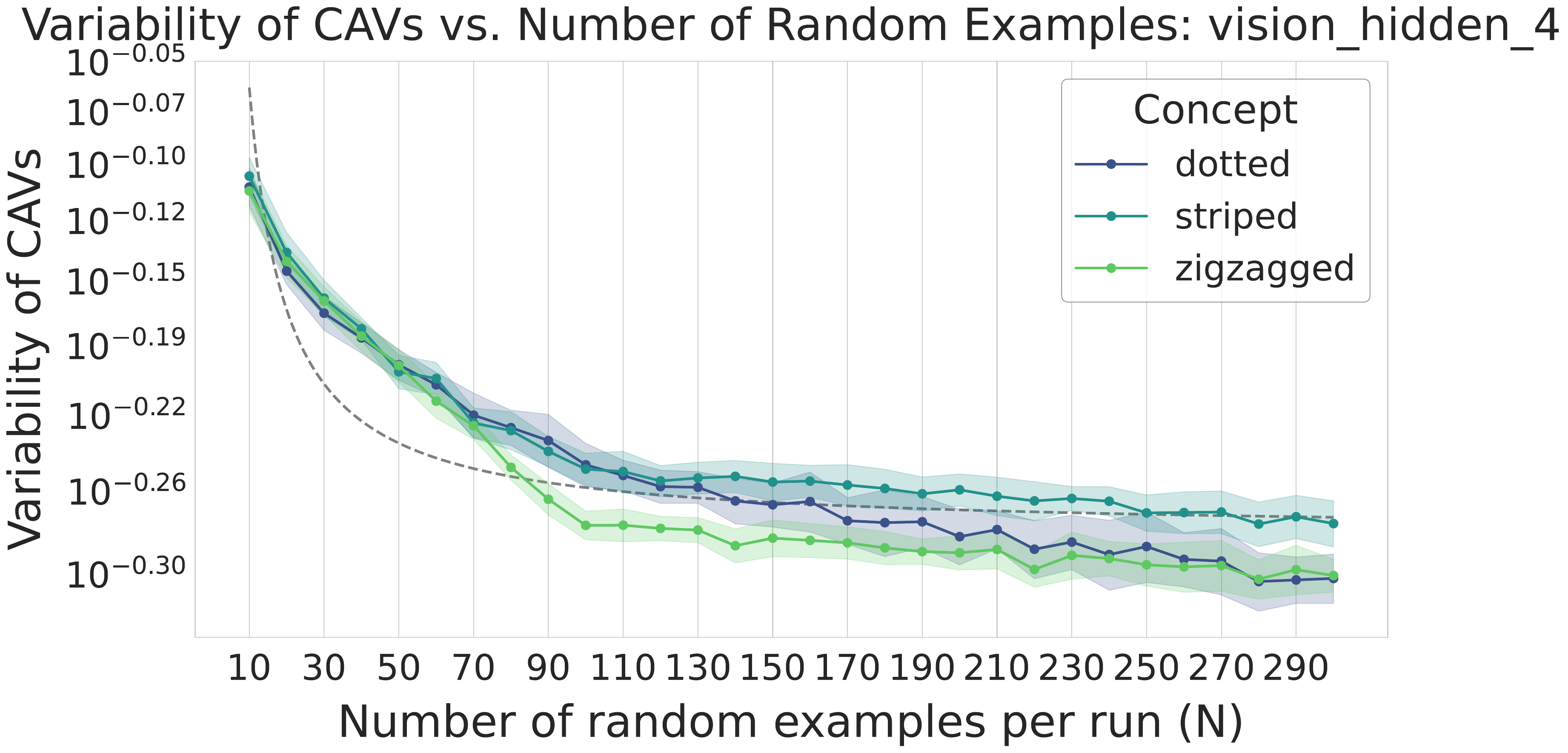}
    \end{minipage}
    \hfill
    \begin{minipage}{0.48\textwidth}
        \centering
        \includegraphics[width=\linewidth]{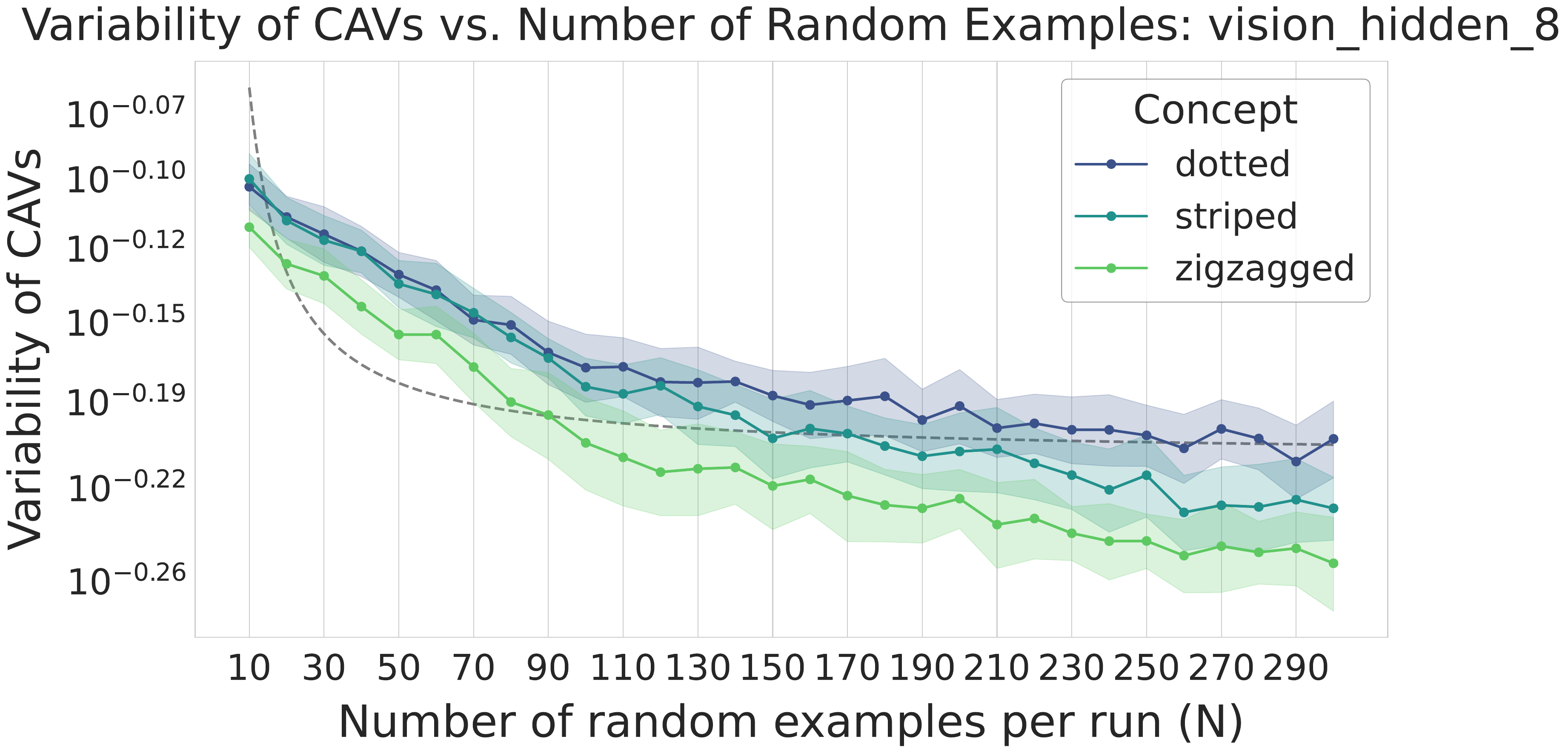}
    \end{minipage}
    \caption[Variance of CAVs (Hinge Loss) for CLIP]{\label{fig:variance_of_cavs_hinge_multimodal}
    Mean variability of \text{\textsc{Cav}\xspace}s for the visual concepts ``striped'', ``zigzagged'', and ``dotted'' as a function of the number of random examples per run ($N$). 
    Results are shown for CLIP hidden layers $4$ and $8$. Error bars indicate $\pm 1$ SD; the $y$-axis is log-scaled.
    Variance is estimated by the sum of per-feature variances across five independent runs. We fitted a curve of the form $f(N) = a/N + b$ to it.}
\end{figure}

\begin{figure}[H]
    \centering  
    \begin{minipage}{0.48\textwidth}
        \centering
        \includegraphics[width=\linewidth]{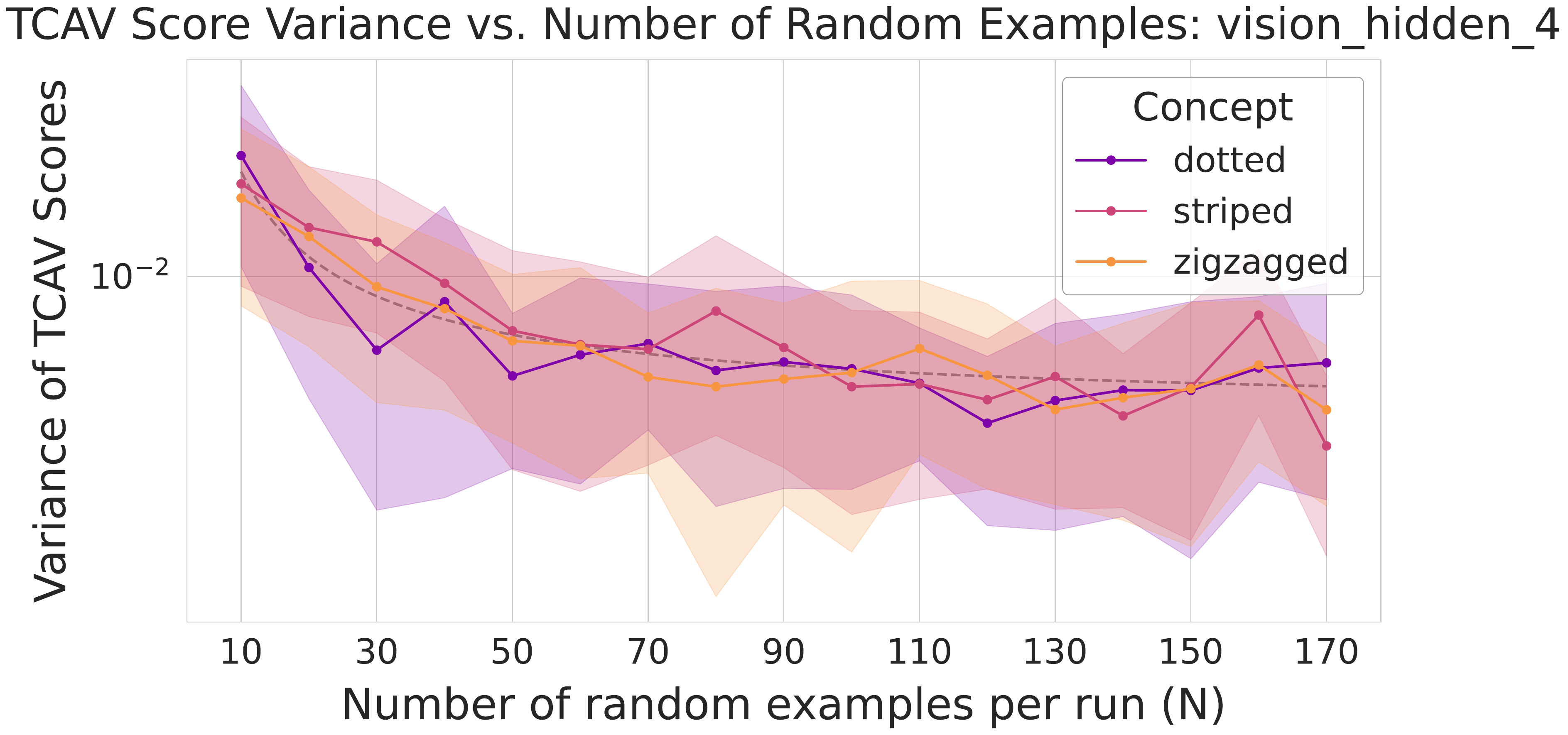}
    \end{minipage}
    \hfill
    \begin{minipage}{0.48\textwidth}
        \centering
        \includegraphics[width=\linewidth]{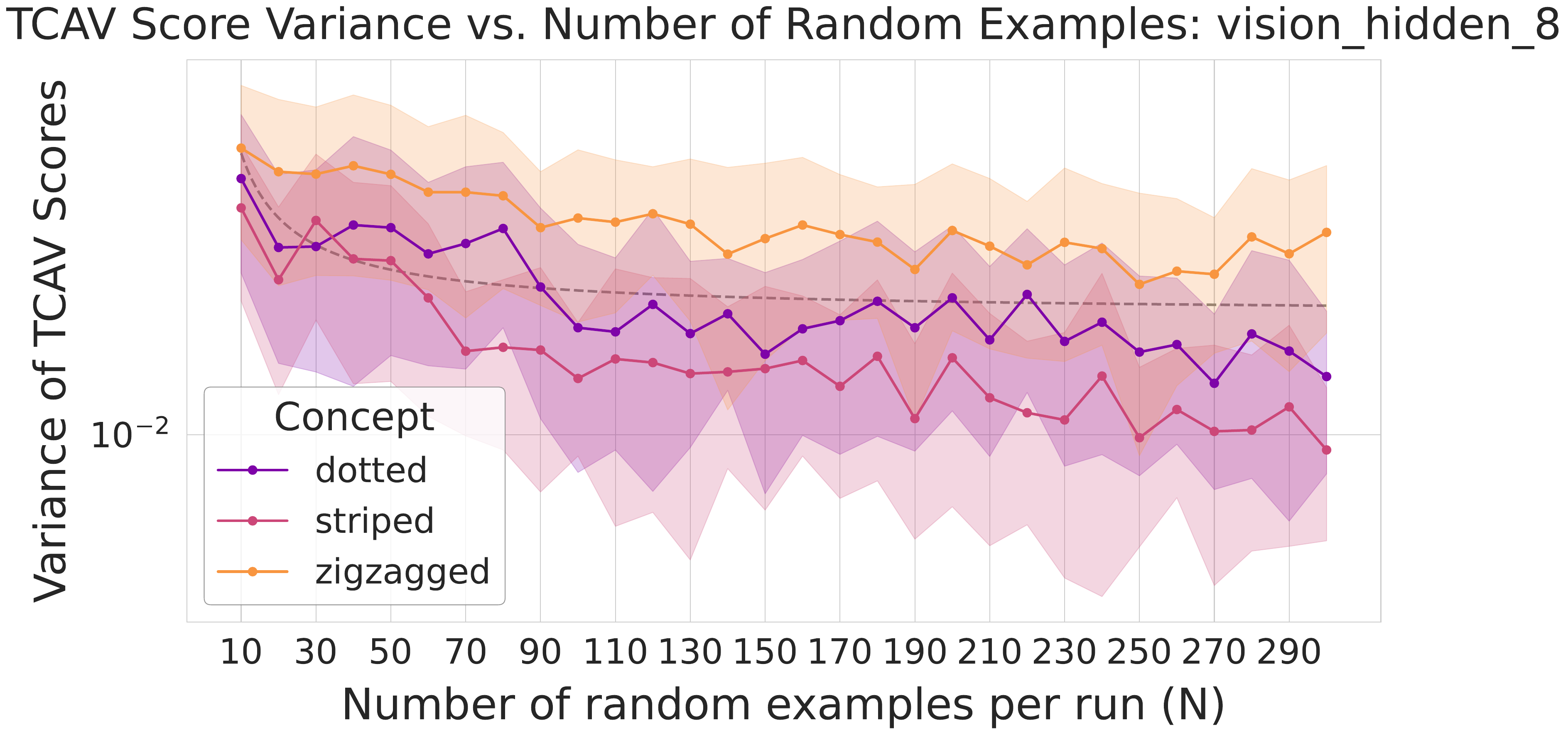}
    \end{minipage}
    \caption[Variance of TCAV scores (Hinge Loss) for CLIP]{\label{fig:variance_of_tcavs_hinge_multimodal}
    Variance of \text{\textsc{Tcav}\xspace} scores at CLIP hidden layers $4$ and $8$ vs.\ the number of examples per concept set ($N$) for the visual concepts ``striped'', ``zigzagged'', and ``dotted''. 
    The underlying \text{\textsc{Cav}\xspace}s were trained using \textbf{hinge loss}. Error bars denote $\pm 1$ standard deviation.
    We fitted a curve of the form $f(N) = a/N + b$ to it. For \texttt{vision\_hidden\_4} $ 0.122/N + 0.00496 $, for \texttt{vision\_hidden\_8} $ 0.223/N + 0.0183 $.}
\end{figure}

\begin{figure}[H]
    \centering
    \begin{minipage}{0.48\textwidth}
        \centering
        \includegraphics[width=\linewidth]{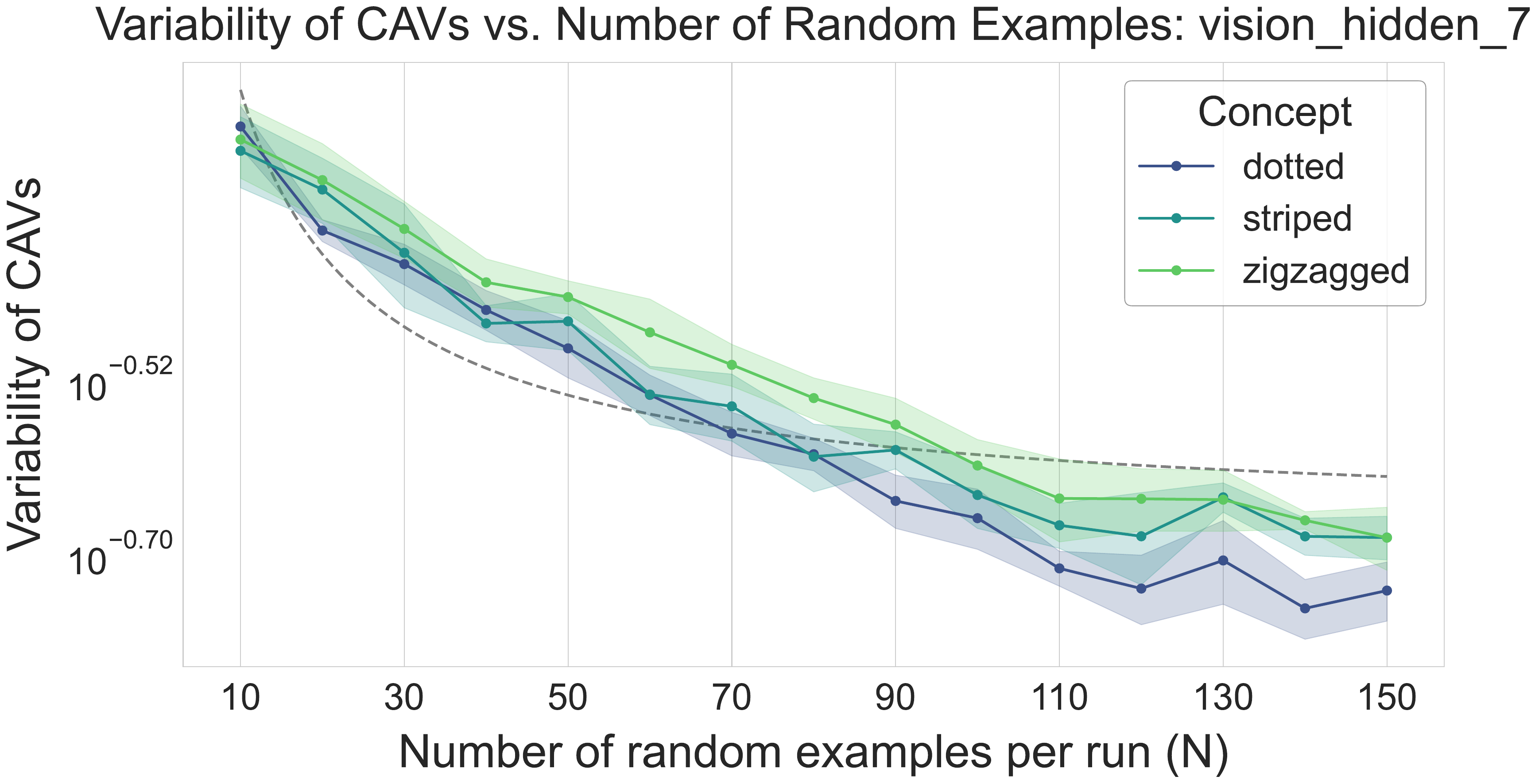}
    \end{minipage}
    \hfill
    \begin{minipage}{0.48\textwidth}
        \centering
        \includegraphics[width=\linewidth]{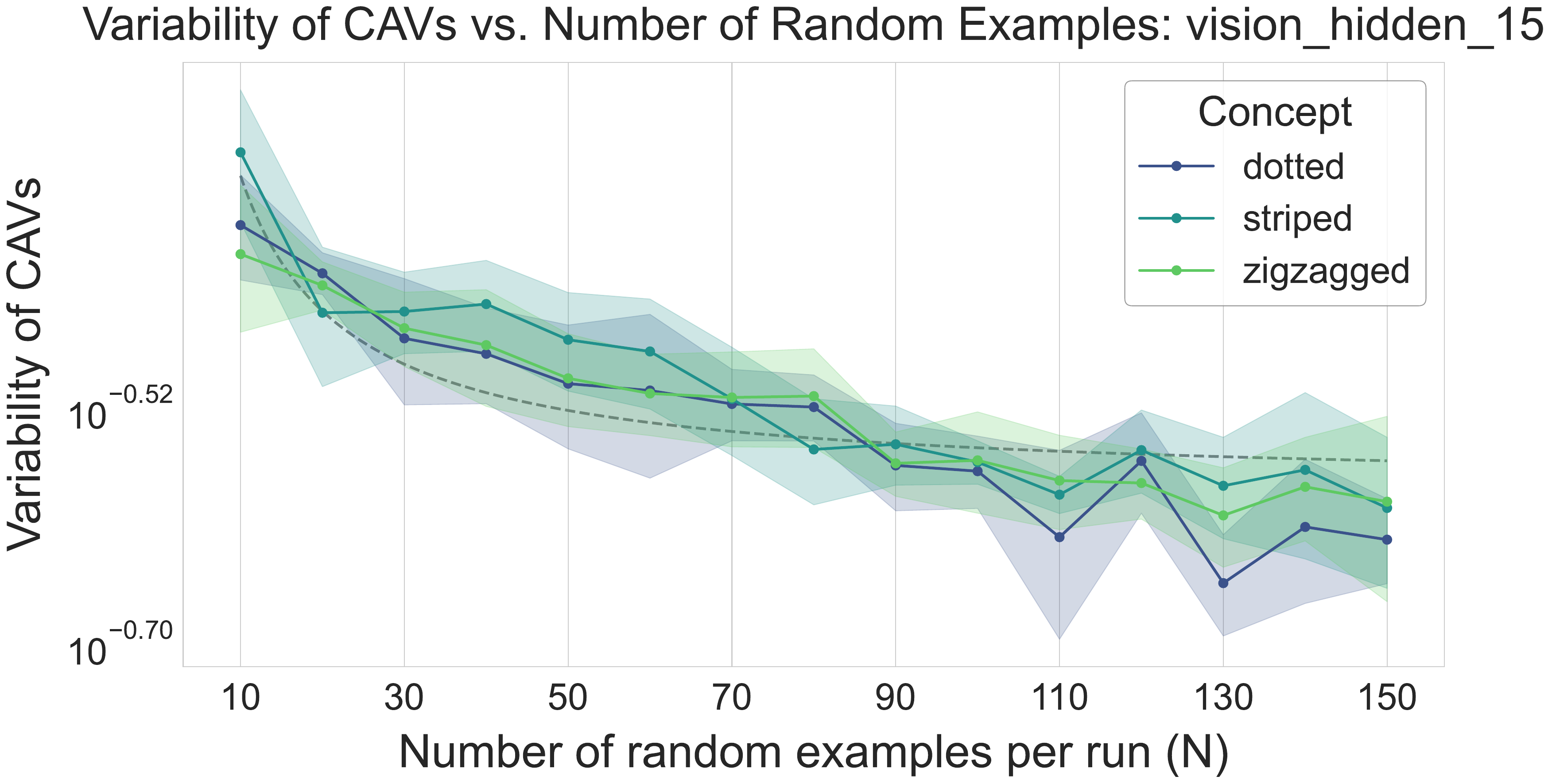}
    \end{minipage}
    \caption[Variability of CAVs (Hinge Loss) for LLaVA-1.5]{\label{fig:variance_of_cavs_hinge_llava_multimodal}
    Mean variability of \text{\textsc{Cav}\xspace}s for LLaVA-1.5 as a function of the number of random examples per run ($N$) for Broden-style visual concepts. Results are shown for vision layers \texttt{vision\_hidden\_7} and \texttt{vision\_hidden\_15}. Error bars indicate $\pm 1$ SD; the $y$-axis is log-scaled.}
\end{figure}

\begin{figure}[H]
    \centering
    \begin{minipage}{0.48\textwidth}
        \centering
        \includegraphics[width=\linewidth]{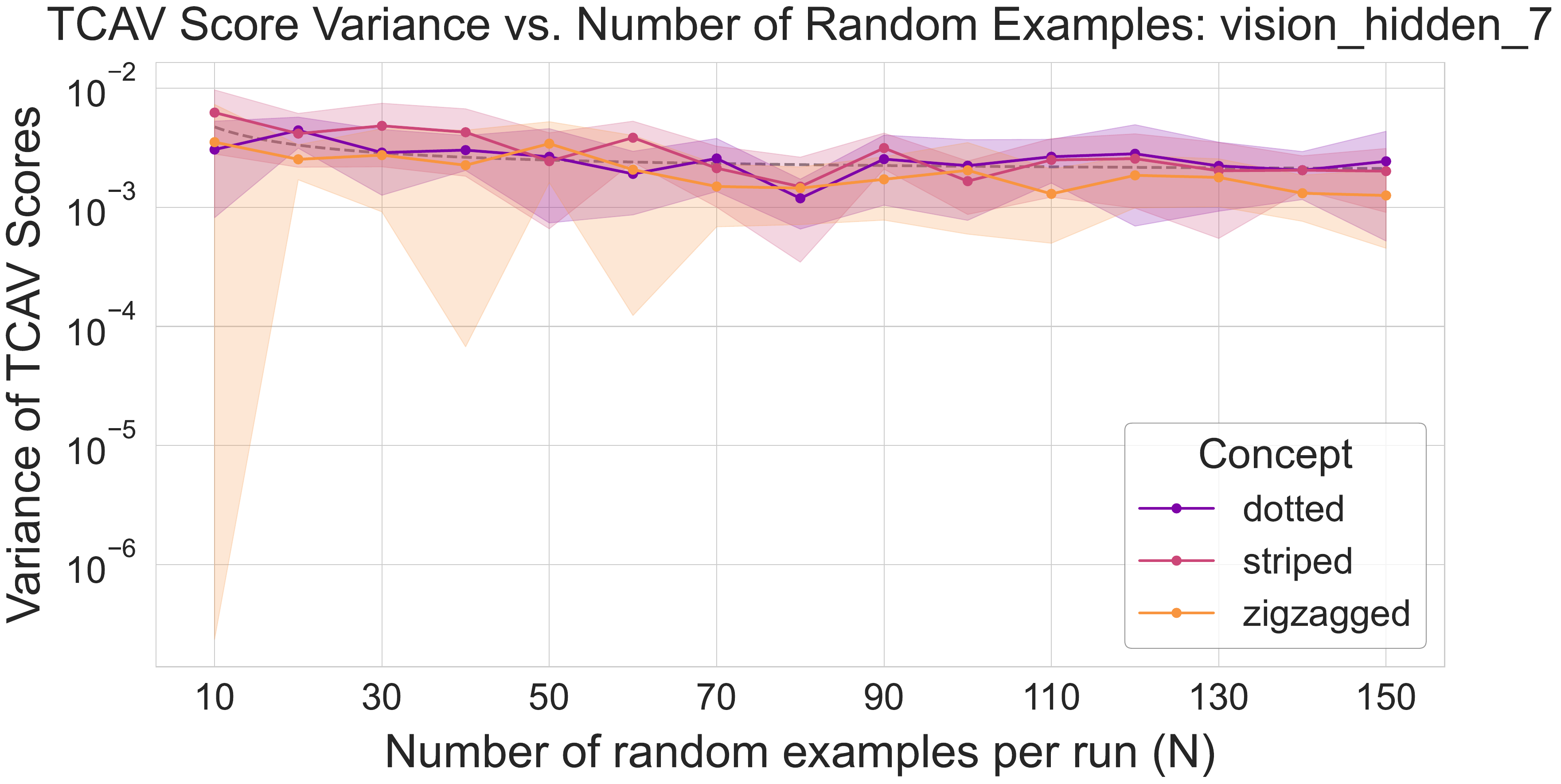}
    \end{minipage}
    \hfill
    \begin{minipage}{0.48\textwidth}
        \centering
        \includegraphics[width=\linewidth]{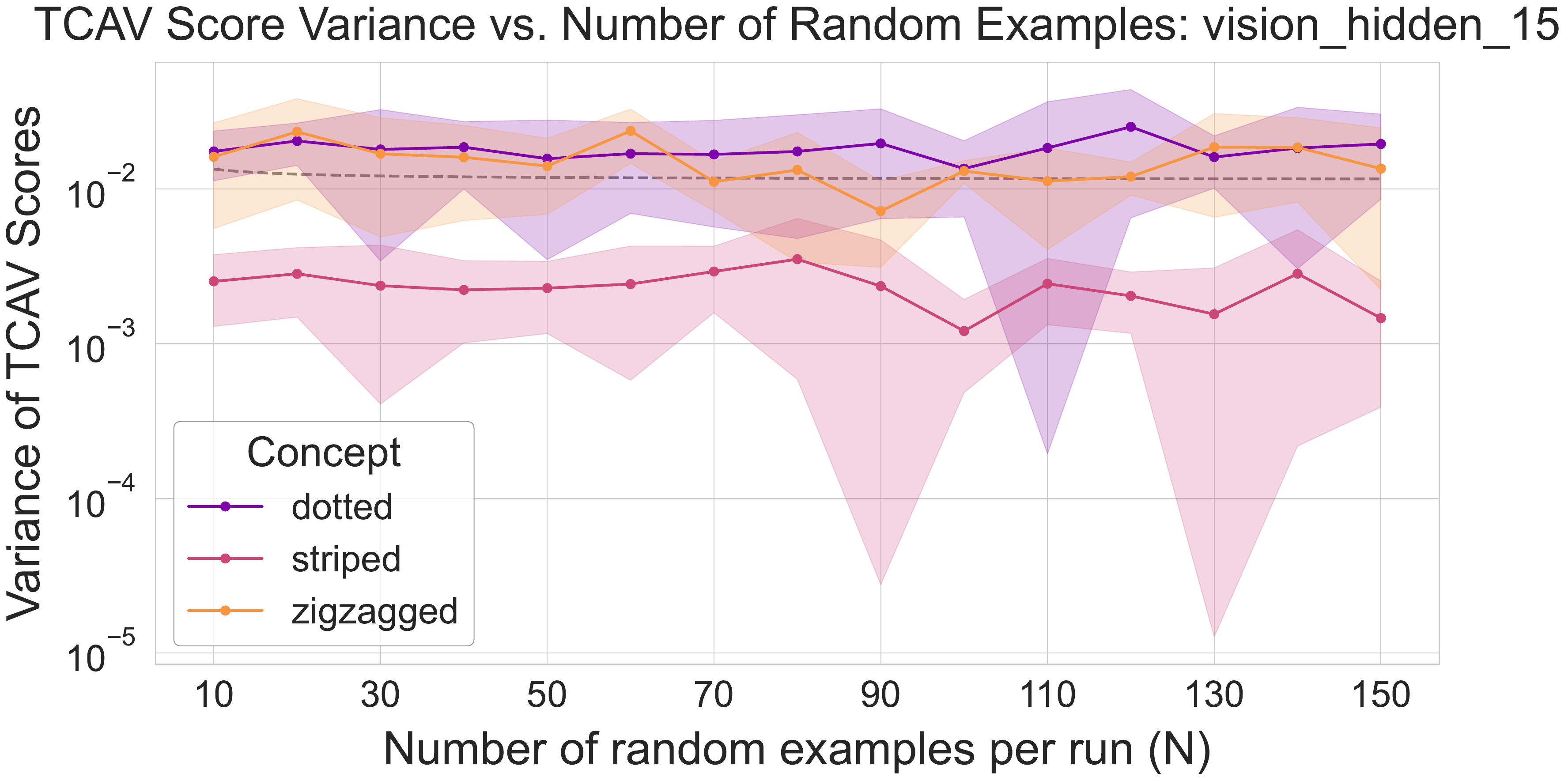}
    \end{minipage}
    \caption[Variance of TCAV scores (Hinge Loss) for LLaVA-1.5]{\label{fig:variance_of_tcavs_hinge_llava_multimodal}
    Variance of \text{\textsc{Tcav}\xspace} scores for LLaVA-1.5 at vision layers \texttt{vision\_hidden\_7} and \texttt{vision\_hidden\_15} vs. the number of examples per concept set ($N$) for Broden-style visual concepts. Error bars denote $\pm 1$ standard deviation.}
\end{figure}

\subsubsection{Empirical Findings with Concept Activation Regions}

Second, we prove that our analysis generalizes to Concept Activation Regions.

\begin{figure}[H]
    \centering
    \begin{minipage}{0.48\textwidth}
        \centering
        \includegraphics[width=\linewidth]{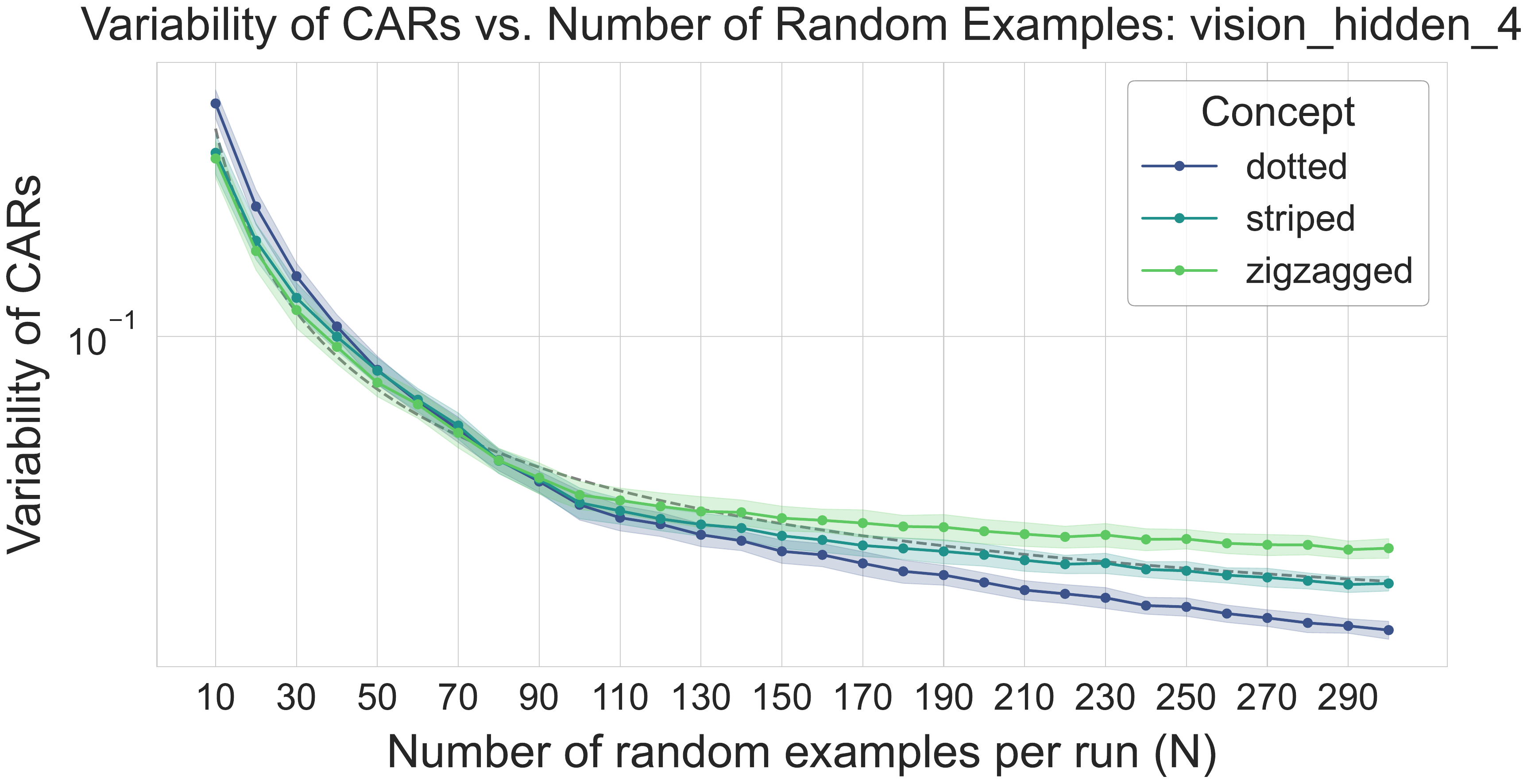}
    \end{minipage}
    \hfill
    \begin{minipage}{0.48\textwidth}
        \centering
        \includegraphics[width=\linewidth]{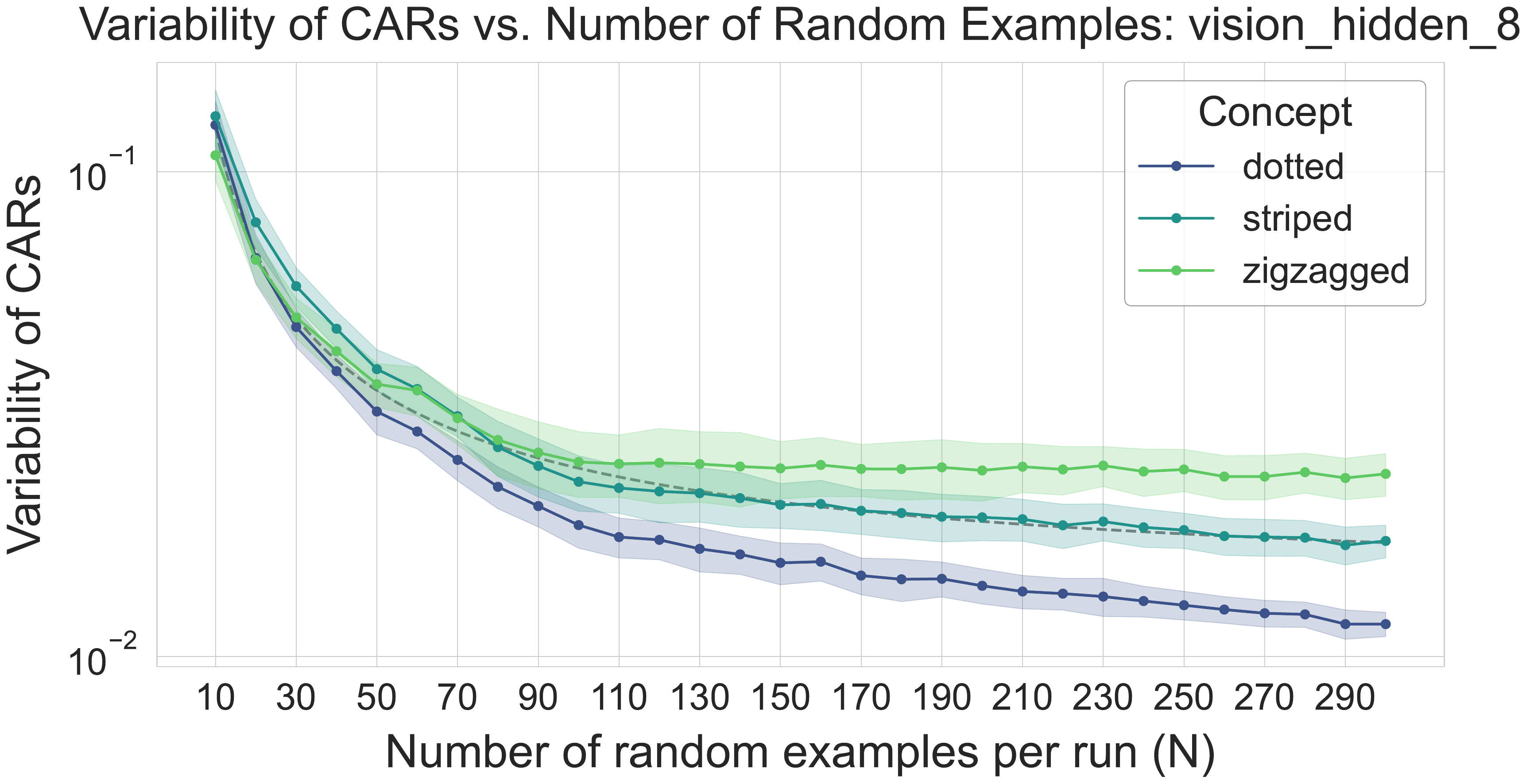}
    \end{minipage}
    \caption[Variance of CARs for CLIP]{\label{fig:variance_of_cars_multimodal}
    Mean variability of \text{\textsc{Car}\xspace}s for the visual concepts ``striped'', ``zigzagged'', and ``dotted'' as a function of the number of random examples per run ($N$). 
    Results are shown for CLIP hidden layers $4$ and $8$. Error bars indicate $\pm 1$ SD; the $y$-axis is log-scaled.
    Variance is estimated by the sum of per-feature variances across five independent runs. We fitted a curve of the form $f(N) = a/N + b$ to it. For \texttt{vision\_hidden\_4} $ 2.91/N + 0.0168 $, for \texttt{vision\_hidden\_8} $ 1.09/N + 0.0135 $.}
\end{figure}

\begin{figure}[H]
    \centering  
    \begin{minipage}{0.48\textwidth}
        \centering
        \includegraphics[width=\linewidth]{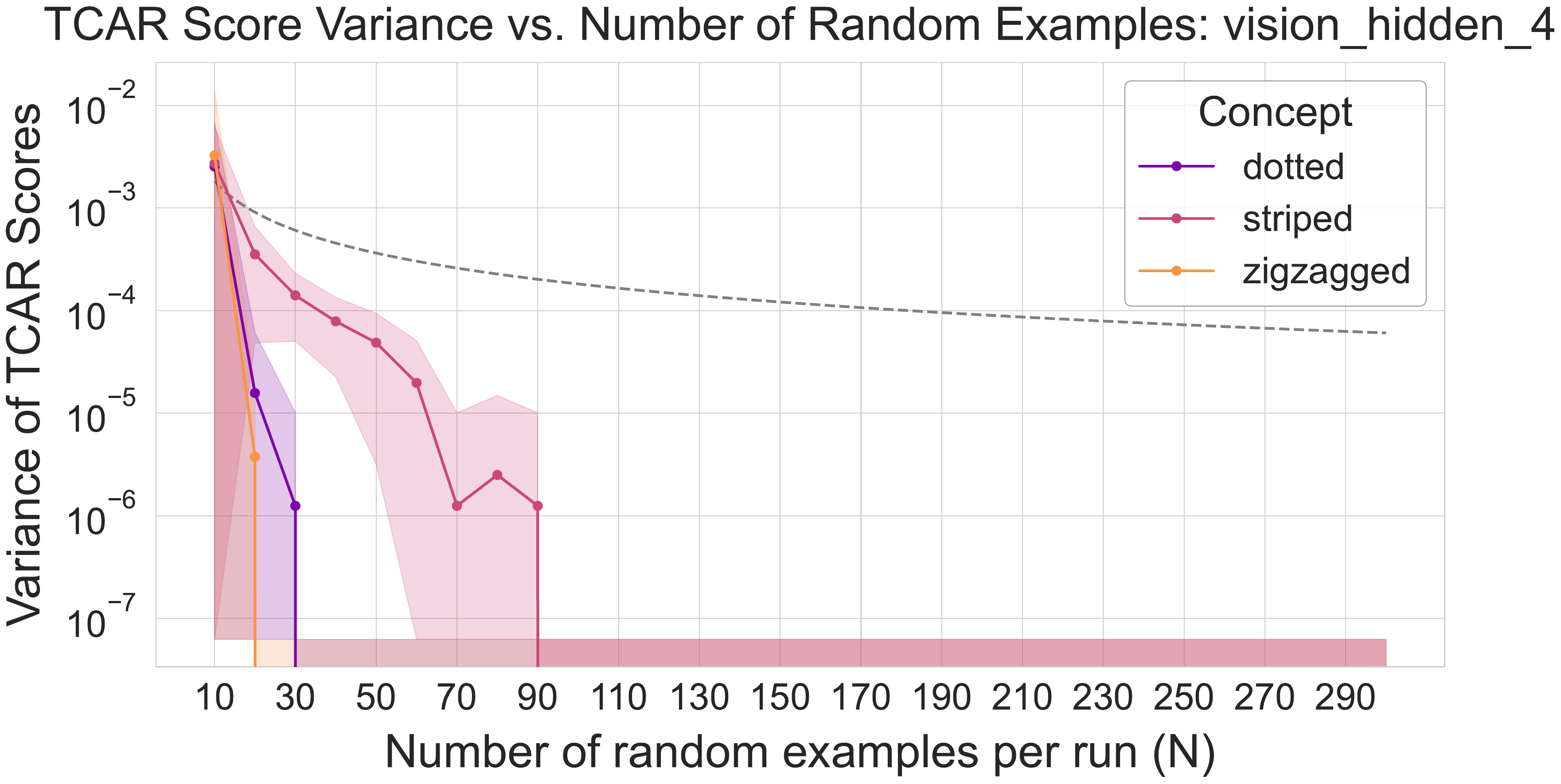}
    \end{minipage}
    \hfill
    \begin{minipage}{0.48\textwidth}
        \centering
        \includegraphics[width=\linewidth]{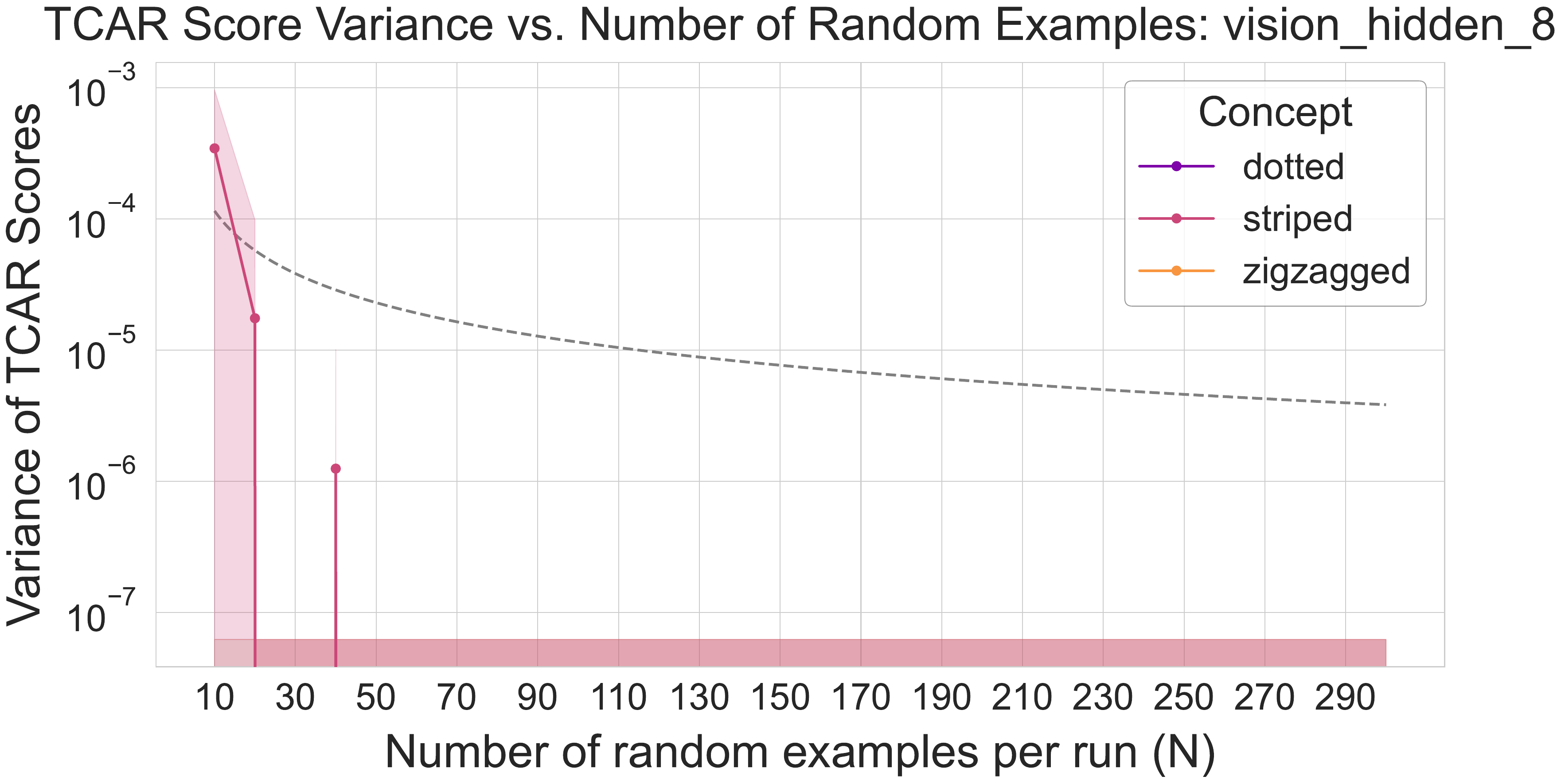}
    \end{minipage}
    \caption[Variance of TCAR scores for CLIP]{\label{fig:variance_of_tcars_multimodal}
    Variance of \text{\textsc{Tcar}\xspace} scores at CLIP hidden layers $4$ and $8$ vs.\ the number of examples per concept set ($N$) for the visual concepts ``striped'', ``zigzagged'', and ``dotted''. 
    The underlying \text{\textsc{Car}\xspace}s were trained using a non-linear SVM. Error bars denote $\pm 1$ standard deviation.
    We fitted a curve of the form $f(N) = a/N + b$ to it.}
\end{figure}

\begin{figure}[H]
    \centering
    \begin{minipage}{0.48\textwidth}
        \centering
        \includegraphics[width=\linewidth]{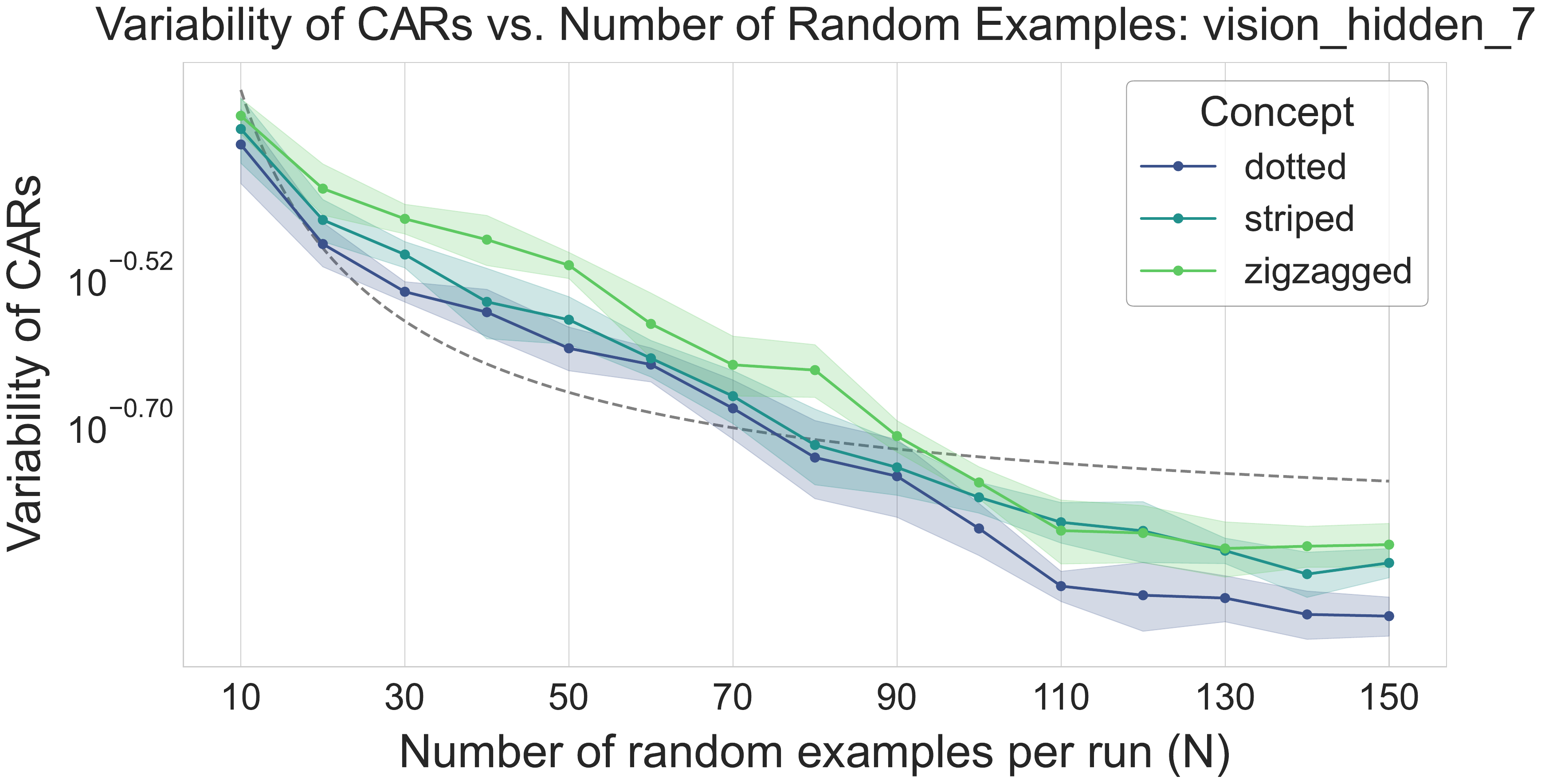}
    \end{minipage}
    \hfill
    \begin{minipage}{0.48\textwidth}
        \centering
        \includegraphics[width=\linewidth]{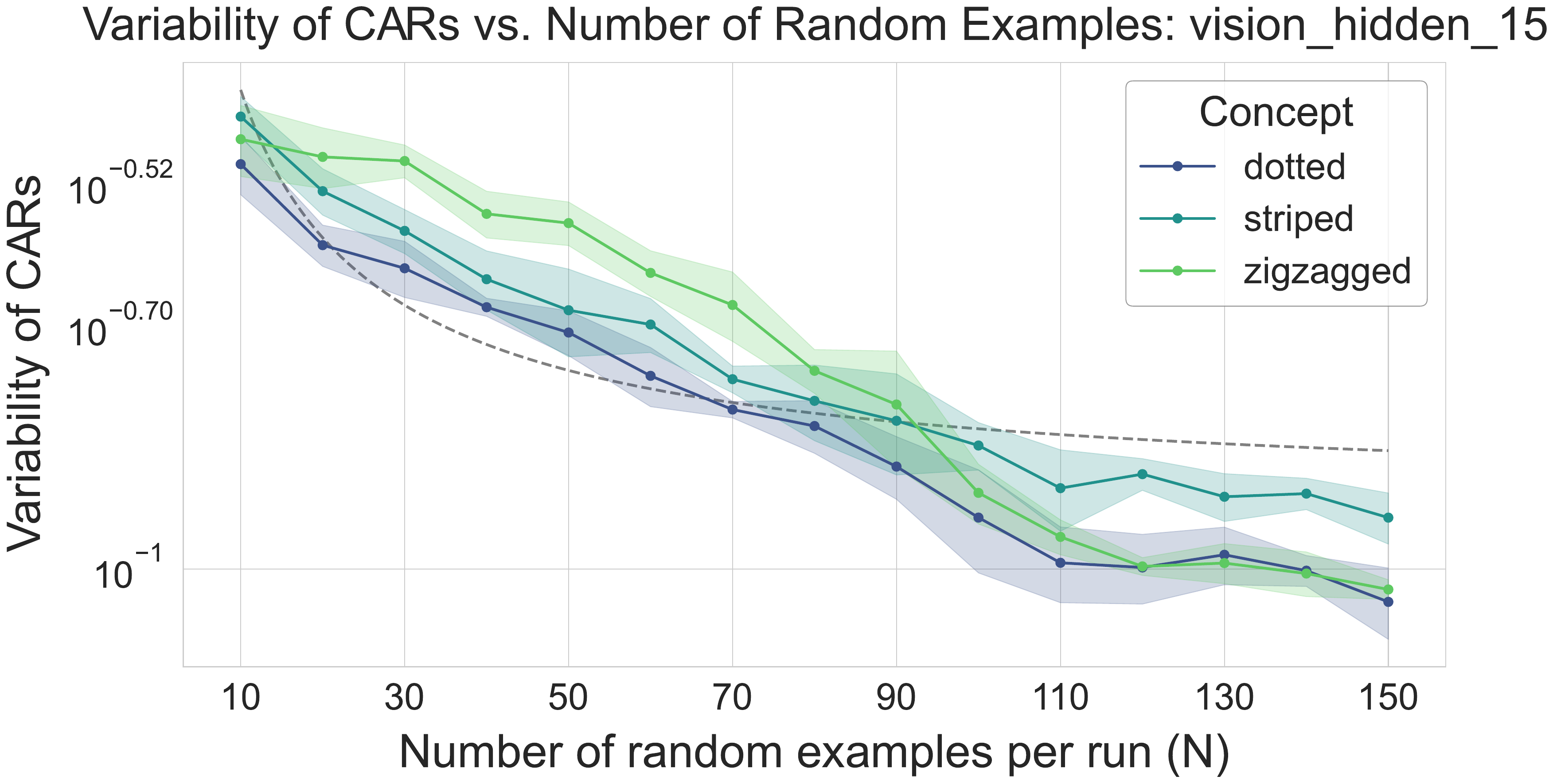}
    \end{minipage}
    \caption[Variability of CARs (Concept Activation Regions) for LLaVA-1.5]{\label{fig:variance_of_cars_svm_llava_multimodal}
    Mean variability of \text{\textsc{Car}\xspace}s for LLaVA-1.5 as a function of the number of random examples per run ($N$) for Broden-style visual concepts. Results are shown for vision layers \texttt{vision\_hidden\_7} and \texttt{vision\_hidden\_15}. Error bars indicate $\pm 1$ SD; the $y$-axis is log-scaled.}
\end{figure}

\begin{figure}[H]
    \centering
    \begin{minipage}{0.48\textwidth}
        \centering
        \includegraphics[width=\linewidth]{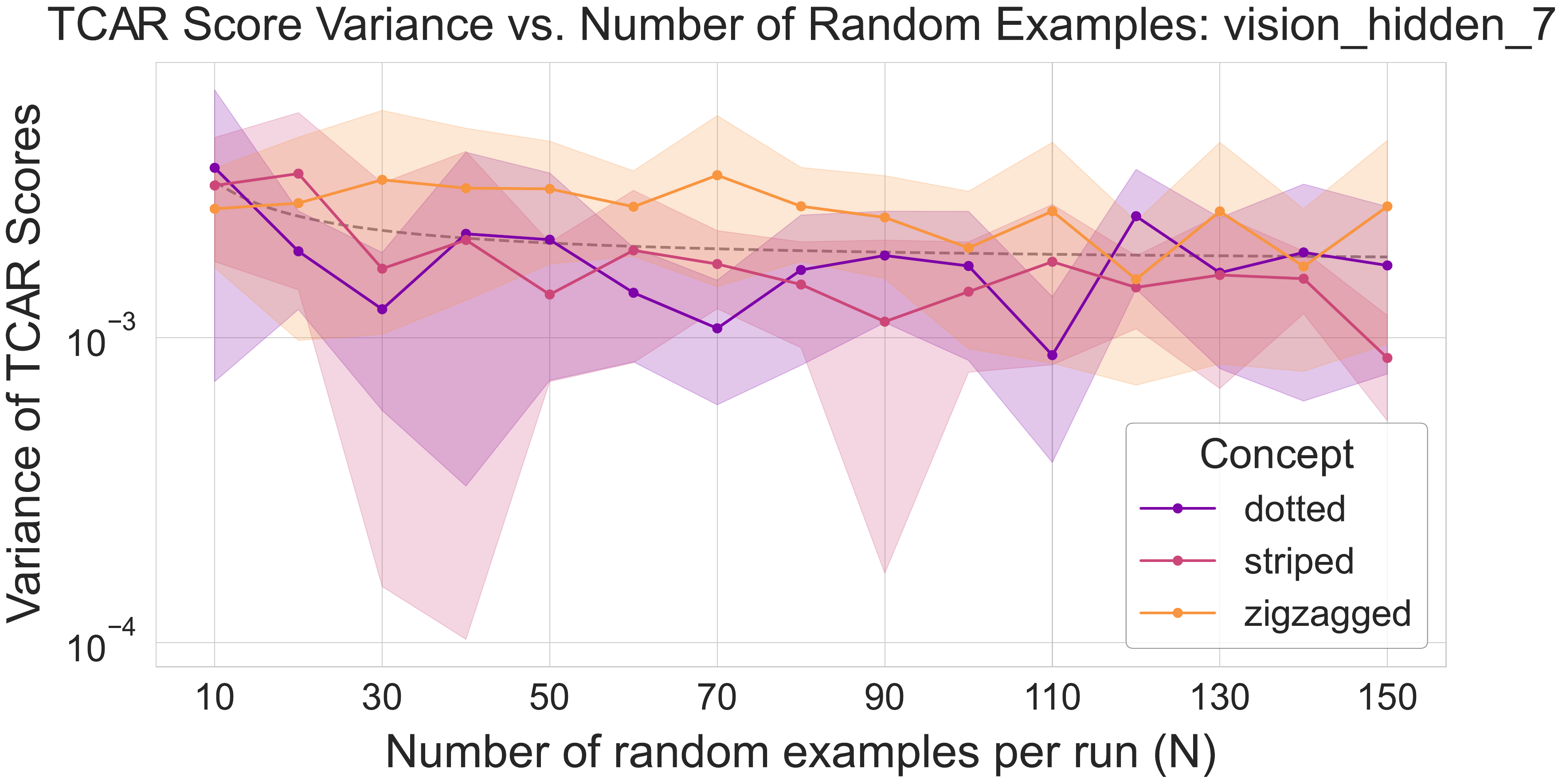}
    \end{minipage}
    \hfill
    \begin{minipage}{0.48\textwidth}
        \centering
        \includegraphics[width=\linewidth]{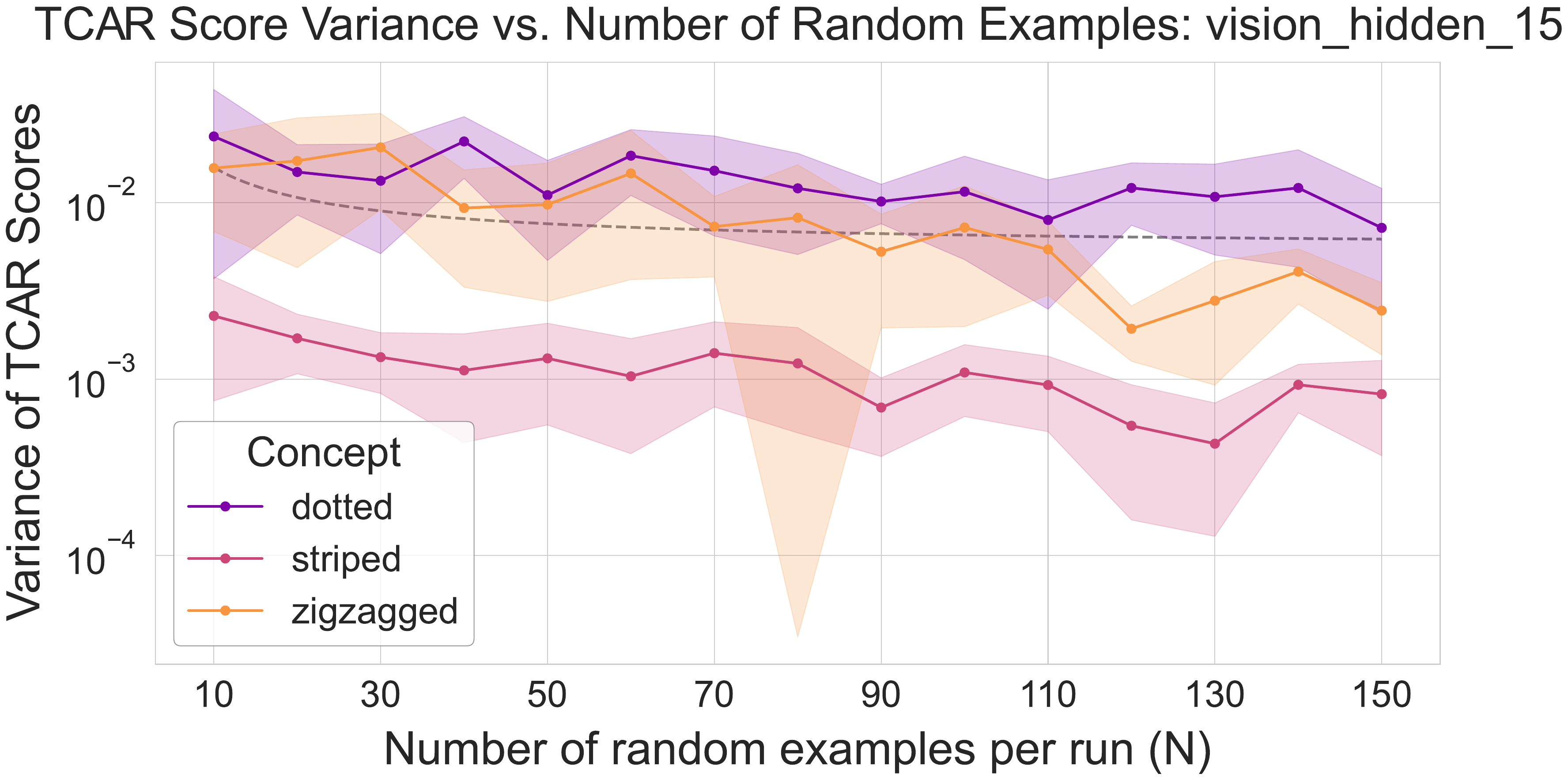}
    \end{minipage}
    \caption[Variance of TCAR scores (Concept Activation Regions) for LLaVA-1.5]{\label{fig:variance_of_tcars_svm_llava_multimodal}
    Variance of \text{\textsc{Tcar}\xspace} scores for LLaVA-1.5 at vision layers \texttt{vision\_hidden\_7} and \texttt{vision\_hidden\_15} vs. the number of examples per concept set ($N$) for Broden-style visual concepts. Error bars denote $\pm 1$ standard deviation.}
\end{figure}

\subsubsection{Empirical Findings with Difference of Means}

Again, we found that \text{\textsc{Cav}\xspace}s computed via the \emph{Difference of Means} method most closely follow a variance decline of $\mathcal{O}(1/N)$.

\begin{figure}[H]
    \centering
    \begin{minipage}{0.48\textwidth}
        \centering
        \includegraphics[width=\linewidth]{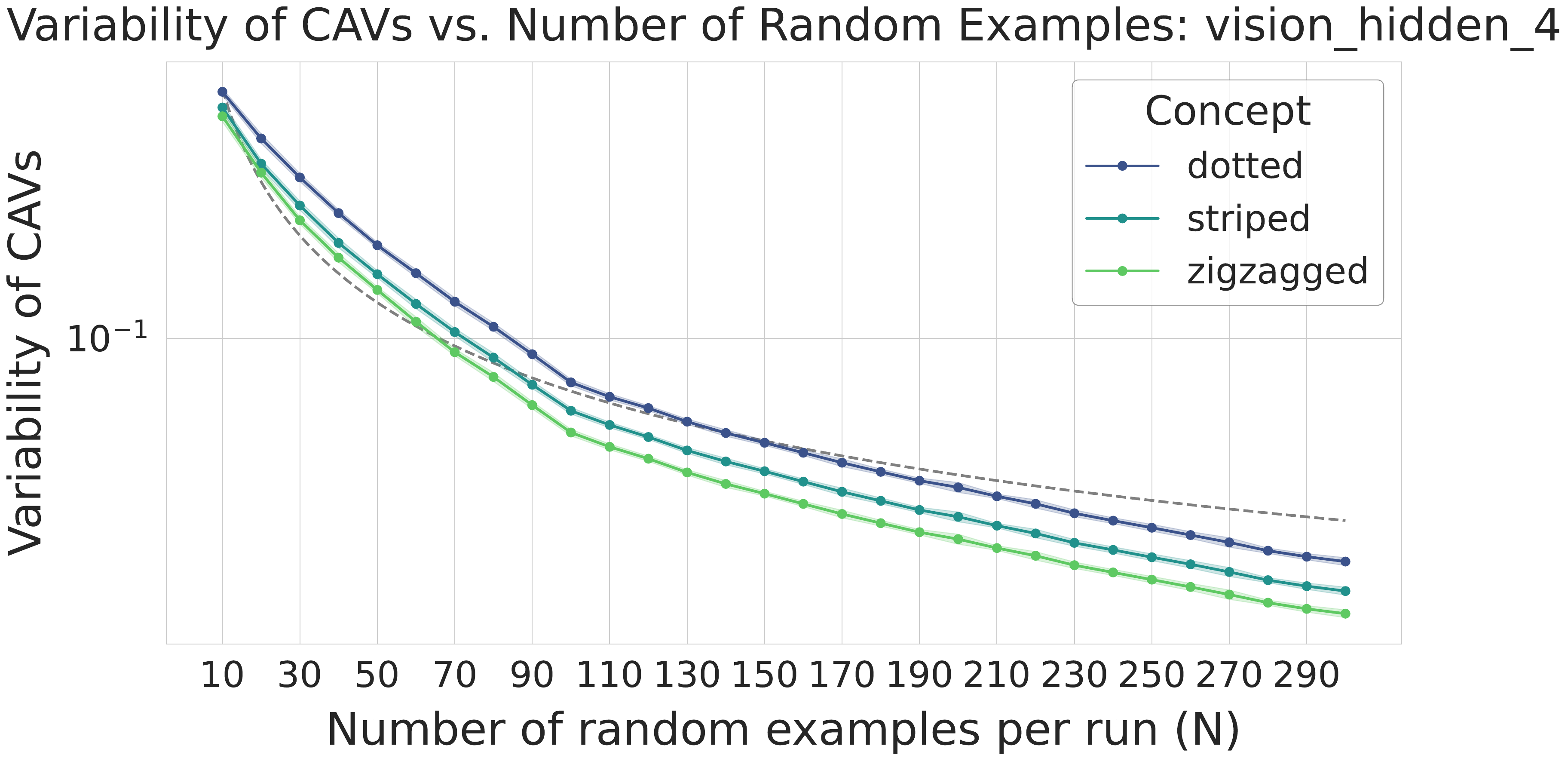}
    \end{minipage}
    \hfill
    \begin{minipage}{0.48\textwidth}
        \centering
        \includegraphics[width=\linewidth]{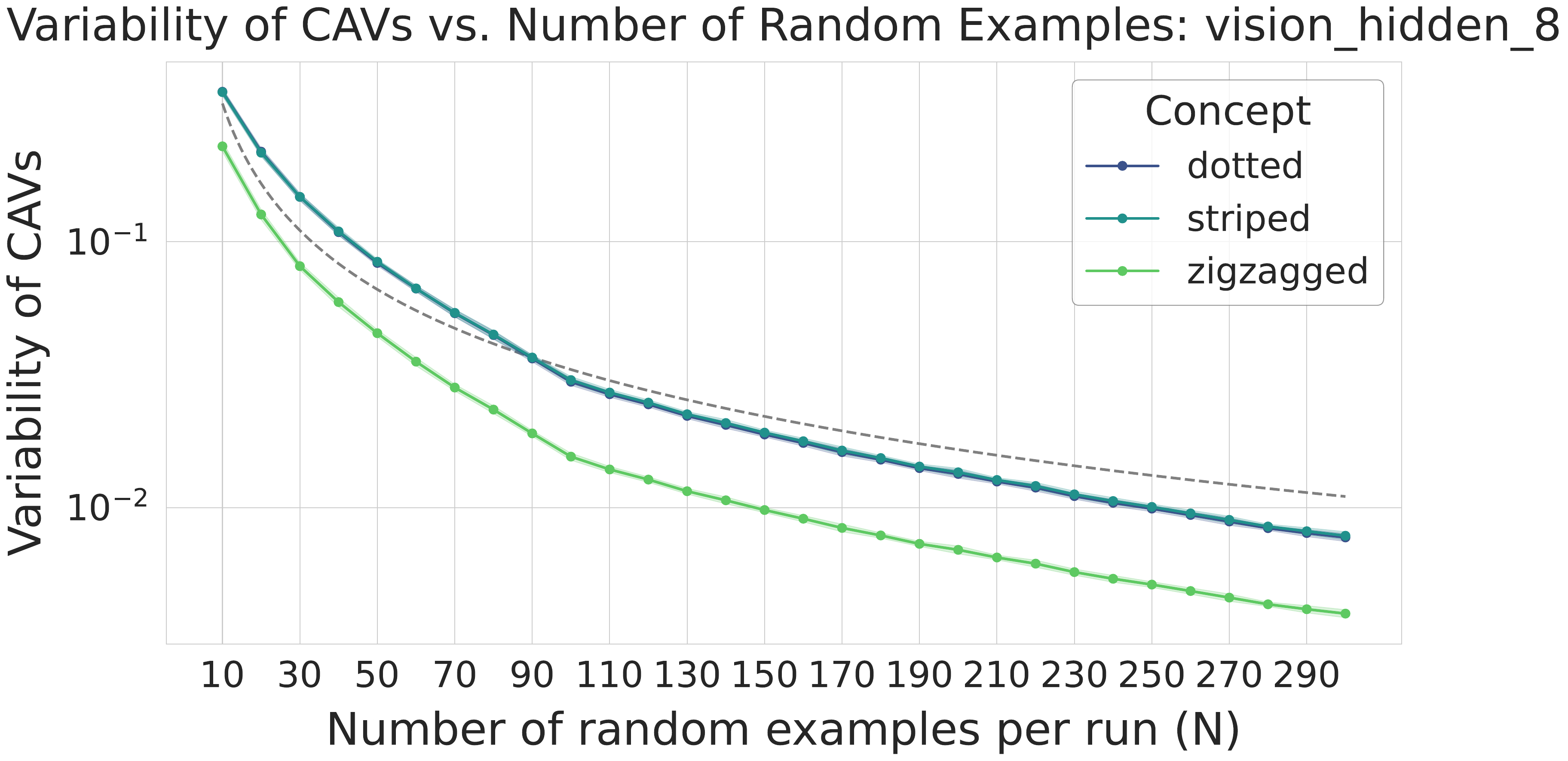}
    \end{minipage}
    \caption[Variance of CAVs (Difference of Means) for CLIP]{\label{fig:variance_of_cavs_dom_multimodal}
    Mean variability of \text{\textsc{Cav}\xspace}s in dependence of the number of random examples per run $N$ for the CLIP model at hidden layers $4$ and $8$.
    The \text{\textsc{Cav}\xspace}s were generated using \textbf{difference-of-means}. Error bars indicate $\pm 1$ SD; the $y$-axis is log-scaled.
    Variance is estimated by the sum of per-feature variances across five independent runs. We fitted a curve of the form $f(N) = a/N + b$ to it.}
\end{figure}

\begin{figure}[H]
    \centering
    \begin{minipage}{0.48\textwidth}
        \centering
        \includegraphics[width=\linewidth]{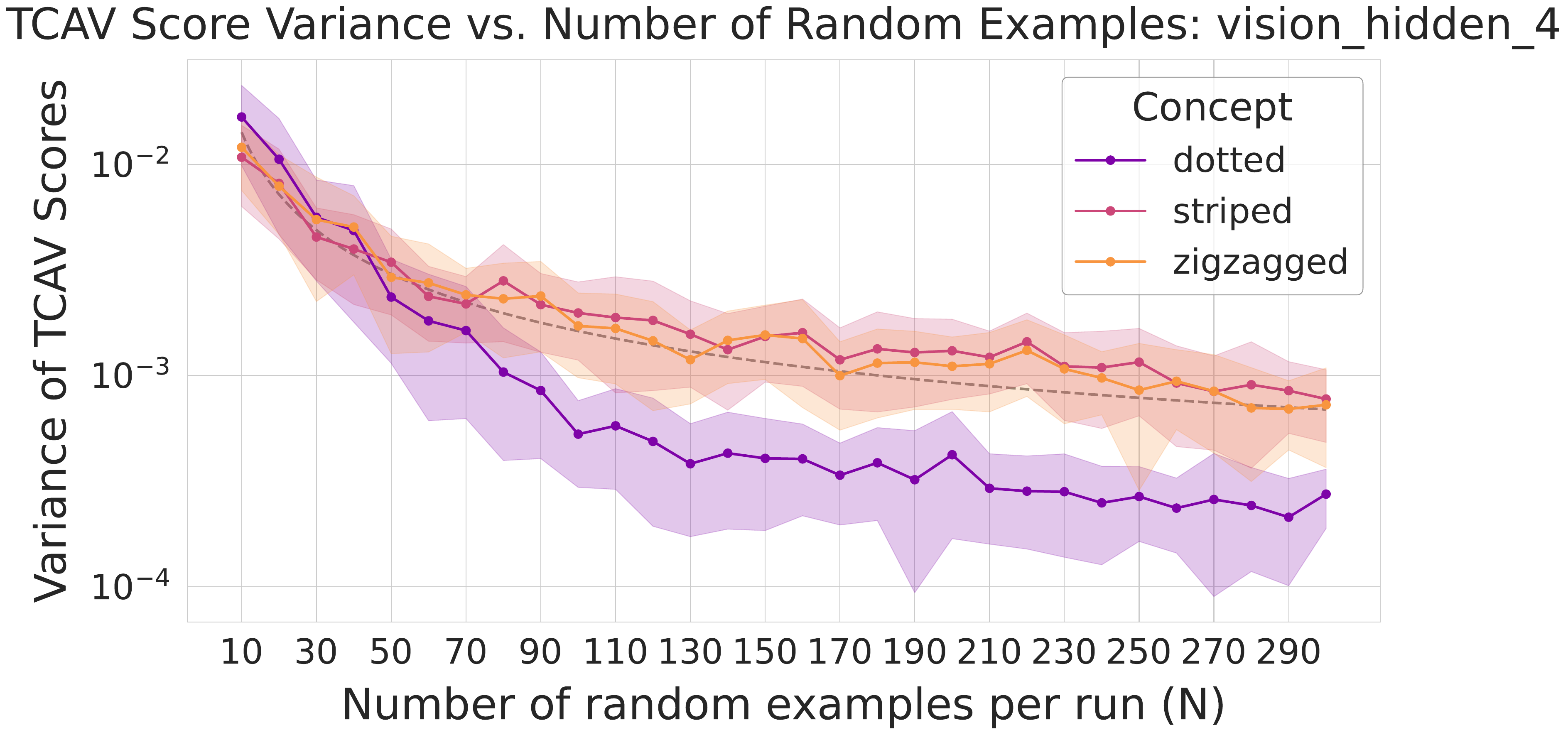}
    \end{minipage}
    \hfill 
    \begin{minipage}{0.48\textwidth}
        \centering
        \includegraphics[width=\linewidth]{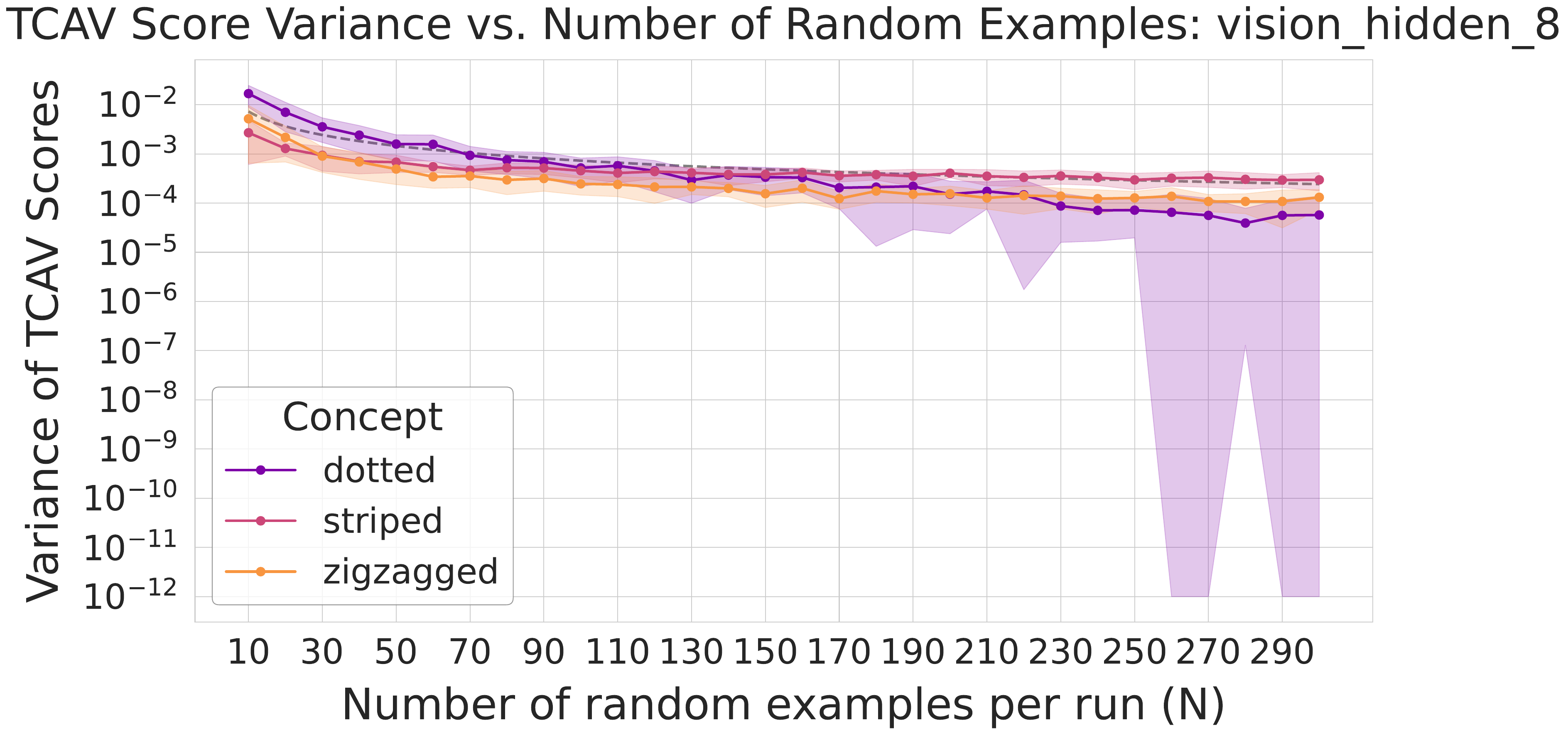}
    \end{minipage}
    \caption[Variance of TCAV scores (Difference of Means) for CLIP]{\label{fig:variance_of_tcavs_dom_multimodal}
    Variance of \text{\textsc{Tcav}\xspace} scores at CLIP hidden layers $4$ and $8$ vs.\ the number $N$ of random examples used for the visual concepts ``striped'', ``zigzagged'', and ``dotted''. 
    The underlying \text{\textsc{Cav}\xspace}s were generated using \textbf{difference-of-means}. Error bars denote $\pm 1$ standard deviation.
    We fitted a curve of the form $f(N) = a/N + b$ to it.}
\end{figure}

\begin{figure}[H]
    \centering
    \begin{minipage}{0.48\textwidth}
        \centering
        \includegraphics[width=\linewidth]{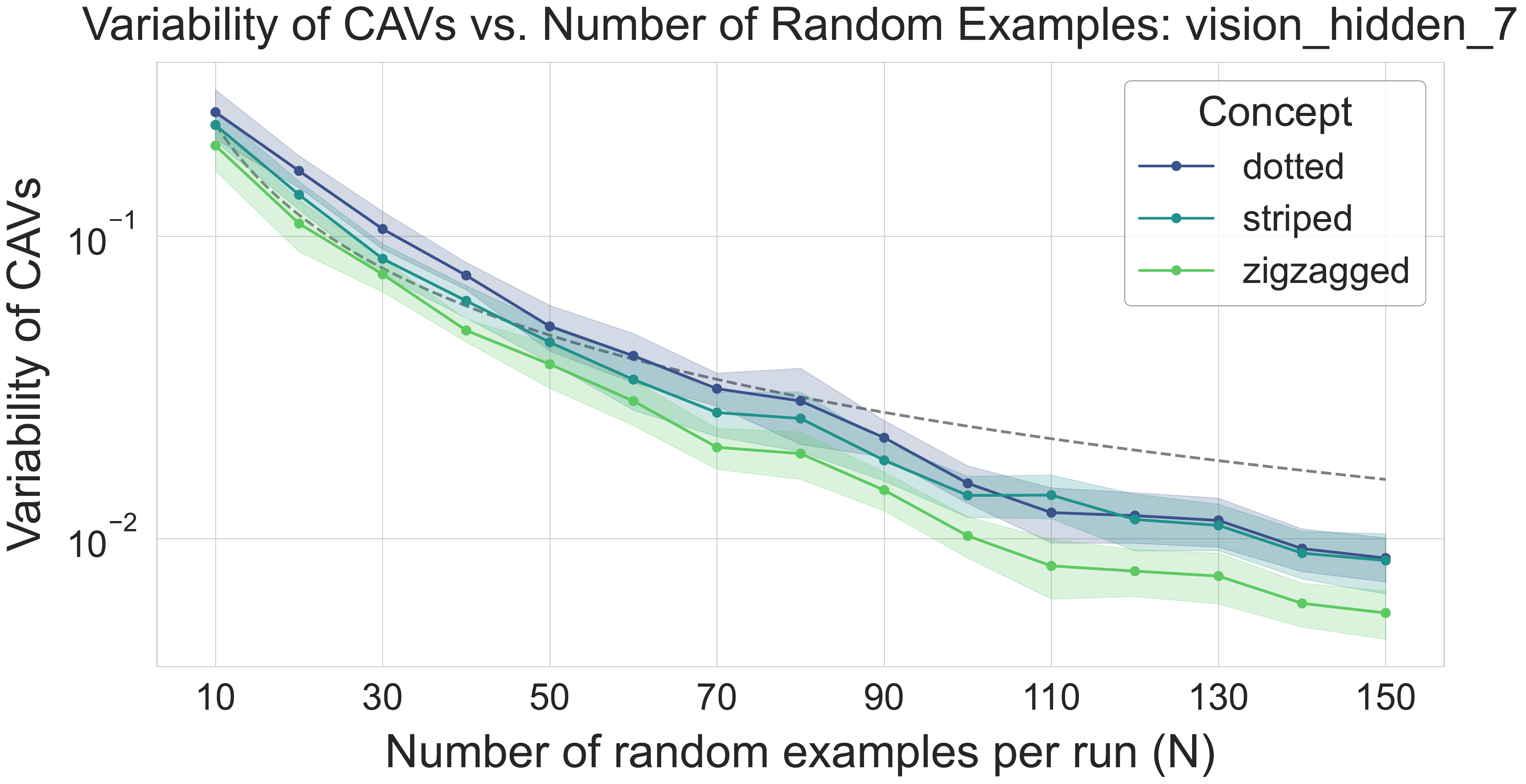}
    \end{minipage}
    \hfill
    \begin{minipage}{0.48\textwidth}
        \centering
        \includegraphics[width=\linewidth]{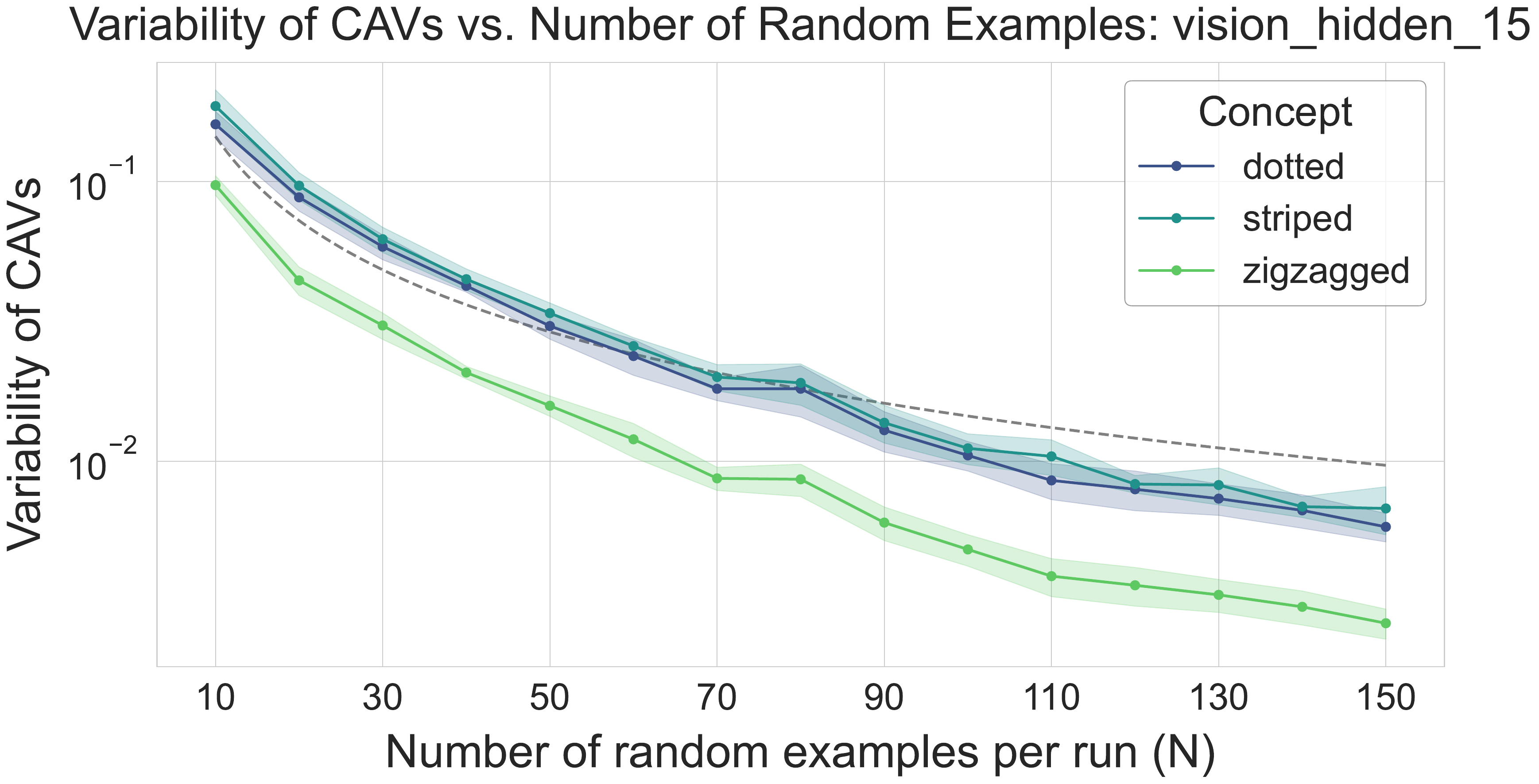}
    \end{minipage}
    \caption[Variability of CAVs (Difference of Means) for LLaVA-1.5]{\label{fig:variance_of_cavs_dom_llava_multimodal}
    Mean variability of \text{\textsc{Cav}\xspace}s for LLaVA-1.5 as a function of the number of random examples per run ($N$) for Broden-style visual concepts. Results are shown for vision layers \texttt{vision\_hidden\_7} and \texttt{vision\_hidden\_15}. Error bars indicate $\pm 1$ SD; the $y$-axis is log-scaled.}
\end{figure}

\begin{figure}[H]
    \centering
    \begin{minipage}{0.48\textwidth}
        \centering
        \includegraphics[width=\linewidth]{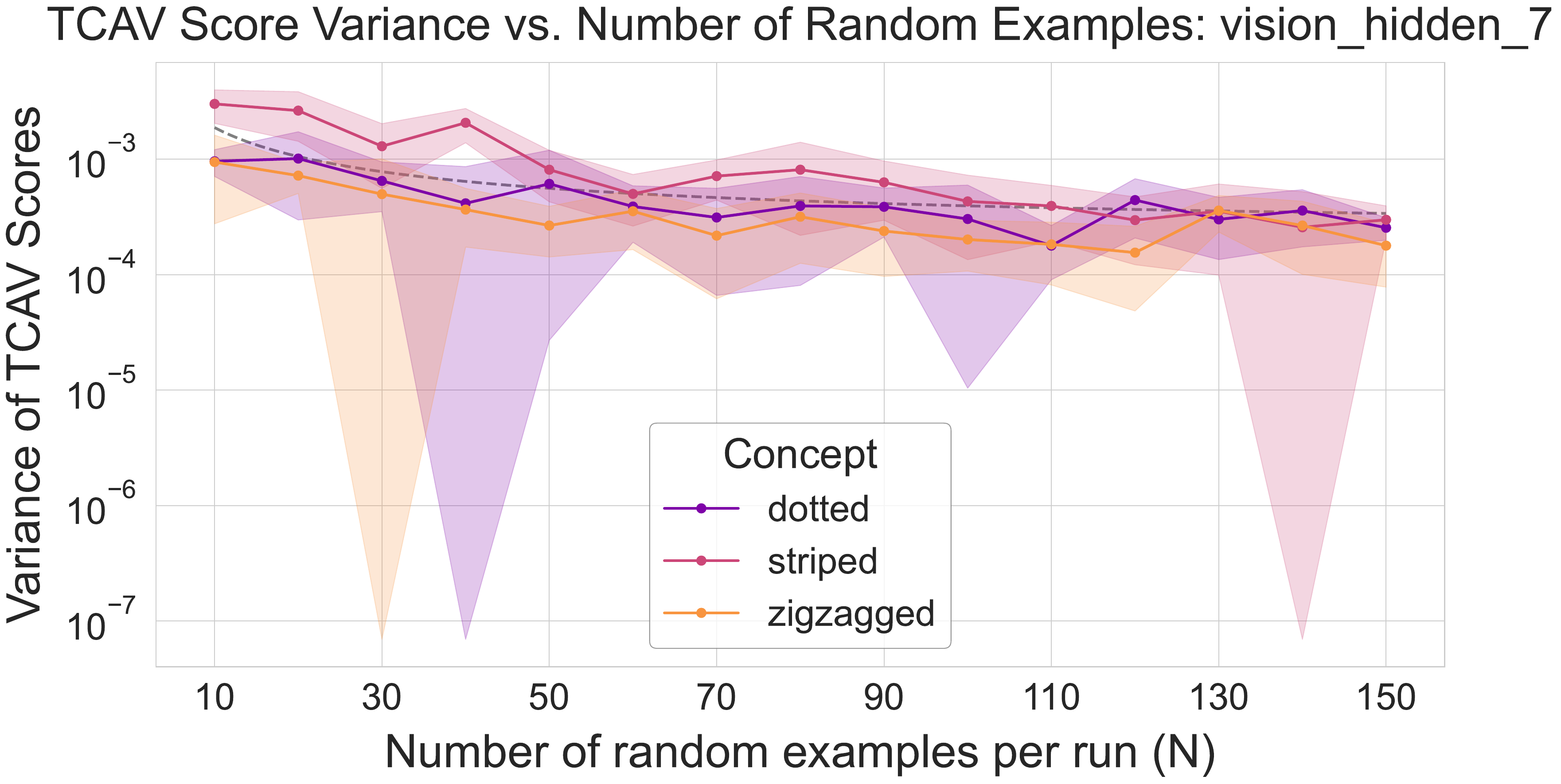}
    \end{minipage}
    \hfill
    \begin{minipage}{0.48\textwidth}
        \centering
        \includegraphics[width=\linewidth]{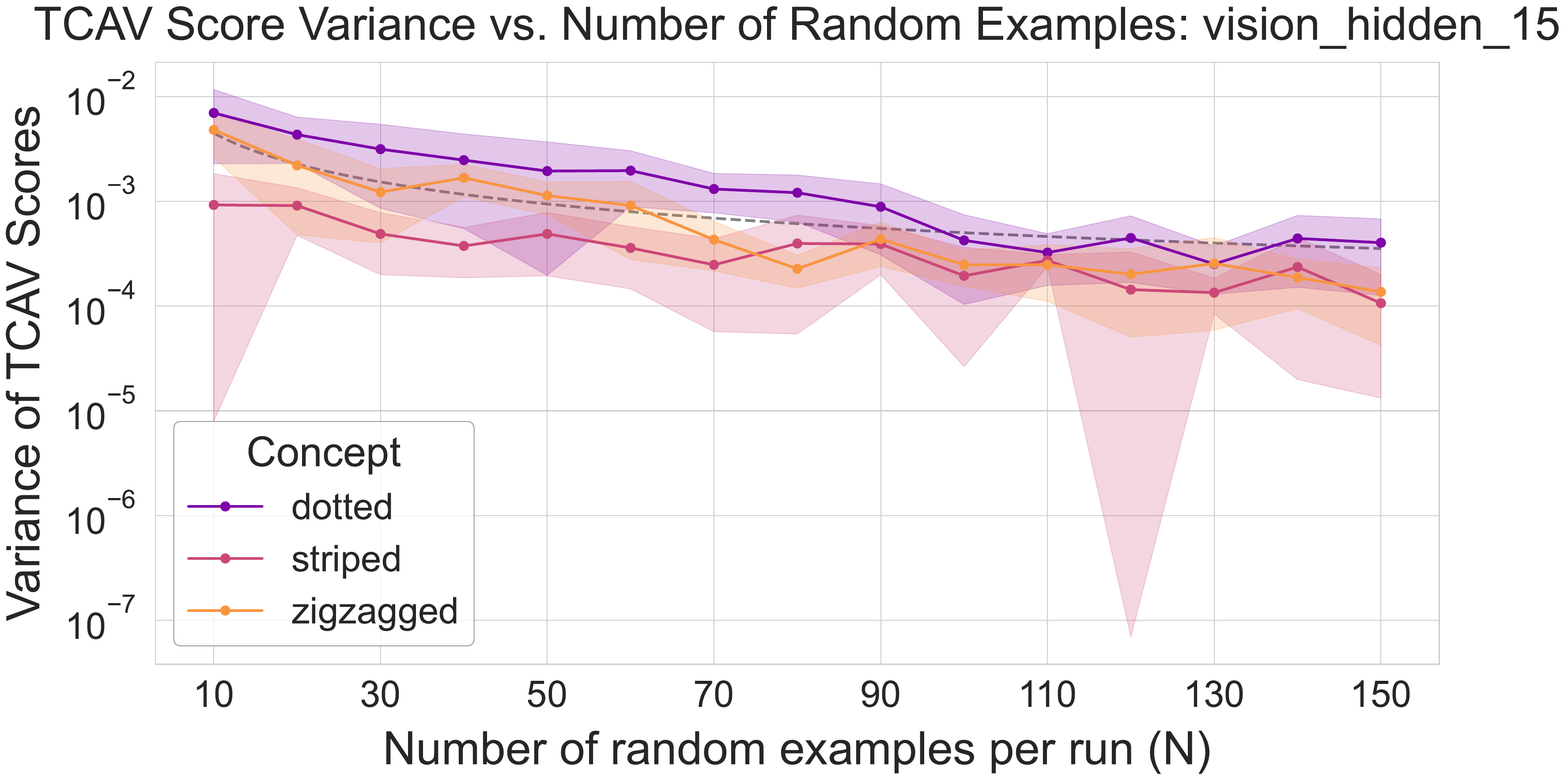}
    \end{minipage}
    \caption[Variance of TCAV scores (Difference of Means) for LLaVA-1.5]{\label{fig:variance_of_tcavs_dom_llava_multimodal}
    Variance of \text{\textsc{Tcav}\xspace} scores for LLaVA-1.5 at vision layers \texttt{vision\_hidden\_7} and \texttt{vision\_hidden\_15} vs. the number of examples per concept set ($N$) for Broden-style visual concepts. Error bars denote $\pm 1$ standard deviation.}
\end{figure}

\subsubsection{Compute Trade-off}

\begin{figure}[H]
    \centering

    \begin{subfigure}[t]{0.48\textwidth}
        \centering
        \includegraphics[width=\linewidth]{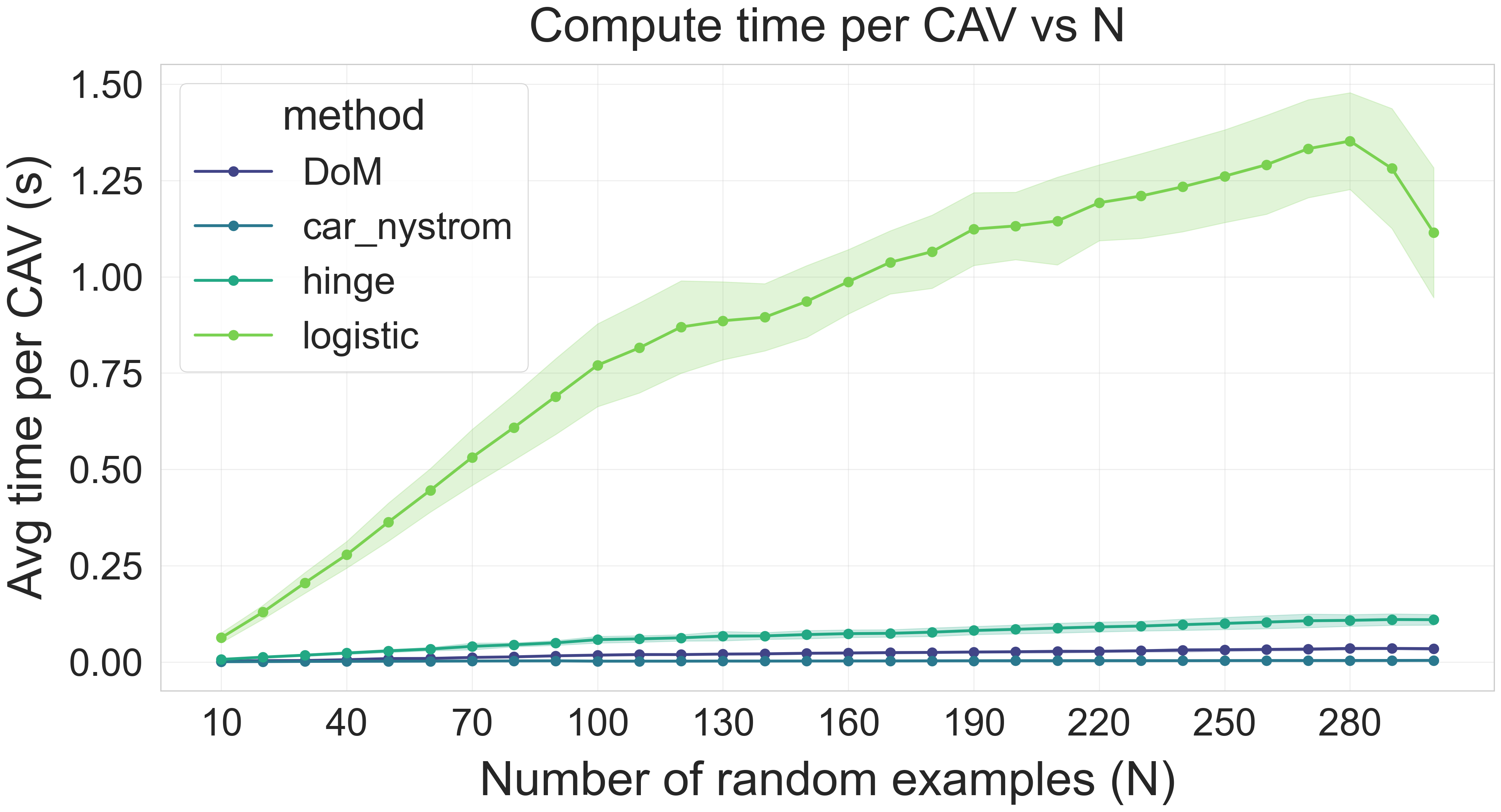}
        \caption{\textbf{Compute vs.\ $N$.} Average wall-clock time per \text{\textsc{Cav}\xspace} as a function of the number of random examples per run ($N$). Shaded band denotes $\pm 1$ SD across runs.}
        \label{fig:compute_vs_n_multimodal}
    \end{subfigure}
    \hfill
    \begin{subfigure}[t]{0.48\textwidth}
        \centering
        \includegraphics[width=\linewidth]{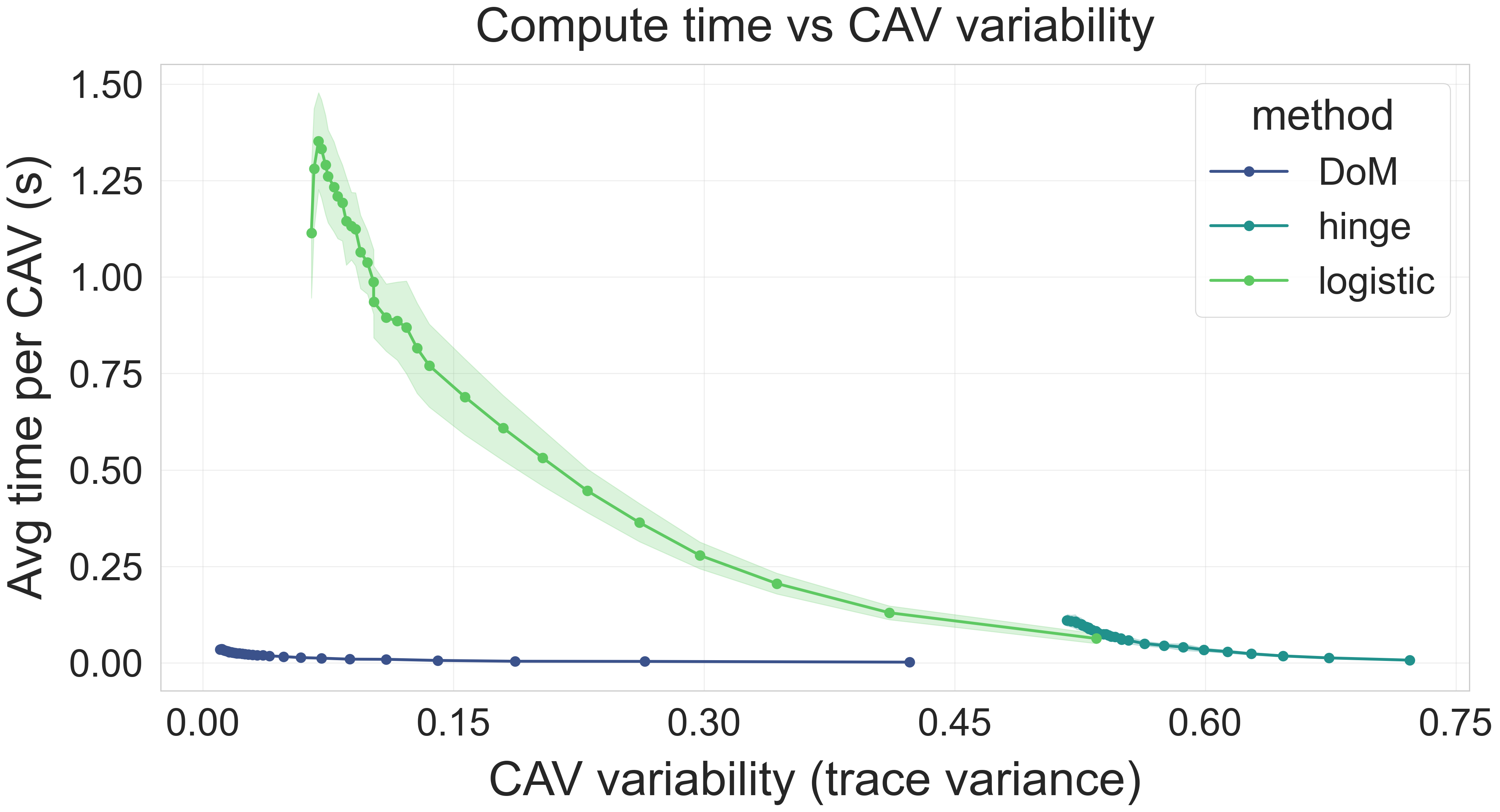}
        \caption{\textbf{Compute vs.\ variability.} Compute-stability trade-off: average time per \text{\textsc{Cav}\xspace} versus \text{\textsc{Cav}\xspace} variability (trace variance). Points correspond to different $N$ values (connected in increasing $N$). Shaded band denotes $\pm 1$ SD for compute time.}
        \label{fig:compute_vs_variability_multimodal}
    \end{subfigure}

    \caption[Compute cost vs.\ stability of CAVs]{\label{fig:compute_tradeoff_multimodal}
    Compute cost and stability of \text{\textsc{Cav}\xspace} generation (see legend for the selected setting, e.g.\ layer/method). Left: compute time per \text{\textsc{Cav}\xspace} vs.\ $N$. Right: compute time vs.\ resulting \text{\textsc{Cav}\xspace} variability (trace variance), illustrating the compute-stability trade-off.}
\end{figure}

\subsection{Concept-Set-Size Variation}
\label{appendix:concept_size_variation}

The main analysis in the paper studies the regime where the number of concept examples is fixed and the number of
random reference examples $N$ increases. This reflects the common \TCAV setting in which annotated concept examples
are scarce, while random reference examples are comparatively easy to sample. In this appendix, we study the complementary
regime: we fix the reference protocol and vary the number of concept examples.

We denote by $M$ the number of concept examples used in a run. For each value of $M$, we subsample concept examples
from the same concept pool, recompute the \CAV, and evaluate the same three quantities as in the main experiments:
\CAV variability, concept-sensitivity variance, and \TCAV-score variance. This experiment isolates the finite-sample effect
of the concept set size. It should not be confused with concept quality: increasing $M$ can reduce sampling variability,
but a larger concept set can still be semantically poor if the examples do not faithfully represent the intended concept.

Across modalities, the results are consistent with the main message of the paper. Increasing the number of concept examples
generally stabilizes the \CAV direction and the continuous concept-sensitivity scores. \TCAV scores are again more sensitive,
because the hard sign threshold can amplify the effect of borderline evaluation points.

\subsubsection{Image Data}
\label{app:concept_size_image}

We first vary the number of concept examples in the image setting. As in the main image experiments, the concepts are
visual concepts such as \emph{striped}, \emph{zigzagged}, and \emph{dotted}. We report results for the same ResNet layers
used throughout the image appendix, namely \texttt{layer2} and \texttt{layer3}.

\begin{figure*}[!ht]
    \centering
    \begin{minipage}{0.48\textwidth}
        \centering
        \includegraphics[width=\linewidth]{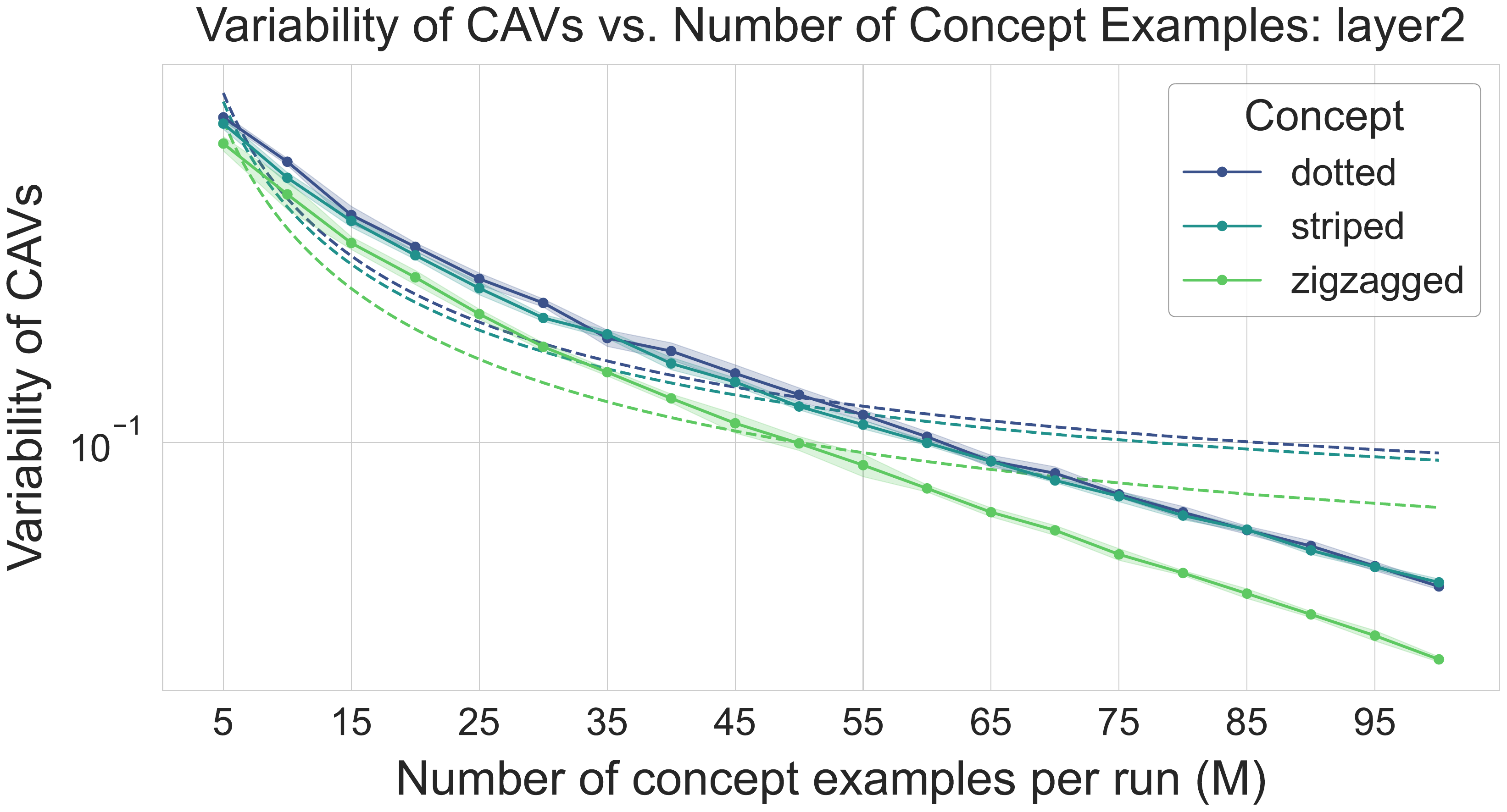}
        \centerline{\small \texttt{layer2}: \CAV variability}
    \end{minipage}
    \hfill
    \begin{minipage}{0.48\textwidth}
        \centering
        \includegraphics[width=\linewidth]{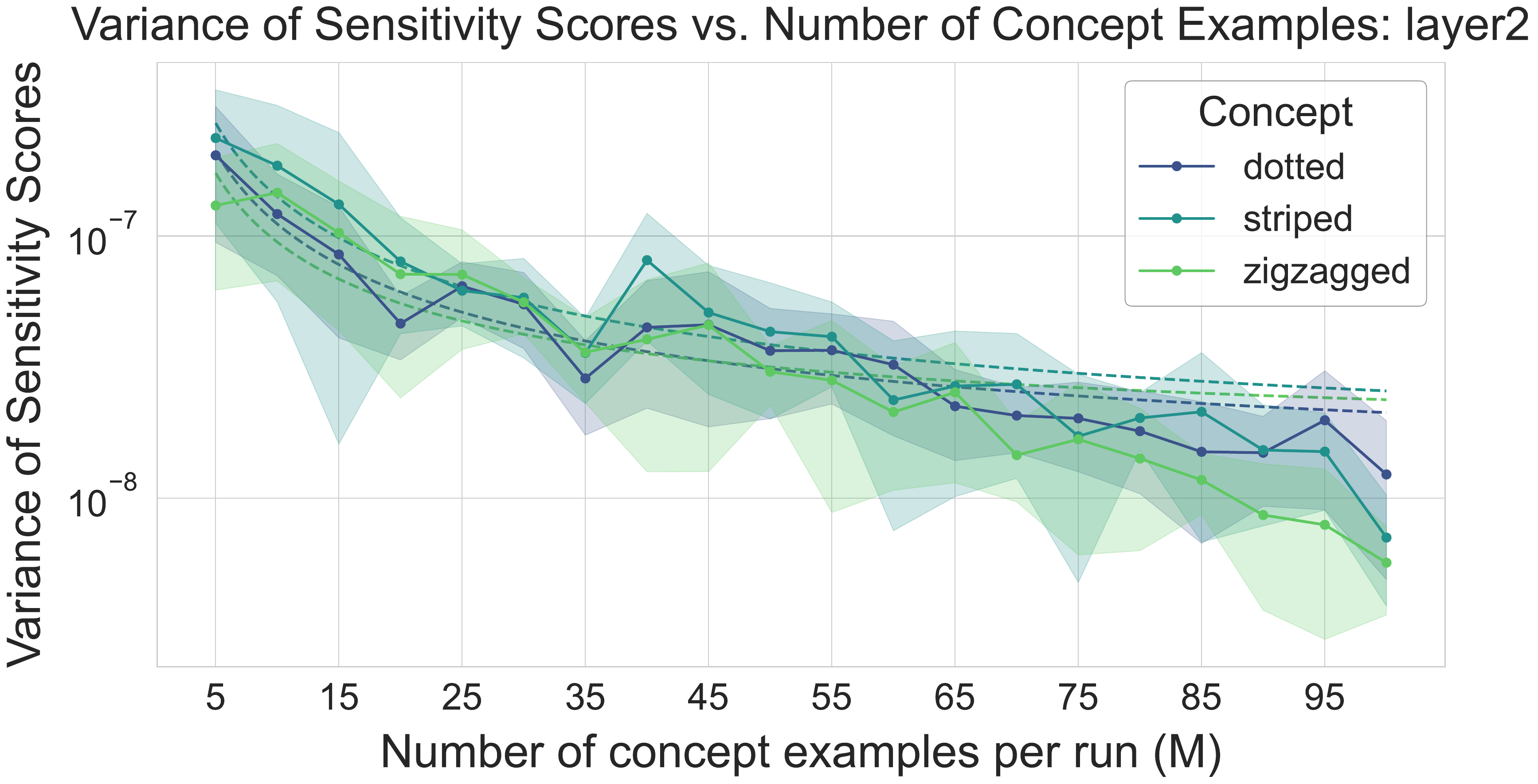}
        \centerline{\small \texttt{layer2}: Sensitivity variance}
    \end{minipage}

    \vspace{0.6em}

    \begin{minipage}{0.55\textwidth}
        \centering
        \includegraphics[width=\linewidth]{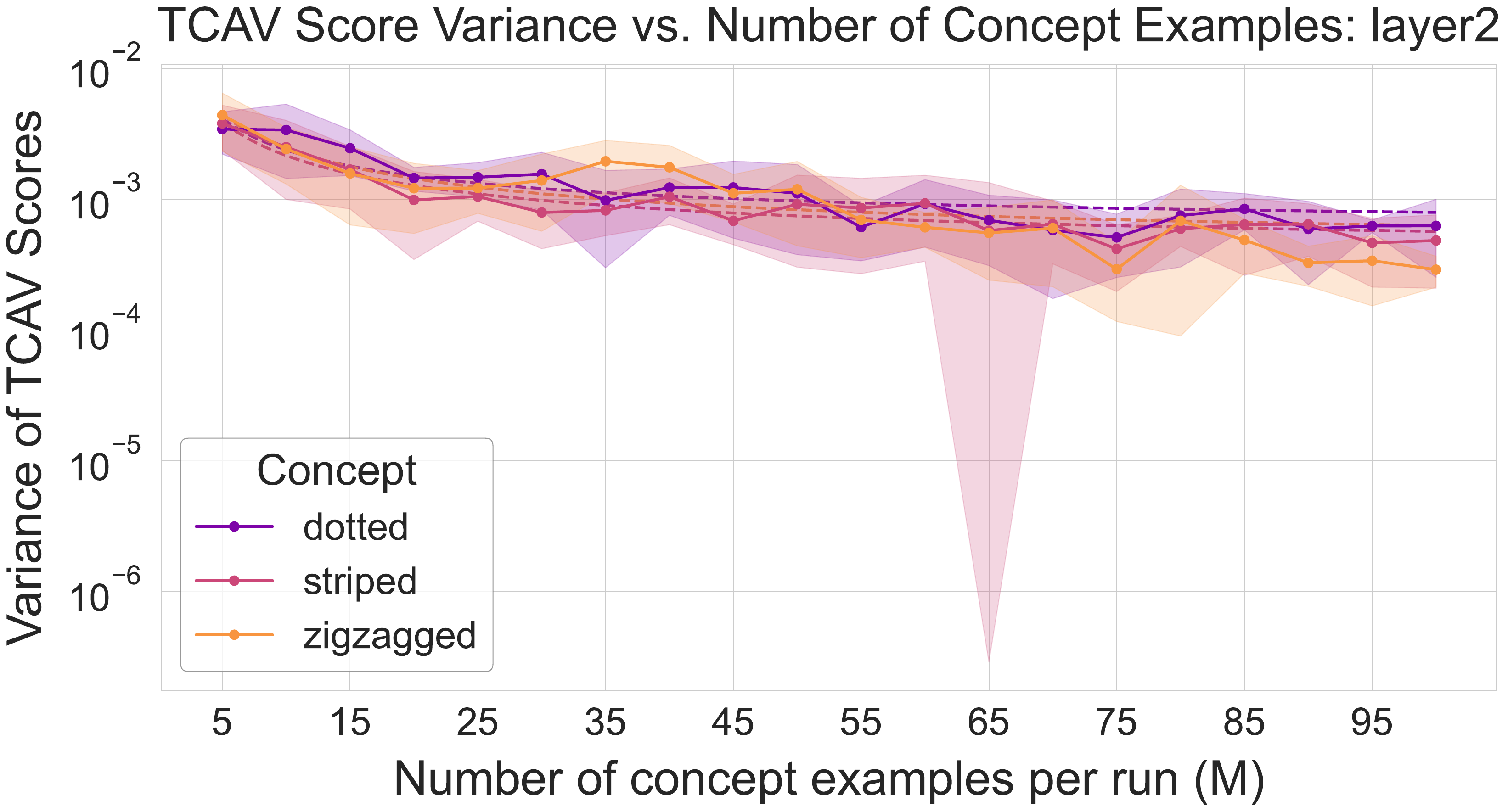}
        \centerline{\small \texttt{layer2}: \TCAV score variance}
    \end{minipage}

    \caption{
    \textbf{Concept-set-size ablation for image data at ResNet \texttt{layer2}.}
    We vary the number $M$ of concept examples while keeping the reference protocol fixed. \CAV variability and
concept-sensitivity variance decrease as more concept examples are used. \TCAV-score variance is less smooth because
it depends on the sign of the sensitivity scores.
    }
    \label{fig:concept_size_image_layer2}
\end{figure*}

\begin{figure*}[!ht]
    \centering
    \begin{minipage}{0.48\textwidth}
        \centering
        \includegraphics[width=\linewidth]{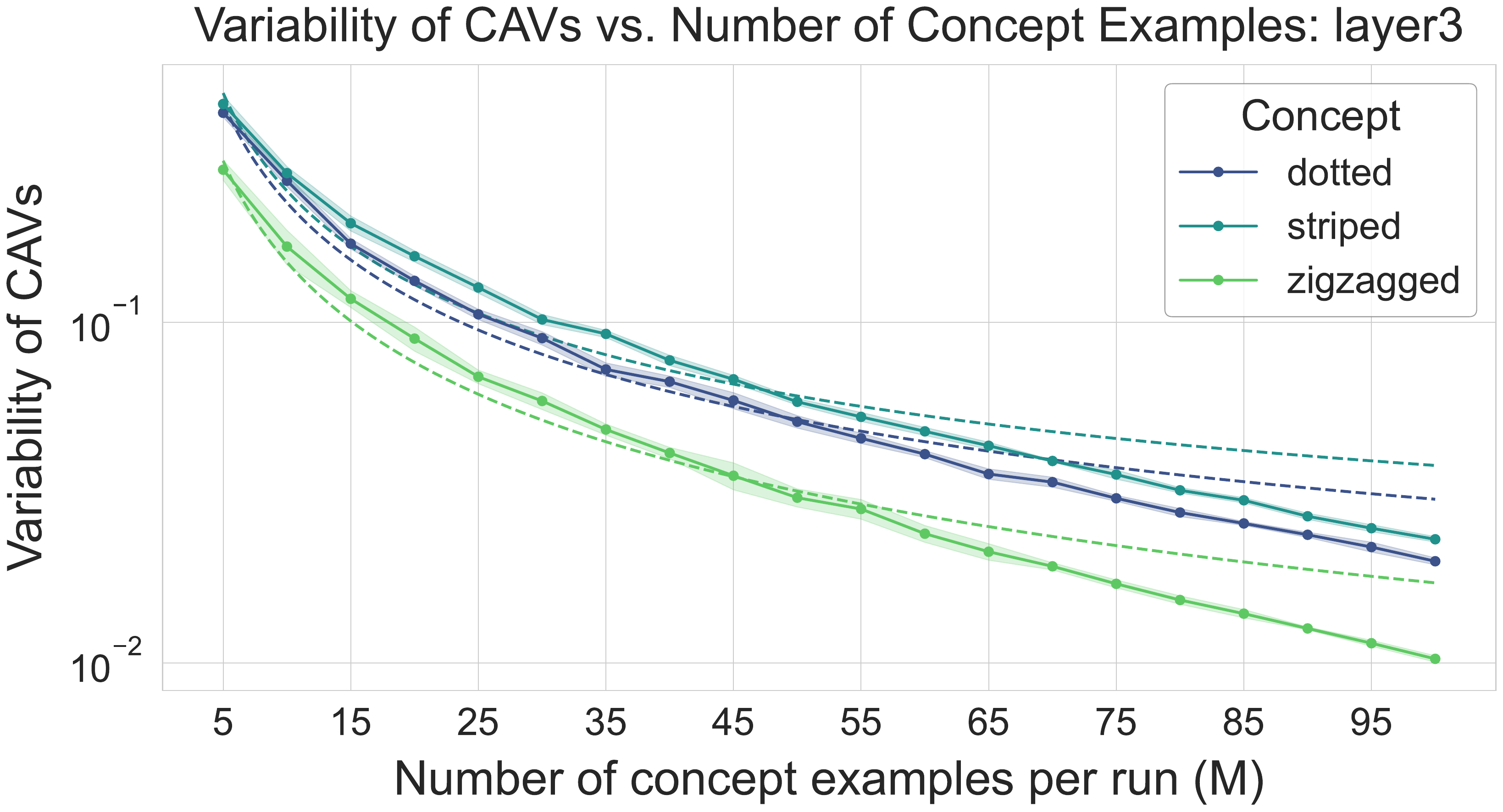}
        \centerline{\small \texttt{layer3}: \CAV variability}
    \end{minipage}
    \hfill
    \begin{minipage}{0.48\textwidth}
        \centering
        \includegraphics[width=\linewidth]{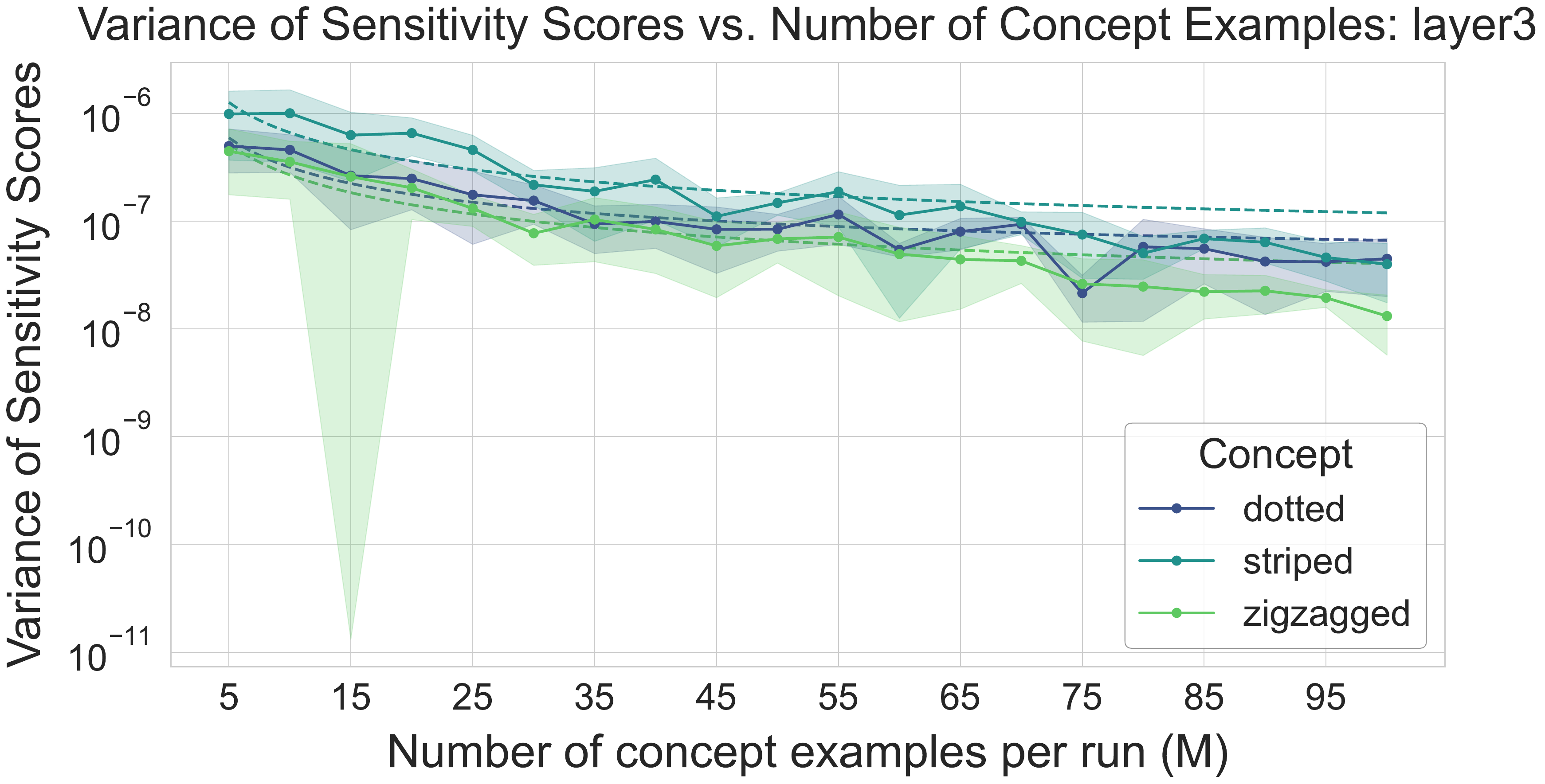}
        \centerline{\small \texttt{layer3}: Sensitivity variance}
    \end{minipage}

    \vspace{0.6em}

    \begin{minipage}{0.55\textwidth}
        \centering
        \includegraphics[width=\linewidth]{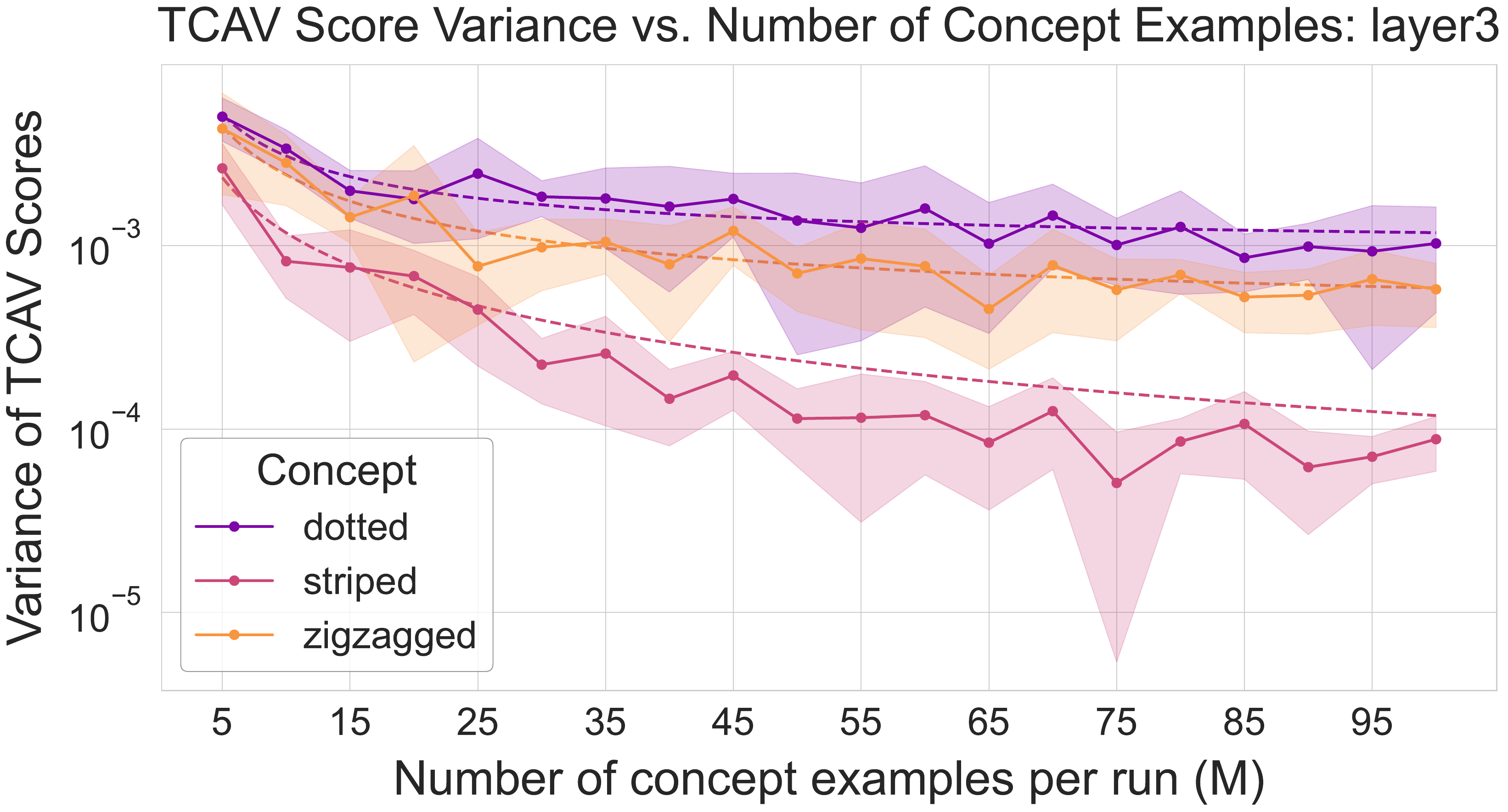}
        \centerline{\small \texttt{layer3}: \TCAV score variance}
    \end{minipage}

    \caption{
    \textbf{Concept-set-size ablation for image data at ResNet \texttt{layer3}.}
    At \texttt{layer3}, increasing the number of concept examples again stabilizes the \CAV direction.
    }
    \label{fig:concept_size_image_layer3}
\end{figure*}

\clearpage

\subsubsection{Tabular Data}
\label{app:concept_size_tabular}

We next consider the UCI Adult tabular setting. The concepts are derived from the gender attribute, namely
\texttt{sex=Female} and \texttt{sex=Male}. We report results for the tabular layers that also appear in the other tabular
experiments: the input projection layer \texttt{fc\_in} and the hidden layer \texttt{hidden\_layers.0}.

\begin{figure*}[!ht]
    \centering
    \begin{minipage}{0.48\textwidth}
        \centering
        \includegraphics[width=\linewidth]{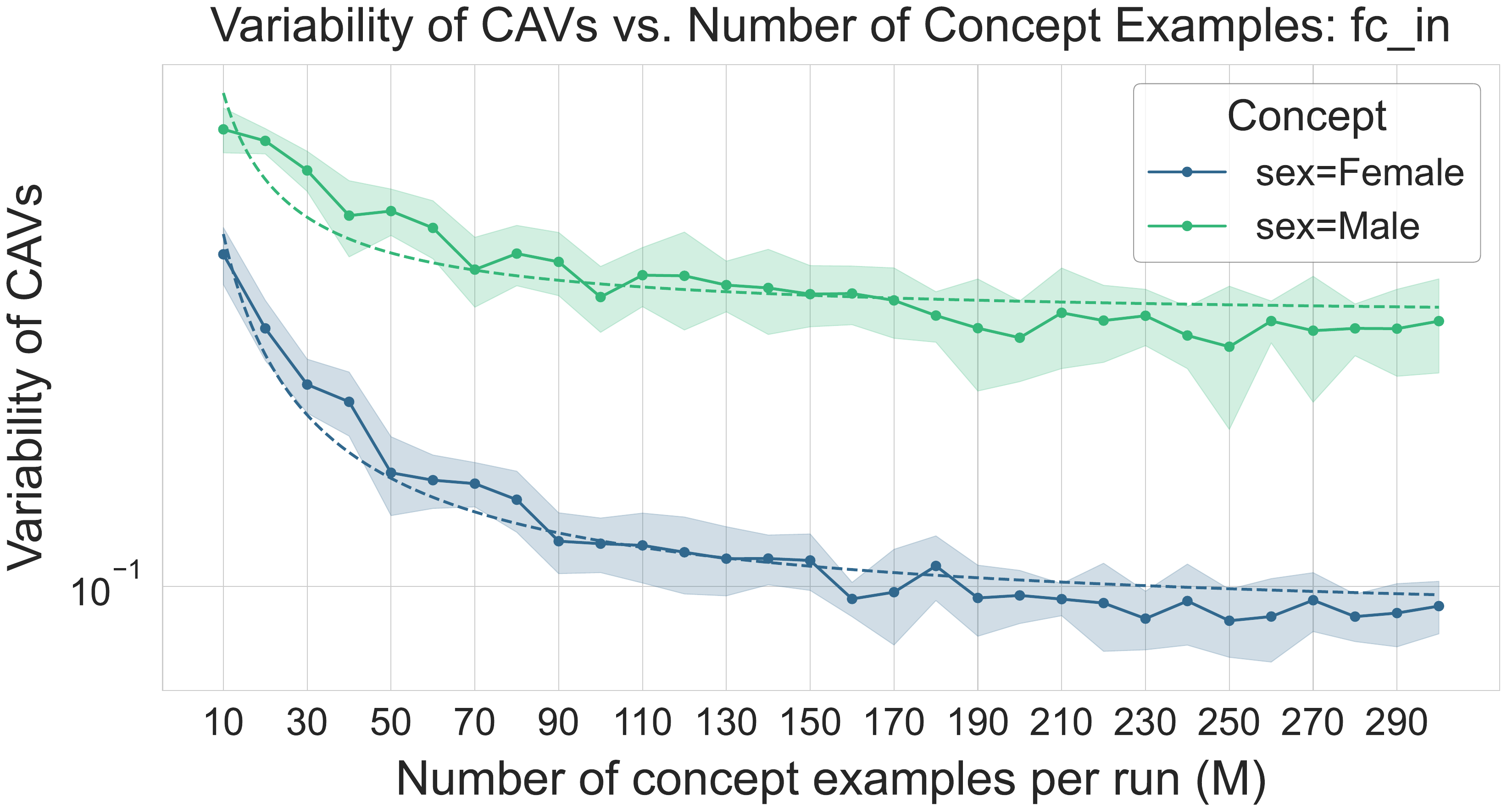}
        \centerline{\small \texttt{fc\_in}: \CAV variability}
    \end{minipage}
    \hfill
    \begin{minipage}{0.48\textwidth}
        \centering
        \includegraphics[width=\linewidth]{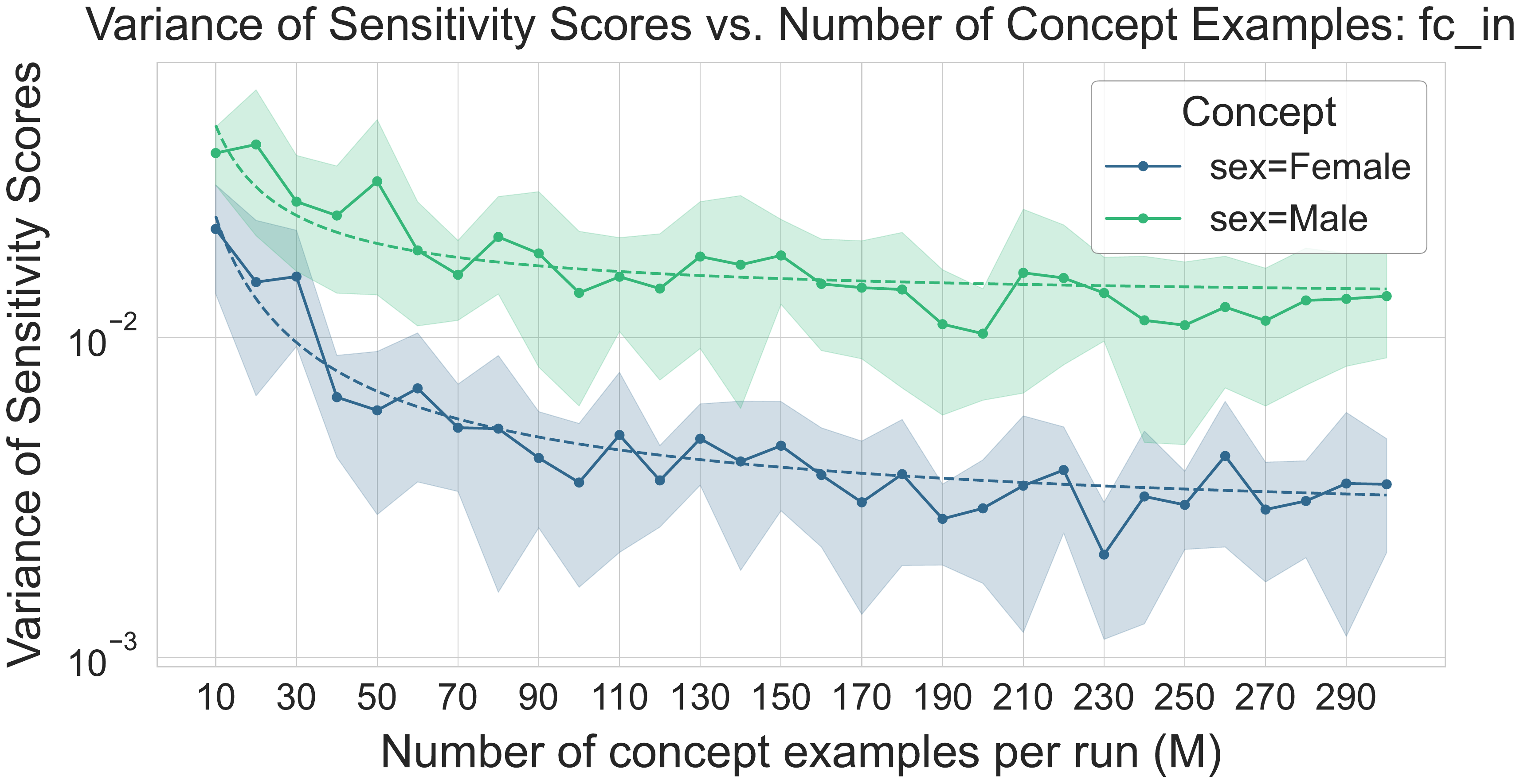}
        \centerline{\small \texttt{fc\_in}: Sensitivity variance}
    \end{minipage}

    \vspace{0.6em}

    \begin{minipage}{0.55\textwidth}
        \centering
        \includegraphics[width=\linewidth]{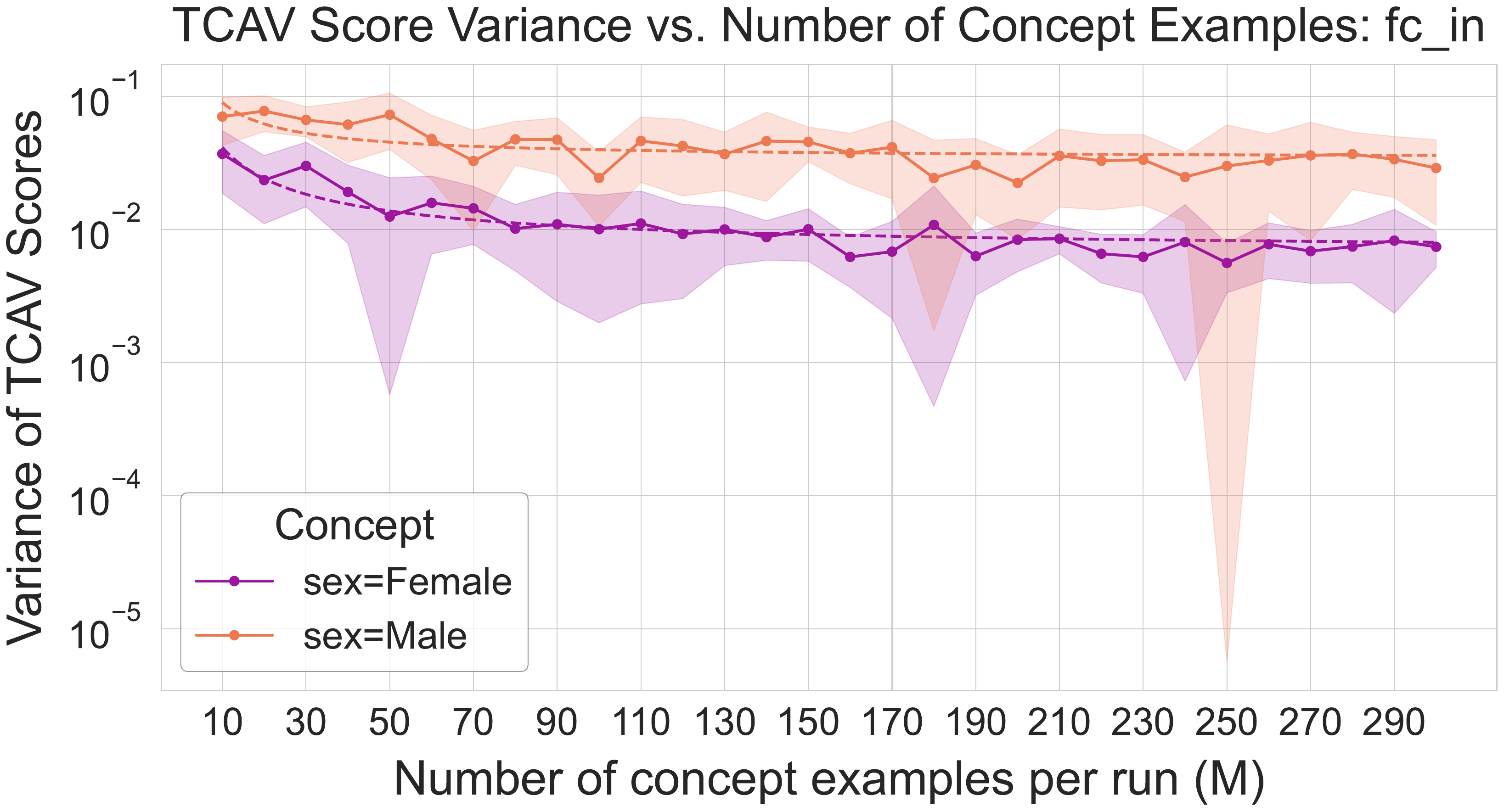}
        \centerline{\small \texttt{fc\_in}: \TCAV score variance}
    \end{minipage}

    \caption{
    \textbf{Concept-set-size ablation for tabular data at \texttt{fc\_in}.}
    For the input projection layer, increasing $M$ reduces the variability of the learned concept direction. Since the tabular
concepts are defined directly from an interpretable attribute, this setting isolates the effect of concept-set size particularly
clearly.
    }
    \label{fig:concept_size_tabular_fc_in}
\end{figure*}

\begin{figure*}[!ht]
    \centering
    \begin{minipage}{0.48\textwidth}
        \centering
        \includegraphics[width=\linewidth]{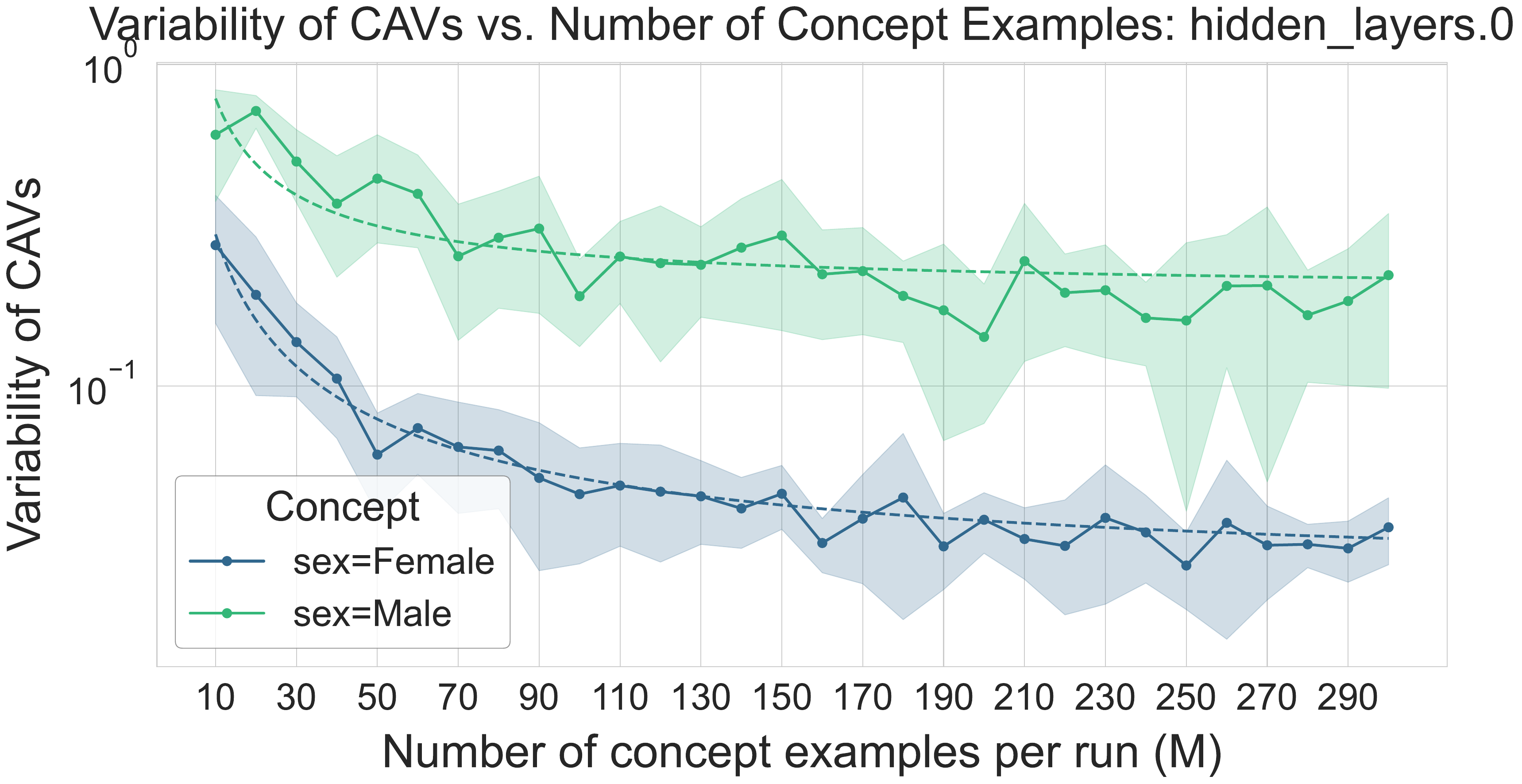}
        \centerline{\small \texttt{hidden\_layers.0}: \CAV variability}
    \end{minipage}
    \hfill
    \begin{minipage}{0.48\textwidth}
        \centering
        \includegraphics[width=\linewidth]{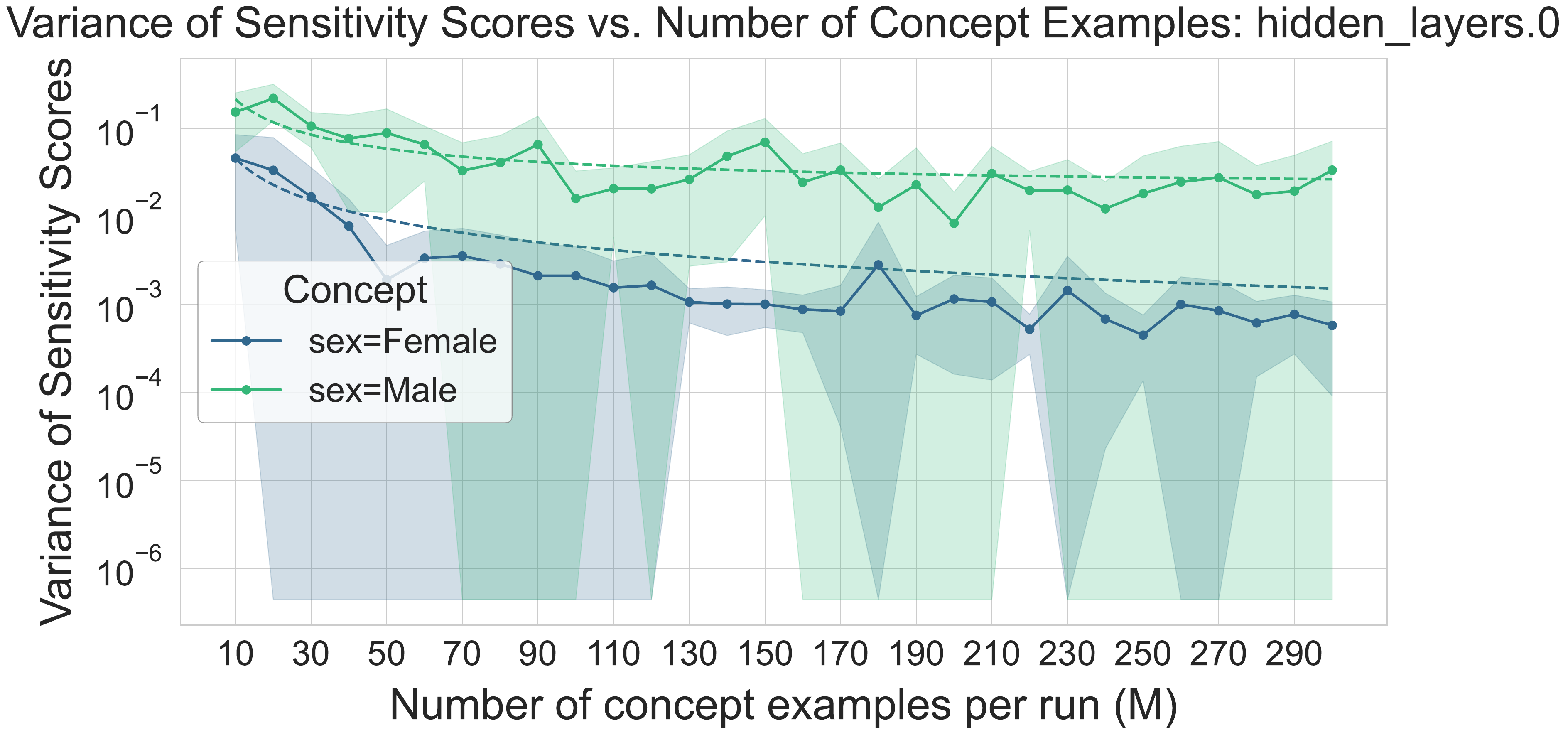}
        \centerline{\small \texttt{hidden\_layers.0}: Sensitivity variance}
    \end{minipage}

    \vspace{0.6em}

    \begin{minipage}{0.55\textwidth}
        \centering
        \includegraphics[width=\linewidth]{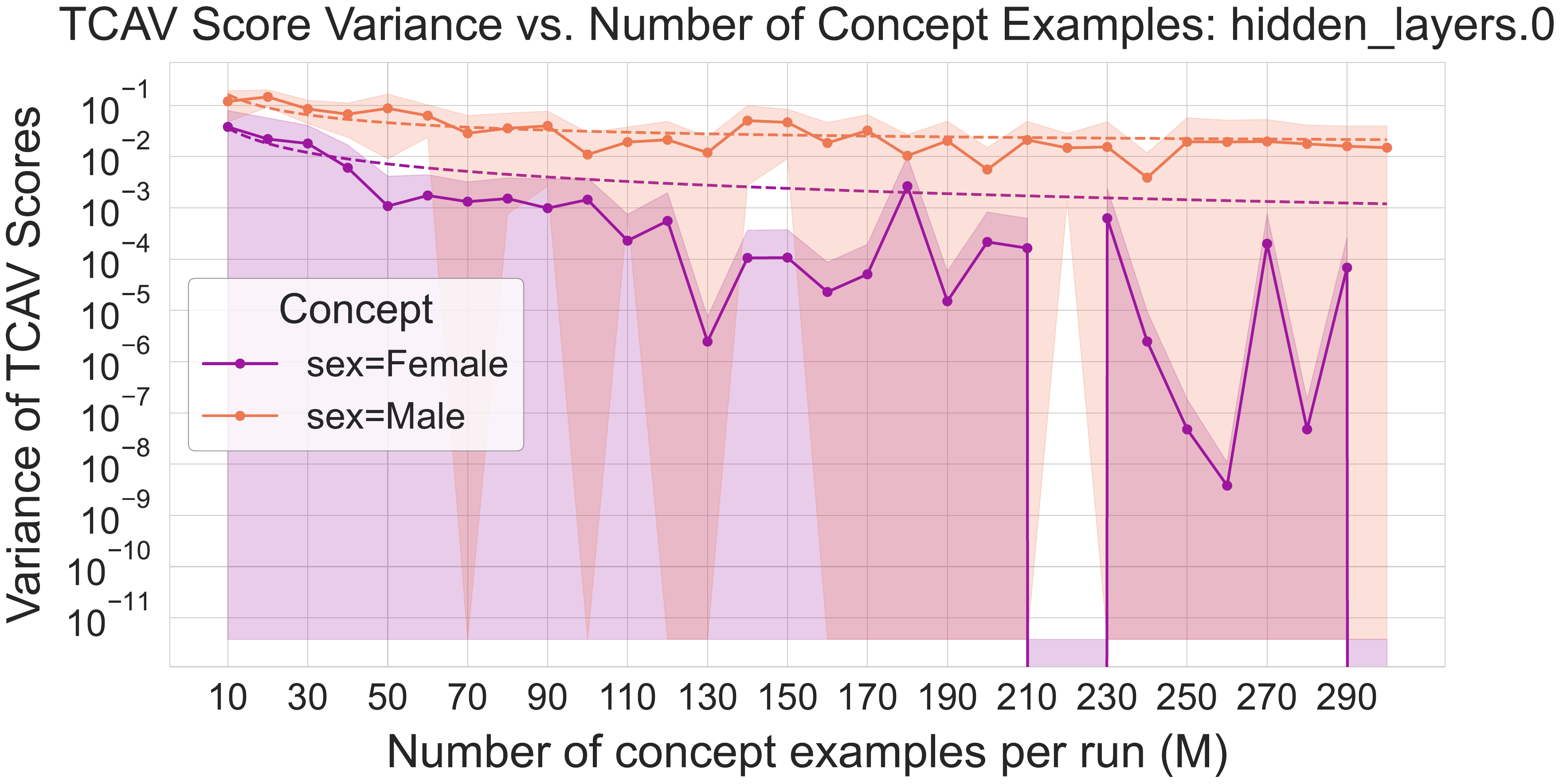}
        \centerline{\small \texttt{hidden\_layers.0}: \TCAV score variance}
    \end{minipage}

    \caption{
    \textbf{Concept-set-size ablation for tabular data at \texttt{hidden\_layers.0}.}
    At the first hidden layer, the same qualitative trend appears. Larger concept sets stabilize the \CAV and the sensitivity
values, whereas the \TCAV score can remain noisy because it aggregates signs rather than magnitudes.
    }
    \label{fig:concept_size_tabular_hidden_layers_0}
\end{figure*}

\clearpage

\subsubsection{Text Data}
\label{app:concept_size_text}

For text data, we use the IMDB sentiment setting. Concepts are the hand-curated adjective groups used throughout the
paper: \emph{positive-adjectives}, \emph{negative-adjectives}, and \emph{neutral-adjectives}. We report the same BERT
encoder layers as in the text appendix, namely \texttt{bert\_encoder\_2} and \texttt{bert\_encoder\_4}.

\begin{figure*}[!ht]
    \centering
    \begin{minipage}{0.48\textwidth}
        \centering
        \includegraphics[width=\linewidth]{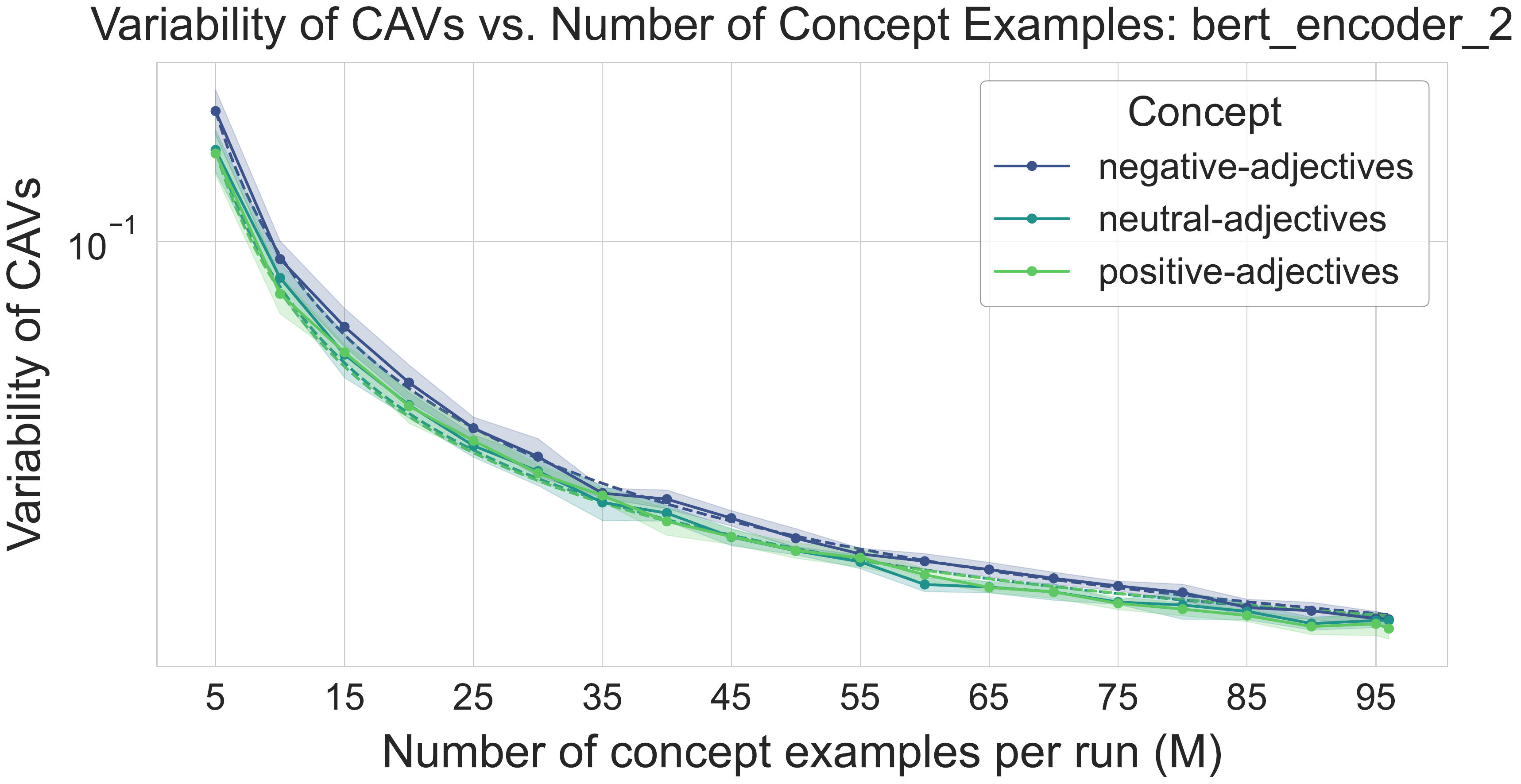}
        \centerline{\small \texttt{bert\_encoder\_2}: \CAV variability}
    \end{minipage}
    \hfill
    \begin{minipage}{0.48\textwidth}
        \centering
        \includegraphics[width=\linewidth]{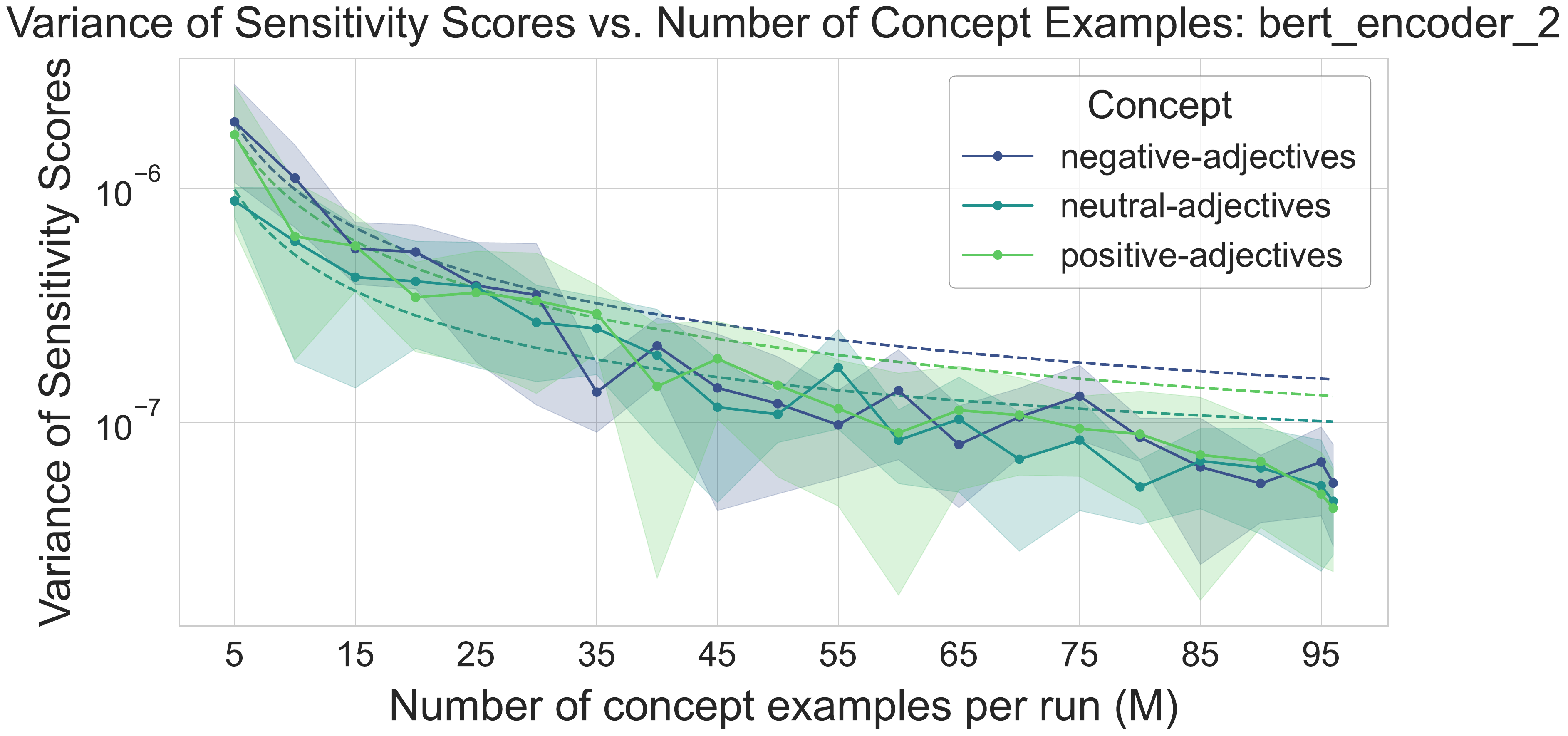}
        \centerline{\small \texttt{bert\_encoder\_2}: Sensitivity variance}
    \end{minipage}

    \vspace{0.6em}

    \begin{minipage}{0.55\textwidth}
        \centering
        \includegraphics[width=\linewidth]{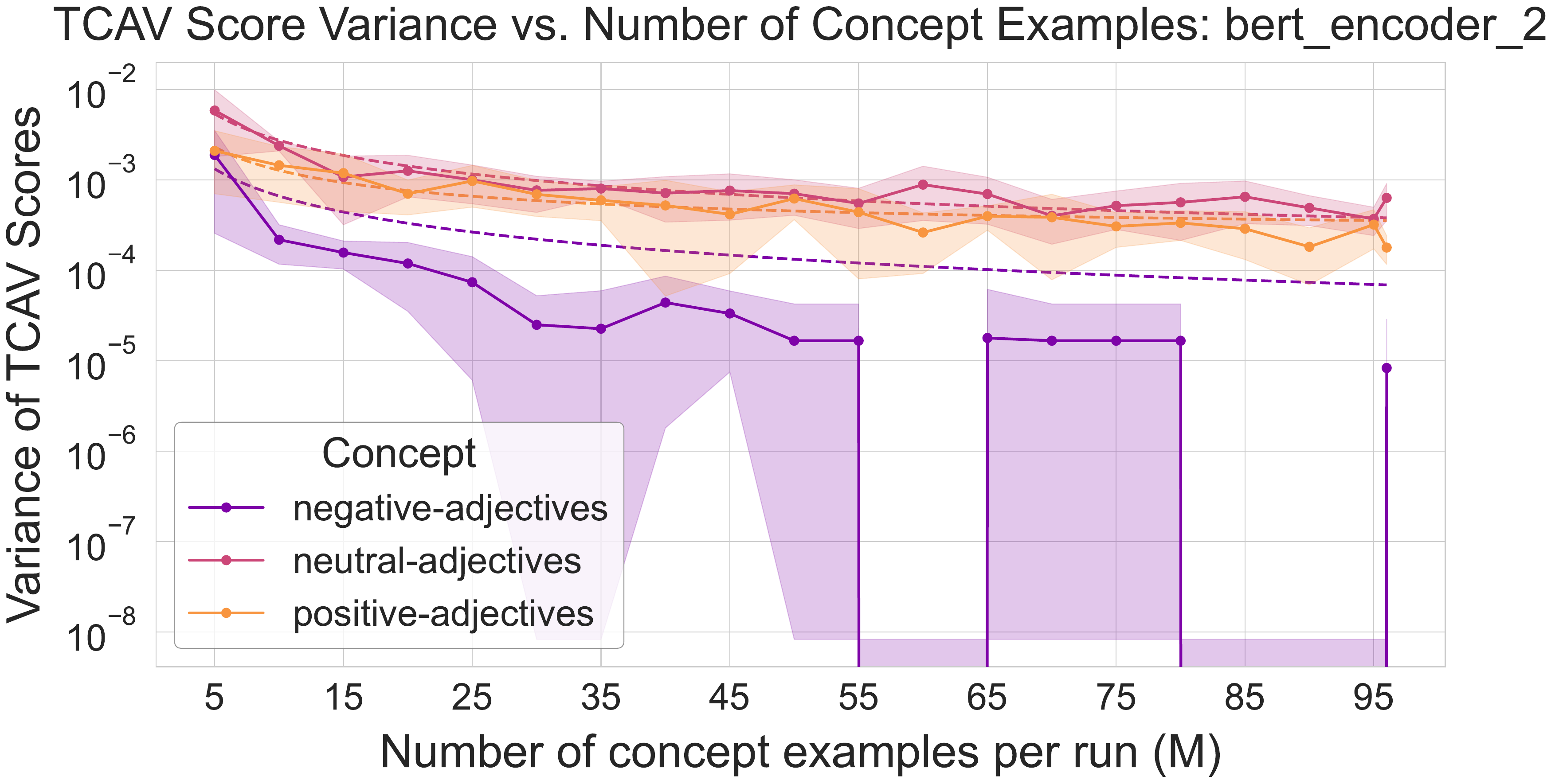}
        \centerline{\small \texttt{bert\_encoder\_2}: \TCAV score variance}
    \end{minipage}

    \caption{
    \textbf{Concept-set-size ablation for text data at \texttt{bert\_encoder\_2}.}
    At \texttt{bert\_encoder\_2}, increasing the number of adjective concept examples stabilizes the \CAV direction. This supports
the interpretation that concept-set sampling contributes to \CAV variability in a way analogous to reference-set sampling.
    }
    \label{fig:concept_size_text_bert_encoder_2}
\end{figure*}

\begin{figure*}[!ht]
    \centering
    \begin{minipage}{0.48\textwidth}
        \centering
        \includegraphics[width=\linewidth]{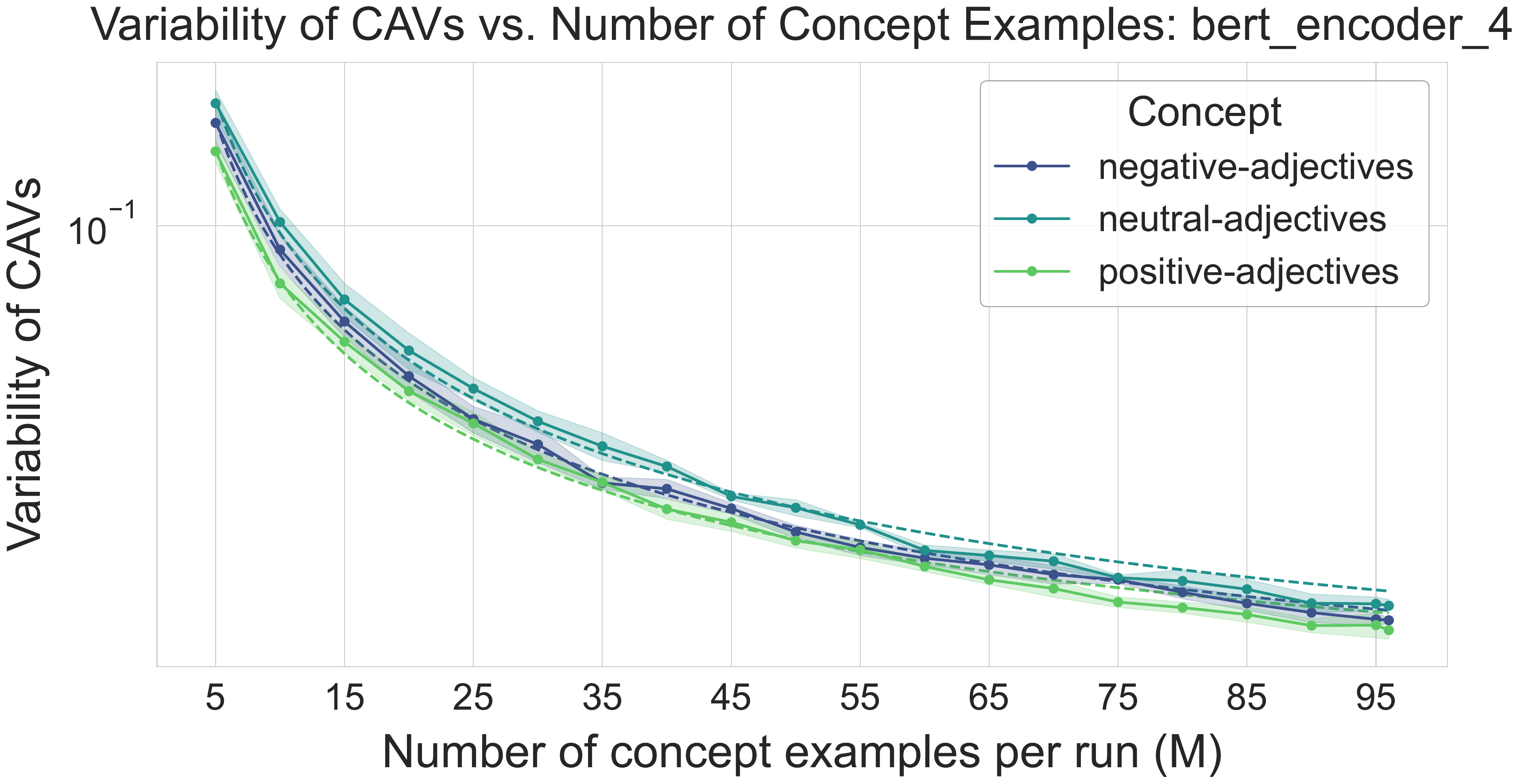}
        \centerline{\small \texttt{bert\_encoder\_4}: \CAV variability}
    \end{minipage}
    \hfill
    \begin{minipage}{0.48\textwidth}
        \centering
        \includegraphics[width=\linewidth]{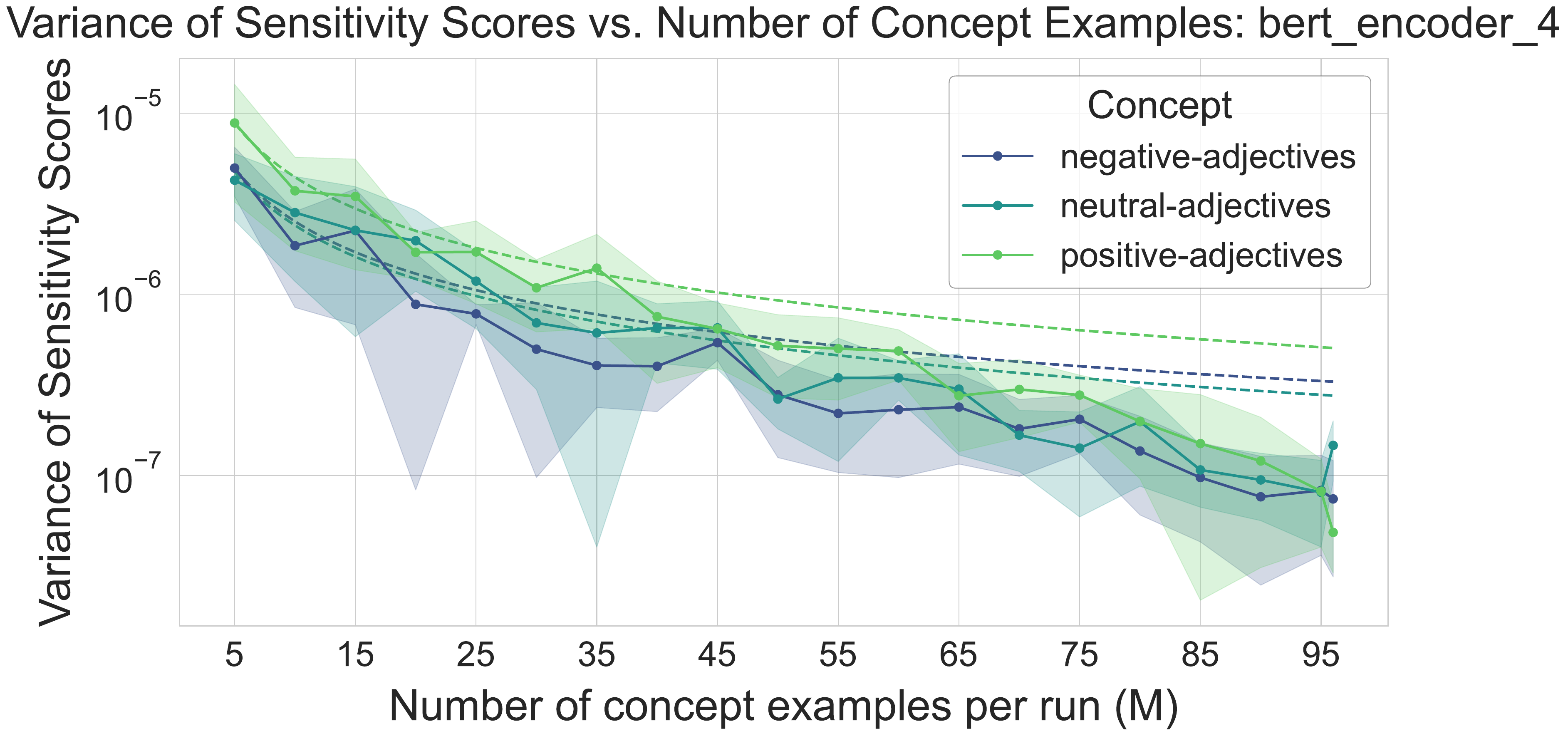}
        \centerline{\small \texttt{bert\_encoder\_4}: Sensitivity variance}
    \end{minipage}

    \vspace{0.6em}

    \begin{minipage}{0.55\textwidth}
        \centering
        \includegraphics[width=\linewidth]{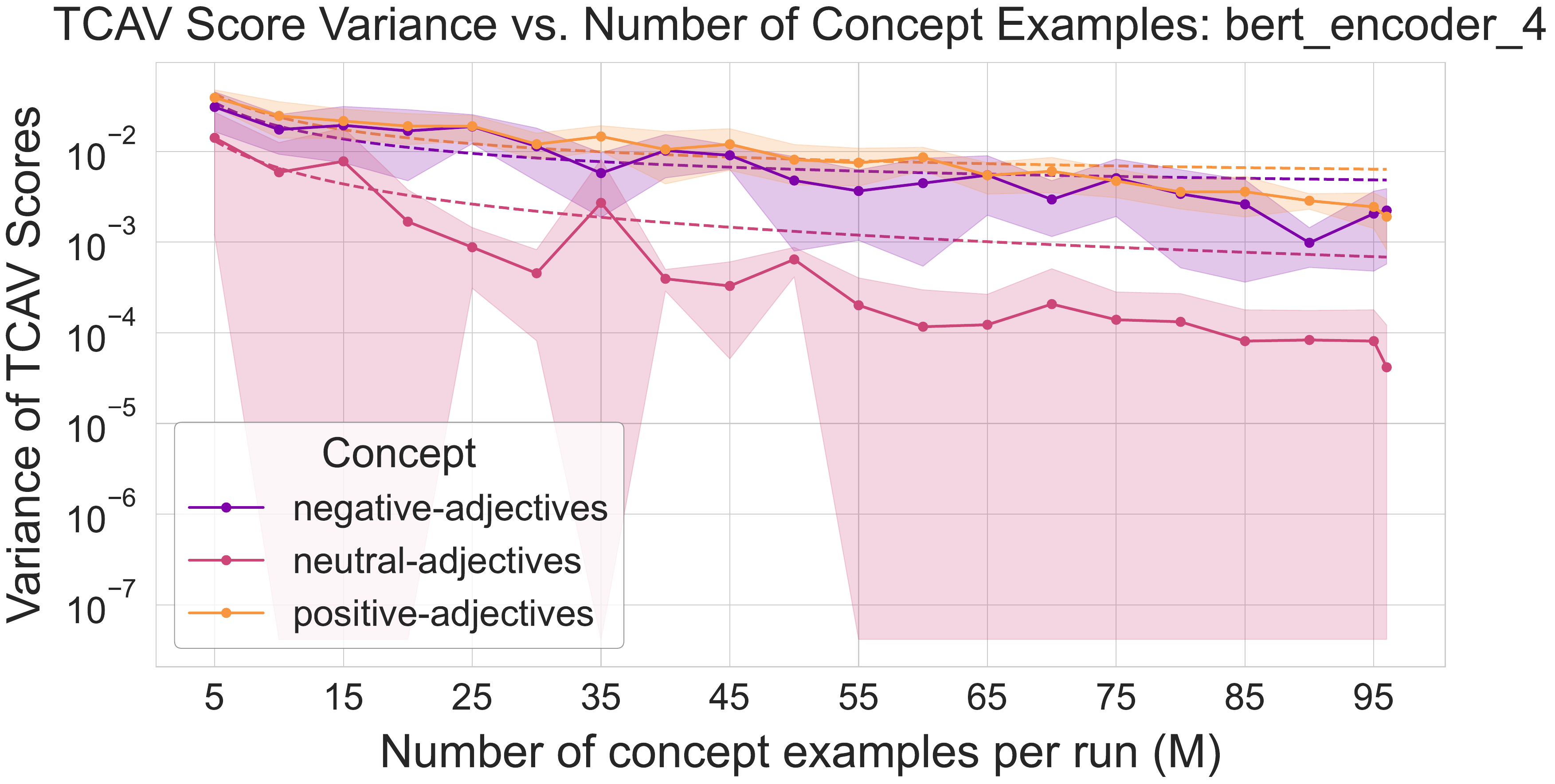}
        \centerline{\small \texttt{bert\_encoder\_4}: \TCAV score variance}
    \end{minipage}

    \caption{
    \textbf{Concept-set-size ablation for text data at \texttt{bert\_encoder\_4}.}
    At \texttt{bert\_encoder\_4}, the same qualitative pattern is visible. The layer-dependent constants differ, but increasing
$M$ reduces the variability of the continuous quantities.
    }
    \label{fig:concept_size_text_bert_encoder_4}
\end{figure*}

\clearpage

\subsubsection{Multimodal Data}
\label{app:concept_size_multimodal}

Finally, we evaluate the effect of concept-set size in the CLIP-based multimodal setting. As in the main multimodal
experiments, the concepts are image sets such as \emph{striped}, \emph{zigzagged}, and \emph{dotted}, and the scalar output
is the image--text similarity to the target prompt. We report the same CLIP vision hidden layers as in the multimodal
appendix, namely \texttt{vision\_hidden\_4} and \texttt{vision\_hidden\_8}.

\begin{figure*}[!ht]
    \centering
    \begin{minipage}{0.48\textwidth}
        \centering
        \includegraphics[width=\linewidth]{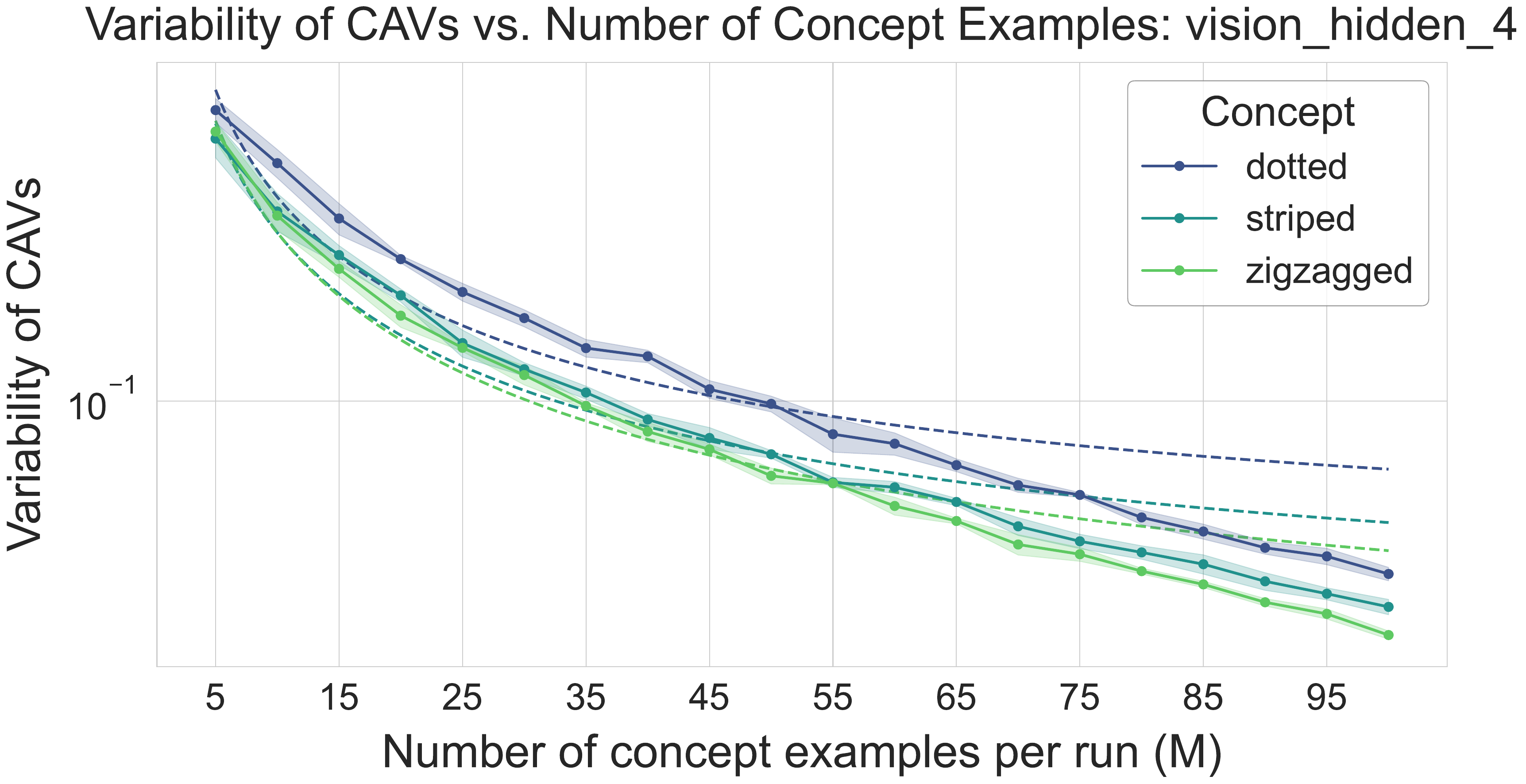}
        \centerline{\small \texttt{vision\_hidden\_4}: \CAV variability}
    \end{minipage}
    \hfill
    \begin{minipage}{0.48\textwidth}
        \centering
        \includegraphics[width=\linewidth]{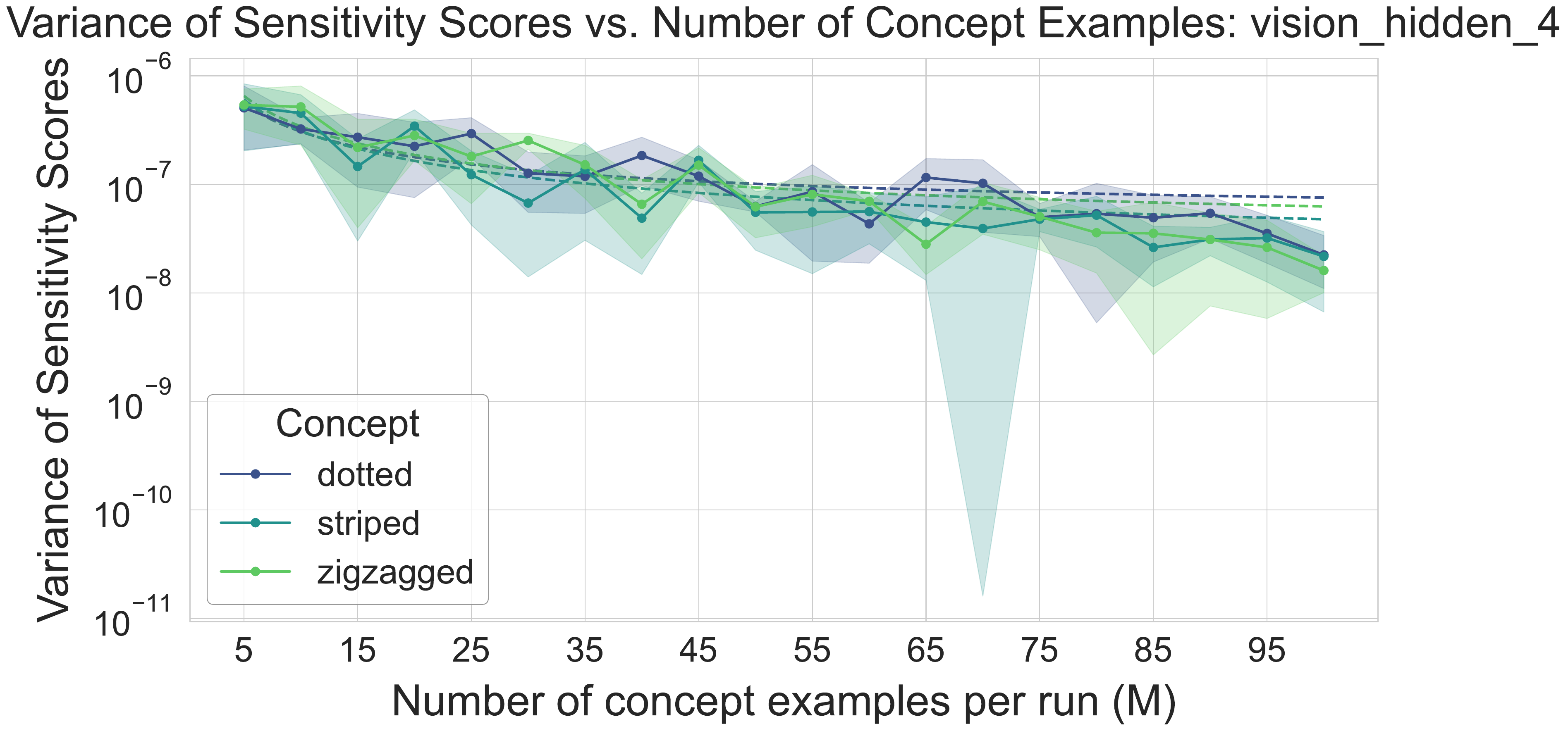}
        \centerline{\small \texttt{vision\_hidden\_4}: Sensitivity variance}
    \end{minipage}

    \vspace{0.6em}

    \begin{minipage}{0.55\textwidth}
        \centering
        \includegraphics[width=\linewidth]{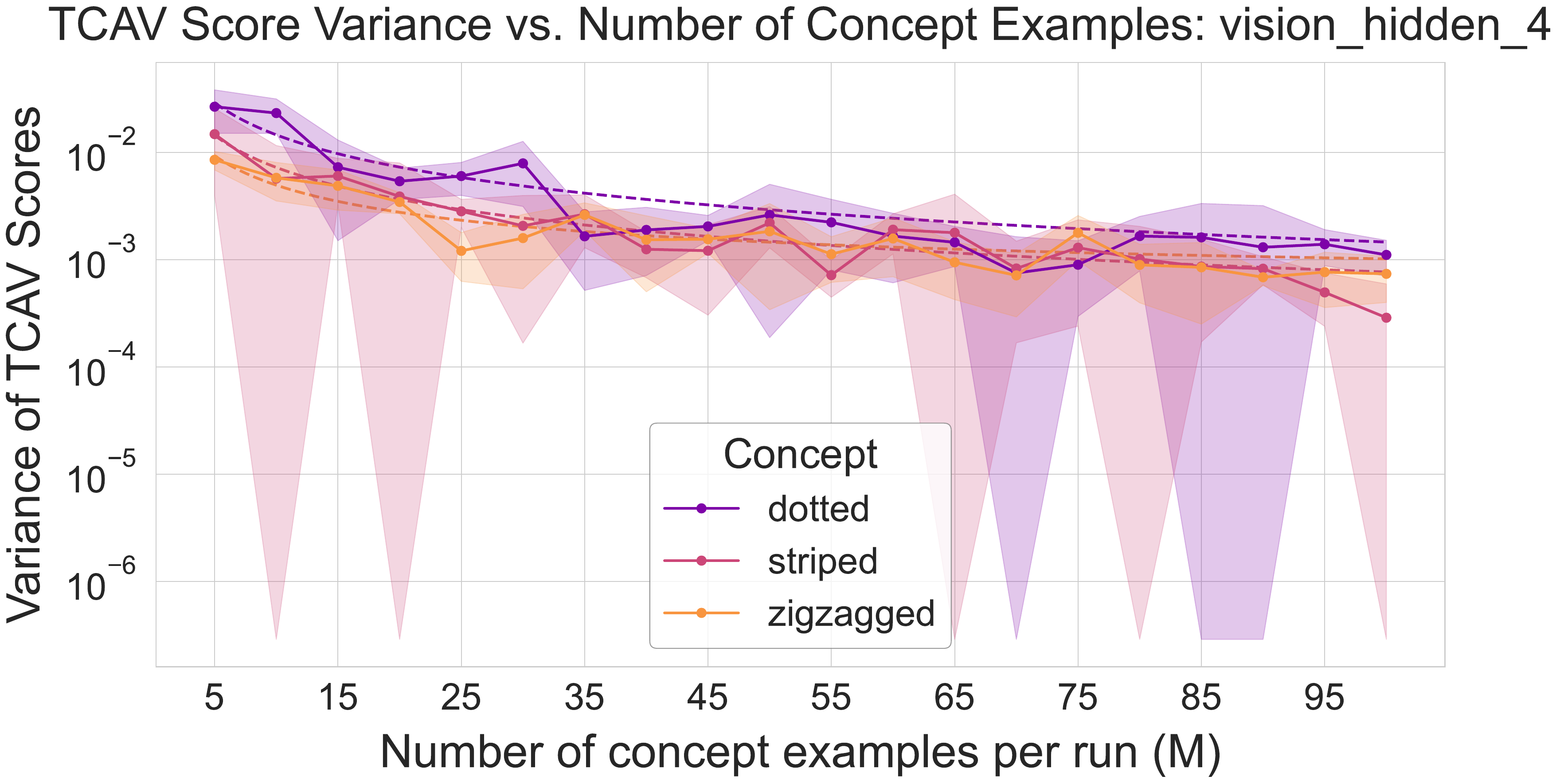}
        \centerline{\small \texttt{vision\_hidden\_4}: \TCAV score variance}
    \end{minipage}

    \caption{
    \textbf{Concept-set-size ablation for multimodal data at \texttt{vision\_hidden\_4}.}
    At CLIP \texttt{vision\_hidden\_4}, larger concept sets reduce \CAV variability and sensitivity variance. This shows that the
concept-set-size effect also appears in the vision--language setting.
    }
    \label{fig:concept_size_multimodal_vision_hidden_4}
\end{figure*}

\begin{figure*}[!ht]
    \centering
    \begin{minipage}{0.48\textwidth}
        \centering
        \includegraphics[width=\linewidth]{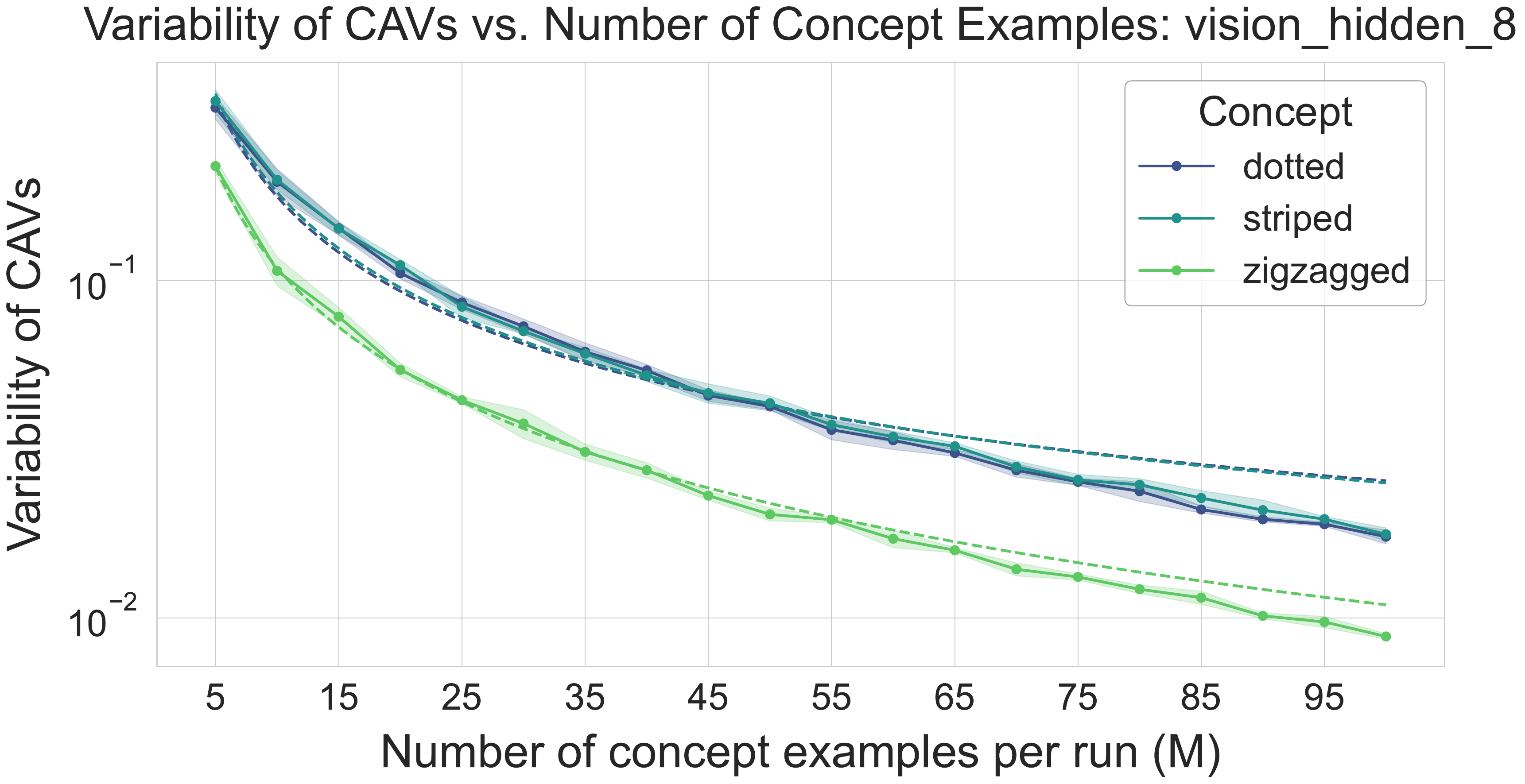}
        \centerline{\small \texttt{vision\_hidden\_8}: \CAV variability}
    \end{minipage}
    \hfill
    \begin{minipage}{0.48\textwidth}
        \centering
        \includegraphics[width=\linewidth]{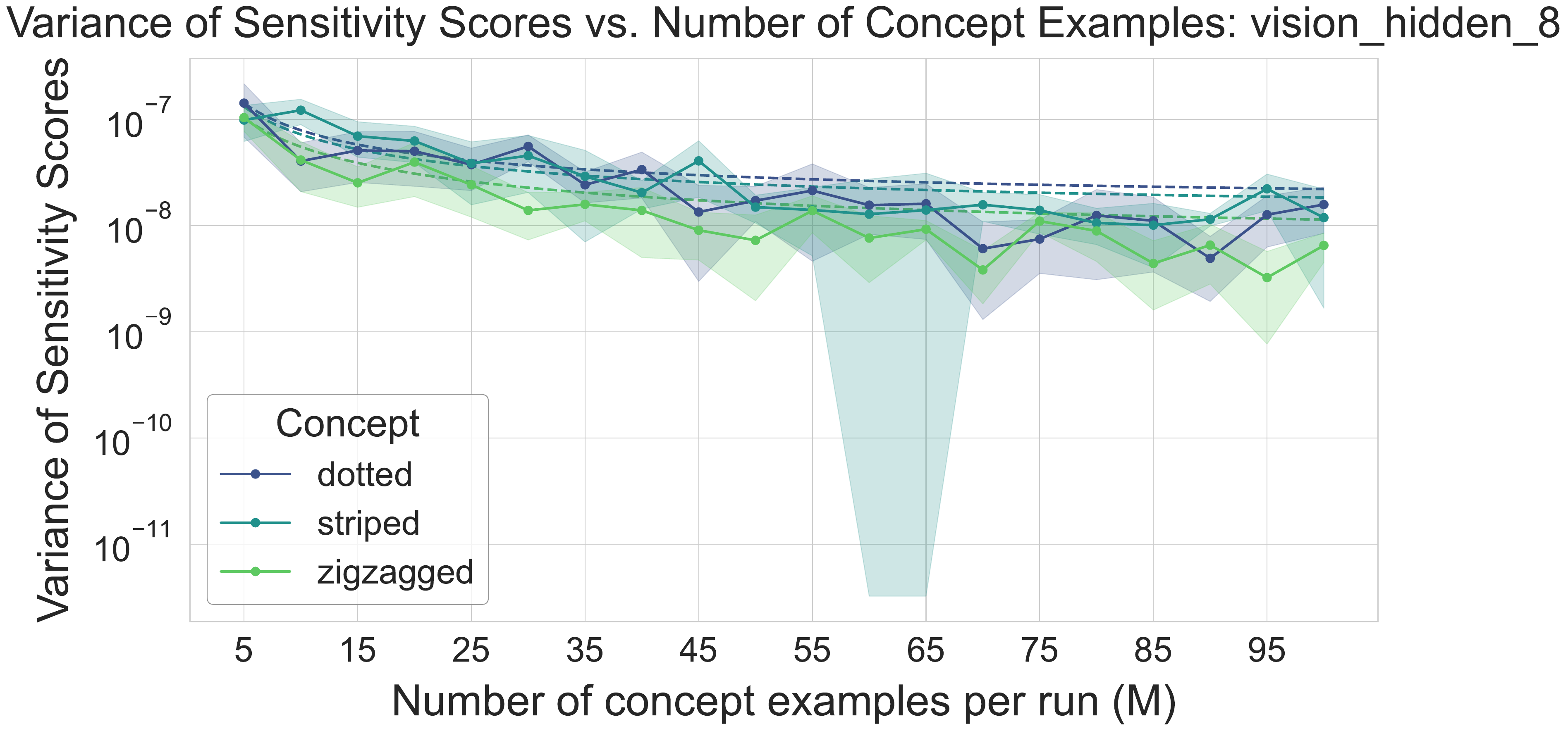}
        \centerline{\small \texttt{vision\_hidden\_8}: Sensitivity variance}
    \end{minipage}

    \vspace{0.6em}

    \begin{minipage}{0.55\textwidth}
        \centering
        \includegraphics[width=\linewidth]{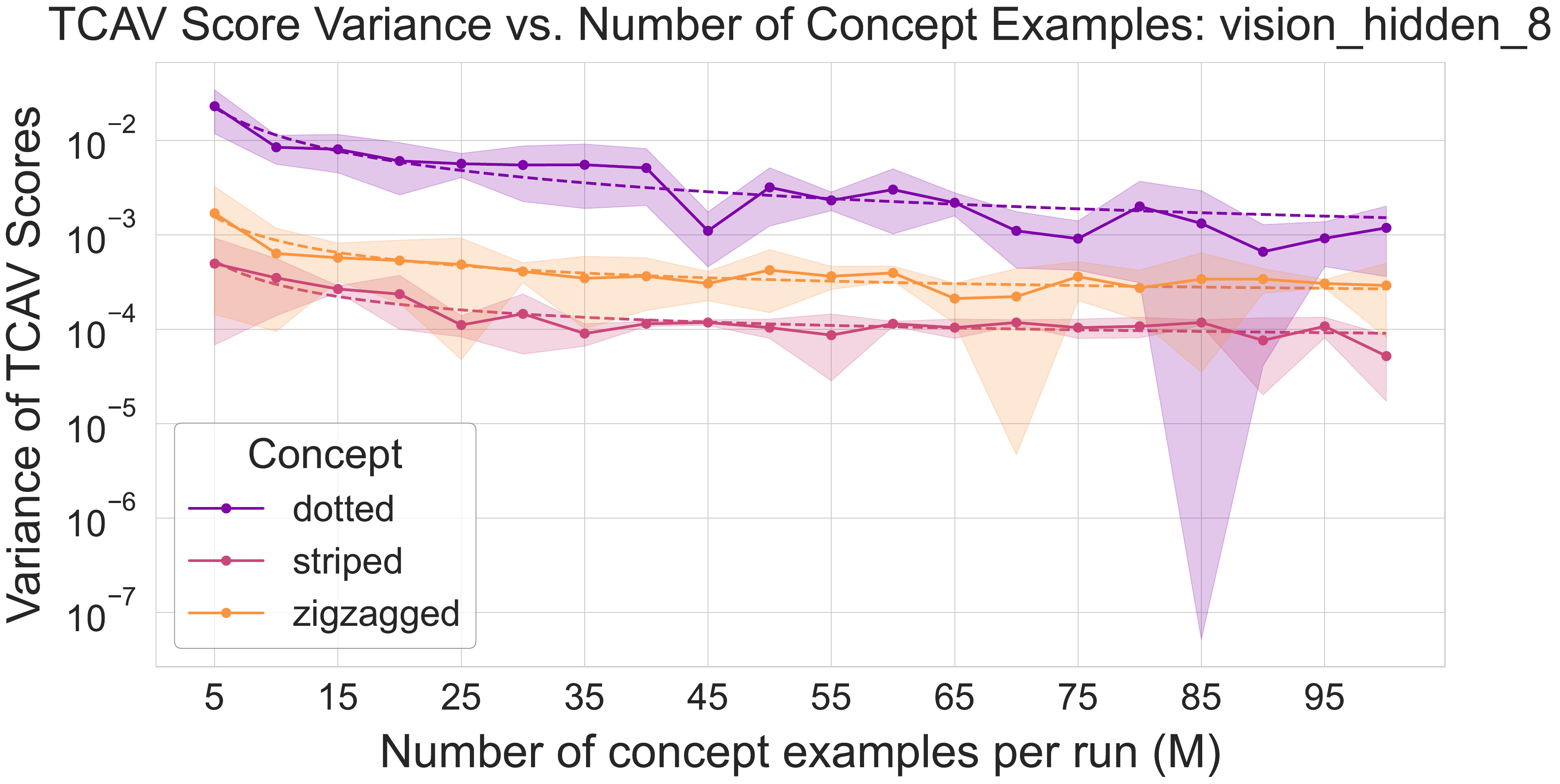}
        \centerline{\small \texttt{vision\_hidden\_8}: \TCAV score variance}
    \end{minipage}

    \caption{
    \textbf{Concept-set-size ablation for multimodal data at \texttt{vision\_hidden\_8}.}
    At CLIP \texttt{vision\_hidden\_8}, the qualitative stabilization trend remains visible. As in the other modalities,
TCAV-score variance is less directly tied to \CAV variance because of the hard threshold in the \TCAV definition.
    }
    \label{fig:concept_size_multimodal_vision_hidden_8}
\end{figure*}

\paragraph{Summary.}
These experiments complement the main fixed-concept-set analysis. Increasing the number $M$ of concept examples tends
to reduce \CAV variability and concept-sensitivity variance across image, tabular, text, and multimodal settings. The effect is
layer- and modality-dependent through the constants, but the qualitative stabilization trend is robust. \TCAV-score variance
is again less predictable because it depends on the sign of the sensitivity, so near-borderline points can remain unstable even
when the underlying \CAV direction becomes more stable.

\subsection{Implementation Ablations}
\label{appendix:implementation_ablations}

In this section, we investigate whether the empirical stability trends reported in the main paper are sensitive to common
implementation choices. We focus on the tabular setting, where the model architecture and training choices can be varied
in a controlled way. Unless stated otherwise, we follow the protocol of Section~5: for each number of random reference
examples $N$, we fit multiple CAVs using independently sampled reference sets and measure the resulting \CAV variability,
concept-sensitivity variance, and \TCAV-score variance.

Overall, the ablations confirm the qualitative conclusions of the main paper. Increasing $N$ stabilizes the \CAV direction
and the continuous concept-sensitivity scores across a range of implementation choices. These choices mainly affect the
constants and the finite-sample regime, rather than changing the qualitative dependence on $N$. In contrast, the
hard-thresholded \TCAV score remains more sensitive to borderline points and can therefore decay more slowly or plateau.

\subsubsection{Layer Ablation}
\label{app:layer_ablation}

We first vary the layer at which the \CAV is computed. Figure~\ref{fig:layer_ablation} compares the input projection layer
\texttt{fc\_in} with two hidden layers, \texttt{hidden\_layers.0} and \texttt{hidden\_layers.2}. The \CAV variability decreases
with $N$ across all layers, but the constants differ substantially. This shows that the representation geometry of the chosen
layer affects finite-sample stability.

The sensitivity variance also decreases with $N$, again with layer-dependent constants. The \TCAV score variance is more
layer-dependent than the \CAV direction itself: even when the \CAV direction stabilizes, hard-thresholded \TCAV scores can
remain noisy if the corresponding layer contains many borderline evaluation points.

\begin{figure*}[!ht]
    \centering
    \begin{minipage}{0.48\textwidth}
        \centering
        \includegraphics[width=\linewidth]{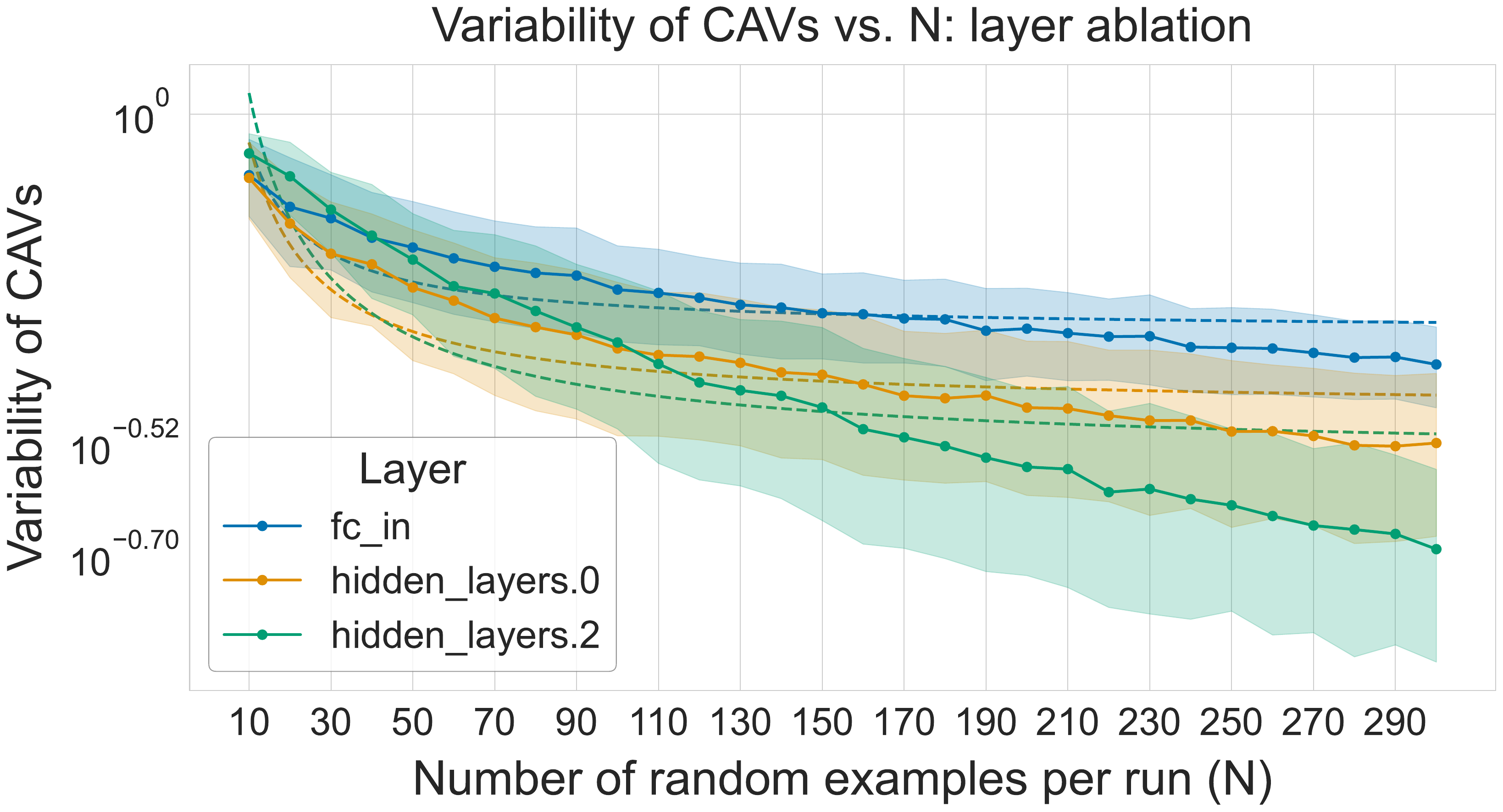}
        \centerline{\small \CAV variability}
    \end{minipage}
    \hfill
    \begin{minipage}{0.48\textwidth}
        \centering
        \includegraphics[width=\linewidth]{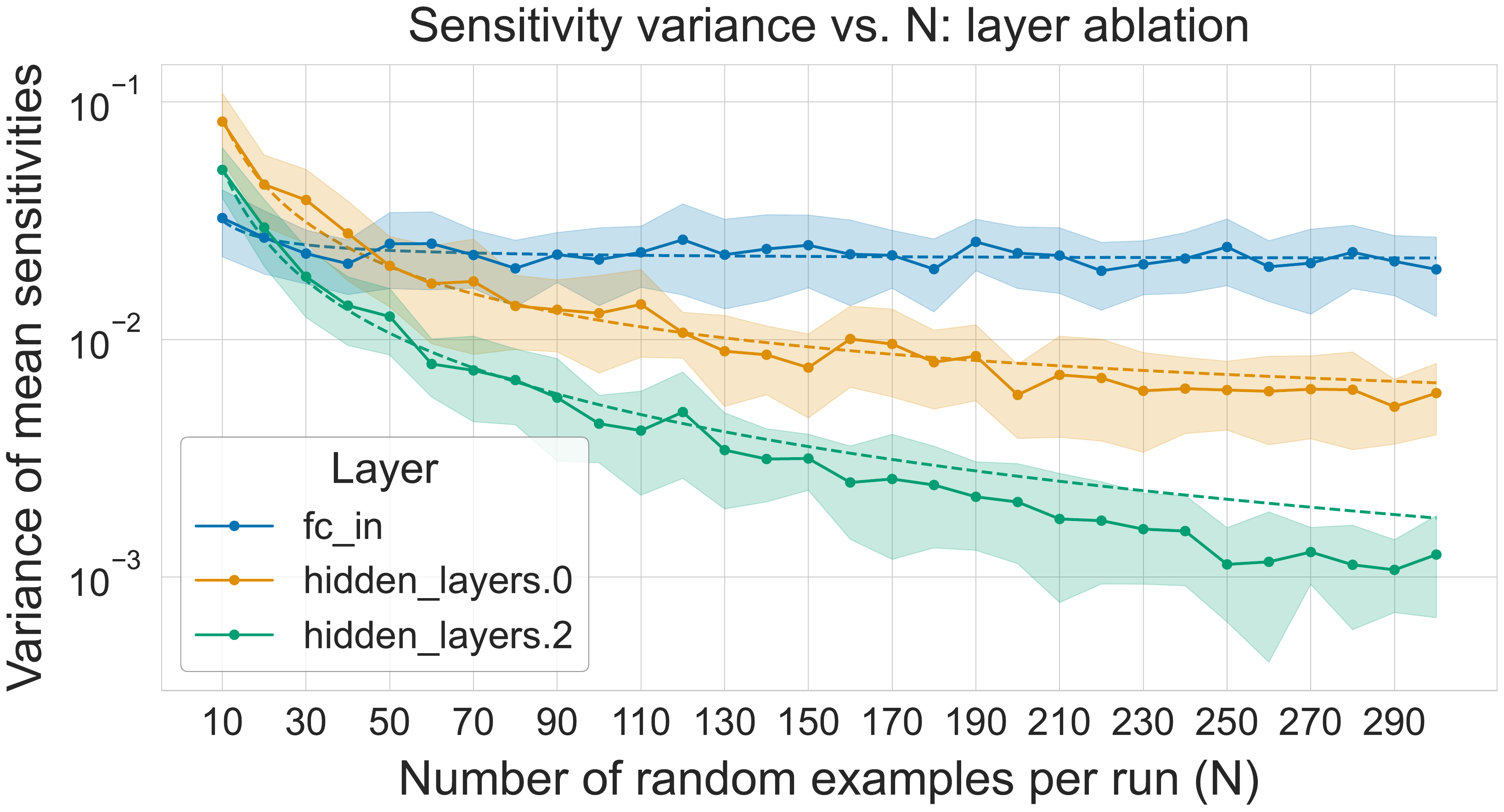}
        \centerline{\small Sensitivity variance}
    \end{minipage}

    \vspace{0.6em}

    \begin{minipage}{0.55\textwidth}
        \centering
        \includegraphics[width=\linewidth]{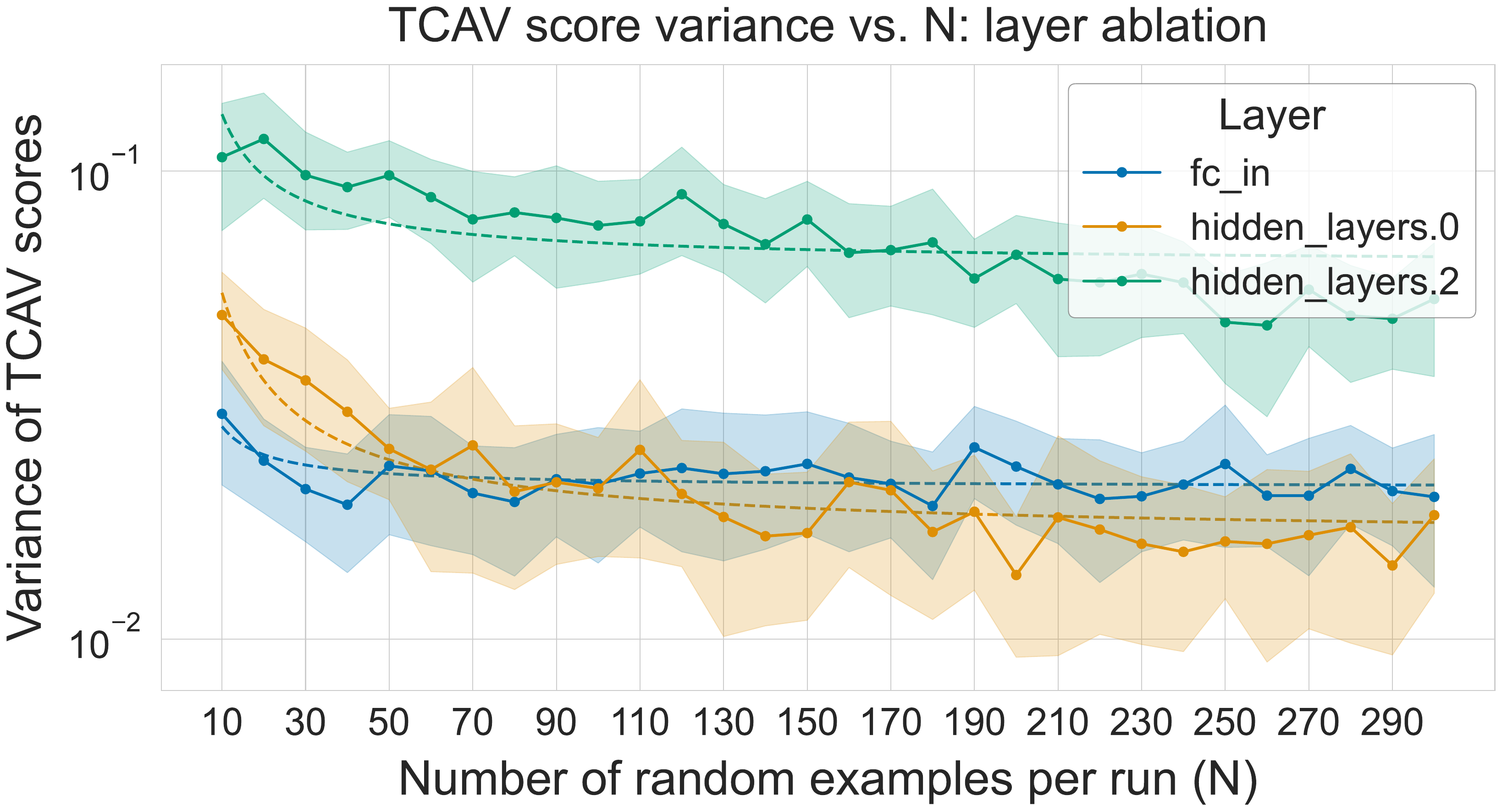}
        \centerline{\small \TCAV score variance}
    \end{minipage}

    \caption{
    \textbf{Layer ablation.}
    We vary the layer used to compute CAVs in the tabular model. \CAV variability decreases with $N$ across all layers,
    while the constants differ. \TCAV score variance is more sensitive to the chosen layer because it depends on the signs
    of the concept sensitivities.
    }
    \label{fig:layer_ablation}
\end{figure*}

\subsubsection{Feature Normalization}
\label{app:normalization_ablation}

Next, we vary the feature normalization applied before fitting the \CAV. Figure~\ref{fig:normalization_ablation} compares
raw activations, max-absolute normalization, and z-score normalization. All three choices preserve the qualitative trend that
CAV variability decreases as $N$ grows. However, the magnitude of the variability changes across normalization schemes,
which is expected because normalization rescales activation dimensions before the separating direction is fitted.

The sensitivity variance is especially affected by normalization. In this experiment, max-absolute normalization produces
larger sensitivity variance than raw or z-score normalized activations. Thus, normalization should be treated as an
implementation choice that affects constants and finite-sample stability, even though it does not invalidate the overall
stabilization trend.

\begin{figure*}[!ht]
    \centering
    \begin{minipage}{0.48\textwidth}
        \centering
        \includegraphics[width=\linewidth]{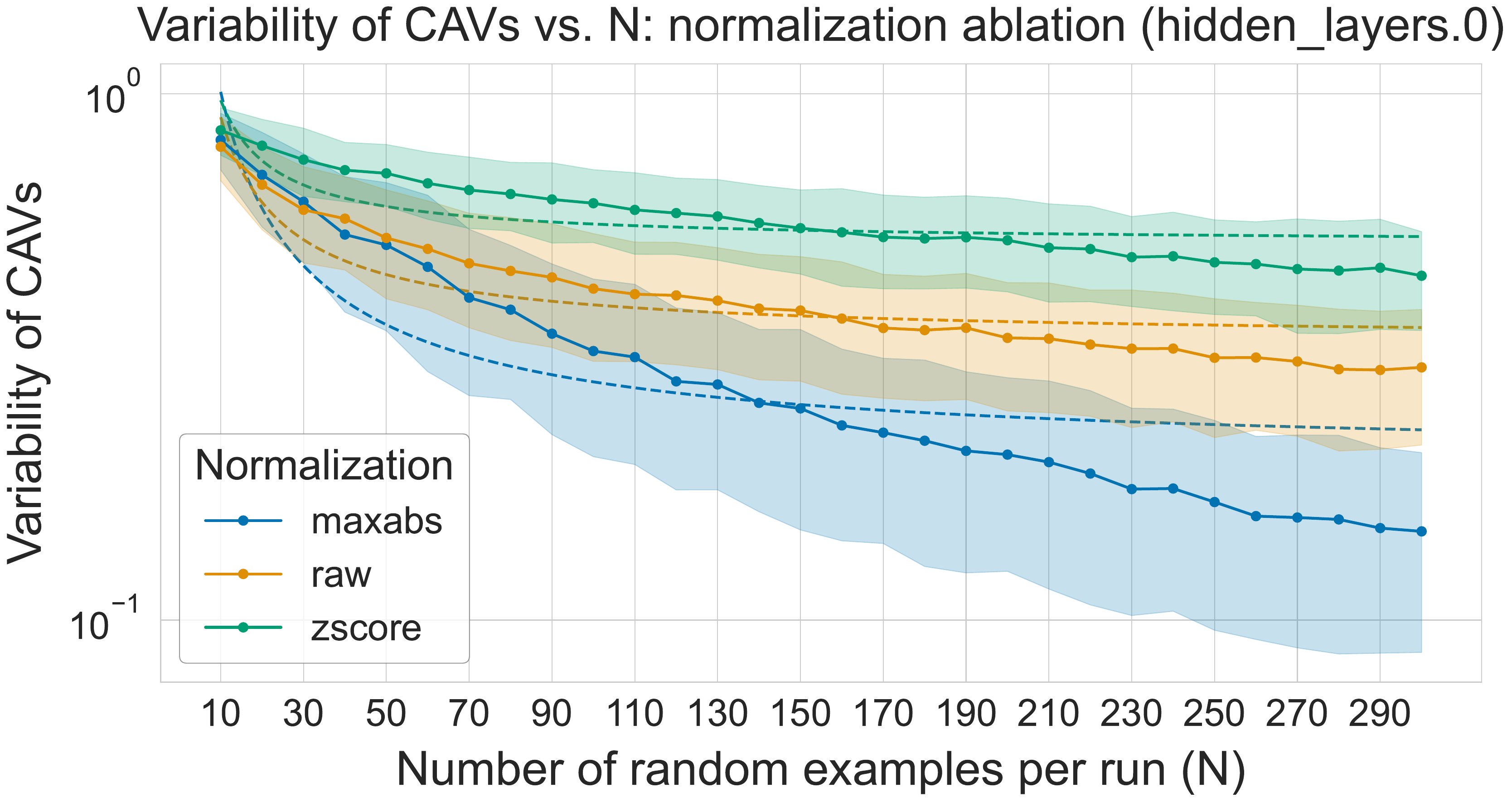}
        \centerline{\small \CAV variability}
    \end{minipage}
    \hfill
    \begin{minipage}{0.48\textwidth}
        \centering
        \includegraphics[width=\linewidth]{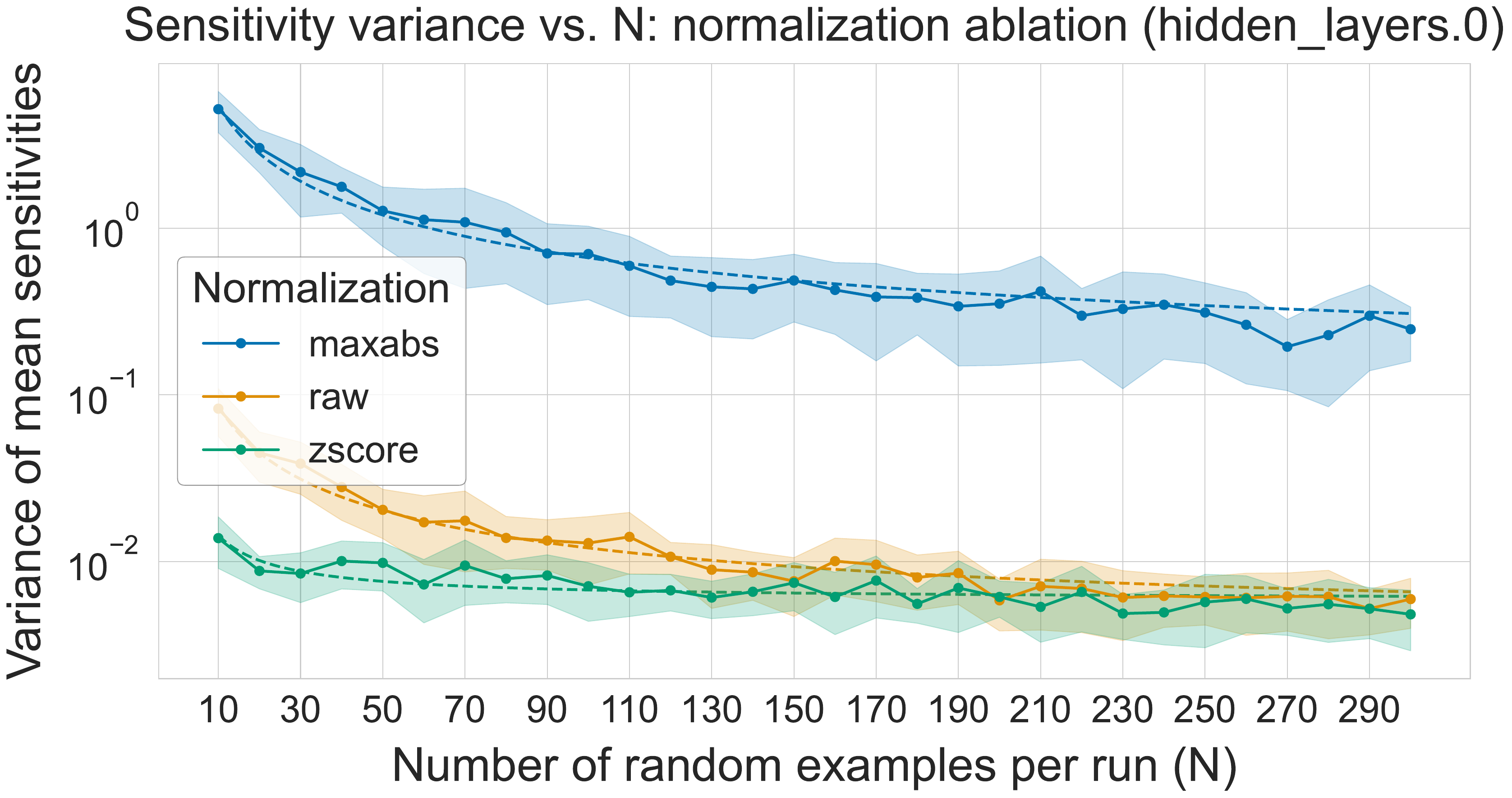}
        \centerline{\small Sensitivity variance}
    \end{minipage}

    \vspace{0.6em}

    \begin{minipage}{0.55\textwidth}
        \centering
        \includegraphics[width=\linewidth]{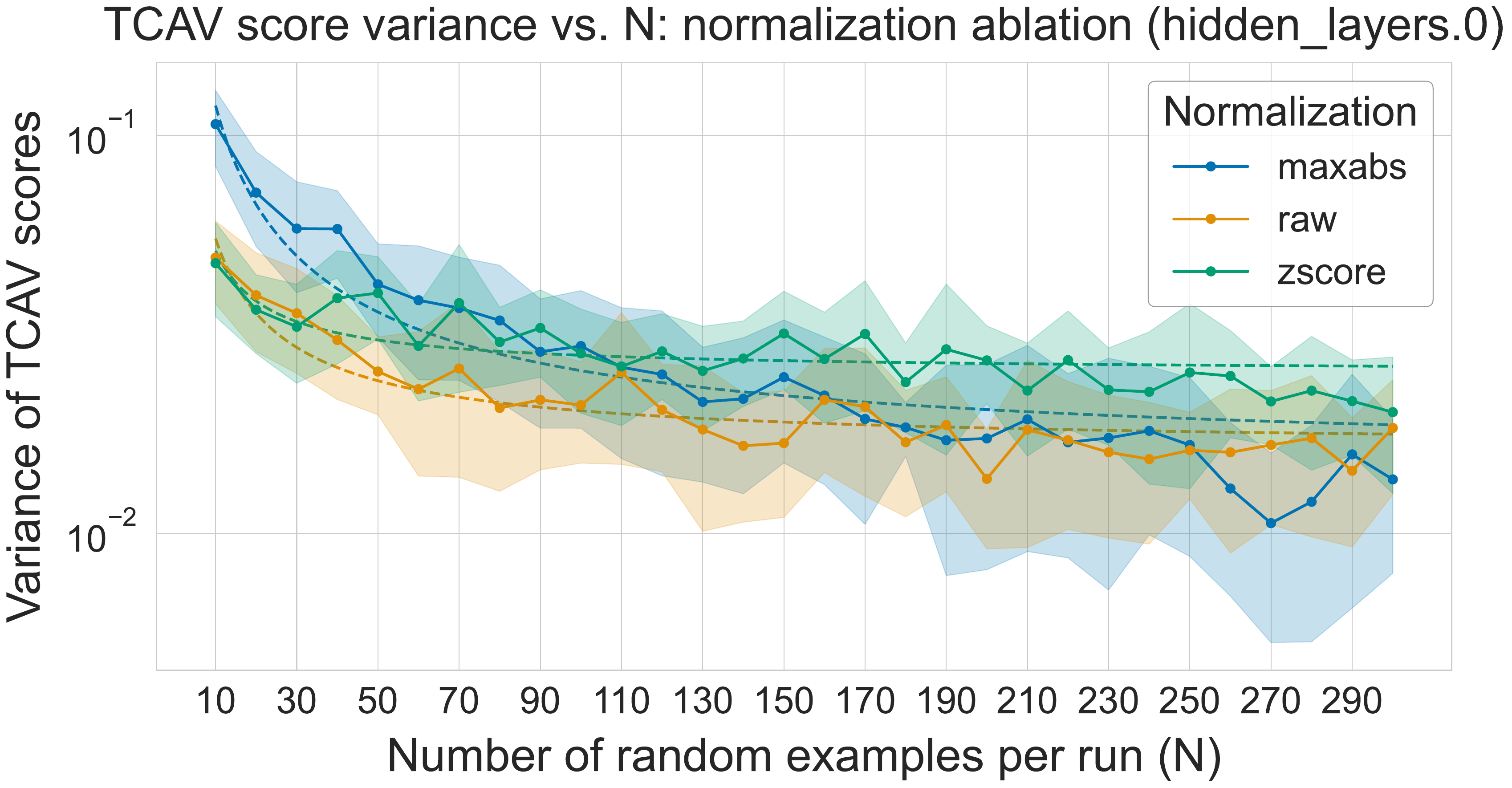}
        \centerline{\small \TCAV score variance}
    \end{minipage}

    \caption{
    \textbf{Normalization ablation.}
    We compare raw activations, max-absolute normalization, and z-score normalization at \texttt{hidden\_layers.0}.
    The qualitative decrease of \CAV variability with $N$ remains visible, but normalization changes the constants and can
    strongly affect sensitivity variance.
    }
    \label{fig:normalization_ablation}
\end{figure*}

\subsubsection{Regularization Strength}
\label{app:regularization_ablation}

We also vary the regularization strength of the logistic-regression \CAV estimator. In the scikit-learn convention, the
parameter $C$ is the inverse regularization strength: smaller values of $C$ correspond to stronger regularization.
Figure~\ref{fig:regularization_ablation} compares $C=0.25$, $C=1$, and $C=4$.

The results show that stronger regularization reduces \CAV variability. This agrees with the role of the $\ell_2$ penalty in
the theory: regularization restricts the effective parameter space and improves conditioning. The sensitivity variance follows
the same qualitative stabilization trend. The \TCAV score variance also benefits from stronger regularization in this
experiment, although the hard sign threshold prevents the same clean behaviour observed for the \CAV direction.

\begin{figure*}[!ht]
    \centering
    \begin{minipage}{0.48\textwidth}
        \centering
        \includegraphics[width=\linewidth]{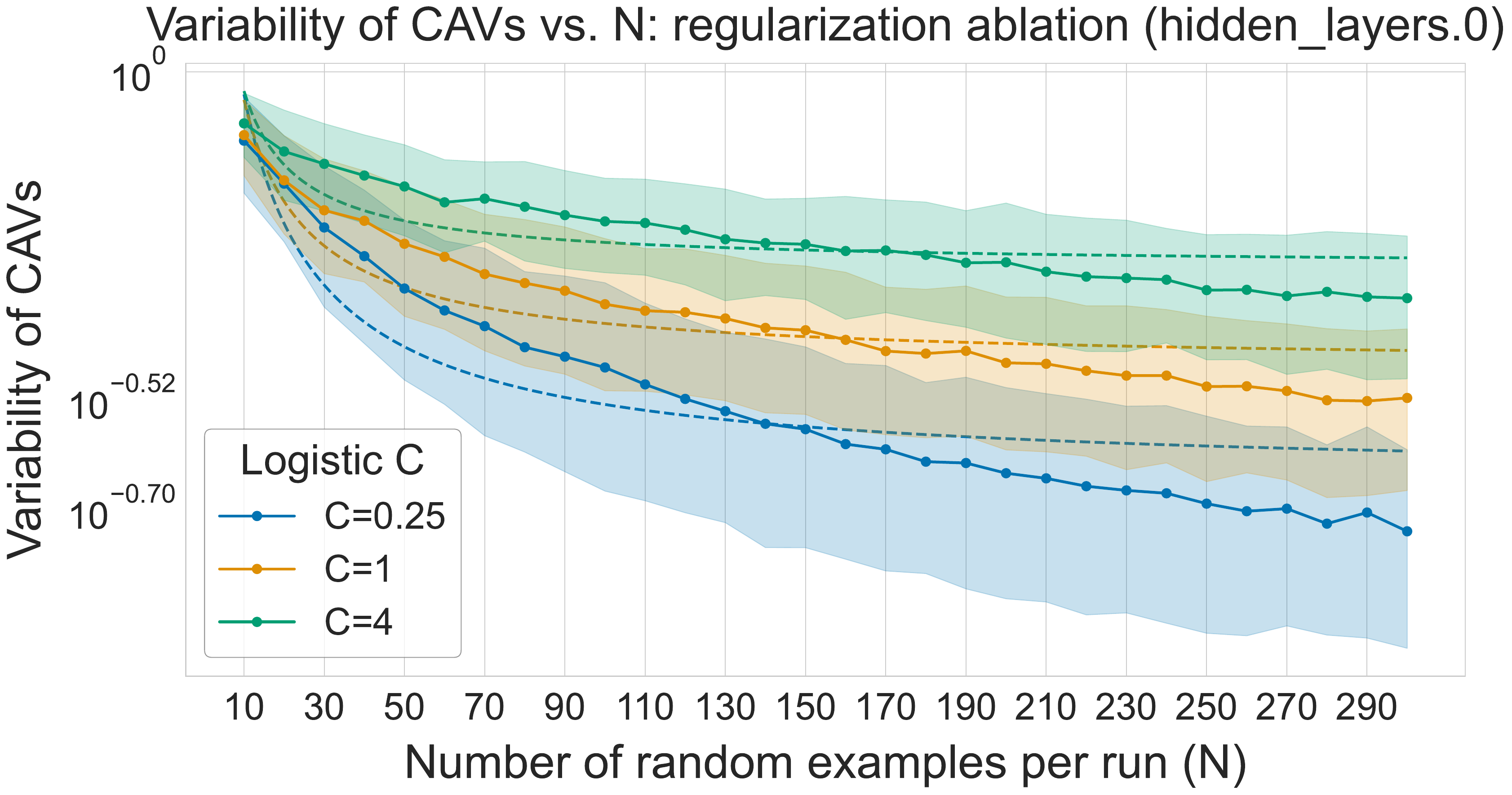}
        \centerline{\small \CAV variability}
    \end{minipage}
    \hfill
    \begin{minipage}{0.48\textwidth}
        \centering
        \includegraphics[width=\linewidth]{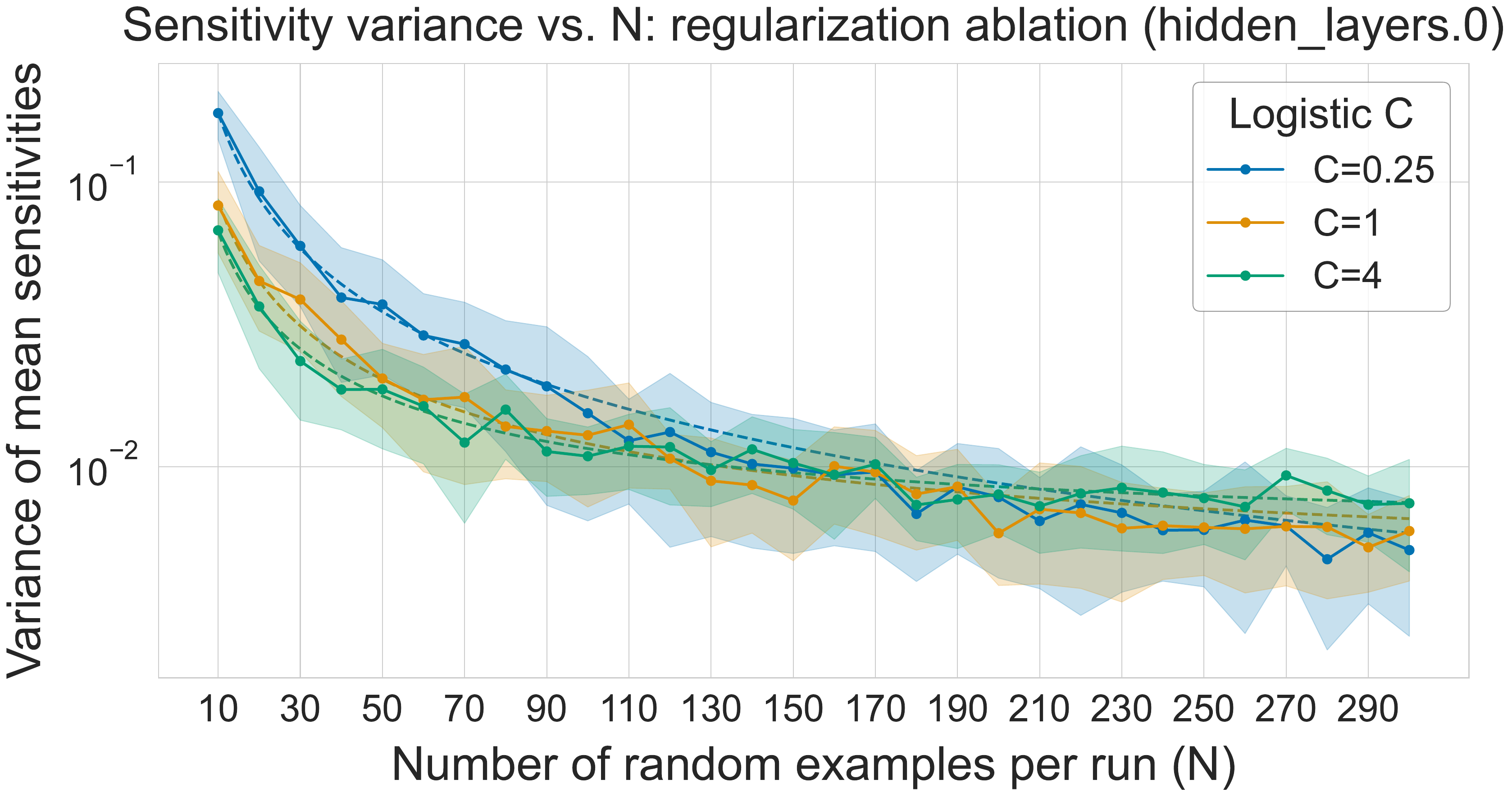}
        \centerline{\small Sensitivity variance}
    \end{minipage}

    \vspace{0.6em}

    \begin{minipage}{0.55\textwidth}
        \centering
        \includegraphics[width=\linewidth]{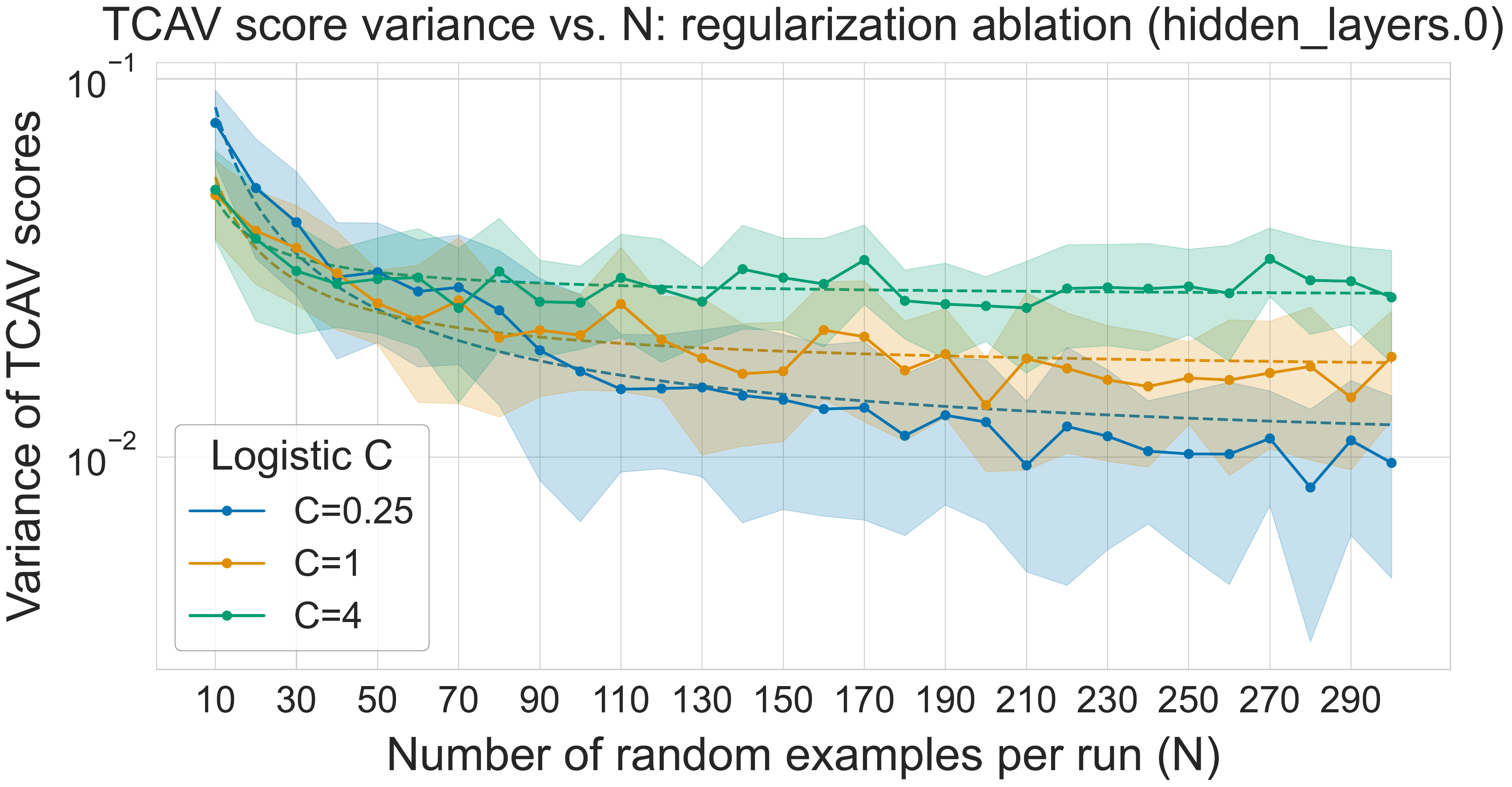}
        \centerline{\small \TCAV score variance}
    \end{minipage}

    \caption{
    \textbf{Regularization ablation.}
    We vary the inverse regularization parameter $C$ of the logistic-regression \CAV estimator. Smaller $C$ corresponds
    to stronger regularization. Stronger regularization lowers \CAV variability and improves finite-sample stability, while the
    qualitative dependence on $N$ remains unchanged.
    }
    \label{fig:regularization_ablation}
\end{figure*}

\subsubsection{CAV Initialization}
\label{app:initialization_ablation}

Figure~\ref{fig:initialization_ablation} studies the effect of the random initialization used when fitting the \CAV estimator.
Across the tested seeds, the curves nearly overlap for \CAV variability, sensitivity variance, and \TCAV score variance. This
suggests that, in this setting, the dominant source of variability is the random reference set rather than the initialization of
the optimizer.

This observation supports the interpretation of the main experiments: the variability analyzed in the paper is not merely an
optimization artifact, but a genuine sampling effect caused by the random reference examples.

\begin{figure*}[!ht]
    \centering
    \begin{minipage}{0.48\textwidth}
        \centering
        \includegraphics[width=\linewidth]{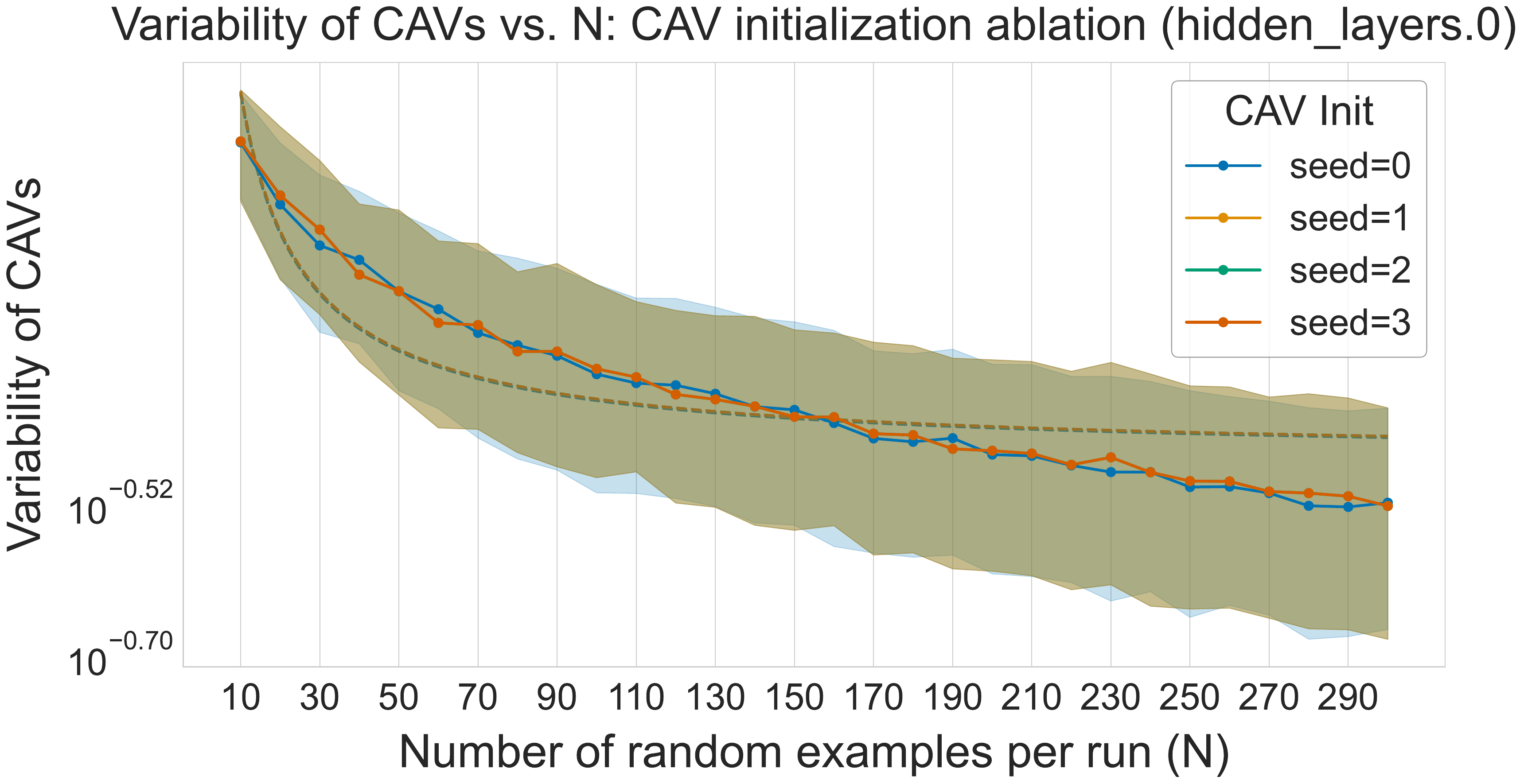}
        \centerline{\small \CAV variability}
    \end{minipage}
    \hfill
    \begin{minipage}{0.48\textwidth}
        \centering
        \includegraphics[width=\linewidth]{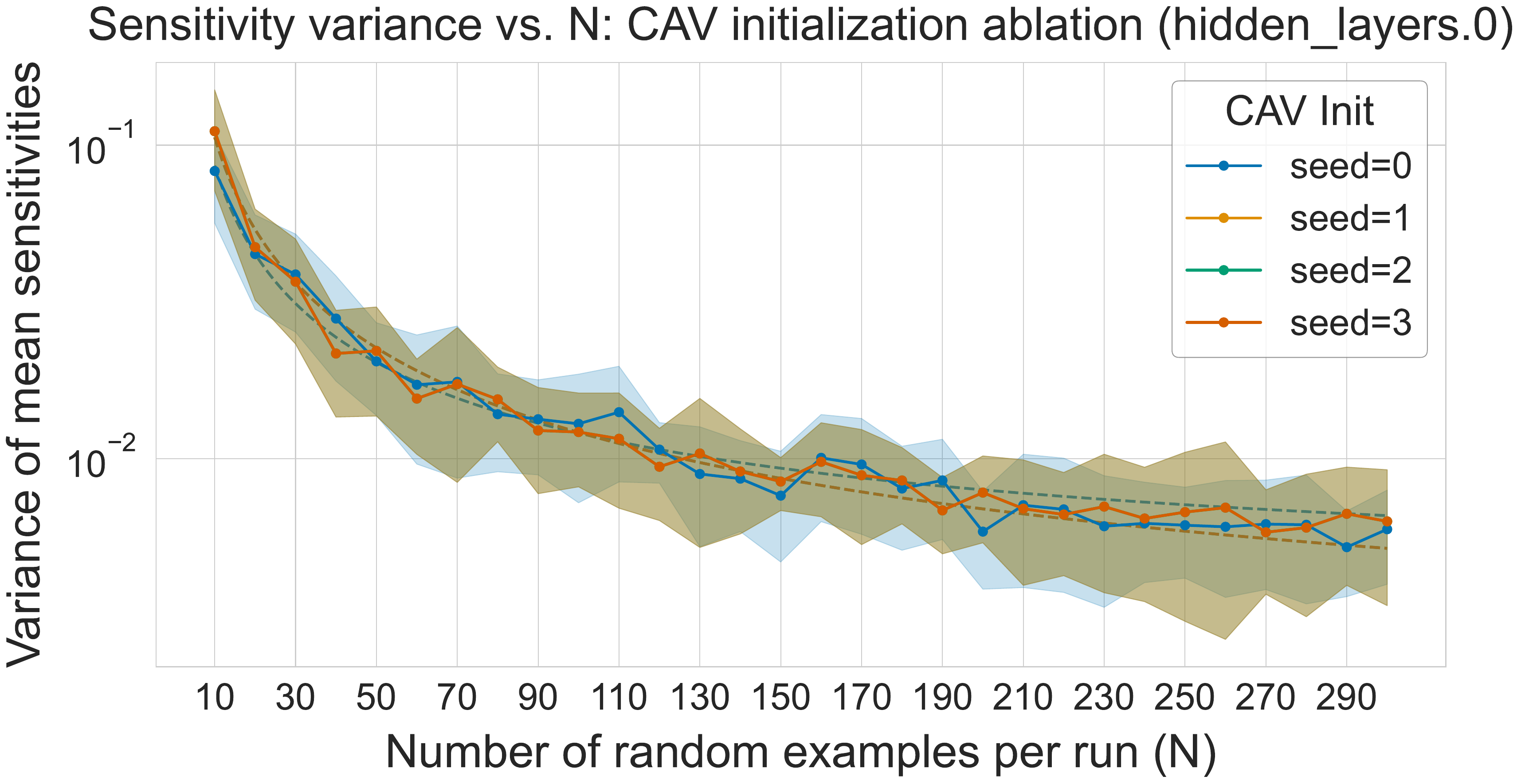}
        \centerline{\small Sensitivity variance}
    \end{minipage}

    \vspace{0.6em}

    \begin{minipage}{0.55\textwidth}
        \centering
        \includegraphics[width=\linewidth]{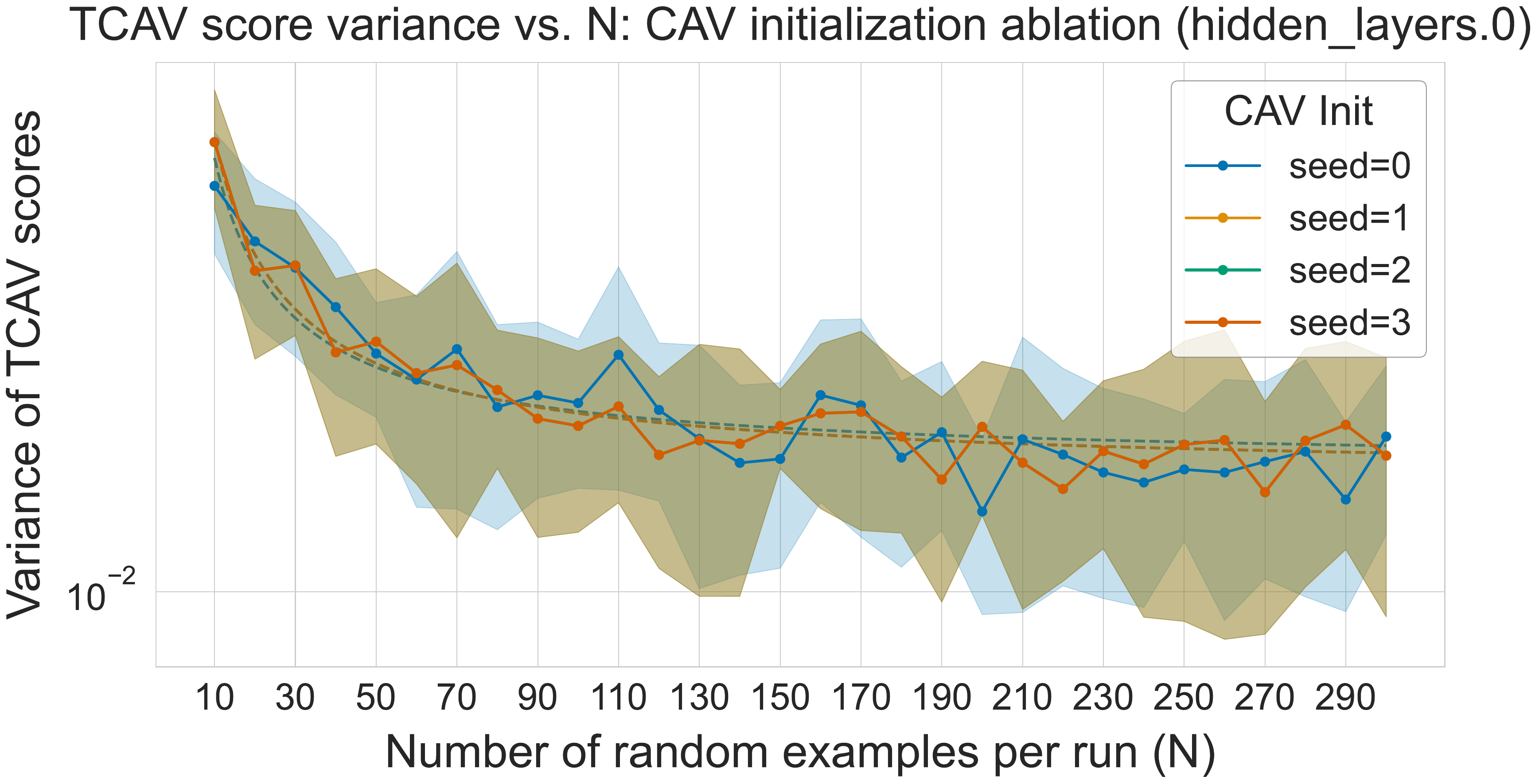}
        \centerline{\small \TCAV score variance}
    \end{minipage}

    \caption{
    \textbf{CAV-initialization ablation.}
    We vary the random seed used for \CAV initialization. The resulting curves nearly overlap, indicating that initialization
    has little effect compared to the variability induced by sampling random reference examples.
    }
    \label{fig:initialization_ablation}
\end{figure*}

\subsubsection{Optimizer}
\label{app:optimizer_ablation}

Finally, we compare different optimizers for fitting the \CAV estimator: Adam, AdamW, and SGD.
Figure~\ref{fig:optimizer_ablation} shows that all optimizers exhibit the same qualitative stabilization trend as $N$
increases. The constants differ, especially for sensitivity variance, but the dependence on $N$ remains qualitatively
consistent.

\begin{figure*}[!ht]
    \centering
    \begin{minipage}{0.48\textwidth}
        \centering
        \includegraphics[width=\linewidth]{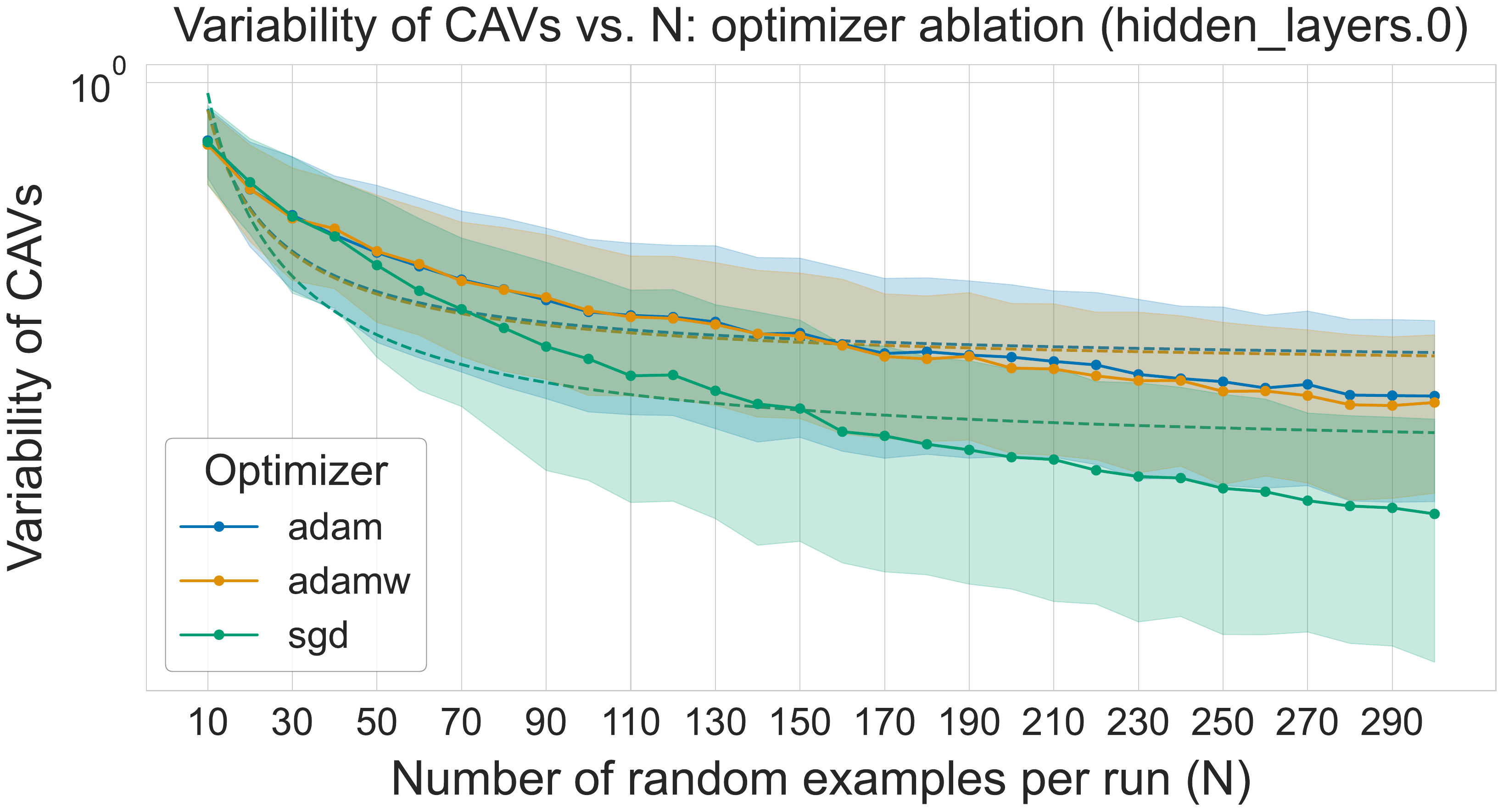}
        \centerline{\small \CAV variability}
    \end{minipage}
    \hfill
    \begin{minipage}{0.48\textwidth}
        \centering
        \includegraphics[width=\linewidth]{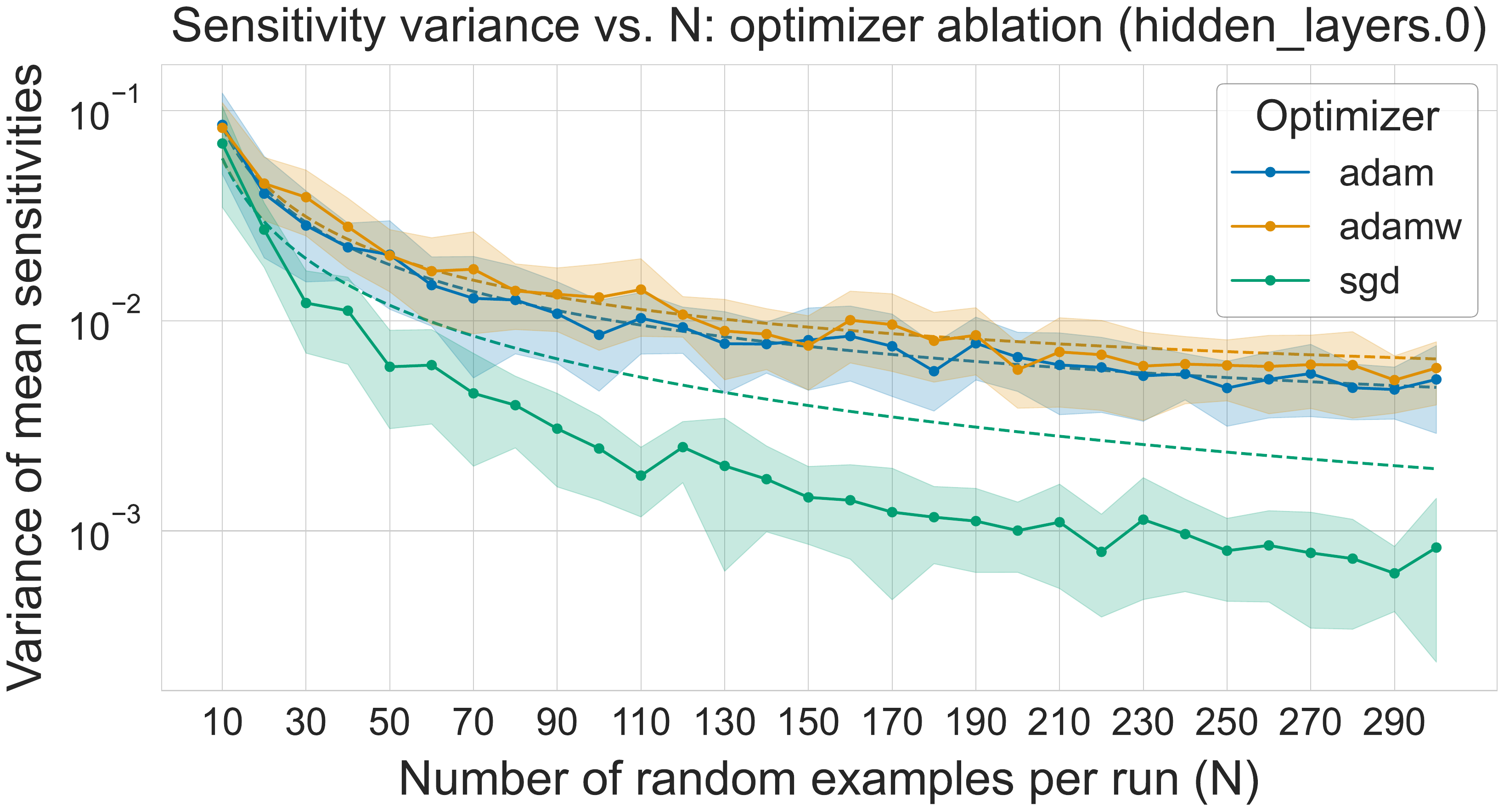}
        \centerline{\small Sensitivity variance}
    \end{minipage}

    \vspace{0.6em}

    \begin{minipage}{0.55\textwidth}
        \centering
        \includegraphics[width=\linewidth]{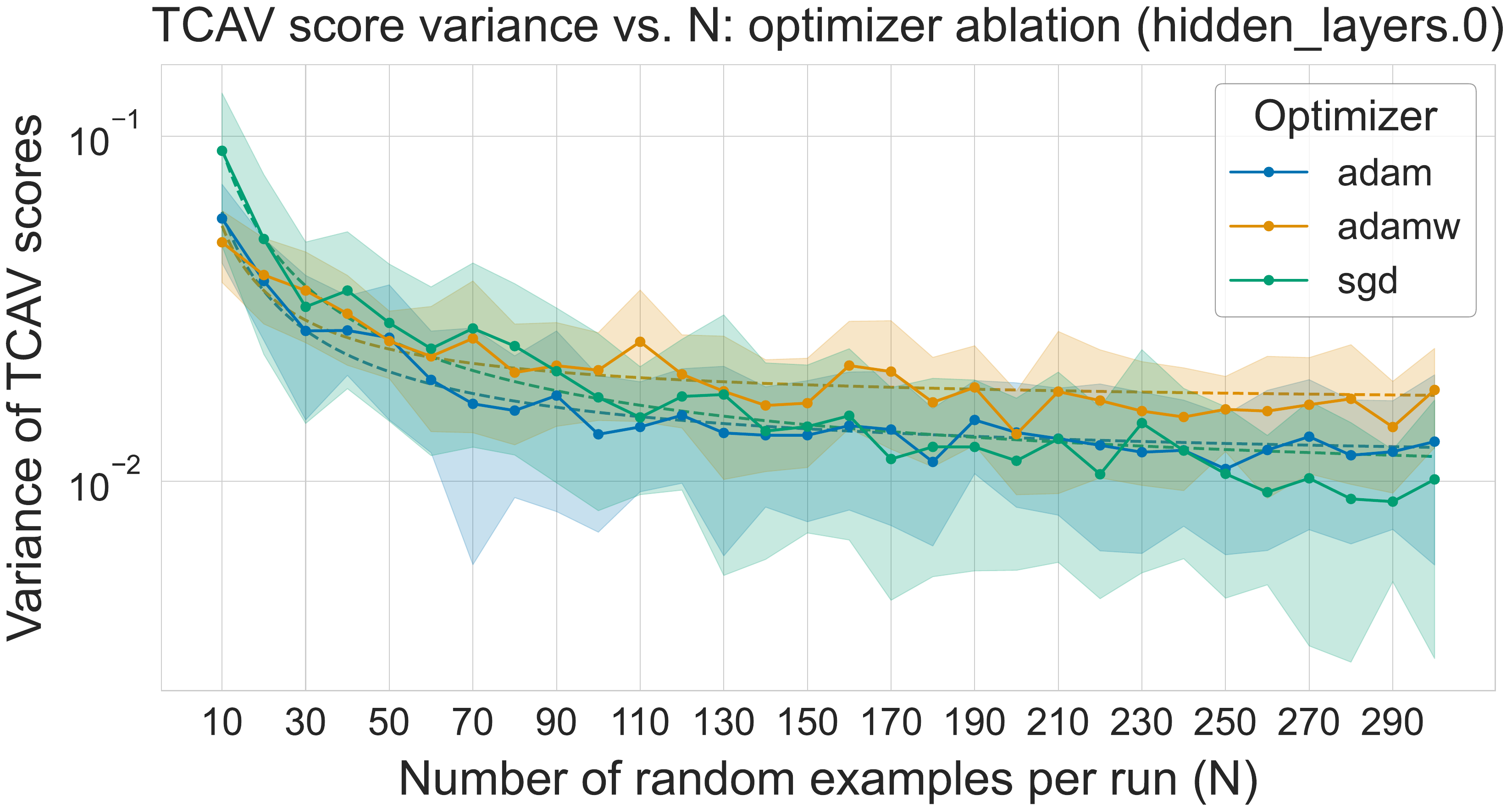}
        \centerline{\small \TCAV score variance}
    \end{minipage}

    \caption{
    \textbf{Optimizer ablation.}
    We compare Adam, AdamW, and SGD for fitting the \CAV estimator. All optimizers show the same qualitative
    stabilization trend with increasing $N$. The optimizer mainly changes constants and finite-sample behaviour rather than
    the overall dependence on the number of random reference examples.
    }
    \label{fig:optimizer_ablation}
\end{figure*}

\paragraph{Summary.}
Across all ablations, the conclusions of the main paper remain stable. \CAV variability and concept-sensitivity
variance decrease as the number of random reference examples increases. Layer choice, feature normalization,
regularization, and optimizer selection affect the constants, with normalization and regularization having  visible
effects in this tabular experiment. Initialization has little effect. These results support the view that the main source of
\CAV instability is the sampling of random reference examples, while implementation choices primarily modulate the
magnitude of the variance.

\clearpage
\section{Proofs of the Main Results}
\label{sec:apx-proofs-of-main-results}
\label{sec:appendix-proof}

\subsection{Asymptotic Normality with Binary Cross-Entropy Loss}
\label{appendix:cross_entropy}

In Section~\ref{sec:main-results}, we established the asymptotic normality of \CAV{} estimators trained with binary cross-entropy loss. This appendix provides the complete proof.

\paragraph{Main Idea and Proof Strategy.}
Our goal is to show that the \CAV{} estimator is asymptotically
normal. The key challenge is the \emph{infinitely imbalanced
regime}: the number of negative samples satisfies $N\to\infty$,
while the number of concept samples $n$ remains fixed.

Our proof strategy, adapted from \citet{owen_2007} and
\citet{goldman_zhang_2022}, proceeds as follows:
\begin{enumerate}
    \item \textbf{Establish preliminary limits}
    (Section~\ref{subsec:setup-and-preliminaries}):
    We use the assumed convergence of
    $A_N=Ne^{\alpha_N}$ and $\beta_N$, together with uniform laws
    of large numbers, to obtain the limiting intercept and slope
    score identities.

    \item \textbf{Prove asymptotic normality}
    (Section~\ref{subsec:proof-asymptotic-normality}):
    We Taylor-expand the slope score around $\beta_\infty$ and use
    the intercept score to account for the first-order fluctuation
    of the estimated intercept. We then show that the resulting
    coefficient matrix converges to the positive definite matrix
    $H_c$, while the corrected score is asymptotically normal by
    the multivariate central limit theorem. Slutsky's theorem
    yields the asymptotic normality of
    $\sqrt{N}(\beta_N-\beta_\infty)$.
\end{enumerate}

\subsubsection{Setup and Preliminary Results}
\label{subsec:setup-and-preliminaries}

\paragraph{Notation.}
Throughout this appendix, we use the equivalent parametrization of the samples centred at $\bar x$ and continue to denote the corresponding intercept by $\alpha$. Thus, for any given $N$, the default \TCAV{} implementation fits the logistic regression model by maximizing the $L^2$-penalized log-likelihood
\begin{align}
\mathcal{L}_N^{(\lambda)}(\alpha,\beta)
&\defeq \sum_{i=1}^n \log \sigma\!\left(\alpha+\beta^\top(x_i-\bar x)\right) \notag\\
&\quad+\sum_{j=1}^N \log\!\left(1-\sigma\!\left(\alpha+\beta^\top(z_j-\bar x)\right)\right)
-\frac{\lambda}{2}\norm{\beta}^2 .
\label{eq:objective-function}
\end{align}
Here $\alpha\in\Reals$, $\beta\in\Reals^d$, and $\lambda>0$. For every $N$, let $(\alpha_N,\beta_N)$ denote the unique maximizer of \eqref{eq:objective-function}, and define
\[
A_N\defeq N e^{\alpha_N}.
\]
This reparametrization leaves $\beta_N$, and hence the \CAV, unchanged.

We require four assumptions. The first ensures that the concept mean is geometrically well-positioned relative to the random distribution.

\begin{assumption}[Surrounded Mean]
\label{assumption:surround_appendix}
The distribution $F_0$ on $\mathbb{R}^d$ has the point $\bar{x}$ ``surrounded,'' that is,
\begin{equation}
\int_{(z-\bar{x})^\top\omega>\varepsilon} \Diff F_0(z)>\delta
\label{eq:surrounded-mean-appendix}
\end{equation}
for some $\varepsilon>0$, some $\delta>0$, and every $\omega\in\Omega$, where
$\Omega=\{\omega\in\mathbb{R}^d:\omega^\top\omega=1\}$ and
$\bar{x}=n^{-1}\sum_{i=1}^n x_i$.
\end{assumption}

This assumption ensures that no hyperplane through $\bar x$ separates the concept mean from the bulk of the random distribution. Figure~\ref{fig:surround_condition_tabular} provides a geometric illustration for the tabular setting.

The second and third assumptions concern parameter convergence.

\begin{assumption}[Intercept Scaling Limit]
\label{assumption:limit_A_appendix}
There exists a deterministic finite constant $A_\infty$ such that
$A_N\xrightarrow{p}A_\infty$ as $N\to\infty$.
\end{assumption}

\begin{assumption}[Consistency of $\beta_N$]
\label{assumption:beta_convergence_appendix}
There exists a deterministic vector $\beta_\infty\in\Reals^d$ such that
$\beta_N\xrightarrow{p}\beta_\infty$ as $N\to\infty$.
\end{assumption}

\begin{figure}[!ht]
    \centering
    \begin{minipage}{0.45\textwidth}
        \centering
        \includegraphics[width=\linewidth]{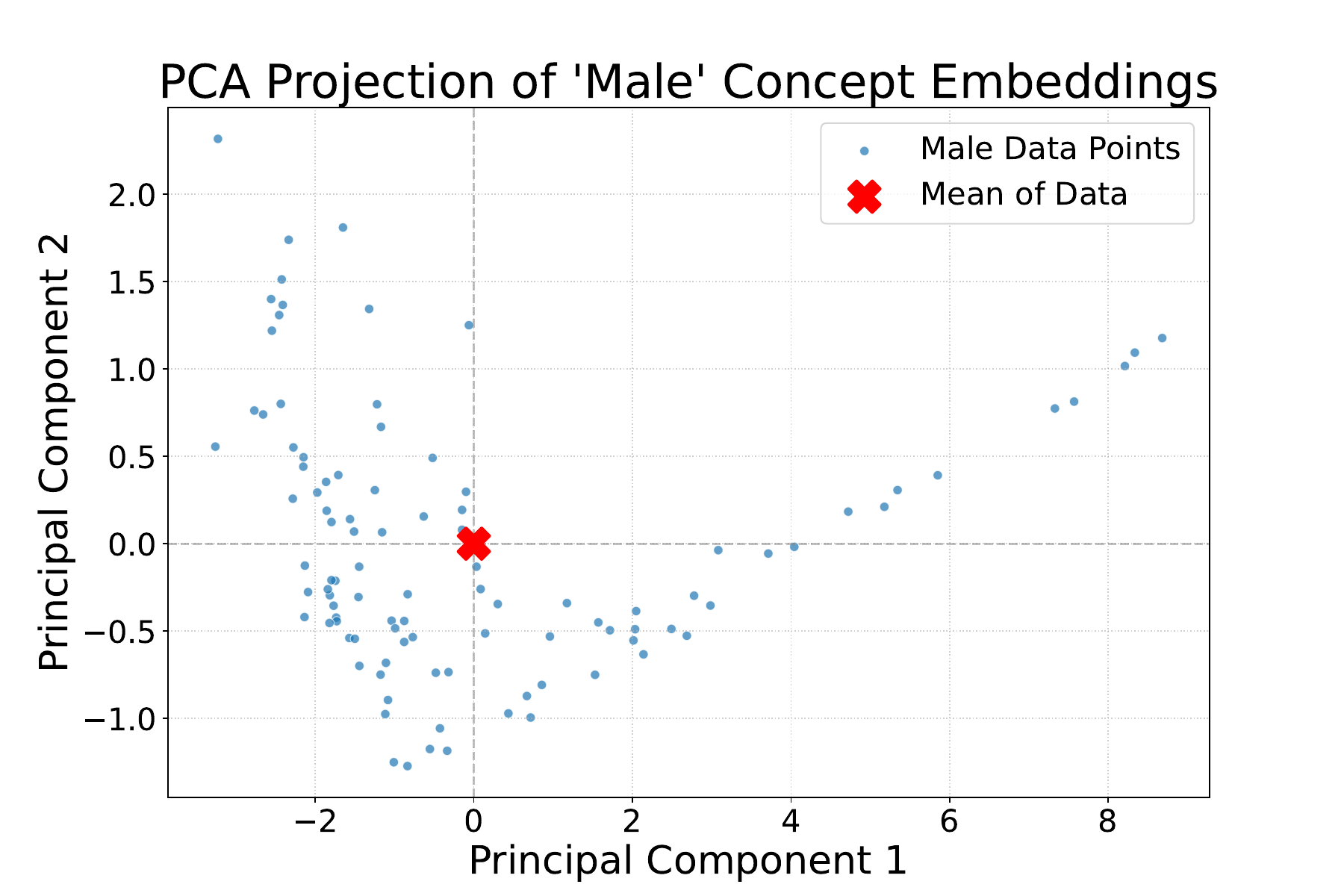}
    \end{minipage}
    \hfill
    \begin{minipage}{0.45\textwidth}
        \centering
        \includegraphics[width=\linewidth]{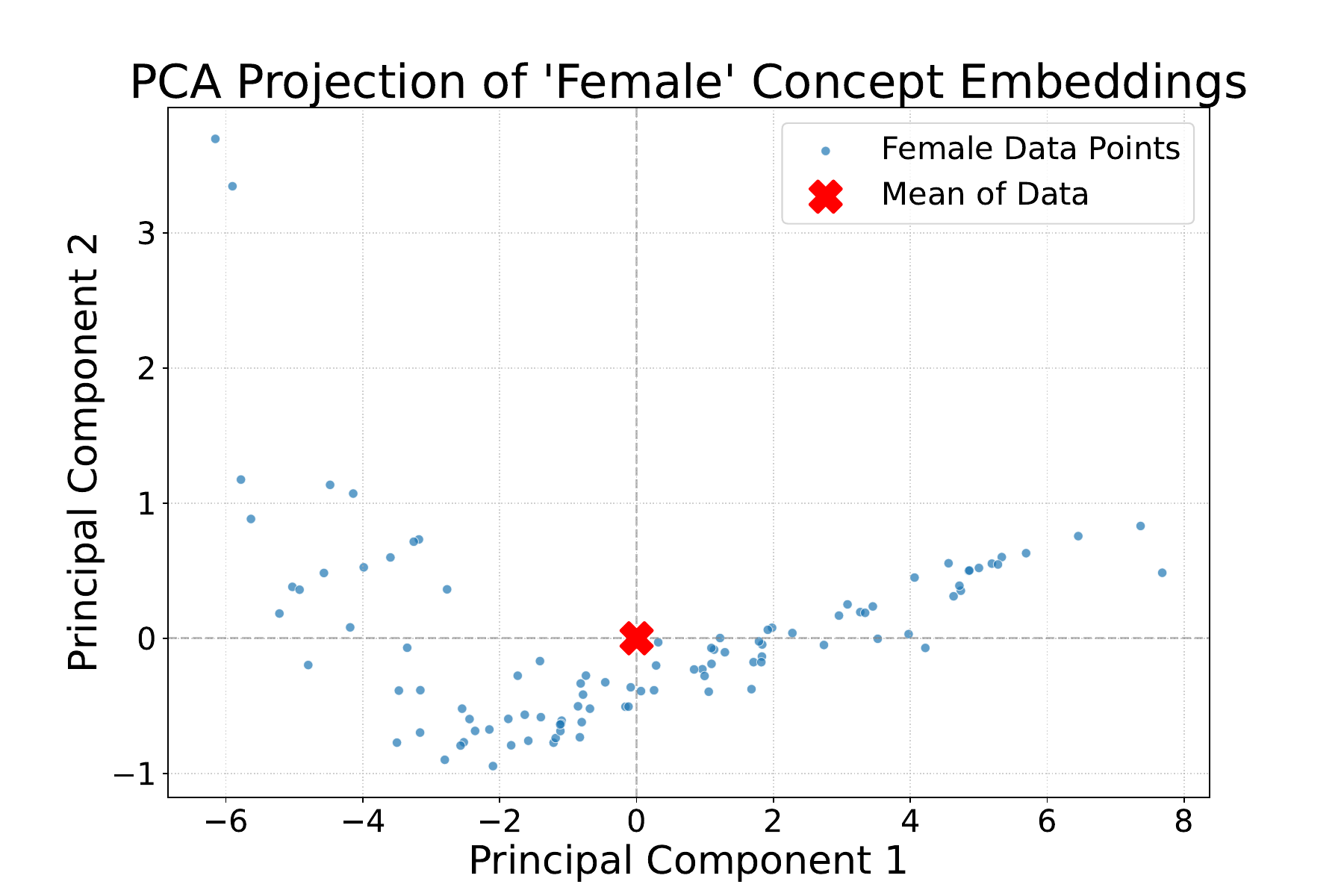}
    \end{minipage}
    \caption[Geometric visualization of the surround assumption]{\label{fig:surround_condition_tabular}
        Geometric visualization of the surround assumption (Assumption~\ref{assumption:surround_appendix}) in the latent space. The embeddings of the ``Male'' and ``Female'' concepts are projected onto their first two principal components. The points are distributed on all sides of the sample mean (red \texttt{'X'}), illustrating that they are not confined to a single half-space. The empirical check is provided in the code.}
\end{figure}

\begin{assumption}[Exponential moments under $F_0$]
\label{assumption:exp_moment_appendix}
For every $\beta\in\Reals^d$,
\begin{equation}
\mathbb{E}_{F_0}\!\left[\|Z\|^2\exp\!\left(2\beta^\top Z\right)\right]<\infty.
\label{eq:exp_moment_appendix}
\end{equation}
\end{assumption}

We provide empirical verification of these assumptions in the accompanying notebooks under ``\textit{Checking Assumptions}.''

\begin{lemma}[Consequences of Assumption~\ref{assumption:exp_moment_appendix}]
\label{lem:exp_moment_consequences}
Assume Assumption~\ref{assumption:exp_moment_appendix}. Then:
\begin{enumerate}
\item For every fixed $\beta\in\Reals^d$,
\[
\mathbb{E}\!\left[\|Z-\bar x\|^2 e^{2\beta^\top(Z-\bar x)}\right]<\infty,
\qquad
\mathbb{E}\!\left[\|Z-\bar x\| e^{2\beta^\top(Z-\bar x)}\right]<\infty.
\]
\item For every $B<\infty$,
\[
\mathbb{E}\!\left[\|Z-\bar x\|^2 e^{B\|Z-\bar x\|}\right]<\infty,
\qquad
\mathbb{E}\!\left[\|Z-\bar x\| e^{B\|Z-\bar x\|}\right]<\infty.
\]
\end{enumerate}
\end{lemma}

\begin{proof}
For the first assertion, use
$e^{2\beta^\top(Z-\bar x)}=e^{-2\beta^\top\bar x}e^{2\beta^\top Z}$ and
$\|Z-\bar x\|^2\le 2\|Z\|^2+2\|\bar x\|^2$; the first-moment statement follows from $\|u\|\le 1+\|u\|^2$.
For the second assertion, put $u=Z-\bar x$ and note that
\[
e^{B\|u\|}\le e^{B\|u\|_1}
\le \sum_{s\in\{\pm1\}^d}e^{B s^\top u}.
\]
After multiplication by $\|u\|^2$, each term has finite expectation by Assumption~\ref{assumption:exp_moment_appendix} applied with $\beta=(B/2)s$. The corresponding first-moment bound follows as above.
\end{proof}

For the remainder of this section, set
\[
u_i\defeq x_i-\bar x,
\qquad U_j\defeq z_j-\bar x,
\qquad U\defeq Z-\bar x,
\]
and define
\[
P f\defeq\mathbb{E}[f(U)],
\qquad
P_N f\defeq\frac1N\sum_{j=1}^N f(U_j),
\qquad
\mathbb{G}_N f\defeq\sqrt N\,(P_N-P)f.
\]

\begin{lemma}[Local boundedness]
\label{lem:local-boundedness}
Under Assumptions~\ref{assumption:limit_A_appendix} and
\ref{assumption:beta_convergence_appendix}, for every $\varepsilon>0$ there exist finite constants $A,B>0$ and $N_0$ such that, for every $N\ge N_0$,
\begin{equation}
\mathbb{P}\!\left(A_N\le A,\ \|\beta_N\|\le B\right)\ge 1-\varepsilon.
\label{eq:local-boundedness}
\end{equation}
\end{lemma}

\begin{proof}
This follows immediately from $A_N\xrightarrow{p}A_\infty<\infty$ and
$\beta_N\xrightarrow{p}\beta_\infty$.
\end{proof}

\begin{lemma}[Uniform laws of large numbers]
\label{lem:local-ulln}
Assume Assumption~\ref{assumption:exp_moment_appendix}. For every finite $B>0$ and $q\in\{0,1,2\}$,
\begin{equation}
\sup_{\|\beta\|\le B}
\left\|
P_N\!\left[U^{\otimes q}e^{\beta^\top U}\right]
-P\!\left[U^{\otimes q}e^{\beta^\top U}\right]
\right\|\xrightarrow{p}0,
\label{eq:ulln-one}
\end{equation}
\begin{equation}
\sup_{\|\beta\|\le B}
\left\|
P_N\!\left[U^{\otimes q}e^{2\beta^\top U}\right]
-P\!\left[U^{\otimes q}e^{2\beta^\top U}\right]
\right\|\xrightarrow{p}0,
\label{eq:ulln-two}
\end{equation}
where $U^{\otimes0}=1$, $U^{\otimes1}=U$, and $U^{\otimes2}=UU^\top$. Moreover,
\begin{equation}
\sup_{\substack{0\le a\le1\\\|\beta\|\le B}}
\left\|
P_N\!\left[\frac{UU^\top e^{\beta^\top U}}{(1+a e^{\beta^\top U})^2}\right]
-P\!\left[\frac{UU^\top e^{\beta^\top U}}{(1+a e^{\beta^\top U})^2}\right]
\right\|\xrightarrow{p}0.
\label{eq:ulln-hessian}
\end{equation}
\end{lemma}

\begin{proof}
Each displayed integrand is pointwise continuous in its parameter on the indicated compact set. Its coordinatewise absolute value is bounded, up to a fixed constant, by
$(1+\|U\|^2)e^{2B\|U\|}$, which is integrable by Lemma~\ref{lem:exp_moment_consequences}(2). The result follows from the dominated parametric uniform law of large numbers; see \citet[Lemma~2.4]{newey_large_1994}. Vector- and matrix-valued statements follow coordinatewise.
\end{proof}

\begin{lemma}[Limiting intercept and slope identities]
\label{lemma:asymptotic_consistency_of_weighted_mean}
Assume Assumptions~\ref{assumption:limit_A_appendix}-\ref{assumption:exp_moment_appendix}, and put
\[
W\defeq e^{\beta_\infty^\top U},
\qquad
m_0\defeq\mathbb{E}[W],
\qquad
m_1\defeq\mathbb{E}[UW].
\]
Then
\begin{equation}
A_\infty m_0=n,
\qquad
A_\infty m_1+\lambda\beta_\infty=0,
\qquad
\mu\defeq\frac{m_1}{m_0}=-\frac{\lambda}{n}\beta_\infty.
\label{eq:beta_star_def}
\end{equation}
In particular, $A_\infty=n/m_0\in(0,\infty)$.
\end{lemma}

\begin{proof}
\par\smallskip\noindent\textbf{Logistic expansion.}
For every $u,\beta\in\Reals^d$,
\[
\sigma(\alpha_N+\beta^\top u)
=\frac{A_N}{N}e^{\beta^\top u}
-\frac{A_N^2}{N^2}\,
\frac{e^{2\beta^\top u}}{1+(A_N/N)e^{\beta^\top u}}.
\]

\par\smallskip\noindent\textbf{Score equations.}
At $(\alpha_N,\beta_N)$, the intercept and slope scores vanish. Since
$\sum_{i=1}^n u_i=0$, the preceding identity, Lemma~\ref{lem:local-boundedness}, and Eqs.~\eqref{eq:ulln-one}-\eqref{eq:ulln-two} give
\begin{equation}
n=A_NP_N e^{\beta_N^\top U}+O_p(N^{-1}),
\qquad
0=-A_NP_N\!\left[Ue^{\beta_N^\top U}\right]-\lambda\beta_N+O_p(N^{-1}).
\label{eq:joint-score-approximations}
\end{equation}
More explicitly, $\sigma(\alpha_N+\beta_N^\top u_i)\le (A_N/N)e^{\beta_N^\top u_i}$, so the fixed concept contributions are $O_p(N^{-1})$ because $n$ is fixed. Since the denominators in the random-sample remainders are at least one, those remainders are bounded in absolute value or Euclidean norm by $(A_N^2/N)P_N[(1+\|U\|)e^{2\beta_N^\top U}]=O_p(N^{-1})$, by Lemma~\ref{lem:local-boundedness} and Eq.~\eqref{eq:ulln-two}.

\par\smallskip\noindent\textbf{Passage to the limit.}
Fix $B>\|\beta_\infty\|$. By Assumption~\ref{assumption:beta_convergence_appendix},
$\mathbb{P}(\|\beta_N\|\le B)\to1$. On the event $\{\|\beta_N\|\le B\}$, for $q\in\{0,1\}$,
\begin{align*}
&\left\|P_N\!\left[U^{\otimes q}e^{\beta_N^\top U}\right]
-P\!\left[U^{\otimes q}e^{\beta_\infty^\top U}\right]\right\|\\
&\quad\le
\sup_{\|\beta\|\le B}
\left\|(P_N-P)\!\left[U^{\otimes q}e^{\beta^\top U}\right]\right\|
+\left\|P\!\left[U^{\otimes q}
\bigl(e^{\beta_N^\top U}-e^{\beta_\infty^\top U}\bigr)\right]\right\|.
\end{align*}
The first term is $o_p(1)$ by Lemma~\ref{lem:local-ulln}, which is the dominated parametric uniform law of large numbers of \citet[Lemma~2.4]{newey_large_1994}. The second term is $o_p(1)$ by Assumption~\ref{assumption:beta_convergence_appendix}, continuity of the population moment in $\beta$, and dominated convergence, using  $e^{B\|U\|}$ for $q=0$ and $\|U\|e^{B\|U\|}$ for $q=1$ from Lemma~\ref{lem:exp_moment_consequences}(2). Consequently,
\[
P_Ne^{\beta_N^\top U}\xrightarrow{p}m_0,
\qquad
P_N[Ue^{\beta_N^\top U}]\xrightarrow{p}m_1.
\]
Taking probability limits in \eqref{eq:joint-score-approximations} proves the first two identities in \eqref{eq:beta_star_def}. Since $m_0>0$, the first identity yields $A_\infty=n/m_0\in(0,\infty)$, and the formula for $\mu$ follows by division.
\end{proof}

\subsubsection{Proof of Asymptotic Normality}
\label{subsec:proof-asymptotic-normality}

\begin{theorem}[Asymptotic Normality of \CAV s]
\label{thm:asymptotic_normality_cav_appendix}
Assume Assumptions~\ref{assumption:surround_appendix}-\ref{assumption:exp_moment_appendix}. Then
\begin{equation}
\sqrt{N}\,(\beta_N-\beta_\infty)\xrightarrow{D}\mathcal{N}(0,\Sigma),
\label{eq:asymptotic_normality_appendix}
\end{equation}
where
\begin{equation}
\Sigma\defeq H_c^{-1}V_c(H_c^{-1})^\top,
\qquad
V_c\defeq\mathbb{E}\!\left[W^2(U-\mu)(U-\mu)^\top\right],
\label{eq:sigma}
\end{equation}
and
\begin{equation}
H_c\defeq
\mathbb{E}\!\left[W(U-\mu)(U-\mu)^\top\right]
+\frac{\lambda}{A_\infty}I.
\label{eq:H_infty_def}
\end{equation}
Here $W$ and $\mu$ are defined in Lemma~\ref{lemma:asymptotic_consistency_of_weighted_mean}.
\end{theorem}

\begin{proof}
\par\smallskip\noindent\textbf{Taylor identity.}
Define
\[
\mathcal{F}_N(\beta)\defeq\nabla_\beta\mathcal{L}_N^{(\lambda)}(\alpha_N,\beta),
\qquad
\beta_t\defeq\beta_\infty+t(\beta_N-\beta_\infty).
\]
The fundamental theorem of calculus gives the exact identity
\begin{equation}
0=\mathcal{F}_N(\beta_N)
=\mathcal{F}_N(\beta_\infty)+H_N(\beta_N-\beta_\infty),
\qquad
H_N\defeq\int_0^1\nabla_\beta^2\mathcal{L}_N^{(\lambda)}(\alpha_N,\beta_t)\,\Diff t.
\label{eq:taylor_expansion_appendix_clean}
\end{equation}
Put
\begin{equation}
\Delta_N\defeq\sqrt N(\beta_N-\beta_\infty),
\qquad
K_N\defeq-\frac{e^{-\alpha_N}}{N}H_N,
\qquad
S_N\defeq\frac{e^{-\alpha_N}}{\sqrt N}\mathcal{F}_N(\beta_\infty).
\qquad
K_N\Delta_N=S_N.
\label{eq:taylor_clean_final}
\end{equation}

\par\smallskip\noindent\textbf{Hessian limit.}
Writing $u_i\defeq x_i-\bar x$, the exact scaled Hessian decomposition is
\[
\begin{aligned}
K_N
&=\int_0^1\frac1N\sum_{i=1}^n
\frac{u_i u_i^\top e^{\beta_t^\top u_i}}
{\{1+(A_N/N)e^{\beta_t^\top u_i}\}^2}\,\Diff t\\
&\quad+\int_0^1P_N\!\left[
\frac{UU^\top e^{\beta_t^\top U}}
{\{1+(A_N/N)e^{\beta_t^\top U}\}^2}
\right] \,\Diff t+\frac{\lambda}{A_N}I.
\end{aligned}
\]
On the event $\sup_{t\in[0,1]}\|\beta_t\|\le B$, the norm of the first integral is at most $N^{-1}\sum_{i=1}^n\|u_i\|^2e^{B\|u_i\|}=O(N^{-1})$. By Lemma~\ref{lem:local-boundedness}, this event has probability tending to one, so the fixed-concept contribution is $o_p(1)$. The penalty term equals $(\lambda/A_N)I$ exactly because $e^{-\alpha_N}/N=1/A_N$.
By Assumption~\ref{assumption:beta_convergence_appendix}, Lemma~\ref{lem:local-boundedness}, Eq.~\eqref{eq:ulln-hessian}, and dominated convergence, the second integral converges in probability. Thus $K_N\xrightarrow{p}H_0$, where $H_0\defeq\mathbb{E}[WUU^\top]+(\lambda/A_\infty)I$.

\par\smallskip\noindent\textbf{Slope-score expansion.}
The fixed concept contribution to $S_N$ is $o_p(1)$. Applying the logistic expansion displayed in the proof of Lemma~\ref{lemma:asymptotic_consistency_of_weighted_mean} at $\beta_\infty$ gives
\[
S_N
=-\frac1{\sqrt N}\sum_{j=1}^N U_jW_j
-\frac{\lambda\sqrt N}{A_N}\beta_\infty+o_p(1),
\qquad W_j\defeq e^{\beta_\infty^\top U_j}.
\]
The omitted random remainder is bounded in norm by
\[
\frac{A_N}{\sqrt N}\,
P_N\!\left[\|U\|e^{2\beta_\infty^\top U}\right]=o_p(1)
\]
by Assumption~\ref{assumption:limit_A_appendix} and Lemma~\ref{lem:exp_moment_consequences}(1). Using
$m_1=-\lambda\beta_\infty/A_\infty$ from \eqref{eq:beta_star_def}, we obtain
\[
S_N
=-\mathbb{G}_N(UW)
+\lambda\beta_\infty\sqrt N\left(\frac1{A_\infty}-\frac1{A_N}\right)
+o_p(1).
\]
The second term cannot be discarded merely from $A_N\xrightarrow{p}A_\infty$.

\par\smallskip\noindent\textbf{Intercept-score expansion.}
The first equation in \eqref{eq:joint-score-approximations}, together with $A_\infty>0$, implies
\[
\frac1{A_N}=\frac1nP_Ne^{\beta_N^\top U}+O_p(N^{-1}).
\]
Let
\[
\bar m_{1,N}\defeq\int_0^1P_N\!\left[Ue^{\beta_t^\top U}\right] \,\Diff t.
\]
The fundamental theorem of calculus gives
\[
P_Ne^{\beta_N^\top U}
=P_NW+\bar m_{1,N}^\top(\beta_N-\beta_\infty),
\]
and Lemmas~\ref{lem:local-boundedness}-\ref{lem:local-ulln} imply
$\bar m_{1,N}\xrightarrow{p}m_1$. Since $A_\infty^{-1}=m_0/n$,
\begin{equation*}
\sqrt N\left(\frac1{A_\infty}-\frac1{A_N}\right)
=-\frac1n\mathbb{G}_NW
-\frac1n\bar m_{1,N}^\top\Delta_N+o_p(1).
\tag{23a}\label{eq:intercept-score-first-order}
\end{equation*}

\par\smallskip\noindent\textbf{Projected score and corrected Hessian.}
Substituting the preceding expansion into the slope score and using
$\mu=-\lambda\beta_\infty/n$ yields
\begin{equation*}
\left(K_N-\mu\bar m_{1,N}^\top\right)\Delta_N
=-\mathbb{G}_N\!\left[W(U-\mu)\right]+o_p(1).
\tag{23b}\label{eq:projected-score-equation}
\end{equation*}
Furthermore,
\begin{equation*}
\begin{aligned}
K_N-\mu\bar m_{1,N}^\top
&\xrightarrow{p}H_0-\mu m_1^\top\\
&=\mathbb{E}[WUU^\top]-\frac{m_1m_1^\top}{m_0}
+\frac{\lambda}{A_\infty}I
=H_c.
\end{aligned}
\tag{23c}\label{eq:corrected-hessian-limit}
\end{equation*}
For every nonzero $v\in\Reals^d$,
\[
v^\top H_cv
=\mathbb{E}\!\left[W\{v^\top(U-\mu)\}^2\right]
+\frac{\lambda}{A_\infty}\|v\|^2>0,
\]
so $H_c$ is positive definite.

\par\smallskip\noindent\textbf{Central limit theorem and conclusion.}
By the definition of $\mu$,
$\mathbb{E}[W(U-\mu)]=m_1-\mu m_0=0$, and Assumption~\ref{assumption:exp_moment_appendix} implies
$\mathbb{E}\|W(U-\mu)\|^2<\infty$. Therefore,
\begin{equation*}
\mathbb{G}_N\!\left[W(U-\mu)\right]
\xrightarrow{D}\mathcal{N}(0,V_c)
\tag{23d}\label{eq:projected-score-clt}
\end{equation*}
by the multivariate central limit theorem. Since
$K_N-\mu\bar m_{1,N}^\top\xrightarrow{p}H_c$ and $H_c$ is positive definite,
this matrix is invertible with probability tending to one, so we may solve the projected score equation for $\Delta_N$. Continuity of matrix inversion and Slutsky's theorem then give
\begin{equation*}
\Delta_N
=-H_c^{-1}\mathbb{G}_N\!\left[W(U-\mu)\right]+o_p(1)
\xrightarrow{D}\mathcal{N}(0,\Sigma),
\tag{23e}\label{eq:asymptotic-linear-representation}
\end{equation*}
with $\Sigma$ as in \eqref{eq:sigma}.
\end{proof}

\subsubsection{Convergence Rate of the Covariance Trace}
\label{appendix:convergence_rate_of_covariance_trace}
Finally, we establish the convergence rate of the covariance trace.

\begin{definition}[Uniform Integrability]
A sequence of random variables $\{X_N\}_{N\ge1}$ is uniformly integrable if
\[
\lim_{K\to\infty}\sup_{N\ge1}
\mathbb{E}\!\left[|X_N|\,\mathbf{1}_{\{|X_N|>K\}}\right]=0.
\]
\end{definition}

\begin{corollary}[Asymptotic Behavior of the Covariance Trace]
\label{cor:covariance_trace_appendix}
Under the assumptions of Theorem~\ref{thm:asymptotic_normality_cav_appendix} and uniform integrability of
$\{\|\sqrt N(\beta_N-\beta_\infty)\|^2\}_{N\ge1}$,
\[
\operatorname{Cov}(\beta_N)=\frac1N\Sigma+o(N^{-1}).
\]
Consequently, $\operatorname{tr}(\operatorname{Cov}(\beta_N))=O(N^{-1})$.
\end{corollary}

\begin{proof}
Theorem~\ref{thm:asymptotic_normality_cav_appendix} and uniform integrability imply convergence of first and second moments; see \citet{vaart_asymptotic_2007}. Hence
\[
\operatorname{Cov}\!\left(\sqrt N(\beta_N-\beta_\infty)\right)\longrightarrow\Sigma.
\]
Since the covariance on the left equals $N\operatorname{Cov}(\beta_N)$, the asserted matrix expansion follows. Taking traces gives the final statement.
\end{proof}

\subsubsection{Asymptotic Normality of Concept Sensitivity Scores}
\label{appendix:sensitivity}

We prove the asymptotic normality of concept sensitivity scores using the linear case of the multivariate delta method. This argument depends only on the asymptotic normality of $\beta_N$, not on the method used to construct it.

\begin{lemma}[Linear maps preserve asymptotic normality]
\label{lem:linear_map_asymp_normal}
Let $(\beta_N)_{N\ge1}$ be $\mathbb{R}^d$-valued, and let
$\beta_\infty\in\mathbb{R}^d$ and a deterministic symmetric positive semidefinite matrix $\Sigma\in\mathbb{R}^{d\times d}$ satisfy
\begin{equation}
\sqrt N(\beta_N-\beta_\infty)\xrightarrow{D}\mathcal{N}(0,\Sigma).
\label{eq:beta_asymp_normal_generic}
\end{equation}
Then, for every fixed matrix $A\in\mathbb{R}^{m\times d}$,
\begin{equation}
\sqrt N\bigl(A\beta_N-A\beta_\infty\bigr)
\xrightarrow{D}\mathcal{N}(0,A\Sigma A^\top).
\label{eq:linear_map_asymp_normal_generic}
\end{equation}
\end{lemma}

\begin{proof}
Apply the continuous mapping theorem to the continuous linear map $u\mapsto Au$. If $Z\sim\mathcal{N}(0,\Sigma)$, then $AZ\sim\mathcal{N}(0,A\Sigma A^\top)$.
\end{proof}

\begin{corollary}[Asymptotic Normality of Concept Sensitivity]
\label{cor:asymptotic_normality_sensitivity}
Assume \eqref{eq:beta_asymp_normal_generic}. Fix an input $\mathbf{x}$ and define
\[
a(\mathbf{x})\defeq\nabla g_{\ell,k}\!\left(h_\ell(\mathbf{x})\right)\in\mathbb{R}^d,
\qquad
S(\mathbf{x},\beta)\defeq a(\mathbf{x})^\top\beta.
\]
Then
\begin{equation}
\sqrt N\Bigl(S(\mathbf{x},\beta_N)-S(\mathbf{x},\beta_\infty)\Bigr)
\xrightarrow{D}\mathcal{N}\!\left(0,V(\mathbf{x})\right),
\label{eq:sensitivity_asymp_normal}
\end{equation}
where
\begin{equation}
V(\mathbf{x})\defeq a(\mathbf{x})^\top\Sigma a(\mathbf{x}).
\label{eq:Vx_def}
\end{equation}
\end{corollary}

\begin{proof}
Apply Lemma~\ref{lem:linear_map_asymp_normal} with $m=1$ and
$A=a(\mathbf{x})^\top$.
\end{proof}

\subsection{Asymptotic Normality with Hinge Loss}
\label{appendix:hinge_loss}
By default, the \TCAV implementations in both TensorFlow and PyTorch use \texttt{sklearn}'s \texttt{SGDClassifier} \citep{pedregosa_scikit-learn_2011}, which operates by minimizing the hinge loss function. It achieves this by penalizing predictions that are not only incorrect but also those that are correct but fall within a specified ``margin'' around the decision boundary. 

For a dataset of $N+n$ samples $\{(u_i, y_i)\}_{i=1}^{N+n}$ with true labels $y_i \in \{-1, 1\}$, the final loss $J_N^{(\lambda)}(\beta, \alpha)$ minimized is defined as the sum of the average hinge loss and an $L^2$ regularization term:
  \[J^{(\lambda)}_N(\beta, \alpha) = \left( \frac{1}{N+n} \sum_{i=1}^{N+n} \max[0, 1 - y_i (\beta^\top u_i + \alpha)] \right) + \frac{\lambda}{2} \|\beta\|^2.  \]
Here, $\lambda$ represents the regularization strength hyperparameter, $\beta$ is the weight vector, and $\alpha$ is the intercept term.
Again we consider our known setting. The only difference is that we use a different labeling system. 
The first class consists of $n$ fixed points, $\{x_i\}_{i=1}^n \subset \mathbb{R}^d$, referred to as ``concept'' samples with label $y=1$. 
The second class consists of $N$ random points, $\{z_j\}_{j=1}^N \subset \mathbb{R}^d$, drawn independently and identically from a distribution $F_0$ with label $y=-1$.

For our main theorem of asymptotic normality to hold, we require the following assumptions:
\begin{enumerate}
\item The distribution $F_0$ of the controls must be continuous. 
Specifically, the projection $\beta_\infty^\top Z$ must have a continuous probability density function, $f_{\beta_\infty^\top Z}(\cdot)$, in a neighbourhood of $-1$.
\item The distribution $F_0$ must have finite second moments, \emph{i.e.}, $\mathbb{E}_{Z \sim F_0}[\|Z\|^2] < \infty$.
\end{enumerate}

\paragraph{}
Given these assumptions we can now state the theorem. 

\begin{theorem}[Asymptotic Normality with Hinge Loss]
\label{thm:asymptotic_normality_hinge}
Let $\beta_N$ be the minimizer of the objective function $J_N(\beta)$. Let $\beta_\infty$ be the unique minimizer of the limiting objective function
\begin{equation}
\lim_{N\to\infty} J_N^{(\lambda)}(\beta, \alpha) =  \mathbb{E}_{Z \sim F_0}[\max(0, 1 + \beta^T Z + \alpha)] + \frac{\lambda}{2} \|\beta\|^2.
\end{equation}
Under the assumptions listed above in Section \ref{appendix:hinge_loss}, as $N \to \infty$ with $n$ fixed, the estimator is asymptotically normal:
\begin{equation}
\sqrt{N}(\beta_N - \beta_\infty) \xrightarrow{D} \mathcal{N}(0, M^{-1} \Sigma_Z M^{-1}), 
\end{equation}
where $M = \lambda I + \mathbb{E}[ZZ^\top | \beta_\infty^\top Z = -1] f_{\beta_\infty^\top Z}(-1)$, and $\Sigma_Z = \text{Var}\left(Z \cdot I(\beta_\infty^\top Z > -1)\right)$.
\end{theorem}
\begin{proof}

As $N \to \infty$, the objective function $J_N(\beta)$ converges pointwise in probability to $\lim_{N\to\infty} J_N^{(\lambda)}(\beta)$. Under standard M-estimation arguments, the minimizer $\beta_N$ of $J_N(\beta)$ converges in probability to the minimizer $\beta_\infty$ of $\lim_{N\to\infty} J_N^{(\lambda)}(\beta)$. 

We first perform a Taylor expansion of the gradient of the objective function around $\beta_\infty$, and set $\nabla J_N(\beta_N) = 0$:
\begin{equation} \label{eq:taylor_expansion_hinge}
0 = \nabla_\beta J_N(\beta_\infty) + \nabla^2_\beta J_N(\beta_\infty) (\beta_N - \beta_\infty) + o_p(1).
\end{equation}
Rearranging gives 
  \[\sqrt{N}(\beta_N - \beta_\infty) =  -[\nabla^2_\beta J_N(\beta_\infty)]^{-1} \sqrt{N} \nabla_\beta J_N(\beta_\infty) + o_p(1).
  \]
We analyze the two terms on the right.
The optimality of $\beta_\infty$ for the limiting problem implies $\lim_{N\to\infty} \nabla_\beta J^{(\lambda)}(\beta) = 0$. 
This gives the condition
$
\mathbb{E}[Z \cdot I(\beta_\infty^T Z > -1)] + \lambda \beta_\infty = 0.
$ 
Let now $\mu_0 = \mathbb{E}[Z \cdot I(\beta_\infty^\top Z > -1)].$
The gradient of the finite-sample objective at $\beta_\infty$ is
\begin{align*}
    \nabla_\beta J_N(\beta_\infty) ={}& \frac{N}{n+N} \left( \frac{1}{N}\sum_{j=1}^N z_j I(\beta_\infty^\top z_j > -1) - \mu_0 \right) \\
    & - \frac{n}{n+N} \left( \frac{1}{n}\sum_{i=1}^n x_i I(\beta_\infty^\top x_i < 1) + \mu_0 \right).
\end{align*}

When scaled by $\sqrt{N}$, the second term vanishes as $N\to\infty$. By the Central Limit Theorem and our assumption on $\mathbb{E}_{Z \sim F_0}[\|Z\|^2] < \infty$, the first term converges in distribution. Thus, the scaled score has a normal limit
\begin{equation}
\sqrt{N} \nabla_\beta J_N(\beta_\infty) \xrightarrow{D} \mathcal{N}(0, \Sigma_Z)
\end{equation}
where $\Sigma_Z = \text{Var}\left(Z \cdot I(\beta_\infty^\top Z > -1)\right)$.
The Hessian matrix of the finite-sample objective, $H_N(\beta) = \nabla^2_\beta J_N(\beta)$, converges in probability to the Hessian of the limiting objective
\begin{equation}
    H := \lim_{N\to\infty} \nabla^2_\beta J_N(\beta_\infty) = \lambda I + \mathbb{E}[ZZ^\top | \beta_\infty^\top Z = -1] f_{\beta_\infty^\top Z}(-1).
\end{equation}
 Slutsky's Theorem \citep{slutzky_summation_1937} and the assumed probability density function, $f_{\beta_\infty^\top Z}(\cdot)$ give us the asymptotic distribution of the estimator
\[
    \sqrt{N}(\beta_N - \beta_\infty) \xrightarrow{D} \mathcal{N}(0, H^{-1} \Sigma_Z H^{-1}).
\]
This completes the proof.
\end{proof}

\subsection{Variability of Difference of Means}
\label{appendix:difference_of_means}

The Difference of Means (DoM) \citep{martin_interpretable_2019} method identifies a concept's direction within a model's activation space by simply taking the vector difference between the average activation for concept examples and the average activation for random examples. As this method was also adapted for calculating \text{\textsc{Cav}\xspace}s \citep{santis_visual-tcav_2024}, we analyze the stability of the direction vector $\beta_N = \bar{x} - \bar{z}$, where $\bar{x}$ is the mean of the $n$ fixed concept samples and $\bar{z}$ is the mean of $N$ random samples (from an independent distribution $F_0$ with covariance $\Sigma_z$).

\begin{theorem}[Asymptotic Variance of the Difference of Means Vector]
\label{thm:asymptotic_variance_dom}
Let $\{x_i\}_{i=1}^n$ be $n$ fixed points, and let $\{z_j\}_{j=1}^N$ be $N$ random samples drawn i.i.d. from an independent distribution $F_0$ with finite covariance $\Sigma_z$. Define the direction vector $\beta_N = \bar{x} - \bar{z}$ as the difference of the respective means.

Then, the total variance of $\beta_N$, is given by:
\begin{equation}
\label{eq:asymptotic-variance-dom}
\tr(\Cov(\beta_N)) = \frac{1}{N}\tr(\Sigma_z)
\end{equation}
\end{theorem}

\begin{proof}
The variance of $\beta_N$, measured by $\tr(\Cov(\beta_N))$, is derived as follows. Since the set $\{x_i\}_{i=1}^n$ is fixed, their mean $\bar{x}$ is a deterministic constant vector, and thus $\Cov(\bar{x}) = 0$. The variance of $\beta_N$ is therefore determined exclusively by the random component $\bar{z}$.

\begin{align}
\tr(\Cov(\beta_N)) &= \tr(\Cov(\bar{x} - \bar{z})) \\
&= \tr\left(\frac{1}{N^2} \sum_{j=1}^N \Cov(z_j)\right) \quad (\text{by i.i.d. assumption}) \\
&= \tr\left(\frac{1}{N^2} (N \cdot \Sigma_z)\right) \\
&= \frac{1}{N}\tr(\Sigma_z)
\end{align}

Since $\tr(\Sigma_z)$ is a fixed constant, the variance of $\beta_N$ declines at a rate of $\mathcal{O}(1/N)$ as the number of random samples $N$ increases.

\end{proof}

As the number of random samples $N \to \infty$ (while $n$ remains fixed), the variance of $\beta_N$ converges to zero at a rate of $\mathcal{O}(1/N)$. The \emph{Difference of Means}-method thus exhibits the same convergence behaviour as the other two classifiers discussed.

\subsection{Variability of Concept Activation Regions}
\label{appendix:concept_activation_regions}

Concept Activation Regions extend the linear SVM classifier described in Appendix~\ref{appendix:hinge_loss} to nonlinear settings through the use of positive-definite kernels. In this section, we first introduce the necessary background on reproducing kernel Hilbert spaces, then describe the kernel SVM formulation used by \text{\textsc{Car}\xspace}s, and finally establish asymptotic concentration of the \text{\textsc{Car}\xspace} estimator.

Let $\kappa: H \times H \to \mathbb{R}$ be a positive-definite kernel. By the Moore-Aronszajn theorem, there exists a unique reproducing kernel Hilbert space (RKHS) $\mathcal{H}_\kappa$ of functions $f: H \to \mathbb{R}$ satisfying the reproducing property $f(h) = \langle f, \kappa(\cdot, h) \rangle_{\mathcal{H}_\kappa}$. Equivalently, there exists a feature map $\Phi: H \to \mathcal{H}_\kappa$ such that
\[
\kappa(h, h') = \langle \Phi(h), \Phi(h') \rangle_{\mathcal{H}_\kappa}.
\]
We assume throughout that the kernel is bounded: $\sup_{h \in H} \kappa(h,h) \leq K^2$. For background on RKHS theory, see \citet{steinwart2008support} or \citet{scholkopf2002learning}.

As before, let $\{h_\ell(x_i)\}_{i=1}^n$ denote the positive-class embeddings and $\{h_\ell(z_j)\}_{j=1}^N$ the negative-class embeddings. We define a labeled training set of size $m = n + N$ by
\[
(s_r, y_r) = 
\begin{cases}
\bigl(h_\ell(x_r),\, +1\bigr), & r = 1, \ldots, n, \\[4pt]
\bigl(h_\ell(z_{r-n}),\, -1\bigr), & r = n+1, \ldots, n+N.
\end{cases}
\]
\citet{crabbe_concept_2022} train a soft-margin SVM on this labeled set. The standard constrained formulation seeks $\mathbf{w} \in \mathcal{H}_\kappa$, bias $b \in \mathbb{R}$, and slack variables $\xi \in \mathbb{R}^m$ solving
\[
\min_{\mathbf{w}, b, \xi} \quad \frac{1}{2} \|\mathbf{w}\|_{\mathcal{H}_\kappa}^2 + C \sum_{r=1}^m \xi_r
\qquad \text{subject to} \quad
y_r \bigl( \langle \mathbf{w}, \Phi(s_r) \rangle_{\mathcal{H}_\kappa} + b \bigr) \geq 1 - \xi_r, \quad \xi_r \geq 0.
\]
Equivalently (up to rescaling of the regularization parameter), one may write the unconstrained regularized hinge-loss form:
\[
\min_{\mathbf{w} \in \mathcal{H}_\kappa,\, b \in \mathbb{R}} \quad 
\frac{\lambda}{2} \|\mathbf{w}\|_{\mathcal{H}_\kappa}^2 
+ \frac{1}{m} \sum_{r=1}^m \max\bigl(0,\, 1 - y_r (\langle \mathbf{w}, \Phi(s_r) \rangle_{\mathcal{H}_\kappa} + b)\bigr).
\]

\paragraph{Dual formulation.}
This optimization is solved via its dual. Introducing Lagrange multipliers $\{\alpha_r\}_{r=1}^m$ and applying the KKT conditions yields the representer expansion
\[
\mathbf{w} = \sum_{r=1}^m \alpha_r y_r \Phi(s_r) \;\in\; \mathcal{H}_\kappa.
\]
The dual problem includes box constraints $0 \leq \alpha_r \leq C$ and, when the bias $b$ is included, the equality constraint $\sum_{r=1}^m \alpha_r y_r = 0$. Substituting the expansion into the decision function gives the kernel classifier:
\[
f(h) = \langle \mathbf{w}, \Phi(h) \rangle_{\mathcal{H}_\kappa} + b 
= \sum_{r=1}^m \alpha_r y_r \kappa(s_r, h) + b,
\qquad
\hat{y}(h) = \operatorname{sgn}\bigl(f(h)\bigr).
\]

Only points with $\alpha_r\neq 0$ contribute; these are the \emph{support vectors}.
The same mechanism extends to non-linear decision boundaries by choosing a non-linear feature map $\Phi$
(implicitly) via a kernel $\kappa(h,h')=\langle \Phi(h),\Phi(h')\rangle$ as done by \citet{crabbe_concept_2022}.

For concept activation methods we seek a \emph{direction} in representation space that captures the concept—analogous to the \CAV vector $\beta_N \in \mathbb{R}^d$.

We therefore define the \text{\textsc{Car}\xspace} vector as the bias-free normal vector of the separating hyperplane:

\begin{definition}[Concept Activation Region (CAR)]\label{def:car}
Given the labeled training set $\{(s_r, y_r)\}_{r=1}^m$ constructed from positive and negative concept embeddings, the \emph{\text{\textsc{Car}\xspace} vector} is
\[
\beta_N := \mathbf{w} = \sum_{r=1}^m \alpha_r y_r \Phi(s_r) \;\in\; \mathcal{H}_\kappa,
\]
where $\{\alpha_r\}$ are the dual SVM coefficients. The associated \emph{bias-free score} is
\[
f_N(h) := \langle \beta_N, \Phi(h) \rangle_{\mathcal{H}_\kappa} = \sum_{r=1}^m \alpha_r y_r \kappa(s_r, h).
\]
\end{definition}

Setting the bias $b = 0$ yields a direction-like object directly comparable to a \CAV, at the cost of discarding the optimal decision threshold.

Since $\beta_N \in \mathcal{H}_\kappa$, we extend Definition~\ref{def:var-car} to the CAR setting:
\begin{definition}[Variance of CARs]\label{def:var-car}
The variance of the \text{\textsc{Car}\xspace} vector $\beta_N$ is
\[
\var{\beta_N} := \mathbb{E}\bigl\|\beta_N - \mathbb{E}[\beta_N]\bigr\|_{\mathcal{H}_\kappa}^2 = \tr\bigl(\Cov(\beta_N)\bigr).
\]
\end{definition}

This is the natural analogue of the total variance used for \text{\textsc{Cav}\xspace}s. When $\kappa$ is the linear kernel, we have $\mathcal{H}_\kappa \simeq \mathbb{R}^d$ and $\|\cdot\|_{\mathcal{H}_\kappa} = \|\cdot\|_2$, so Definition~\ref{def:var-car} reduces exactly to Definition~\ref{def:var-cav}.

Our goal is to show that the variance $\var{\beta_N}$ decreases as the number of negative samples $N$ grows. To obtain this bound, we analyze how $\beta_N$ depends on the random negative samples $\{z_j\}_{j=1}^N$. The key technical tool is the Efron-Stein inequality, which bounds the variance of a function of independent random variables in terms of the sensitivity to each input.

\begin{lemma}[Efron-Stein inequality for Hilbert spaces \citep{efron_the_1981}]
\label{lemma:efron_stein}
Let $f: \mathcal{X}^N \to \mathcal{H}$ be measurable, $W = f(X_1, \dots, X_N) \in \mathcal{H}$,
and assume $\mathbb{E}[ \norm{W}^2 ] < \infty$.
For each $j$, let $X^{(j)}$ be the sequence where $X_j$ is replaced by an independent copy, and set
$W^{(j)} = f(X^{(j)})$.
Then
\[
\Var(W)
=
\mathbb{E}\bigl\Vert W - \mathbb{E}W \bigr\Vert_{\mathcal{H}}^2
\;\le\;
\frac{1}{2}\sum_{j=1}^N \mathbb{E}\bigl\Vert W - W^{(j)} \bigr\Vert_{\mathcal{H}}^2.
\]
\end{lemma}

Using this Lemma from \citet{efron_the_1981} we can now state the following theorem. 

\begin{theorem}[Variance of class-weighted kernel CARs]
\label{thm:car_variance}
Let $\{h_\ell(x_i)\}_{i=1}^n \subset H$ be fixed positive embeddings and let $\{h_\ell(z_j)\}_{j=1}^N$ be i.i.d.\ from $F_0$.
Assume $\sup_{h \in H} \kappa(h,h) \leq K^2$, and let $\beta_N \in \mathcal{H}_\kappa$ be trained by the
class-weighted, $\lambda$-regularized hinge-loss objective. Then
\[
\Var(\beta_N) \leq \frac{2K^2}{\lambda^2} \cdot \frac{1}{N}
\qquad\text{and hence}\qquad
\Var(\beta_N) = \mathcal{O}(1/N)\ \text{as } N \to \infty.
\]
\end{theorem}

\begin{proof}
Let $Z = (z_1, \ldots, z_N)$ denote the negative sample and let $Z^{(j)}$ be the same sample except that
$z_j$ is replaced by an independent copy $z_j' \sim F_0$.
Write $\beta_N = \beta(Z)$ for the \text{\textsc{Car}\xspace} vector trained on $Z$, and $\beta_N^{(j)} = \beta(Z^{(j)})$ for the \text{\textsc{Car}\xspace} vector trained on the perturbed sample $Z^{(j)}$.

These are the minimizers of the respective empirical objectives
\[
J_Z(\beta)
=
\frac{\lambda}{2} \|\beta\|_{\mathcal{H}_\kappa}^2
+ \frac{1}{n} \sum_{i=1}^n \ell\bigl(\langle \beta, \Phi(h_\ell(x_i)) \rangle_{\mathcal{H}_\kappa}\bigr)
+ \frac{1}{N} \sum_{j=1}^N \ell\bigl(-\langle \beta, \Phi(h_\ell(z_j)) \rangle_{\mathcal{H}_\kappa}\bigr),
\]
where $\ell(u) = \max(0, 1 - u)$ is the hinge loss.

Because of the quadratic regularizer, $J_Z(\beta)$ is $\lambda$-strongly convex in $\beta$. Recall that $\beta_N = \beta(Z)$ minimizes $J_Z$ and $\beta_N^{(j)} = \beta(Z^{(j)})$ minimizes $J_{Z^{(j)}}$. By the strong convexity inequality applied at each respective minimizer, we have
\[
J_Z(\beta_N^{(j)}) - J_Z(\beta_N) \geq \frac{\lambda}{2} \|\beta_N^{(j)} - \beta_N\|^2_{\mathcal{H}_\kappa},
\qquad
J_{Z^{(j)}}(\beta_N) - J_{Z^{(j)}}(\beta_N^{(j)}) \geq \frac{\lambda}{2} \|\beta_N^{(j)} - \beta_N\|^2_{\mathcal{H}_\kappa}.
\]
Adding these two inequalities yields
\[
\lambda \|\beta_N^{(j)} - \beta_N\|^2_{\mathcal{H}_\kappa}
\leq
\Bigl( J_Z(\beta_N^{(j)}) - J_{Z^{(j)}}(\beta_N^{(j)}) \Bigr)
+
\Bigl( J_{Z^{(j)}}(\beta_N) - J_Z(\beta_N) \Bigr).
\]
The objectives $J_Z$ and $J_{Z^{(j)}}$ differ only in the $j$-th negative-class loss term: $J_Z$ includes $\ell\bigl(-\langle \beta, \Phi(h_\ell(z_j)) \rangle\bigr)$ while $J_{Z^{(j)}}$ includes $\ell\bigl(-\langle \beta, \Phi(h_\ell(z_j')) \rangle\bigr)$, each scaled by $1/N$. Therefore
\[
\lambda \|\beta_N^{(j)} - \beta_N\|^2_{\mathcal{H}_\kappa}
\leq
\frac{1}{N} \Bigl( \ell_{z_j}(\beta_N^{(j)}) - \ell_{z_j}(\beta_N) \Bigr)
+
\frac{1}{N} \Bigl( \ell_{z_j'}(\beta_N) - \ell_{z_j'}(\beta_N^{(j)}) \Bigr),
\]
where we write $\ell_u(\beta) := \ell\bigl(-\langle \beta, \Phi(h_\ell(u)) \rangle_{\mathcal{H}_\kappa}\bigr)$ for the hinge loss evaluated at embedding $h_\ell(u)$.

The hinge loss $\ell(t) = \max(0, 1-t)$ is $1$-Lipschitz in its scalar argument. Moreover, for any $u$,
\[
\bigl| \langle \beta, \Phi(h_\ell(u)) \rangle - \langle \beta', \Phi(h_\ell(u)) \rangle \bigr|
\leq \|\beta - \beta'\|_{\mathcal{H}_\kappa} \, \|\Phi(h_\ell(u))\|_{\mathcal{H}_\kappa}
= \|\beta - \beta'\|_{\mathcal{H}_\kappa} \sqrt{\kappa(h_\ell(u), h_\ell(u))}
\leq K \|\beta - \beta'\|_{\mathcal{H}_\kappa}.
\]
Hence $\ell_u(\cdot)$ is $K$-Lipschitz in $\beta$. Applying this to both terms on the right-hand side gives
\[
\lambda \|\beta_N^{(j)} - \beta_N\|^2_{\mathcal{H}_\kappa}
\leq
\frac{2K}{N} \|\beta_N^{(j)} - \beta_N\|_{\mathcal{H}_\kappa}.
\]
If $\|\beta_N^{(j)} - \beta_N\|_{\mathcal{H}_\kappa} > 0$, dividing both sides yields
\[
\|\beta_N^{(j)} - \beta_N\|_{\mathcal{H}_\kappa}
\leq
\frac{2K}{\lambda N}.
\]
The inequality holds trivially when $\|\beta_N^{(j)} - \beta_N\|_{\mathcal{H}_\kappa} = 0$.

Apply Lemma~\ref{lemma:efron_stein} with $W=\beta_N$:
\[
\mathcal \Var(\beta_N)
=
\mathbb{E}\|\beta-\mathbb{E}[\beta]\|^2_{\mathcal H_\kappa}
\le
\frac12\sum_{j=1}^N \mathbb{E}\|\beta-\beta^{(j)}\|^2_{\mathcal H_\kappa}.
\]
Using the Step~1 bound,
\[
\mathcal \Var(\beta_N)
\le
\frac12 \sum_{j=1}^N \left(\frac{2K}{\lambda N}\right)^2
=
\frac12 \cdot N \cdot \frac{4K^2}{\lambda^2 N^2}
=
\frac{2K^2}{\lambda^2}\cdot\frac1N.
\]
\end{proof}

\begin{lemma}[Variance bound at a fixed point]
\label{lemma:var_bound_point}
For any fixed $h\in H$,
\begin{equation}
\label{eq:var_bound_point}
\Var\!\bigl(f_N(h)\bigr)
=
\Var\!\bigl(\langle \beta_N,\Phi(h)\rangle\bigr)
\;\le\;
\kappa(h,h)\,\mathcal \Var(\beta_N).
\end{equation}
\end{lemma}

\begin{proof}
Let $\bar\beta := \mathbb{E}[\beta_N]$. Then
\(
f_N(h)-\mathbb{E}[f_N(h)]
=
\langle \beta_N-\bar\beta,\Phi(h)\rangle.
\)
Hence
\(
\Var\!\bigl(f_N(h)\bigr)
=
\mathbb{E}\bigl[\langle \beta_N-\bar\beta,\Phi(h)\rangle^2\bigr].
\)

By Cauchy--Schwarz,
\[
\langle \beta_N-\bar\beta,\Phi(h)\rangle^2
\le
\|\beta_N-\bar\beta\|_{\mathcal H_\kappa}^2\ \|\Phi(h)\|_{\mathcal H_\kappa}^2.
\]
Taking expectations and using $\|\Phi(h)\|^2=\kappa(h,h)$ yields
\[
\Var \bigl(f_N(h)\bigr)
\le
\mathbb{E}\|\beta_N-\bar\beta\|^2_{\mathcal H_\kappa}\ \kappa(h,h)
=
\mathcal \Var(\beta_N)\,\kappa(h,h).
\]
\end{proof}

For the RBF kernel, $\kappa(h,h)=1$ for all $h$, so \eqref{eq:var_bound_point} simplifies to
\[
\Var \bigl(f_N(h)\bigr)\le \mathcal \Var(\beta_N).
\]

\begin{corollary}[Variance of kernel \text{\textsc{Car}\xspace} decision values]
For any $h\in H$, the decision-value variance decays as $\mathcal O(1/N)$:
\[
\Var \bigl(f_N(h)\bigr)
\le
\kappa(h,h)\,\mathcal \Var(\beta_N)
\le
K^2 \cdot \frac{2K^2}{\lambda^2}\cdot \frac1N
=
\frac{2K^4}{\lambda^2}\cdot\frac1N.
\]
\end{corollary}

\paragraph{Finite-Dimensional Approximation via Nyström}

While products between \text{\textsc{Car}\xspace} vectors ($\beta$)  can be computed via the dual form, this requires storing all support vectors from each resampled \text{\textsc{Car}\xspace} and $\mathcal{O}(R^2 \cdot n_{\text{sv}}^2)$ kernel evaluations for variance estimation across $R$ runs.

The Nyström method \citep{williams_using_2000} approximates the kernel via a low-rank factorization using $q$ landmark points, which gives an explicit feature map $\tilde{\Phi}: H \to \mathbb{R}^q$ with $\tilde{\Phi}(h)^\top \tilde{\Phi}(h') \approx \kappa(h, h')$. Training a linear SVM on transformed features produces $\tilde{\beta}_N \in \mathbb{R}^q$, reducing variance computation to $\mathcal{O}(R \cdot q^2)$.

The guarantees of Theorem~\ref{thm:car_variance} transfer to this approximation: the objective remains $\lambda$-strongly convex on the restricted subspace, so the Efron-Stein argument yields identical $\mathcal{O}(1/N)$ variance bounds.

\end{document}